\documentclass{article}
\usepackage{times}
\usepackage{natbib}
\setcitestyle{authoryear,round,citesep={;},aysep={,},yysep={;}}

\usepackage{graphicx}
\usepackage{wrapfig}
\usepackage{needspace}
\usepackage{subfiles}
\usepackage{float}
\usepackage{placeins}
\usepackage[utf8]{inputenc}
\usepackage[T1]{fontenc}
\usepackage{xr-hyper}
\usepackage[colorlinks=true,citecolor=blue,linkcolor=blue,urlcolor=blue,hypertexnames=false]{hyperref}
\usepackage{url}
\usepackage{booktabs}
\usepackage{amsfonts}
\usepackage{amssymb}
\usepackage{amsmath}
\usepackage{amsthm}
\usepackage{mathtools}
\usepackage{bm}
\usepackage{nicefrac}
\usepackage{microtype}
\usepackage{xcolor}
\usepackage{cleveref}
\usepackage{enumitem}
\usepackage{array}
\usepackage{tabularx}
\usepackage{longtable}
\usepackage{pdflscape}
\usepackage{fancyhdr}
\usepackage{tcolorbox}
\usepackage{tikz}
\usetikzlibrary{arrows.meta}

\tcbset{
  optimizerblock/.style={
    colback=black!2,
    colframe=black!45,
    boxrule=0.4pt,
    arc=0pt,
    left=6pt,
    right=6pt,
    top=5pt,
    bottom=5pt,
    before skip=0pt,
    after skip=0pt
  }
}

\floatstyle{ruled}
\newfloat{algorithm}{tbp}{loa}
\floatname{algorithm}{Algorithm}
\crefname{algorithm}{algorithm}{algorithms}
\crefname{assumption}{Assumption}{Assumptions}
\Crefname{algorithm}{Algorithm}{Algorithms}
\theoremstyle{definition}
\newtheorem{definition}{Definition}

\theoremstyle{plain}
\newtheorem{theorem}{Theorem}[section]
\newtheorem{lemma}[theorem]{Lemma}

\newtheorem{proposition}[theorem]{Proposition}

\theoremstyle{definition}
\newtheorem{assumption}[theorem]{Assumption}

\theoremstyle{remark}
\newtheorem{remark}[theorem]{Remark}

\newcommand{\E}{\mathbb{E}}

\newcommand{\Var}{\mathrm{Var}}
\newcommand{\Cov}{\mathrm{Cov}}

\newcommand{\diag}{\mathrm{diag}}

\newcommand{\grad}{\nabla}

\newcommand{\w}{\theta}

\newcommand{\m}{m}
\newcommand{\g}{g}
\newcommand{\ghat}{\widehat g}

\newcommand{\Hmat}{H}
\newcommand{\Sig}{\Sigma}

\newcommand{\lr}{\eta}
\newcommand{\mom}{\mu}
\newcommand{\wdc}{\lambda}

\newcommand{\batch}{B}
\newcommand{\kappaBatch}{\kappa}

\newcommand{\gam}{\gamma_{\mathrm{samp}}}
\newcommand{\rhoMom}{\rho}
\newcommand{\tausamp}{\tau}

\newcommand{\dd}{\,\mathrm{d}}

\newcommand{\Dir}{\Psi}

\newcommand{\normal}{\mathcal{N}}
\newcommand{\identity}{\mathcal{I}}

\usepackage{shoreline}
\shorelinestatus{Preprint}
\shorelinerunningtitle{The Best Optimizer Depends on Batch Size}
\tcbset{optimizerblock/.style={shoreline/base,unbreakable,before skip=0pt,after skip=0pt}}
\title{The Best Optimizer Depends on Batch Size}
\author{%
  \textbf{Xingyu Dang\textsuperscript{\dag}}\\
  {\small Princeton University}\\
  {\small\texttt{xingyu.dang@princeton.edu}}
  \and
  \textbf{Kaiyue Wen\textsuperscript{\dag}}\\
  {\small Stanford University}\\
  {\small\texttt{kaiyuew@stanford.edu}}
  \and
  \textbf{Sadhika Malladi}\\
  {\small University of California, San Diego}\\
  {\small\texttt{samalladi@ucsd.edu}}%
}
\date{}

\shorelineboxenvironment{definition}{definition}
\shorelineboxenvironment{theorem}{theorem}
\begin{document}

\maketitle
\begingroup
\renewcommand{\thefootnote}{\fnsymbol{footnote}}
\footnotetext[2]{Equal contribution.}
\endgroup

\begin{abstract}
A plethora of new adaptive optimizers are designed to efficiently estimate and use minibatch gradient statistics to shape parameter updates, but they are typically benchmarked at a single batch size.
Hyperparameter scaling rules promise to preserve performance as batch size and gradient noise change, suggesting that the best optimizer at one batch size should remain the best at another.
We challenge this approach to developing and evaluating optimizers by showing: (1) no principled scaling rule for Muon works consistently across training settings, and (2) the best optimizer for language model pretraining changes with batch size even after extensive hyperparameter tuning.
To understand these failures, we analyze a noisy quadratic model and show that the batch-size scaling rule that preserves performance is direction-dependent, shifting from square-root to linear as a direction's curvature-to-noise ratio rises.
We further prove that gradient noise drives a bias--variance tradeoff in preconditioning that can reverse optimizer rankings across batch sizes, even under optimal learning rate and momentum tuning.
Consistent with this, when we increase the batch size $16\times$ partway through language model pretraining, keeping small-batch updates in only the sharpest $0.001\%$ of directions removes up to $59.5\%$ of the resulting loss gap over the next $8\%$ of training, while a random subspace of the same size barely helps.
Taken together, our results show that principled optimizer design must account for batch size and empirical benchmarking should be performed across batch sizes.
\end{abstract}

\section{Introduction}
\label{sec:introduction}

Larger batches can increase training throughput, but retaining optimization performance is challenging and often requires adjusting hyperparameters~\citep{goyal2017accurate,shallue2019measuring}.
Batch-size scaling rules prescribe these adjustments without requiring additional tuning (\Cref{sec:scaling-setup}).
These rules are derived analytically and intended to apply broadly across training settings~\citep{malladi2022sdes,mlodozeniec2025completed,shulgin2026deriving}, but different derivations yield conflicting prescriptions for the same optimizer (\Cref{tab:scaling-recipes}).
We ask whether any of these prescriptions, or alternatives beyond them, can fulfill this promise:

\vspace{-0.5em}
\begin{center}
\itshape
\textbf{Q1:} Can a hyperparameter scaling rule consistently work in different training settings?
\end{center}
\vspace{-0.5em}

We study this question with Muon~\citep{jordan2024muon} (\Cref{def:muonw-adamw}).
We derive scaling rules by applying prior theoretical frameworks designed to either preserve training dynamics via SDE or minimize convergence bounds (\Cref{def:s3-sde-rule,def:s3-bound-rule}), and compare them with a broad grid of candidate rules in language modeling and image classification (\Cref{sec:scaling-results}).
Neither theoretical prescription consistently gives the best performance (\Cref{fig:muonw_transfer_comparison}) and in language modeling both fall behind the trivial baseline of leaving the hyperparameters unchanged at the largest batch sizes.
Even among the broader set of candidates, the best rule differs between tasks (\Cref{tab:scaling-recipes}) and between target batch sizes within a task (\Cref{fig:muonw_transfer_comparison}).

Beyond how to adjust Muon's hyperparameters, we ask whether batch size changes its performance relative to other optimizers.
Adaptive optimizers differ in how they shape updates using minibatch estimates of quantities such as gradient second moments~\citep{kingma2015adam,gupta2018shampoo,morwani2024shampooperspective,wang2024batchinvariant}.
Batch size controls the noise in these estimates, potentially changing which form of adaptation is most effective~\citep{wang2024batchinvariant,zhang2019algorithmicchoicesmatterbatch}.
Established benchmarks like the Modded-NanoGPT optimization benchmark~\citep{jordan2026moddednanogpt}, AlgoPerf~\citep{dahl2023algoperf} and Fantastic Optimizers~\citep{wen2025fantastic} rank optimizers without testing their performance across batch sizes, leaving open the question:

\vspace{-0.5em}
\begin{center}
\itshape
\textbf{Q2:} Do optimizer rankings persist across batch sizes after extensive retuning?
\end{center}
\vspace{-0.5em}

\begin{figure}[t]
    \centering
    \includegraphics[width=\textwidth]{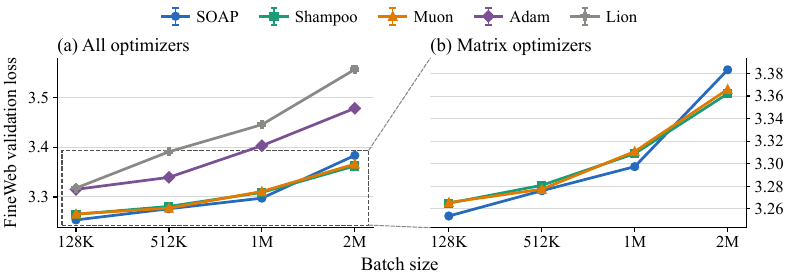}
    \caption{\textbf{The best optimizer changes with batch size.}
    FineWeb validation loss from the Modded-NanoGPT optimization benchmark, pretrained on a fixed budget of $1.7$B tokens with decoupled weight decay; every optimizer is tuned independently at every batch size.
    (a)~Validation loss with respect to batch size for all  five optimizers.
    (b)~The performance of three matrix optimizers on a zoomed loss axis.
    SOAP is the best optimizer from $128\mathrm K$ to $1\mathrm M$ tokens, ahead by up to $0.012$ nats, yet at $2\mathrm M$ it falls to third, $0.021$ nats behind Shampoo.
    Rankings cross between other pairs of optimizers too: Muon beats Shampoo by $0.0035$ nats at $512\mathrm K$, but Shampoo beats Muon by $0.0037$ nats at $2\mathrm M$.
    The HyperBall~\citep{wen2026hyperball} counterpart is \Cref{fig:track3_batch_transfer_hyperball}.}
    \label{fig:track3_batch_transfer}
    \par\vspace{-1.5em}
\end{figure}

We compare Adam~\citep{kingma2015adam}, Lion~\citep{chen2024lion}, Muon~\citep{jordan2024muon}, SOAP~\citep{vyas2024soap}, and Shampoo~\citep{gupta2018shampoo} in the Modded-NanoGPT optimization benchmark setting, independently tuning each at every batch size using coordinate descent over hyperparameters~\citep{wen2025fantastic} (\Cref{sec:batch_size_comparison}).
We observe the \emph{optimizer crossover} phenomenon: optimizer rankings reverse across batch sizes: SOAP performs best at smaller batches, whereas Shampoo performs best at the largest batch size tested.
This crossover occurs under both decoupled weight decay and HyperBall norm control~\citep{wen2026hyperball} (\Cref{fig:track3_batch_transfer,fig:track3_batch_transfer_hyperball}).

Motivated by recent evidence that quadratic models capture the local optimization dynamics of language model pretraining~\citep{meterez2026defensequadraticmodel}, we study a noisy quadratic model (NQM,~\citet{zhang2019algorithmicchoicesmatterbatch}) with direction-dependent curvature and gradient noise to understand why scaling rules do not transfer consistently and why optimizer rankings reverse (\Cref{sec:theory}).
Tuning the learning rate on SignSGD-M in one direction shows that its scaling with batch size shifts from approximately square-root toward linear as the direction's curvature-to-noise ratio increases (\Cref{sec:nqm_signsgd_scaling}).
A single scaling exponent therefore cannot match the individually tuned learning rates across directions.
When we compare SGD and Newton's method, we find that Newton's preconditioning trades faster reduction of initialization bias for greater variance from gradient noise~\citep{zhang2019algorithmicchoicesmatterbatch} (\Cref{sec:nqm_main}).
We prove that this bias--variance tradeoff can favor SGD at small batches and Newton's method at large batches, even with independently optimal learning rate and momentum (\Cref{thm:nqm_crossover_main}).

We test whether this direction-dependent scaling appears in language-model pretraining (\Cref{sec:lm_directional}).
Following \citet{merrill2025criticalbatchsizerevisited}, we branch from intermediate checkpoints of a tuned MuonW run, increasing the batch size from $128\mathrm K$ to $2\mathrm M$ tokens with square-root learning-rate scaling.
Retaining small-batch updates in the sharpest $768$ of roughly $85$M directions of a preconditioned Hessian proxy, while using large-batch updates elsewhere, removes $35$--$59.5\%$ of the loss gap to a small-batch control over the next $8\%$ of training (\Cref{fig:lm_directional}).
A random subspace of the same size changes the gap by at most $3\%$.
Thus, a tiny set of sharp directions carries a disproportionate share of the local penalty from applying one learning-rate exponent to every direction.
Because the quadratic approximation is local, we assess loss shortly after the switch rather than final training loss~\citep{wen2025riverValley}.

Altogether, our results challenge established practices in optimizer evaluation and design.
Comparisons at a single batch size are not sufficient to establish the efficacy of an optimizer, and we recommend performing extensive hyperparameter tuning across several batch sizes.
Optimizer design must likewise go beyond efficiently estimating and using landscape statistics to account for how batch size changes the benefits and costs of adaptation.

\section{Preliminaries}
\label{sec:preliminaries}
\subsection{Scaling Rules}
\label{sec:scaling-setup}
A batch-size scaling rule prescribes how to adjust hyperparameters tuned at a reference batch size $B_0$ for use at $B=\kappa B_0$.
Scaling rules are thought to be specific to a given optimizer but generalize across empirical learning settings.
For a hyperparameter $\gamma$ at $B_0$, let $\gamma^{(\kappa)}$ denote its prescribed value at $\kappa B_0$.
Typically, $\kappa>1$ due to the potential throughput gains from larger batches.
We fix the training budget $\tau$ (tokens for language modeling or examples for image classification), giving $\tau/B$ updates.
The \emph{sample time}, $s=kB$, is the data processed after $k$ updates.
\Cref{tab:scaling-recipes} compares prescriptions from two approaches; \Cref{app:related_work} discusses broader related work.

\textbf{Invariant-preservation View.}
This view adjusts hyperparameters to preserve a training property, such as an approximate parameter trajectory or the amount of data remembered by a moving average.
Different invariants can yield different scaling rules.

An SDE (stochastic differential equation) approximates an optimizer's random updates by continuous-time dynamics.
For SGD on loss $L$, let $\Sigma(x)$ be the individual-gradient covariance at parameters $x$; averaging $B$ independent gradients gives covariance $\Sigma(x)/B$.
Under suitable conditions~\citep{li2018stochastic,li2021validity}, at time $t=k\eta$, an SDE approximation of SGD is
$
\mathrm dX_t=-\nabla L(X_t)\,\mathrm dt+\sqrt{\eta/B}\,\Sigma(X_t)^{1/2}\,\mathrm d\mathcal W_t.$
Here $X_t$ denotes the parameters and $\mathcal W_t$ is Brownian motion.
Preserving $\eta/B$ gives the \emph{linear scaling rule}, $\eta^{(\kappa)}=\kappa\eta$~\citep{goyal2017accurate,smith2018dont,li2021validity}.
Since $t=(\eta/B)s$ is also preserved, the SDE predicts the same parameter distribution and expected loss at equal sample time.
Matching the coupled SDE for parameters and optimizer state (e.g. moment estimates) yields scaling rules for Adam~\citep{malladi2022sdes} and AdamW~\citep{mlodozeniec2025completed}.

However, the invariant need not be the full trajectory or even come from an SDE.
\citet{compagnoni2025adaptive} use an SDE to preserve an upper bound on long-time expected loss and the relative speeds of the parameter and moment states.
\citet{marek2025smallbatch} keep Adam's second-moment memory fixed in sample time, while \citet{bergsma2025powerlines} keep AdamW's weight-decay averaging timescale fixed relative to the training budget.

\textbf{Bound-minimization View.}
A second approach starts from a convergence bound on a specified quantity, such as suboptimality or stationarity, expressed in terms of the optimizer's hyperparameters, batch size, and number of updates \citep{kovalev2025understanding,pethick2025training}.
Minimizing this bound over the scaled hyperparameters at each batch size, with sample-time budget fixed, gives a scaling rule.
Varying the sample budgets or jointly optimizing the hyperparameters can further yield different rules \citep{shulgin2026deriving}.

\begin{table}[t]
  \centering
  \begingroup
  \defcitealias{bergsma2025powerlines}{[BDG$^{+}$25]}
  \defcitealias{shulgin2026deriving}{[SRZ$^{+}$26]}
  \defcitealias{islamov2026role}{[IML$^{+}$26]}
  \defcitealias{bu2026convexdominance}{[BXM26]}
  \defcitealias{malladi2022sdes}{[MLP$^{+}$22]}
  \defcitealias{marek2025smallbatch}{[MLS$^{+}$25]}
  \defcitealias{compagnoni2025adaptive}{[MLI$^{+}$25]}
  \defcitealias{schaipp2025surprising}{[SHT$^{+}$25]}
  \defcitealias{granziol2022learningrates}{[GZR22]}
  \defcitealias{mlodozeniec2025completed}{[MAB$^{+}$26]}
  \parbox{\textwidth}{\scriptsize
    Primes denote values at $B'=\kappa B$. Here $\eta$ is the learning rate, $\mu,\beta_1,\beta_2$ are momentum and moment retention coefficients, and $\lambda$ is the weight-decay coefficient (\Cref{def:muonw-adamw}); $p$ is a free learning-rate exponent in the Power Lines product constraint.}\par\smallskip
  \scriptsize
  \setlength{\tabcolsep}{2pt}
  \renewcommand{\arraystretch}{1.12}
  \begin{minipage}[t]{0.39\textwidth}
    \vspace{0pt}
    \centering
    \begin{tabular*}{\linewidth}{@{\extracolsep{\fill}}lccc@{}}
      \toprule
      \textbf{MuonW} & $\eta'$ & $\mu'$ & $\lambda'$ \\
    \midrule
    \multicolumn{4}{@{}l}{\textit{Invariant preservation}} \\
    \citetalias{islamov2026role} & $\kappa\eta$ & $\mu^\kappa$ & $\lambda$ \\
    \citetalias{bergsma2025powerlines}\textsuperscript{*} & $\kappa^p\eta$ & $\mu$ & $\kappa^{1-p}\lambda$ \\
    Ours & $\kappa\eta$ & $\mu^\kappa$ & $\lambda$ \\
    \midrule
    \multicolumn{4}{@{}l}{\textit{Bound minimization}} \\
    \citetalias{shulgin2026deriving} & $\sqrt\kappa\,\eta$ & $\mu$ & -- \\
    \citetalias{bu2026convexdominance} & $\sqrt\kappa\,\eta$ & -- & -- \\
    Ours & $\sqrt\kappa\,\eta$ & $\mu$ & $\lambda$ \\
    \midrule
    \multicolumn{4}{@{}l}{\textit{Empirical}} \\
    \citetalias{islamov2026role} & $\eta$ & $\mu$ & $\lambda$ \\
    LLM & $\eta$ & $\mu$ & $\kappa\lambda$ \\
    CIFAR & $\sqrt\kappa\,\eta$ & $\mu$ & $\lambda$ \\
      \bottomrule
    \end{tabular*}
  \end{minipage}\hfill
  \begin{minipage}[t]{0.59\textwidth}
    \vspace{0pt}
    \centering
    \begin{tabular*}{\linewidth}{@{\extracolsep{\fill}}lcccc@{}}
      \toprule
      \textbf{AdamW} & $\eta'$ & $\beta_1'$ & $\beta_2'$ & $\lambda'$ \\
    \midrule
    \multicolumn{5}{@{}l}{\textit{Invariant preservation}} \\
    \citetalias{malladi2022sdes} & $\sqrt\kappa\,\eta$ & $\beta_1^\kappa$ & $\beta_2^\kappa$ & -- \\
    \citetalias{mlodozeniec2025completed} & $\sqrt\kappa\,\eta$ & $\beta_1^\kappa$ & $\beta_2^\kappa$ & $\sqrt\kappa\,\lambda$ \\
    \citetalias{marek2025smallbatch} & -- & $\beta_1$ & $\beta_2^\kappa$ & -- \\
    \citetalias{bergsma2025powerlines} & $\kappa^p\eta$ & $\beta_1$ & $\beta_2$ & $\kappa^{1-p}\lambda$ \\
    \citetalias{compagnoni2025adaptive} & $\sqrt\kappa\,\eta$ & $\beta_1^{\sqrt\kappa}$ & $\beta_2^{\sqrt\kappa}$ & $\sqrt\kappa\,\lambda$ \\
    \midrule
    \multicolumn{5}{@{}l}{\textit{Bound minimization}} \\
    \citetalias{bu2026convexdominance} & $\sqrt\kappa\,\eta$ & -- & -- & -- \\
    \citetalias{schaipp2025surprising} & $\sqrt\kappa\,\eta$ & -- & -- & -- \\
    \midrule
    \multicolumn{5}{@{}l}{\textit{Stability}} \\
    \citetalias{granziol2022learningrates} & $\sqrt\kappa\,\eta$ & -- & -- & -- \\
    \midrule
    \multicolumn{5}{@{}l}{\textit{Empirical}} \\
    LLM & $\eta$ & $\beta_1$ & $\beta_2^\kappa$ & $\lambda$ \\
    CIFAR & $\sqrt\kappa\,\eta$ & $\beta_1$ & $\beta_2$ & $\lambda$ \\
      \bottomrule
    \end{tabular*}
  \end{minipage}
  \endgroup
  \caption{\textbf{Batch-size scaling prescriptions at $B'=\kappa B$.}
    Entries give target values; ``--'' means unspecified.
    We include two scaling rules derived from \citet{islamov2026role} here.
    One is derived by keeping the convergence bound constant across batch sizes.
    The other follows from their empirical experiments.
    We extend Power Lines~\citep{bergsma2025powerlines} to Muon assuming fixed momentum. The original work favors $p=0$ for AdamW, which we also found to be effective for MuonW in our empirical sweep.
    Fixed Adam moments in Power Lines and fixed $\beta_1$ in \citet{marek2025smallbatch} are empirical choices.   \citet{bu2026convexdominance,schaipp2025surprising} concern SGD on convex optimization, with empirical extensions to the optimizers shown.
    For a moment retention coefficient $\beta$ near $1$ and fixed exponent $a$, we use $\beta^a\approx1-a(1-\beta)$ to display moment transfers as powers; Appendix~\ref{sec:scaling_rule_survey} gives the details and original formulas.
    The LLM and CIFAR rows are from results in \Cref{sec:scaling-results}.}
  \label{tab:scaling-recipes}
  \par\vspace{-1em}
\end{table}

\subsection{MuonW-AdamW Optimizer}
\label{sec:hybrid-optimizer}
Following standard practice, we use Muon with decoupled weight decay (MuonW) for matrix parameters and AdamW for embeddings, the output projection, and one-dimensional parameters.
Both apply weight decay separately from the gradient update.

\begin{definition}[MuonW--AdamW]
\label{def:muonw-adamw}
Let $W_t,\theta_t$ be MuonW and AdamW parameters with minibatch gradients $G_t,g_t$, respectively.
For $j\in\{{\rm M},{\rm A}\}$, $\eta_j$ and $\lambda_j$ denote learning rates and weight-decay coefficients; $\mu,\beta_{1,\rm A},\beta_{2,\rm A}$ are momentum and moment retention coefficients. The Nesterov mixing weight $\omega$ multiplies the updated Muon momentum before orthogonalization; this implementation sets $\omega=\mu$.
Starting from $M_0=m_0=v_0=0$, for $t\geq1$,
\[
\begin{aligned} M_t &= \mu M_{t-1}+(1-\mu)G_t, & \widetilde M_t &= \omega M_t+(1-\omega)G_t,\\ m_t &= \beta_{1,\rm A}m_{t-1}+(1-\beta_{1,\rm A})g_t, & v_t &= \beta_{2,\rm A}v_{t-1}+(1-\beta_{2,\rm A})g_t^{\odot2},\\ \widehat m_t &= m_t/(1-\beta_{1,\rm A}^{t}), & \widehat v_t &= v_t/(1-\beta_{2,\rm A}^{t}). \end{aligned}
\]
\[
\begin{aligned} W_{t+1} &= (1-\eta_{\rm M}\lambda_{\rm M})W_t-\eta_{\rm M}\operatorname{Orth}(\widetilde M_t),\\ \theta_{t+1} &= (1-\eta_{\rm A}\lambda_{\rm A})\theta_t-\eta_{\rm A}\widehat m_t/(\sqrt{\widehat v_t}+\epsilon). \end{aligned}
\]
Here $\operatorname{Orth}$ is the Newton--Schulz map with fixed matrix-shape rescaling (\Cref{alg:llm_muon}), $\epsilon>0$ is a fixed stabilizer, and powers, roots, and division in AdamW are elementwise.
\end{definition}

\section{Deriving and Testing Batch-Size Scaling Rules}
\label{sec:scaling-rules}
\label{sec:s2}
\label{sec:fixed_optimizer_scaling_rules}
We apply the two approaches in Section~\ref{sec:scaling-setup} to derive prescriptions for Muon's learning rate, momentum, and weight decay.
We test these theoretically motivated prescriptions within a search of over $200$ candidate scaling rules for MuonW--AdamW, combining plausible choices for how its learning rates, weight decay, and momentum and moment-estimation coefficients change with batch size.
Surprisingly, despite the setting-agnostic promise of both derivation styles, the best scaling rule depends on the precise training setting.
Neither theoretical prescription identifies a rule that is best in both settings.
Appendix~\ref{app:scaling-rules} gives the derivations, assumptions, and detailed comparisons.

\subsection{Two Scaling Rules for Muon}
\label{sec:scaling-approaches}
\label{sec:s2-views}
\label{sec:s2-muon}
We apply both frameworks discussed in \Cref{sec:preliminaries} to derive scaling rules for the MuonW-AdamW hybrid optimizer.
We summarize our candidate scaling rules alongside ones proposed by prior works in~\Cref{tab:scaling-recipes}.

\textbf{Candidate 1: Matching a Muon SDE.}
\label{sec:s2-diffusion-sketch}
Relative to the SGD example, our MuonW update carries momentum, orthogonalizes its update direction, and applies decoupled weight decay.
We therefore model the matrix parameter $W_s$ and momentum $M_s$ together at sample time $s=kB$.
Write $\bar\eta=\eta_{\rm M}/B$ for the learning rate per sample and $\rho=-\log\mu/B$ for the momentum decay per sample.
The mean input to orthogonalization is $q_s=\omega M_s+(1-\omega)\nabla L(W_s)$.
For this calculation, $\omega$ is held fixed as $\mu$ changes; at $B_0$, it equals $\mu$.
Using the individual-gradient covariance $\Sigma$ from Section~\ref{sec:scaling-setup}, a first-order expansion of orthogonalization, with derivative $D\operatorname{Orth}$, gives
\begin{align}
\mathrm dM_s &=\rho\bigl(\nabla L(W_s)-M_s\bigr)\,\mathrm ds+\rho\Sigma(W_s)^{1/2}\,\mathrm d\mathcal W_s,\label{eq:s2-muon-sde-m}\\
\mathrm dW_s &=-\bar\eta\bigl(\operatorname{Orth}(q_s)+\lambda_{\rm M}W_s\bigr)\,\mathrm ds
-\bar\eta(1-\omega)D\operatorname{Orth}(q_s)\Sigma(W_s)^{1/2}\,\mathrm d\mathcal W_s.\label{eq:s2-muon-sde-w}
\end{align}
At a fixed parameter--momentum state, the gradient, its covariance, and the orthogonalization map do not change with $B$.
With $\omega$ fixed, we match the displayed SDEs by keeping $\rho$, $\bar\eta$, and the weight-decay coefficient $\bar\eta\lambda_{\rm M}$ fixed.
At $B=\kappa B_0$, fixed $\rho=-\log\mu/B$ gives $\mu^{(\kappa)}=\mu^\kappa$, and fixed $\bar\eta=\eta_{\rm M}/B$ gives $\eta_{\rm M}^{(\kappa)}=\kappa\eta_{\rm M}$.
Because $\bar\eta$ is fixed, $\bar\eta\lambda_{\rm M}$ stays fixed when $\lambda_{\rm M}$ does too.
\begin{definition}[SDE scaling candidate for Muon]
\label{def:s3-sde-rule}
When changing batch size from $B_0$ to $\kappa B_0$, use
\begin{equation}
\eta_{\rm M}^{(\kappa)}=\kappa\eta_{\rm M},\qquad \mu^{(\kappa)}=\mu^\kappa,\qquad \lambda_{\rm M}^{(\kappa)}=\lambda_{\rm M}.
\label{eq:s2-diffusion-rule-main}
\end{equation}
\end{definition}
Our MuonW implementation instead ties $\omega$ to $\mu$ at every batch size, so this rule is a candidate for that optimizer rather than an exact invariance guarantee.
Appendix~\ref{app:s2-diffusion} proves the coefficient-matching result and gives a finite-horizon approximation bound under explicit noise and regularity assumptions.

\textbf{Candidate 2: Minimizing the Muon convergence bound.}
\label{sec:s2-bound-sketch}
For this candidate, we compare the minimizers of a bound on Muon's Frank--Wolfe gap, a measure of stationarity.
With exact orthogonalization and decoupled weight decay, Muon selects the matrix in a spectral-norm ball most aligned with the negative momentum direction, then moves toward it by an effective step size $\gamma_{\rm FW}=\eta_{\rm M}\lambda_{\rm M}$ \citep{sfyraki2025lions}.
We hold momentum and weight decay fixed.
Let $\mathcal C$ denote this spectral-norm ball and $\mathcal G_{\mathcal C}$ its Frank--Wolfe gap, measuring the largest first-order decrease within the ball (formally defined in Appendix~\ref{app:s2-fw}).
For $W_a$ sampled uniformly from the $\tau/B$ iterates, we prove that there exist constants $c_1,c_2>0$ independent of $B$ and $\gamma_{\rm FW}$, and a remainder $r(B)$ independent of $\gamma_{\rm FW}$, such that
$
\mathbb E\mathcal G_{\mathcal C}(W_a)\leq \frac{c_1 B}{\tau\gamma_{\rm FW}}+c_2\gamma_{\rm FW}+r(B)
$.

The first term spreads the initial objective gap over only $\tau/B$ updates, so it falls as $\gamma_{\rm FW}$ grows; the second is a smoothness cost that rises with the step size.
Ignoring other constraints, their balance gives $\gamma_{{\rm FW},\star}(B)=\sqrt{c_1 B/(\tau c_2)}$, so the bound-minimizing effective step size grows by $\sqrt{\kappa}$ when $B$ grows by $\kappa$.
With weight decay fixed, the same ratio applies to the learning rate.
\begin{definition}[Bound-based scaling candidate for Muon]
\label{def:s3-bound-rule}
When changing batch size from $B_0$ to $\kappa B_0$, use
\begin{equation}
\eta_{\rm M}^{(\kappa)}=\sqrt{\kappa}\,\eta_{\rm M},\qquad \mu^{(\kappa)}=\mu,\qquad \lambda_{\rm M}^{(\kappa)}=\lambda_{\rm M}.
\label{eq:s2-bound-rule-main}
\end{equation}
\end{definition}
The constrained minimizer follows square-root scaling until the effective step-size cap at $1$ becomes active.
Appendix~\ref{app:s2-fw} gives the full bound, its constraints, and the distinction between exact orthogonalization and the finite Newton--Schulz implementation.

\textbf{Relationship to batch--sequence--token (BST) scaling.}
\citet{islamov2026role} study batch-size scaling by deriving a convergence bound for stochastic conditional gradient methods that accounts for momentum and batch size and perform experiments in a language modeling setting similar to us.
Their analysis assumes a Kurdyka--\L{}ojasiewicz (KL) condition relating the dual gradient norm to objective suboptimality and yields a different bound from our \Cref{thm:sfw_gap_bound}.
Concretely, at a fixed sample budget, we extract one rule by preserving their convergence bound, which coincides with the SDE scaling rule, and another rule from their empirical results, which coincides with no scaling (\Cref{tab:scaling-recipes}).
Both underperform the Power Lines prescription across target batches in our language-model sweep described below.
Appendix~\ref{sec:scaling_rule_survey} gives the derivations and detailed comparison. A notable difference is that we use AdamW for non-matrix parameters, whereas \citet{islamov2026role} use optimizers based on steepest descent, which we will discuss further in the next section.

\begin{figure}[t]
    \centering
    \includegraphics[width=\textwidth]{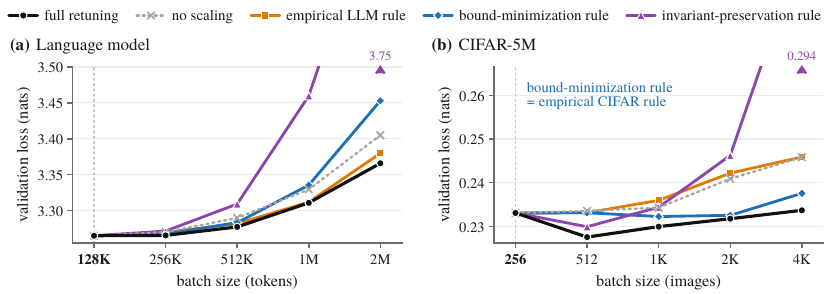}
    \caption{\textbf{No single scaling rule transfers well in both settings.}
    We transfer a tuned configuration from a reference batch to larger batches using the rules in \Cref{tab:scaling-recipes}:
    (a)~MuonW--AdamW language-model pretraining on $1.7$B tokens, with a reference batch of $128\mathrm K$ tokens;
    (b)~a ResNet with GroupNorm on CIFAR-5M, with a reference batch of $256$ images.
    The theoretical rules follow \Cref{def:s3-bound-rule,def:s3-sde-rule} for the matrix parameters, with AdamW counterparts chosen following prior work on SDE and bound minimization.
    \emph{Full retuning} retunes every hyperparameter at each batch by coordinate descent.
    In language modeling, the empirical LLM rule is consistent with Power Lines~\citep{bergsma2025powerlines} and stays closest to full retuning, while both theoretical rules fall behind no scaling from $1\mathrm M$ tokens on.
    On CIFAR-5M, the bound-minimization rule instead stays within $0.006$ nats of full retuning, while the empirical LLM rule does no better than no scaling.
    CIFAR-5M points are medians over three seeds.
    \Cref{app:scaling_experiments} gives full results.}
    \label{fig:muonw_transfer_comparison}
    \par\vspace{-1em}
\end{figure}

\subsection{Sweeping Scaling Rules}
\label{sec:scaling-results}
\label{sec:s2-empirical}
We test whether a single scaling rule transfers well across training settings and target batch sizes. For MuonW--AdamW, we evaluate all combinations of fixed, square-root, or linear scaling for the two learning-rate groups and for each weight-decay coordinate, together with fixed or decay-preserving scaling for Muon momentum and Adam's two moment coefficients. This gives $3^3\times2^3=216$ complete rules in language modeling, where auxiliary weight decay is held fixed\footnote{\Cref{app:scaling_experiments} shows that performance is insensitive to how auxiliary weight decay is scaled.}, and $3^4\times2^3=648$ on CIFAR-5M~\citep{nakkiran2021deepbootstrap}, where the matrix and auxiliary weight decays are scaled separately.

\textbf{Setting.}
We transfer a tuned reference configuration to four larger batches in language modeling and six target batches in one-pass CIFAR-5M training with a GroupNorm ResNet. Here we choose CIFAR-5M to focus on optimization effect instead of generalization effect of batch size.
The reference batches are $128\mathrm K$ tokens and $256$ examples, respectively. Within each setting, all rules share the model, data budget, and schedule shape. We select one rule across target batches by its mean batch-size regret, defined as the equally weighted mean of its validation-loss gap to the best tested rule at each batch. Appendix~\ref{app:s2-protocol} gives the full protocol and scaling conventions.

\textbf{The best rule depends on the setting.}
The selected language-model rule keeps both learning rates fixed and scales matrix weight decay linearly with batch size. The selected CIFAR-5M rule instead scales both learning rates with the square root of batch size and keeps both weight decays fixed. Thus, even for the same optimizer, the two settings favor different scaling prescriptions.
The empirically best CIFAR rule coincides with the \textbf{bound-minimization} rule and the empirically best language modeling rule at 2M is consistent with the \textbf{Power Lines} rule~\citep{bergsma2025powerlines}. One possible explanation is that the run on CIFAR-5M is much closer to convergence compared to the language model experiments, which stay far from convergence~\citep{wen2025riverValley}.
Within language modeling, the best scaling rule also changes across target batch sizes within each setting. The best rules use square-root scaling for both learning rates at $256\mathrm K$ and $512\mathrm K$ tokens, but fixed learning rates and linear matrix weight-decay scaling at $1\mathrm M$ and $2\mathrm M$ (\Cref{fig:muonw_transfer_comparison,app:scaling_experiments}).

\label{sec:scaling-discussion}
\textbf{Keeping momentum coefficients fixed is generally preferable in our sweeps.}
The best common rules in both settings keep Adam's first-moment coefficient $\beta_1$ and Muon's momentum coefficient $\mu$ fixed as batch size changes (\Cref{app:scaling_experiments}).
This suggests keeping these coefficients fixed as a default when transferring hyperparameters across batch sizes.
The observation is consistent with prior theory showing that momentum offers limited benefit over SGD in the small-batch regime when small learning rates and gradient noise govern the dynamics~\citep{wang2024marginalmomentum}, and with noisy-quadratic analyses predicting that momentum's gains emerge primarily at larger batch sizes~\citep{zhang2019algorithmicchoicesmatterbatch}.
One possible explanation is that momentum becomes useful at large batches by increasing the maximum learning rate that training can tolerate.
River-valley theory supports this interpretation by showing that heavy-ball momentum stabilizes larger learning rates in sharp directions, allowing faster progress along a flat, low-loss valley~\citep{lu2026momentumrivervalley}.
This mechanism motivates our interpretation of the Adam and Muon results, although our scaling sweeps do not directly measure their stability thresholds.

\textbf{Relationship to batch--sequence--token (BST) scaling.}
\citet{islamov2026role} study batch-size scaling by deriving a convergence bound for stochastic conditional gradient methods that accounts for momentum and batch size and perform empirical experiments in a language modelling setting similar to us.
Their analysis assumes a Kurdyka--\L{}ojasiewicz (KL) condition relating the dual gradient norm to objective suboptimality and yields a different bound from our \Cref{thm:sfw_gap_bound}.
Regardless, in our empirical comparisons of these scaling rules among hundreds of candidates, we find that the best-performing scaling rule differs fundamentally from either prescription. Concretely, at a fixed sample budget, we extract one rule by preserving their convergence bound, which coincides with the SDE scaling rule, and another from their empirical experiments, which coincides with no scaling (\Cref{tab:scaling-recipes}).
Both underperform the Power Lines prescription across target batches in our language-model sweep.
Appendix~\ref{sec:scaling_rule_survey} gives the derivations and detailed comparison. A notable difference is that we use AdamW for non-matrix parameters, whereas \citet{islamov2026role} use optimizers based on steepest descent, which we will discuss further in the next section.

\section{Batch Size Changes the Best Optimizer}
\label{sec:batch_size_comparison}
\label{sec:llm_experiments}

In this section, we test how batch size affects the ranking of optimizers in the NanoGPT benchmark.
Because the scaling rules are unreliable, we perform extensive hyperparameter tuning for each optimizer at each batch size.
We find frequent \emph{optimizer crossovers}, where the best optimizer at one batch size underperforms at another one.
Our results provide substantive evidence that optimizer benchmarking should be performed over a range of batch sizes.

\textbf{Setting.}
We follow the setting of the Modded-NanoGPT
optimization benchmark~\citep{jordan2026moddednanogpt} and vary the batch size $B\in\{128\mathrm K,512\mathrm K,1\mathrm M,2\mathrm M\}$
tokens while holding the training budget to be roughly
$1.7$B tokens.
We use coordinate descent to tune the hyperparameters extensively~\citep{wen2025fantastic} and report the details in \Cref{app:llm_experiment_details}.
As HyperBall~\citep{wen2026hyperball} performs competitively on the Modded NanoGPT benchmark, we compare two weight-norm control mechanisms applied to the hidden matrices.
With learning rate $\eta$ and weight decay $\lambda$, \textbf{weight decay (W)} shrinks a matrix parameter $W_t$ by $1-\eta\lambda$ at each step.
\textbf{HyperBall (H)}~\citep{wen2026hyperball} instead steps along the normalized update and renormalizes the result, leaving each matrix's Frobenius norm unchanged and thereby constraining optimization to a fixed-norm hypersphere.

\textbf{Experiment Results.} After independent retuning, SOAP leads at the smallest batch, while Shampoo leads at the largest. We observe that crossovers occur at different batch sizes and across different pairs of optimizers. \Cref{fig:track3_batch_transfer_hyperball} shows the same change in winner under HyperBall. We also observe the crossover from SOAP to Shampoo in nanochat setting where we perform hyperparameter transfer following the empirical LLM scaling (\Cref{app:nanochat_optimizer_ce}). These results show that comparing optimizers at a single batch size provides a limited understanding of their efficacy.

\begin{figure}[H]
    \vspace{-1em}
    \centering
    \includegraphics[width=\textwidth]{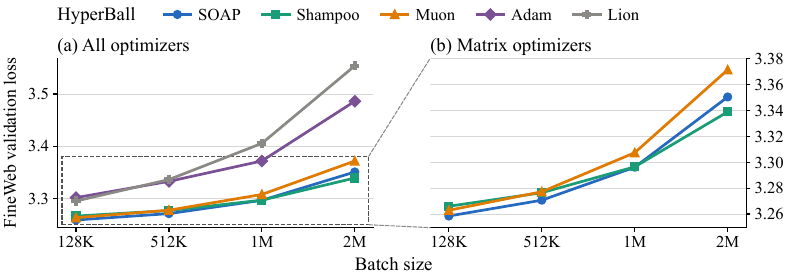}
    \caption{\textbf{The lowest-loss optimizer changes with batch size under HyperBall.}
    Language-model pretraining with an approximately \(1.7\)B-token budget
    and independent tuning of every optimizer at every batch size.
    Left: all five optimizers. Right: Muon, KL-SOAP (labeled SOAP), and
    Shampoo on a zoomed loss axis.
    KL-SOAP has the lowest reported loss from \(128\mathrm K\) through
    \(1\mathrm M\), with a gap of only about \(0.0003\) to Shampoo at
    \(1\mathrm M\); Shampoo has the lowest loss at \(2\mathrm M\).}
    \label{fig:track3_batch_transfer_hyperball}
\end{figure}

 \paragraph{Comparison with~\citet{pethick2025training}.} In this section and the preceding one, we follow common practice by using Adam for embeddings and other non-matrix parameters when benchmarking matrix optimizers. Experiments under similar setup in Figure 2 of \citet{pethick2025training} show that norm-constrained steepest descent can be preferable for these parameters for large batch size training when using matrix optimizers such as Muon or Scion. We leave a broader investigation of how this choice interacts with matrix optimizers to future work.

\suppressfloats[t]
\section{Curvature and Noise Shape Batch-Size Scaling}
\label{sec:theory}
\begin{figure}[t]
    \centering
    \setlength{\abovecaptionskip}{2pt}
    \includegraphics[width=\linewidth]{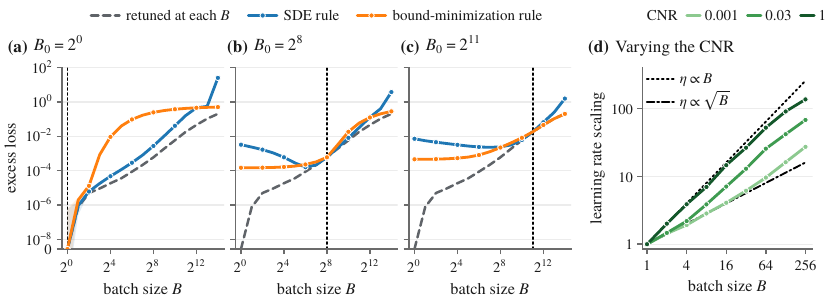}
    \caption{\textbf{Which scaling rule works depends on the reference batch size and on the curvature-to-noise ratio (CNR).}
    (a)--(c)~SignSGD with momentum on a three-direction NQM with a fixed sample budget.
    Gray: hyperparameters retuned at each batch size. Blue and orange: hyperparameters
    tuned at $B_0$ (dotted line), then scaled to each batch size using the SDE and bound-minimization
    rules, respectively. Excess loss is terminal loss minus the retuned loss at $B=1$.
    Shaded bands show one standard error across eight evaluation groups.
    (d)~Learning rates of SignSGD with momentum tuned for terminal loss on a one-dimensional NQM,
    each relative to its own value at $B=1$; we control the CNR by choosing $c$ and $h$.
    Dotted and dash-dotted lines mark linear and square-root scaling.
    See \Cref{app:nqm_oracle_transfer_protocol,app:signsgd_cnr_protocol} for details.}
    \label{fig:nqm_scaling_rules}
    \par\vspace{-1em}
\end{figure}

We study the noisy quadratic model (NQM) to isolate how curvature and gradient noise shape batch-size effects.
Prior work used this model to analyze the benefits of momentum and preconditioning across batch sizes~\citep{zhang2019algorithmicchoicesmatterbatch}, and recent work supports quadratic approximations of local language-model training dynamics~\citep{meterez2026defensequadraticmodel}.

\begin{definition}[$(c,h)$-noisy quadratic model]
\label{def:nqm}
For $\w\in\mathbb R^d$, let $h_i>0$ and $c_i\geq0$ denote the curvature and single-example gradient-noise variance in coordinate $i$.
The $(c,h)$-NQM objective and unbiased minibatch gradient at batch size $B\in\mathbb N_{>0}$ are
\begin{equation}
 L(\w)=\frac12\sum_{i=1}^d h_i\w_i^2,
 \qquad \widehat g_{B,i}(\w)=h_i\w_i+\sqrt{\frac{c_i}{B}}\,z_i,
 \qquad z_i\overset{\mathrm{iid}}{\sim}\mathcal N(0,1).
 \label{eq:overview_nqm_setting}
\end{equation}
Each update uses fresh Gaussian draws, independent of initialization and all previous draws.
\end{definition}

In the NQM setting, we note that the Muon optimizer approximately reduces to SignSGD with momentum, because orthogonalization on a $1\times 1$ block is exactly the coordinate-wise sign operation.
As such, we primarily analyze SignSGD with momentum in this section, and we ignore weight decay here for simplicity of analysis.

We first show how properties of the optimization problem can change the best scaling rule, supporting our prior finding that different training settings favor different prescriptions (\Cref{fig:muonw_transfer_comparison} in \Cref{sec:scaling-rules}).
Specifically, increasing the \emph{curvature-to-noise ratio} (CNR), defined below, shifts the best scaling rule for SignSGD with momentum from approximately square-root toward linear.

\begin{definition}[Curvature-to-noise ratio (CNR)]
\label{def:nqm-cnr}
For a $(c,h)$-NQM, the curvature-to-noise ratio in coordinate $i$ is squared curvature divided by single-example gradient-noise variance:
\begin{equation}
    \mathrm{CNR}_i := h_i^2/c_i, \qquad \text{with } \mathrm{CNR}_i := +\infty \text{ if } c_i=0.
    \label{eq:nqm_cnr}
\end{equation}
\end{definition}

We then address Section~\ref{sec:batch_size_comparison}'s finding that optimizer rankings reverse with batch size despite independent retuning.
We prove that SGD and Newton's method can exhibit such a reversal on a two-dimensional NQM, because each direction exhibits a different CNR and thus a different bias--variance tradeoff in the preconditioner.

\subsection{Scaling Rules Depend on the Training Setting}
\label{sec:nqm_directional_scaling}
\label{sec:nqm_signsgd_scaling}

\textbf{The efficacy of a scaling rule depends on the reference and target batch sizes.}
We study SignSGD with momentum on a three-dimensional noisy quadratic under a fixed sample budget.
At each reference batch size, we tune the learning rate and momentum, then apply scaling rules derived under either the SDE approximation or the convex convergence analysis (\Cref{def:s3-sde-rule,def:s3-bound-rule}).
We compare the resulting terminal loss with that obtained by independently retuning at the target batch size; \Cref{app:nqm_oracle_transfer_protocol} gives further empirical details.
\Cref{fig:nqm_scaling_rules}(a)--(c) shows that the better scaling rule depends on both the reference and target batch sizes.

\textbf{The curvature-to-noise ratio affects the correct scaling.}
On a one-dimensional noisy quadratic, we vary the noise variance while holding curvature, initialization, sample budget, and momentum fixed.
We independently tune the learning rate at each batch size, without weight decay.
\Cref{fig:nqm_scaling_rules}(d) shows that increasing CNR shifts the tuned learning rates from approximately square-root toward linear scaling.
Experimental details are given in \Cref{app:signsgd_cnr_protocol}.

To understand this behavior, recall that SignSGD takes a step of magnitude $\eta$ in the direction determined by the sign of the momentum average.
With momentum retention $\mu\in[0,1)$, the stationary momentum average at a fixed parameter value $w$ is Gaussian.
Its expected sign, $\E[\mathrm{sign}(m)]=\mathrm{erf}\!\left(hw\sqrt{\frac{1+\mu}{1-\mu}\cdot\frac{B}{2c}}\right)$, determines the expected update.
When noise dominates, this term grows as $\sqrt{B}$ and preserving displacement per processed sample motivates square-root scaling. When signal dominates, its magnitude saturates at $1$ and linear scaling compensates for fewer updates per sample.
\Cref{app:signsgd_cnr_mechanism} gives the derivation based on this \emph{movement ansatz}, that is, choosing learning rates that preserve expected displacement per processed sample at a fixed parameter value.

\Needspace{14\baselineskip}
\subsection{Batch Size Can Reverse Optimizer Rankings After Retuning}
\label{sec:nqm_main}

\begin{wrapfigure}{r}{0.42\textwidth}
    \vspace{-1.2em}
    \centering
    \includegraphics[width=\linewidth]{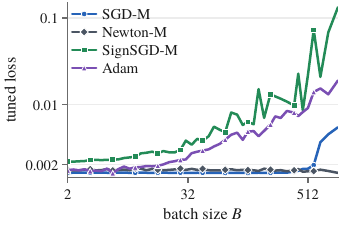}
    \vspace{-1.6em}
    \caption{\textbf{Optimizer crossover on a noisy quadratic.}
    Tuned loss on a three-direction NQM, with each optimizer retuned at every batch size; ``-M'' denotes momentum.}
    \label{fig:nqm_optimizer_crossover}
    \vspace{-1em}
\end{wrapfigure}

The preceding section shows that directions with different CNRs can require different scaling exponents to match their individually tuned learning rates.
We now show that these directional differences can also produce optimizer crossovers even after optimal hyperparameter tuning (\Cref{sec:batch_size_comparison}).
We compare SGD with Newton's method on a two-dimensional noisy quadratic, where the two directions have different CNRs because one is flat and noisy while the other is sharp and noiseless.

First consider the updates without momentum.
Both optimizers take the form
\begin{equation}
    w_{t+1,i}=(1-a_i h_i)w_{t,i}-a_i\sqrt{c_i/B}\,z_{t,i},
    \label{eq:nqm_crossover_update}
\end{equation}
with step size $a_i=\eta$ for SGD and $a_i=\eta/h_i$ for Newton's method.
At the same learning rate, SGD reduces initialization error more slowly in flatter directions, whereas Newton's method reduces it at the same rate in every direction while amplifying noise in the flat direction by $1/h_i$.
Batch size changes both the noise per update and the number of updates available, changing the balance between these effects.
The following theorem establishes that this bias-variance tradeoff can reverse optimizer rankings even after optimal tuning.

\begin{theorem}
\label{thm:nqm_crossover_main}
For every sufficiently large sample budget $T$, there exists a two-dimensional NQM on which optimally tuned SGD outperforms optimally tuned Newton at $B=1$, but the ranking reverses at $B=T$.
This holds both without momentum and with optimally tuned momentum.
\end{theorem}
\Cref{app:nqm_crossover_proof} provides a complete proof of the theorem as well as empirical evidence of the bias-variance tradeoff.
Beyond this construction, a three-dimensional noisy quadratic exhibits the same crossover with every optimizer independently tuned at each batch size (\Cref{fig:nqm_optimizer_crossover}), and \Cref{app:matrix_results} shows crossovers on three matrix regression tasks.

\suppressfloats[t]
\begin{figure}[t]
    \centering
    \includegraphics[width=\textwidth]{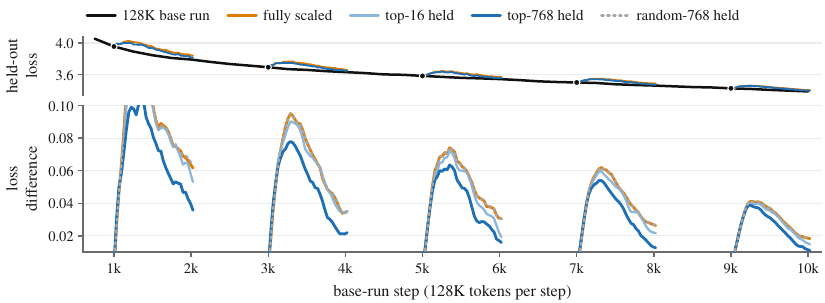}
    \caption{\textbf{Holding the sharpest $0.001\%$ of directions at $128\mathrm K$ removes up to $59.5\%$ of the $2\mathrm M$ branch penalty.}
    Top: held-out loss of the $128\mathrm K$ base run (black) and of the fully scaled and top-$768$ held $2\mathrm M$ branches started at steps $1{,}000$, $3{,}000$, $5{,}000$, $7{,}000$, and $9{,}000$.
    Bottom: held-out loss minus the control branch's loss at each evaluation for the fully scaled, top-$16$ held, top-$768$ held, and random-$768$ held (dotted) branches, evaluated every $32$ base steps on $1$M held-out tokens and smoothed over three evaluations.}
    \label{fig:lm_directional}
    \par\vspace{-1em}
\end{figure}

\section{Sharp Directions Carry Most of the Local Batch-Size Penalty in Pretraining}
\label{sec:lm_directional}

Based on our analysis in~\Cref{sec:theory}, we conjecture that the complex interaction between batch size and optimizers is an artifact of the heterogeneity of the CNR across directions. We validate this in a real language model in this Section. We show that when a scaling rule transfers the model of \Cref{sec:llm_experiments} to a larger batch, the resulting loss increase is concentrated in the sharp directions.

\textbf{Directional branching.}
Following the critical-batch-size measurement in \citet{merrill2025criticalbatchsizerevisited}, we branch from checkpoints of a base run.
The base run is the tuned MuonW run in \Cref{sec:llm_experiments} at $B=128\mathrm K$ tokens per step, and we call its checkpoints at every $1{,}000$ steps \emph{anchors}.
From each anchor, a branch trains for $1{,}024$ further steps ($134$M tokens) on the same data as the base run.
Let $w_t\in\mathbb{R}^N$ denote the $N\approx85$M hidden-matrix parameters at step $t$, flattened into one vector.
Each branch runs two copies of Muon, both initialized from the anchor's optimizer state and fed the same $B$-token gradients.
The \emph{small-batch} copy computes an update direction $u_t^{(B)}$ from each gradient.
The \emph{large-batch} copy computes an update direction $u_t^{(\kappa B)}$ from the mean of the last $\kappa=16$ gradients once every $\kappa$ steps, which emulates training at $\kappa B=2\mathrm M$ tokens; we set $u_t^{(\kappa B)}=0$ at the other steps.
Let $\eta_t$ and $\lambda_t$ be the base run's matrix learning rate and weight-decay coefficient.
Following the bound-minimization rule (\Cref{def:s3-bound-rule}), the large-batch copy uses the learning rate $\eta_t^{(\kappa B)}$, equal to $\sqrt{\kappa}$ times the mean of $\eta_t$ over those $\kappa$ steps.
Given a subspace of $\mathbb{R}^N$ with orthogonal projector $Q$, called the \emph{held subspace}, the branch updates
\[
  w_{t+1}=(1-\eta_t\lambda_t)\,w_t-\eta_t\,Q\,u_t^{(B)}-\eta_t^{(\kappa B)}\,(I-Q)\,u_t^{(\kappa B)} .
\]
The branch thus takes small-batch updates inside the held subspace and large-batch updates outside it.
The \emph{control branch} ($Q=I$) continues the base run at $128\mathrm K$, and the \emph{fully scaled branch} ($Q=0$) moves every direction to $2\mathrm M$.
The remaining parameters keep the base run's AdamW updates.

\textbf{Choosing the held subspace.}
We hold the $r$ sharpest directions at the anchor, for $r\in\{16,768\}$.
To measure sharpness, we precondition the Gauss--Newton approximation $G$ of the loss Hessian by the Jacobian $P$ of Muon's orthogonalization map at the anchor.
We then take the top-$r$ eigenvectors of $PG$ as the sharpest directions after preconditioning induced by Muon and let $Q$ be the orthogonal projector onto their span (see \Cref{app:lm_directional_protocol} for details).
We prove that a local linearization of Muon's dynamics around the anchor is a noisy quadratic model with the same curvature spectrum (\Cref{app:lm_directional_local_nqm}), so these are the highest-curvature directions of that model.
We ignore the anisotropy in gradient noise in this experiment and leave that for future work.
As a baseline, we also hold a random $768$-dimensional subspace.
The held subspace stays fixed throughout each branch.

\textbf{Result.}
We score each branch by its \emph{local branch penalty}: its validation loss at the end of the branch minus that of the control branch from the same anchor.
\Cref{app:lm_directional_protocol} details the construction and tabulates the numbers behind \Cref{fig:lm_directional}.
Keeping the base batch in only the sharpest $0.001\%$ of directions removes $35$--$59.5\%$ of the $2\mathrm M$ branch penalty for branches started between steps $1{,}000$ and $9{,}000$ (\Cref{tab:lm_directional_numbers}), indicating that the bulk of the branch penalty
was driven by sharp directions.
The recovery share grows with held-rank $r$; the top $16$ directions alone remove up to $27\%$ (at step $2{,}000$ they remove nothing, within run-to-run noise).
A random $768$-dimensional subspace, by contrast, changes the penalty by at most $3\%$ at every starting step. The effect disappears once the learning rate decays at the end of training: for branches started at steps $11{,}000$ and $12{,}000$, no held subspace significantly changes the penalty. This is potentially due to the fact that we ignore gradient noise in the direction selection.
Consistent with \Cref{sec:theory}, the results indicate that the bound-minimization rule induces different batch-size penalties across directions. Our results suggest that future work on batch-size scaling should account for variation in CNR across directions and explore how to use it.

\section{Acknowledgements}
Sadhika Malladi completed this work partially while doing a postdoc at Microsoft Research in NYC. We thank Akshay Krishnamurthy, Kaifeng Lyu, and Will Merrill for helpful discussions. We gratefully acknowledge Princeton Language and Intelligence (PLI) and Microsoft Research for providing computational resources.

\appendix
\clearpage
\section*{Appendix}
\addcontentsline{toc}{section}{Appendix}

\textbf{Appendix roadmap.}
\begin{itemize}[leftmargin=1.5em,itemsep=2pt,topsep=2pt]
  \item \Cref{app:related_work} discusses related work on batch-size scaling rules, the critical batch size, optimizer comparisons, batch-size-invariant optimizers, noisy quadratic models, and preconditioning under gradient noise.
  \item \Cref{app:scaling-rules} fixes the optimizer notation and scaling coordinates, states the invariant-preservation and bound-minimization views of a scaling rule, and proves the results used in \Cref{sec:scaling-rules}.
  \item \Cref{app:nqm_details} interleaves the noisy quadratic analysis and experiments: finite-budget risks and optimizer comparisons, finite-budget SignSGD batch-size transfer, and learning-rate scaling for momentum-SignSGD and RMSProp.
  \item \Cref{app:matrix_results} defines the matrix model and reports the additional results obtained with it.
  \item \Cref{app:cifar_experiments} gives the shared scaling-rule design, both empirical sweeps, the auxiliary weight-decay experiment, the comparison across settings, and the complete loss tables.
  \item \Cref{app:optimizer_crossover_experiments} collects the supporting language-model optimizer-crossover results, optimizer definitions, the tuning protocol, final hyperparameters, and complete one-coordinate sweep tables.
  \item \Cref{app:lm_directional_protocol} gives the directional-branching protocol and numerical results, and establishes a local noisy-quadratic correspondence for Muon with a finite-horizon approximation bound.
  \item \Cref{app:nanochat_optimizer_ce} measures optimizer compute efficiency along a batch-size scaling ladder from \(256\mathrm K\) to \(2\mathrm M\) tokens per update, using optimizer-specific loss--compute curves and a shared MuonW reference.
\end{itemize}

\section{Related Work}
\label{app:related_work}

\textbf{Batch-size scaling rules.}
A scaling rule prescribes how hyperparameters tuned at one batch size should change at another.
\citet{goyal2017accurate} popularized the linear learning-rate rule for SGD with an informal argument, which was later justified as preserving an SDE approximation of SGD \citep{smith2018bayesian,jastrzebski2017three}.
\citet{malladi2022sdes} use SDE approximations to derive a square-root learning-rate rule for RMSprop and Adam that also rescales the moment coefficients and $\epsilon$.
\citet{li2021validity} give a testable necessary condition for the SDE approximation of SGD, and hence its justification of the linear rule, to hold.
Later work preserves narrower invariants: an asymptotic expected-loss bound together with the relative speeds of AdamW's moments and iterates \citep{compagnoni2025adaptive}, the token half-life of Adam's second moment \citep{marek2025smallbatch}, or AdamW's weight-decay averaging timescale at fixed model size and training-data budget \citep{bergsma2025powerlines}; the last two are identified empirically.
A second line takes convergence bounds for LMO-based methods such as Muon \citep{kovalev2025understanding,pethick2025training} and minimizes them over the learning rate and momentum separately at each batch size \citep{shulgin2026deriving}.
\citet{sato2025muoncriticalbatch} instead minimize the sample complexity implied by a Muon convergence bound over the batch size and obtain a lower bound on the critical batch size that depends on momentum and weight decay.
Their experiments set learning rates at other batch sizes by square-root scaling without testing whether this transfer is optimal.
In a blog post, \citet{cesista2025sdcbs} extend this critical-batch-size analysis to steepest descent under arbitrary norms with Nesterov momentum and decoupled weight decay.
They also give a heuristic argument for square-root learning-rate scaling at a fixed number of steps, assuming the map from momentum to update direction is locally Lipschitz.
Their language-model experiments, run at a fixed token budget rather than a fixed number of steps, find that the optimal learning rate grows roughly as the square root of batch size for AdamW and Muon.
Related transfer results rescale hyperparameters with width through a reparameterization \citep{yang2021tensorprogramsv}, compare width-scaling prescriptions under several parameterizations and optimizers \citep{everett2024scalingexponents}, and rescale AdamW's learning rate, weight decay, $\epsilon$, and moment coefficients with batch size and training duration \citep{mlodozeniec2025completed}.
\citet{everett2024scalingexponents} compare a grid of width-scaling prescriptions at a fixed batch size.
We derive batch-size scaling prescriptions from both SDE approximations and convergence bounds for the same optimizer, include them in grids of $216$ and $648$ complete rules, and compare the grids across two training settings.
No tested rule is best in both settings, and the theoretically derived prescriptions disagree with each other.

\textbf{Critical batch size.}
The critical batch size marks where increasing the batch size stops reducing the number of updates to a target loss proportionally: below it, doubling the batch roughly halves the number of updates, and above it, returns diminish.
\citet{mccandlish2018empirical} predict it from a scalar gradient noise scale that weights gradient variance and the mean gradient by curvature; their simpler proxy is the total gradient variance divided by the squared gradient norm.
\citet{shallue2019measuring} measure steps to a target across batch sizes and find that the end of the perfect-scaling region depends on the model, the optimizer, and to a lesser extent the dataset.
\citet{zhang2019algorithmicchoicesmatterbatch} reproduce the optimizer dependence in experiments and in a noisy quadratic model: momentum and preconditioning speed up noise-free convergence but amplify the noise floor, so they extend perfect scaling to larger batches.
Recent work fits how the critical batch size grows with training-data size in language-model pretraining, finding little dependence on model size \citep{zhang2024criticalbatchsize,bergsma2025powerlines}.
For Adam-style updates, the optimal learning rate can first rise and then fall as batch size grows \citep{li2024surge}.
\citet{merrill2025criticalbatchsizerevisited} measure it along a run by branching from checkpoints at larger batch sizes and warm the batch size up accordingly; \Cref{sec:lm_directional} reuses that branch construction direction by direction.
The critical-batch-size estimates in these works summarize training by a single scalar, even when the underlying model resolves individual directions.
In our noisy quadratic analysis, directions with different curvature-to-noise ratios can favor different batch-size scaling rules.
A single prescription can then preserve behavior in some directions while changing it in others, a distinction that a total-loss measurement alone does not resolve.

\textbf{Comparing optimizers.}
Optimizer benchmarks establish rankings under a fixed protocol.
\citet{choi2019empirical} show that a reported ranking can reverse when the hyperparameter search space changes, and \citet{schmidt2021crowdedvalley} find that no optimizer dominates across tasks, even at the largest of their tuning budgets.
AlgoPerf \citep{dahl2023algoperf} standardizes the tuning budget and workload set but leaves the batch size to each submission.
The Modded-NanoGPT optimization benchmark \citep{jordan2026moddednanogpt} fixes the architecture, data, and batch size and ranks optimizers by the number of steps to a target loss.
For language-model pretraining, \citet{zhao2025deconstructing} report that Adam, Adafactor, Lion, and Signum perform comparably once tuned while SGD lags, and \citet{wen2025fantastic} find that measured speedups over AdamW shrink with model size and depend on the tuning protocol.
\citet{qiu2025matrixpreconditioned} report that Muon, SOAP, and Shampoo each keep a roughly $1.4\times$ gain in compute efficiency over AdamW from $190$M to $1.4$B parameters when the learning rate follows $\mu$P and weight decay scales inversely with width, at a fixed batch size.
Their comparison measures forward and backward FLOPs needed to reach the same loss and excludes optimizer-update costs.
Several studies vary batch size directly.
Tuned SGD approaches Adam at small batch sizes \citep{marek2025smallbatch,sreckovic2025batchsize}, and for $124$M-parameter models, sign-based methods, MARS, and Prodigy catch up with AdamW as the batch grows from $16$K to $262$K tokens \citep{semenov2025benchmarkingoptimizers}.
Muon needs fewer tokens than AdamW to reach a target loss, including beyond the critical batch size \citep{shah2025practicalefficiency}, and SOAP's advantage over both AdamW and Shampoo grows from $256$K to $2$M tokens per batch \citep{vyas2024soap}.
\citet{khona2026soapmuonbeyond} report gains over AdamW for Muon and SOAP variants in large-scale language-model pretraining at global batch sizes up to $100$M tokens.
They address SOAP's large-batch instability by updating its preconditioner eigenbasis every step using the current gradient.
Their optimizer comparisons transfer learning rates by matching the root mean square of parameter updates across optimizers.
These studies show that optimizer comparisons depend on the tuning protocol, model scale, and batch size.
We hold the model, data distribution, and token budget fixed and retune five optimizers independently at every batch size.
The best optimizer changes: KL-SOAP \citep{lin2026klshampoo}, a variant of SOAP, leads at $128$K tokens per batch and Shampoo \citep{gupta2018shampoo} leads at $2$M, under both weight decay and HyperBall.
This large-batch shift toward Shampoo runs opposite to the SOAP--Shampoo comparison of \citet{vyas2024soap}, which uses SOAP rather than KL-SOAP and a different model.

\citet{zhang2026beyond} observe shifts from SGD toward Adam as batch size increases at a fixed training-sample budget across vision, language, genomics, and graph tasks.
Their experiments show that data and architecture affect the batch size at which the preferred optimizer changes.
They model these crossovers by comparing nonconvex convergence bounds for normalized SGD and SignSGD, both with momentum, used as proxies for clipped SGD and Adam, respectively.
Our experiments examine crossovers among adaptive optimizers, including matrix preconditioners.
Our noisy quadratic analysis proves a reversal in expected loss between SGD and Newton's method under independently optimal learning rate and momentum.

\textbf{Batch-size-invariant optimizer design.}
Another response to batch size is to modify the optimizer.
\citet{wang2024batchinvariant} propose Batch Size Invariant Adam, which modifies the second-moment update to average squared fixed-microbatch gradients instead of squaring the averaged minibatch gradient.
In their small-update limit, linear scaling of the learning rate and of the moment refresh rates $1-\beta_1$ and $1-\beta_2$ yields the same first-order updates of the parameters and both moment estimates after processing the same microbatches.
\citet{johnson2020adascale} instead adapt the SGD learning rate online from estimates of the gradient variance and norm.
\citet{hilton2022batchsizeinvariance} make PPO-style policy optimization batch-size invariant by decoupling the proximal policy, which limits the update size, from the behavior policy used for importance sampling.
These approaches modify the training algorithm.
Our scaling analysis keeps the MuonW--AdamW hybrid fixed and rescales its hyperparameters across batch sizes, so Batch Size Invariant Adam is excluded from \Cref{tab:scaling-recipes}.

\textbf{Noisy quadratic models.}
The noisy quadratic model represents training as a separable quadratic with a curvature and a gradient-noise variance per direction.
\citet{schaul2013nomorepesky} use it to derive a greedy per-coordinate optimal learning rate: the inverse curvature scaled by the fraction of the expected squared gradient that comes from its mean.
\citet{wu2018shorthorizon} use it to show that when the problem is both noisy and ill-conditioned, greedily optimal learning-rate schedules decay too quickly and are suboptimal over longer horizons.
\citet{zhang2019algorithmicchoicesmatterbatch} use it to predict how momentum, preconditioning, and iterate averaging change the perfect-scaling region, comparing optimizers across batch sizes by the number of steps to a target loss.
For high-dimensional least-squares regression, \citet{lee2022trajectoryminibatch} identify a batch fraction above which minibatch heavy-ball momentum attains the full-batch accelerated rate and larger batches no longer help.
\citet{meterez2026defensequadraticmodel} argue that quadratic models capture local optimization dynamics during language-model training.
We fix the total sample budget, so batch size trades gradient noise against the number of remaining updates.
We prove that this tradeoff reverses the ordering of two optimizers under independently optimal learning rate and momentum (\Cref{thm:nqm_crossover_main}).

\textbf{Preconditioning and gradient noise.}
Adaptive optimizers such as Adam and SOAP build their preconditioners from second-moment statistics of minibatch gradients.
Adam maintains an exponential moving average of elementwise squared gradients, while SOAP also averages $GG^\top$ and $G^\top G$ for matrix-valued gradients $G$.
At a fixed parameter value, averaging $B$ independent per-example gradient vectors with mean $\mu$ and covariance $\Sigma$ gives a minibatch gradient $g_B$ satisfying $\mathbb{E}[g_Bg_B^\top]=\mu\mu^\top+\Sigma/B$.
These second moments therefore contain both the mean-gradient contribution and a noise contribution that decreases with batch size; averaging the statistics over time does not remove that dependence.
\citet{morwani2024shampooperspective} study Shampoo's approximation to the Adagrad and Gauss--Newton matrices.
When Shampoo's factors are built from real-label minibatch gradients, they find that its approximation to the Gauss--Newton matrix degrades as the batch size used to estimate the matrix increases.
These statistics need not measure curvature: \citet{kunstner2019limitations} show that the empirical Fisher does not generally approximate the Fisher, and discuss gradient-noise adaptation as an alternative interpretation of empirical-Fisher preconditioning.
Preconditioning can also hurt generalization.
\citet{wadia2021whitening} show that whitening and certain unregularized, exact second-order methods can discard information useful for generalization in models with a fully connected, isotropically initialized first layer.
\citet{amari2021preconditioning} decompose the generalization error of preconditioned gradient descent in overparameterized ridgeless regression into bias and variance.
Natural gradient minimizes the variance term within their class of preconditioners, while the bias comparison depends on how the signal aligns with the features; the overall comparison therefore depends on label noise, misspecification, and signal alignment.
\citet{kunstner2023noisenotmain} vary batch size up to the full batch and conclude that minibatch noise is not the main cause of the gap between SGD and Adam on transformers.

\section{Derivations and comparison of batch-size scaling rules}
\label{app:scaling-rules}
This appendix derives the two Muon scaling candidates in Section~\ref{sec:scaling-rules} and compares them with previously published prescriptions.
The first preserves a specified continuous-time model at matched data consumption.
The second minimizes a stationarity bound at fixed data budget.
Their different objectives lead to different learning-rate rules.

\subsection{Setup and notation}
\label{sec:scaling_notation}
\label{app:s2-ema}
We compare a reference batch size \(B\) with \(B'=\kappa B\), where \(\kappa>0\), holding the model and data budget \(\tau\) fixed.
The corresponding number of updates is \(n=\tau/B\).
The budget and batch use the same unit: tokens for language modeling and examples for image classification.
When comparing learning-rate schedules, we hold their shape fixed as a function of the fraction of the budget consumed.

We use the optimizer notation of Definition~\ref{def:muonw-adamw},
omitting the parameter-group subscripts \({\rm M}\) and \({\rm A}\) when unambiguous.

Let \(\w\) denote a generic parameter vector, let \(\xi\) be a sample from the data distribution, and let \(\ell(\w;\xi)\) be its loss.
Define \(L(\w)=\E_\xi[\ell(\w;\xi)]\) and \(\g(\w)=\grad L(\w)\).
Index updates by \(k\geq1\), with initial momentum \(\m_1=0\).
Conditional on the optimizer history \(\mathcal F_k\) before step \(k\), let \(z_k\) denote the scaled, zero-mean minibatch-gradient noise and \(\Sig(\w_k)\) its covariance.
Write the minibatch gradient as
\begin{equation} \ghat_{\batch}(\w_k)=\g(\w_k)+\batch^{-1/2}z_k, \qquad \E[z_k\mid\mathcal F_k]=0, \qquad \Cov(z_k\mid\mathcal F_k)=\Sig(\w_k). \label{eq:sde_gradient_model} \end{equation}
For matrix parameters, covariance is defined after vectorization.
Let \(\mom\in(0,1)\) be the momentum retention coefficient, so its refresh rate is \(1-\mom\).
With Nesterov mixing weight \(\omega\in[0,1]\) on the updated momentum, direction map \(\Dir\), learning rate \(\lr>0\), and weight decay \(\wdc\geq0\), the momentum state \(\m_k\), corrected input \(\widetilde\m_{k+1}\), and parameter state \(\w_k\) evolve as
\begin{align} \m_{k+1} &= \mom\m_k+(1-\mom)\ghat_{\batch}(\w_k), \label{eq:ema_momentum_update}\\ \widetilde\m_{k+1} &= \omega\m_{k+1}+(1-\omega)\ghat_{\batch}(\w_k), \label{eq:lionk_stochastic_corrected_m}\\ \w_{k+1} &= \w_k-\lr\Dir(\widetilde\m_{k+1})-\lr\wdc\w_k. \label{eq:weight_update} \end{align}
In this theoretical model, \(\omega\) and \(\mom\) are independent hyperparameters.
Setting \(\Dir=\operatorname{Orth}\) gives the Muon variant studied here.
The experimental MuonW update in Definition~\ref{def:muonw-adamw} is the restricted case \(\omega=\mom\).
That restriction is not imposed on the theoretical batch-size transport.
The generic recursion is also used in the bound calculation in \Cref{sec:convex}.

\subsection{Matching the sample-time Muon diffusion}
\label{sec:sde}
\label{app:s2-diffusion}
We state a sample-time diffusion and prove that the scaling rule below preserves its law.
One step processes \(\batch\) samples, so the sample time before update \(k\) is \(s_k=(k-1)\batch\).
Define the parameter step per sample and exact exponential-moving-average (EMA) decay rate per sample by
\begin{equation} \gam=\frac{\lr}{\batch}, \qquad \rhoMom=-\frac{\log\mom}{\batch}. \label{eq:sde_sample_coefficients} \end{equation}

We will assume \(\Dir\) is a twice continuously differentiable direction map in real arithmetic,\footnote{Note that the idealized Muon instantiation doesn't have this property but the matrix sign operation is empirically simulated by a smooth matrix polynomial.} with Jacobian \(J(q)=D\Dir(q)\).
Let \(\sigma(\w)=\Sig(\w)^{1/2}\) be a covariance square root and let \(W_s\) be standard Brownian motion in the vectorized parameter space.
At the continuous-time state \((\w_s,\m_s)\), denote the mixture of momentum and the mean gradient by \(q_s\).
Consider the sample-time diffusion
\begin{align} q_{s} &=\omega\m_{s}+(1-\omega)\g(\w_{s}), \notag\\ \dd\m_{s} &=\rhoMom(\g(\w_{s})-\m_{s})\dd s +\rhoMom\sigma(\w_{s})\dd W_{s}, \label{eq:sde_m}\\ \dd\w_{s} &=-\gam(\Dir(q_{s})+\wdc\w_{s})\dd s -\gam(1-\omega)J(q_{s})\sigma(\w_{s}) \dd W_{s}. \label{eq:sde_w} \end{align}
Both equations use the same Brownian motion, reflecting the shared minibatch noise in the parameter and momentum updates.
All coefficients are evaluated at the current state, so the model uses the It\^o convention.
\Cref{thm:sde_scaling} proves exact invariance of this SDE, while \Cref{prop:sde_approximation} bounds its error relative to the discrete recursion under additional assumptions.

\begin{proposition}[Coefficient matching with independent momentum coefficients]
\label[proposition]{thm:sde_scaling}
\label[proposition]{prop:s2-coefficient-matching}
Fix the objective, noise covariance, direction map, and initial joint law of parameters and momentum.
Suppose the coefficients of \eqref{eq:sde_m}--\eqref{eq:sde_w} are globally Lipschitz and have at most linear growth.
For \(\batch'=\kappaBatch\batch\), the transformation
\begin{equation} \lr'=\kappaBatch\lr, \qquad \mom'=\mom^{\kappaBatch}, \qquad \omega'=\omega, \qquad \wdc'=\wdc \label{eq:sde_scaling_rule} \end{equation}
preserves the law of this diffusion at matched sample time.
\end{proposition}

\begin{proof}
At a joint state \(x=(\m,\w)\), let \(q=\omega\m+(1-\omega)\g(\w)\).
The joint drift \(b\) and diffusion coefficient \(A\) of \eqref{eq:sde_m}--\eqref{eq:sde_w} are
\[
 b(x)=
 \begin{pmatrix}
   \rhoMom(\g(\w)-\m)\\
   -\gam(\Dir(q)+\wdc\w)
 \end{pmatrix},
 \qquad
 A(x)=
 \begin{pmatrix}
   \rhoMom\sigma(\w)\\
   -\gam(1-\omega)J(q)\sigma(\w)
 \end{pmatrix}.
\]
Thus the joint process \(X_s=(\m_s,\w_s)\) satisfies
\(\dd X_s=b(X_s)\dd s+A(X_s)\dd W_s\).
Write primes for the coefficients and process under the transformed hyperparameters.
Since \(\batch>0\), \(\kappaBatch>0\), and \(\mom\in(0,1)\),
\[
 \gam'=\frac{\lr'}{\batch'}
       =\frac{\kappaBatch\lr}{\kappaBatch\batch}=\gam,
 \qquad
 \rhoMom'=-\frac{\log\mom'}{\batch'}
         =-\frac{\log(\mom^{\kappaBatch})}{\kappaBatch\batch}
         =-\frac{\log\mom}{\batch}=\rhoMom.
\]
Moreover, \(\omega'=\omega\) and \(\wdc'=\wdc\).
The objective, direction map, and covariance square root are fixed, so at every state \(x\) we have \(q'=q\), \(b'(x)=b(x)\), and \(A'(x)=A(x)\).
In particular, the full joint covariance \(A'(x)A'(x)^\top=A(x)A(x)^\top\) is unchanged.

For each deterministic initial state, the global Lipschitz and linear-growth assumptions give a unique global strong solution.
Couple the reference and transformed equations using that same initial state and the same Brownian motion.
They then have identical coefficients and driving paths, so pathwise uniqueness gives \(X_s'=X_s\) for all \(s\geq0\) almost surely.
Their path laws therefore agree for every deterministic initial state.
Integrating these laws against the common initial joint distribution proves equality of the path laws for the prescribed initialization, and hence at matched sample time.
\end{proof}

\textbf{Approximation of the discrete updates.}
Let \(U_k=(\m_{k+1},\w_{k+1})\) be the joint discrete state after \(k\) updates of \eqref{eq:ema_momentum_update}--\eqref{eq:weight_update}, and let \(X_s=(\m_s,\w_s)\) solve \eqref{eq:sde_m}--\eqref{eq:sde_w}.
Write \(b,A\) for the joint drift and diffusion coefficients displayed in the preceding proof.
For a sample horizon \(\tau=nB\), define the coefficients in normalized time \(t=s/\tau\) by \(b_\tau=\tau b\) and \(A_\tau=\sqrt\tau A\).
For a random vector \(V\), write \(\|V\|_{L^2}=(\E\|V\|^2)^{1/2}\).
Vector norms are Euclidean after vectorization, and \(\|\cdot\|_{\rm F}\) denotes the Frobenius norm.

\begin{proposition}[Finite-horizon diffusion approximation]
\label[proposition]{prop:sde_approximation}
Fix \(B>0\), \(\eta>0\), \(\mu\in(0,1)\), \(\omega\in[0,1]\), \(\lambda\geq0\), and an integer \(n\geq1\).
Let \(\xi_k\) be independent standard Gaussian vectors, independent of \(U_0\), and assume the gradient noises satisfy \(z_{k+1}=\sigma(\w_{k+1})\xi_k\).
Suppose \(\|\sigma(\w)\|_{\rm F}\leq S\) for all \(\w\), and the operator norms of \(D\Dir\) and \(D^2\Dir\) are bounded by \(K\) and \(H\) on the entire parameter space.
Assume that for constants \(L,C\geq0\) and all joint states \(x,y\),
\begin{align*}
 \|b_\tau(x)-b_\tau(y)\|^2+\|A_\tau(x)-A_\tau(y)\|_{\rm F}^2
   &\leq L^2\|x-y\|^2,\\
 \|b_\tau(x)\|^2+\|A_\tau(x)\|_{\rm F}^2
   &\leq C(1+\|x\|^2).
\end{align*}
Suppose \(\E\|U_0\|^2<\infty\), and let \(Q<\infty\) bound the momentum lag over this horizon:
\[
 \max_{0\leq k<n}\|\g(\w_{k+1})-\m_{k+1}\|_{L^2}\leq Q.
\]
Define the momentum-coefficient discrepancy \(\delta_B\), the direct noise coefficient \(c\), and the direction-error bound \(\mathcal R_B\) by
\begin{align*}
 \delta_B&=\frac{1-\mu}{B}-\rhoMom,\qquad c=1-\omega\mu,\\
 \mathcal R_B&=K\omega(1-\mu)\left(Q+\frac{S}{\sqrt B}\right)
             +\frac{\sqrt3 Hc^2S^2}{2B}.
\end{align*}
There is a finite constant \(C_{\rm E}=C_{\rm E}(L,C,\E\|U_0\|^2)\) and a coupling with \(X_0=U_0\) such that
\begin{equation}
 \max_{0\leq k\leq n}\|U_k-X_{kB}\|_{L^2}
 \leq C_{\rm E}\sqrt{\frac{B}{\tau}}
 +e^{1+L+L^2}
  \sqrt{4\delta_B^2(\tau^2Q^2+\tau S^2)
            +2\tau^2\gam^2\mathcal R_B^2},
 \label{eq:sde_approximation_bound}
\end{equation}
Moreover, \(|\delta_B|\leq\rhoMom^2 B/2\).
\end{proposition}

\begin{proof}
Set \(h=B/\tau=1/n\), and couple the Gaussian noises to Brownian increments by
\(\xi_k=(W_{(k+1)B}-W_{kB})/\sqrt B\).
Then \(\widetilde W_t=W_{\tau t}/\sqrt\tau\) is standard Brownian motion.
Write \(\Delta\widetilde W_k=\widetilde W_{(k+1)h}-\widetilde W_{kh}\), whose conditional covariance is \(hI\).

At step \(k+1\), define the momentum lag and mean corrected input by
\(r_{k+1}=\g(\w_{k+1})-\m_{k+1}\) and
\(q_{k+1}=\omega\m_{k+1}+(1-\omega)\g(\w_{k+1})\).
The momentum recursion gives exactly
\[
 \widetilde\m_{k+2}
 =q_{k+1}+\omega(1-\mu)r_{k+1}+c z_{k+1}/\sqrt B.
\]
Define the direction remainder \(R_{k+1}\) by
\[
 \Dir(\widetilde\m_{k+2})
 =\Dir(q_{k+1})+(1-\omega)J(q_{k+1})z_{k+1}/\sqrt B+R_{k+1}.
\]
Applying the derivative bound to the momentum shift, Taylor's theorem to the noise increment, and \(c-(1-\omega)=\omega(1-\mu)\) yields
\[
 \|R_{k+1}\|
 \leq K\omega(1-\mu)\left(\|r_{k+1}\|+\frac{\|z_{k+1}\|}{\sqrt B}\right)
       +\frac{Hc^2\|z_{k+1}\|^2}{2B}.
\]
Conditional Gaussianity and \(\|\sigma\|_{\rm F}\leq S\) give
\(\E[\|z_{k+1}\|^2\mid\mathcal F_{k+1}]\leq S^2\) and
\(\E[\|z_{k+1}\|^4\mid\mathcal F_{k+1}]\leq3S^4\).
Minkowski's inequality therefore gives \(\|R_{k+1}\|_{L^2}\leq\mathcal R_B\).

Let \(f_k\), \(D_k\), and \(u_k\) denote the drift, diffusion, and direction defects relative to an Euler--Maruyama step:
\[
 f_k=\begin{pmatrix}\tau\delta_B r_{k+1}\\0\end{pmatrix},
 \quad D_k=\begin{pmatrix}\sqrt\tau\delta_B\sigma(\w_{k+1})\\0\end{pmatrix},
 \quad u_k=\begin{pmatrix}0\\-\tau\gam R_{k+1}\end{pmatrix}.
\]
Direct substitution into the discrete recursion gives
\[
 U_{k+1}=U_k+h b_\tau(U_k)+A_\tau(U_k)\Delta\widetilde W_k
             +h f_k+D_k\Delta\widetilde W_k+h u_k.
\]
In particular,
\(\E\|f_k\|^2\leq\tau^2\delta_B^2Q^2\),
\(\E\|D_k\|_{\rm F}^2\leq\tau\delta_B^2S^2\), and
\(\E\|u_k\|^2\leq\tau^2\gam^2\mathcal R_B^2\).

Let \(Y_0=U_0\) and
\(Y_{k+1}=Y_k+h b_\tau(Y_k)+A_\tau(Y_k)\Delta\widetilde W_k\).
Define \(e_k=U_k-Y_k\), \(\Delta b_k=b_\tau(U_k)-b_\tau(Y_k)\), and \(\Delta A_k=A_\tau(U_k)-A_\tau(Y_k)\).
The increment before adding the direction defect is
\(F_k=e_k+h\Delta b_k+h f_k+(\Delta A_k+D_k)\Delta\widetilde W_k\).
All its coefficients are measurable before the Brownian increment.
Conditional expectation, the Lipschitz bound, and \(h\leq1\) give
\begin{align*}
 \E\|F_k\|^2
 &\leq\bigl[(1+h)(1+2Lh)+2L^2h\bigr]\E\|e_k\|^2
       +h(1+h)\E\|f_k\|^2+2h\E\|D_k\|_{\rm F}^2,\\
 \E\|e_{k+1}\|^2
 &\leq(1+h)\E\|F_k\|^2+h(1+h)\E\|u_k\|^2\\
 &\leq e^{(2+2L+2L^2)h}\E\|e_k\|^2
       +h\bigl(4\E\|f_k\|^2+4\E\|D_k\|_{\rm F}^2+2\E\|u_k\|^2\bigr).
\end{align*}
The second inequality is algebraic and does not require \(u_k\) to be measurable before the increment.
Iterating from \(e_0=0\) and using \(nh=1\) yields
\[
 \max_{k\leq n}\E\|U_k-Y_k\|^2
 \leq e^{2+2L+2L^2}
       \bigl[4\delta_B^2(\tau^2Q^2+\tau S^2)+2\tau^2\gam^2\mathcal R_B^2\bigr].
\]

For completeness, define the continuous Euler interpolation by
\[
 \overline Y_t=U_0+\int_0^t b_\tau(Y_{\lfloor v/h\rfloor})\,\dd v
                   +\int_0^t A_\tau(Y_{\lfloor v/h\rfloor})\,\dd\widetilde W_v,
 \qquad 0\leq t\leq1.
\]
It satisfies \(\overline Y_{kh}=Y_k\).
The growth bound, It\^o isometry, and Gronwall's inequality imply
\[
 \sup_{t\leq1}\E\|\overline Y_t\|^2\leq M,
 \qquad M=3\bigl(\E\|U_0\|^2+C\bigr)e^{3C},
\]
and hence
\(\E\|\overline Y_t-Y_{\lfloor t/h\rfloor}\|^2\leq2Ch(1+M)\).
For \(Z_t=X_{\tau t}\), the integral equations and the Lipschitz bound give
\[
 \E\|Z_t-\overline Y_t\|^2
 \leq4L^2\int_0^t\E\|Z_v-\overline Y_v\|^2\,\dd v
       +8L^2C(1+M)h.
\]
A further application of Gronwall's inequality gives
\(\sup_{t\leq1}\E\|Z_t-\overline Y_t\|^2\leq C_{\rm E}^2h\), with
\(C_{\rm E}^2=8L^2C(1+M)e^{4L^2}\).
The triangle inequality at the grid points proves \eqref{eq:sde_approximation_bound}.
Finally, \(\mu=e^{-\rhoMom B}\) gives
\[
 0\leq-\delta_B
 =\frac1B\int_0^{\rhoMom B}(1-e^{-u})\,\dd u
 \leq\frac1B\int_0^{\rhoMom B}u\,\dd u
 =\frac{\rhoMom^2B}{2}.
\]
The proof is then complete.
\end{proof}

\subsection{Minimizing a stochastic Frank--Wolfe bound}
\label{sec:convex}
\label{app:s2-fw}
A second prescription follows from minimizing an upper bound on the Frank-Wolfe gap after a fixed data budget.
We minimize this bound over the learning rate while holding the weight-decay coefficient \(\wdc>0\) fixed as batch size varies.
We adapt the stochastic Frank--Wolfe analysis of \citet{sfyraki2025lions} to the recursion in \Cref{sec:scaling_notation}.

Let \(\mathcal C\) be a compact convex set with diameter \(D>0\).
\begin{assumption}[Exact linear minimization]
\label[assumption]{ass:fw_direction}
Fix \(\wdc>0\). For every input \(z\), the direction map \(\Dir\) satisfies
\[
 u(z):=-\frac{\Dir(z)}{\wdc}
 \in\arg\min_{v\in\mathcal C}\langle z,v\rangle.
\]
\end{assumption}

When \(z=\g(\w)\), the point \(u(z)\) minimizes the first-order model
\(L(\w)+\langle z,v-\w\rangle\) over \(v\in\mathcal C\).
It therefore gives the best destination in \(\mathcal C\) according to the local linear approximation of the loss.
For a gradient or momentum estimate \(z\), the same interpretation holds for the linear model defined by that estimate.
Writing \(\gamma=\lr\wdc\), the update with weight decay becomes
\[
 \w_{k+1}=(1-\gamma)\w_k+\gamma u(\widetilde\m_{k+1}).
\]
Thus the existing optimizer update moves a fraction \(\gamma\) of the way toward this destination, which is a stochastic Frank--Wolfe step.
Initialization in \(\mathcal C\) and \(0<\gamma\leq1\) keep all iterates in \(\mathcal C\).

\textbf{Interpretation for Muon.}
Consider one matrix parameter block and a fixed shape factor \(a_{\rm shape}>0\).
For a matrix input \(Z\) of rank \(r\), write its compact singular value decomposition as
\(Z=U\operatorname{diag}(s_1,\ldots,s_r)V^\top\), with \(s_i>0\).
Replacing each positive singular value by one gives the polar factor \(UV^\top\).
The idealized Muon direction is \(\Dir(Z)=a_{\rm shape}UV^\top\); take \(\Dir(0)=0\).
Let \(\|\cdot\|_{\rm op}\) denote the spectral norm and use the Frobenius inner product
\(\langle Z,X\rangle=\operatorname{tr}(Z^\top X)\).
Choose
\[
 \mathcal C=\{X:\|X\|_{\rm op}\leq a_{\rm shape}/\wdc\},
 \qquad u(Z)=-\frac{a_{\rm shape}}{\wdc}UV^\top.
\]
The point \(u(Z)\) belongs to \(\mathcal C\), and every \(X\in\mathcal C\) satisfies
\[
 \langle Z,X\rangle\geq-\frac{a_{\rm shape}}{\wdc}\sum_{i=1}^r s_i
 =\langle Z,u(Z)\rangle,
\]
so this direction satisfies \Cref{ass:fw_direction}.
Geometrically, the negative polar direction minimizes the linear model over a spectral-norm ball, and weight decay sets the ball's radius.
For several matrix blocks, the same argument applies to the Cartesian product of their respective balls.
Finite Newton--Schulz iterations and rounding only approximate this exact direction; the theorem below ignores that implementation error.
For fixed matrix shapes and shape factors, the diameter is fixed when weight decay is fixed, but changes if weight decay is optimized.
Define the Frank--Wolfe gap by
\begin{equation} \mathcal G_{\mathcal C}(\w) = \max_{u\in\mathcal C} \langle u-\w,-\g(\w)\rangle. \label{eq:fw_gap} \end{equation}
For a convex objective this quantity bounds the constrained suboptimality.
For a nonconvex objective it is the stationarity measure controlled below.

\citet{islamov2026role} analyze stochastic conditional gradient methods under a Kurdyka--\L{}ojasiewicz (KL) condition relating the dual norm of the gradient to objective suboptimality.
Their analysis yields a different bound from \Cref{thm:sfw_gap_bound}, which controls the Frank--Wolfe gap without a KL assumption.

\begin{theorem}[Stochastic Frank--Wolfe bound]
\label{thm:sfw_gap_bound}
\label{thm:s2-fw}
Suppose \Cref{ass:fw_direction} holds and \(L\) has
\(L_{\mathrm{sm}}>0\) Lipschitz gradient on \(\mathcal C\).  Conditional on the history, assume the minibatch gradient
is unbiased and, for a scalar noise bound \(\sigma\geq0\), satisfies
\(\E[\|\ghat_B(\w_k)-\g(\w_k)\|^2\mid\mathcal F_k]\le\sigma^2/B\).
Run \eqref{eq:ema_momentum_update}--\eqref{eq:weight_update} with
\(\Dir(q)=-\wdc u(q)\), \(\w_1\in\mathcal C\), \(\m_1=0\),
\(\wdc>0\), fixed \(\mom\in(0,1)\) and
\(\omega\in[0,1]\), and \(0<\gamma=\lr\wdc\le1\).
Let \(\tau\) be the sample budget, assume \(n=\tau/B\) is an integer, and
choose \(a\) independently and uniformly from \(\{1,\ldots,n\}\).
For \(\Delta_1=L(\w_1)-\inf_{\w\in\mathcal C}L(\w)\),
\begin{align}
    \E\mathcal G_{\mathcal C}(\w_a)\le{}&
      \frac{B\Delta_1}{\tau\gamma}
      +L_{\mathrm{sm}}D^2\gamma
          \left(\frac12+
              \frac{\omega\mom}{1-\mom}\right)
      +\frac{D\sigma}{\sqrt B}\notag\\
      &+\frac{DB\omega}{(1-\mom)\tau}
          \left((1-\mom)\frac{\sigma}{\sqrt B}
                   +\mom\|\g(\w_1)\|\right).
    \label{eq:sfw_gap_bound}
\end{align}
\end{theorem}

\begin{proof}
The pre-update initialization \(\m_1=0\) corresponds to \(M_0=0\) in Definition~\ref{def:muonw-adamw}.
The exact linear minimization oracle keeps every iterate
in \(\mathcal C\).  The proof couples objective descent to the momentum
tracking error, cancelling that error in a single potential.

Let \(v_k=u(\widetilde\m_{k+1})\) be the point selected by the optimizer,
and let \(v_k^\star=u(\g(\w_k))\) minimize the linear model defined by the true gradient.
Both points belong to \(\mathcal C\), and
\(\w_{k+1}=\w_k+\gamma(v_k-\w_k)\).
Convexity of \(\mathcal C\) and \(0<\gamma\leq1\) therefore give
\[
 \w_{k+1}\in\mathcal C,\qquad
 \|\w_{k+1}-\w_k\|=\gamma\|v_k-\w_k\|\leq\gamma D.
\]
By \Cref{ass:fw_direction} and the definition of the Frank--Wolfe gap,
\[
 \langle\widetilde\m_{k+1},v_k-v_k^\star\rangle\leq0,
 \qquad
 \langle\g(\w_k),v_k^\star-\w_k\rangle
 =-\mathcal G_{\mathcal C}(\w_k).
\]
Adding and subtracting the true-gradient minimizer gives
\begin{align*}
 \langle\g(\w_k),v_k-\w_k\rangle
 &=\langle\g(\w_k),v_k^\star-\w_k\rangle
   +\langle\g(\w_k),v_k-v_k^\star\rangle\\
 &=-\mathcal G_{\mathcal C}(\w_k)
   +\langle\widetilde\m_{k+1},v_k-v_k^\star\rangle
   +\langle\g(\w_k)-\widetilde\m_{k+1},v_k-v_k^\star\rangle\\
 &\leq-\mathcal G_{\mathcal C}(\w_k)
   +D\|\widetilde\m_{k+1}-\g(\w_k)\|.
\end{align*}
Smoothness then yields
\begin{align*}
 L(\w_{k+1})
 &\leq L(\w_k)+\langle\g(\w_k),\w_{k+1}-\w_k\rangle
       +\frac{L_{\mathrm{sm}}}{2}\|\w_{k+1}-\w_k\|^2\\
 &\leq L(\w_k)-\gamma\mathcal G_{\mathcal C}(\w_k)
       +\gamma D\|\widetilde\m_{k+1}-\g(\w_k)\|
       +\frac{L_{\mathrm{sm}}D^2\gamma^2}{2}.
\end{align*}

Define the minibatch error \(\varepsilon_k=\ghat_B(\w_k)-\g(\w_k)\)
and the expected momentum error \(e_k=\E\|\m_{k+1}-\g(\w_k)\|\).
The conditional second-moment bound and Cauchy--Schwarz imply
\[
 \E\|\varepsilon_k\|
 \leq\E\!\left[\sqrt{\E[\|\varepsilon_k\|^2\mid\mathcal F_k]}\right]
 \leq\frac{\sigma}{\sqrt B}.
\]
Using \eqref{eq:lionk_stochastic_corrected_m}, we obtain
\begin{align*}
 \widetilde\m_{k+1}-\g(\w_k)
 &=\omega(\m_{k+1}-\g(\w_k))+(1-\omega)\varepsilon_k,\\
 \E\|\widetilde\m_{k+1}-\g(\w_k)\|
 &\leq\omega e_k+(1-\omega)\frac{\sigma}{\sqrt B}.
\end{align*}
Taking expectations in the descent inequality therefore gives
\begin{equation}
\begin{split}
 \E L(\w_{k+1})\leq{}&\E L(\w_k)-\gamma\E\mathcal G_{\mathcal C}(\w_k)
       +\gamma D\omega e_k\\
 &+\gamma D(1-\omega)\frac{\sigma}{\sqrt B}
       +\frac{L_{\mathrm{sm}}D^2\gamma^2}{2}.
\end{split}
\label{eq:sfw_descent_step}
\end{equation}
For \(k<n\), the momentum recursion also gives the exact decomposition
\[
 \m_{k+2}-\g(\w_{k+1})
 =\mom(\m_{k+1}-\g(\w_k))
  +\mom(\g(\w_k)-\g(\w_{k+1}))
  +(1-\mom)\varepsilon_{k+1}.
\]
The triangle inequality and Lipschitz continuity of the gradient imply
\begin{align*}
 e_{k+1}
 &\leq\mom e_k+\mom\E\|\g(\w_k)-\g(\w_{k+1})\|
             +(1-\mom)\E\|\varepsilon_{k+1}\|\\
 &\leq\mom e_k+\mom L_{\mathrm{sm}}\E\|\w_{k+1}-\w_k\|
             +(1-\mom)\frac{\sigma}{\sqrt B}\\
 &\leq\mom e_k+\mom\gamma L_{\mathrm{sm}}D
             +(1-\mom)\frac{\sigma}{\sqrt B}.
\end{align*}

Define the potential
\(V_k=\E L(\w_k)+\gamma D\omega e_k/(1-\mom)\).
For \(k<n\), adding the weighted momentum-error bound to
\eqref{eq:sfw_descent_step} yields
\begin{align*}
 V_{k+1}
 &\leq\E L(\w_k)-\gamma\E\mathcal G_{\mathcal C}(\w_k)
       +\gamma D\omega e_k
       +\gamma D(1-\omega)\frac{\sigma}{\sqrt B}
       +\frac{L_{\mathrm{sm}}D^2\gamma^2}{2}\\
 &\quad+\frac{\gamma D\omega}{1-\mom}
       \left(\mom e_k+\mom\gamma L_{\mathrm{sm}}D
                       +(1-\mom)\frac{\sigma}{\sqrt B}\right)\\
 &=V_k-\gamma\E\mathcal G_{\mathcal C}(\w_k)
       +L_{\mathrm{sm}}D^2\gamma^2
           \left(\frac12+\frac{\omega\mom}{1-\mom}\right)
       +\frac{\gamma D\sigma}{\sqrt B}.
\end{align*}
Here the coefficients of \(e_k\) and the noise combine according to
\[
 \gamma D\omega+\frac{\gamma D\omega\mom}{1-\mom}
 =\frac{\gamma D\omega}{1-\mom},
 \qquad
 \gamma D(1-\omega)+\frac{\gamma D\omega}{1-\mom}(1-\mom)
 =\gamma D.
\]
Rearranging and summing over \(k=1,\ldots,n-1\) gives
\begin{align*}
 \gamma\sum_{k=1}^{n-1}\E\mathcal G_{\mathcal C}(\w_k)
 &\leq\sum_{k=1}^{n-1}(V_k-V_{k+1})
       +(n-1)L_{\mathrm{sm}}D^2\gamma^2
          \left(\frac12+\frac{\omega\mom}{1-\mom}\right)
       +\frac{(n-1)\gamma D\sigma}{\sqrt B}\\
 &=V_1-V_n+(n-1)L_{\mathrm{sm}}D^2\gamma^2
          \left(\frac12+\frac{\omega\mom}{1-\mom}\right)
       +\frac{(n-1)\gamma D\sigma}{\sqrt B}.
\end{align*}
For the final step, nonnegativity of \(e_n\) gives
\[
 \gamma D\omega e_n
 \leq\frac{\gamma D\omega}{1-\mom}e_n
 =V_n-\E L(\w_n).
\]
Applying \eqref{eq:sfw_descent_step} at \(k=n\) therefore yields
\begin{align*}
 \E L(\w_{n+1})
 &\leq V_n-\gamma\E\mathcal G_{\mathcal C}(\w_n)
       +\gamma D(1-\omega)\frac{\sigma}{\sqrt B}
       +\frac{L_{\mathrm{sm}}D^2\gamma^2}{2}\\
 &\leq V_n-\gamma\E\mathcal G_{\mathcal C}(\w_n)
       +\frac{\gamma D\sigma}{\sqrt B}
       +L_{\mathrm{sm}}D^2\gamma^2
          \left(\frac12+\frac{\omega\mom}{1-\mom}\right).
\end{align*}
Combining this with the preceding sum gives
\begin{align*}
 \gamma\sum_{k=1}^{n}\E\mathcal G_{\mathcal C}(\w_k)
 &\leq V_1-\E L(\w_{n+1})
       +nL_{\mathrm{sm}}D^2\gamma^2
          \left(\frac12+\frac{\omega\mom}{1-\mom}\right)
       +\frac{n\gamma D\sigma}{\sqrt B}\\
 &\leq\Delta_1+\frac{\gamma D\omega}{1-\mom}e_1
       +nL_{\mathrm{sm}}D^2\gamma^2
          \left(\frac12+\frac{\omega\mom}{1-\mom}\right)
       +\frac{n\gamma D\sigma}{\sqrt B},
\end{align*}
where the second inequality uses
\(\E L(\w_{n+1})\geq\inf_{\w\in\mathcal C}L(\w)\).
The same argument applies when \(n=1\), with the preceding sum empty.
Since \(\m_1=0\), the initial momentum error satisfies
\begin{align*}
 e_1
 &=\E\|(1-\mom)\ghat_B(\w_1)-\g(\w_1)\|\\
 &=\E\|(1-\mom)\varepsilon_1-\mom\g(\w_1)\|\\
 &\leq(1-\mom)\E\|\varepsilon_1\|+\mom\|\g(\w_1)\|\\
 &\leq(1-\mom)\frac{\sigma}{\sqrt B}+\mom\|\g(\w_1)\|.
\end{align*}
Finally, uniform independent sampling of \(a\), division by \(n\gamma\),
and \(n=\tau/B\) give
\begin{align*}
 \E\mathcal G_{\mathcal C}(\w_a)
 &=\frac1n\sum_{k=1}^n\E\mathcal G_{\mathcal C}(\w_k)\\
 &\leq\frac{\Delta_1}{n\gamma}
       +\frac{D\omega e_1}{n(1-\mom)}
       +L_{\mathrm{sm}}D^2\gamma
          \left(\frac12+\frac{\omega\mom}{1-\mom}\right)
       +\frac{D\sigma}{\sqrt B}\\
 &\leq\frac{B\Delta_1}{\tau\gamma}
       +L_{\mathrm{sm}}D^2\gamma
          \left(\frac12+\frac{\omega\mom}{1-\mom}\right)
       +\frac{D\sigma}{\sqrt B}\\
 &\quad+\frac{DB\omega}{(1-\mom)\tau}
       \left((1-\mom)\frac{\sigma}{\sqrt B}+\mom\|\g(\w_1)\|\right),
\end{align*}
which is \eqref{eq:sfw_gap_bound}.
\end{proof}

Only the first two terms in \eqref{eq:sfw_gap_bound} depend on \(\gamma\).
To make the tradeoff described in Section~\ref{sec:scaling-approaches} explicit, write \(c_1=\Delta_1\), \(c_2=L_{\mathrm{sm}}D^2[\tfrac12+\omega\mom/(1-\mom)]\), and let \(r\) collect the noise and initialization terms.
At batch size \(\kappa B_0\), the bound is then
\begin{equation}
\mathbb E\mathcal G_{\mathcal C}(\w_a)\leq \frac{c_1\kappa B_0}{\tau\gamma}+c_2\gamma+r(\kappa B_0).
\label{eq:sfw_balance_appendix}
\end{equation}
For \(c_1,c_2>0\), differentiation before imposing the step-size constraints gives \(c_2\gamma^2=c_1\kappa B_0/\tau\), hence
\[
\gamma_\star(\kappa B_0)=\sqrt{\frac{c_1\kappa B_0}{\tau c_2}}=\sqrt{\kappa}\,\gamma_\star(B_0).
\]
At fixed weight decay, \(\gamma=\lr\wdc\) is proportional to learning rate, so the same scaling holds for the learning-rate minimizer.
For \(\Delta_1>0\) and \(L_{\mathrm{sm}}>0\), minimizing them subject to the theorem gives
\begin{equation} \gamma_{\mathrm{cvx}}(\batch) = \min\left\{ \sqrt{ \frac{\batch\Delta_1}{ \tausamp L_{\mathrm{sm}}D^2 \left( \frac12+ \frac{\omega\mom}{1-\mom} \right)} }, 1 \right\}. \label{eq:convex_scaling_rule} \end{equation}
Thus \(\lr_{\mathrm{cvx}}=\gamma_{\mathrm{cvx}}/\wdc\) grows as \(\sqrt{\batch}\) until the cap \(\gamma\leq1\) becomes active.
If \(\Delta_1=0\), the bound is minimized in the limit \(\gamma\downarrow0\); the nontrivial square-root rule concerns \(\Delta_1>0\).
This calculation keeps the momentum parameters fixed.
It gives a learning rate rule and does not imply that an optimizer state statistic is preserved and transferring a tuned reference learning rate by this rule does not assert that the reference rate minimizes the bound.

\subsection{Comparison with the published prescriptions}
\label{sec:scaling_rule_survey}
The MuonW and Adam/W tables in the main text (\Cref{tab:scaling-recipes}) report finite-change formulas at \(B'=\kappa B\), distinguishing published prescriptions from our conversions and empirical selections.
A dash means that the cited criterion does not specify the coordinate.
For a generic retention coefficient \(q\in(0,1)\), exact sample-memory preservation gives \(q'=q^\kappa\), whereas linear refresh scaling gives \(q'=1-\kappa(1-q)\).
The former preserves the half-life measured in processed samples, \(B\log 2/(-\log q)\).
The two formulas agree only to first order in \(1-q\), for fixed \(\kappa\).
For a fixed exponent \(r\), the expansion \(q^r=1-r(1-q)+\tfrac12r(r-1)(1-q)^2+O((1-q)^3)\) as \(q\to1\) gives the power forms displayed in \Cref{tab:scaling-recipes}, with \(r=\kappa\) or \(r=\sqrt\kappa\).
More generally, for an exponent \(a\), affine refresh scaling \(q'=1-\kappa^a(1-q)\) requires \(0<\kappa^a(1-q)<1\) to keep the target retention in \((0,1)\).
The Power Lines entries use a free learning-rate exponent \(p\), with the associated product constraint derived below in \eqref{eq:powerline_scaling_family}.
\Cref{tab:five-recipes} uses the source formulas rather than replacing affine refresh rules by exact powers. We will discuss the related work in more detail below.

\begin{table}[!t]
  \centering
  \scriptsize
  \setlength{\tabcolsep}{2pt}
  \renewcommand{\arraystretch}{1.2}
  \resizebox{\textwidth}{!}{%
  \begin{tabular}{@{}lccccccc@{}}
    \toprule
    \textbf{Method} & \multicolumn{3}{c}{\textbf{MuonW}} & \multicolumn{4}{c}{\textbf{Adam / AdamW}} \\
    \cmidrule(lr){2-4}\cmidrule(lr){5-8}
    & $\eta'$ & $\mu'$ & $\lambda'$ & $\eta'$ & $\beta_1'$ & $\beta_2'$ & $\lambda'$ \\
    \midrule
    \multicolumn{8}{@{}l}{\textit{Invariant preservation}} \\
    \citet{islamov2026role}\textsuperscript{$\ddagger$} & $\kappa\eta$ & $1-\kappa(1-\mu)$ & $\lambda$ & -- & -- & -- & -- \\
    Adam SDE~\citep{malladi2022sdes} & -- & -- & -- & $\sqrt\kappa\,\eta$ & $1-\kappa(1-\beta_1)$ & $1-\kappa(1-\beta_2)$ & -- \\
    Completed transfer~\citep{mlodozeniec2025completed} & -- & -- & -- & $\sqrt\kappa\,\eta$ & $1-\kappa(1-\beta_1)$ & $1-\kappa(1-\beta_2)$ & $\sqrt\kappa\,\lambda$ \\
    Second-moment memory~\citep{marek2025smallbatch} & -- & -- & -- & -- & $\beta_1$ & $\beta_2^\kappa$ & -- \\
    Power Lines~\citep{bergsma2025powerlines}\textsuperscript{*} & $\kappa^p\eta$ & $\mu$ & $\kappa^{1-p}\lambda$ & $\kappa^p\eta$ & $\beta_1$ & $\beta_2$ & $\kappa^{1-p}\lambda$ \\
    Loss and relative speeds~\citep{compagnoni2025adaptive} & -- & -- & -- & $\sqrt\kappa\,\eta$ & $1-\sqrt\kappa(1-\beta_1)$ & $1-\sqrt\kappa(1-\beta_2)$ & $\sqrt\kappa\,\lambda$ \\
    Heuristic noise preservation~\citep{cesista2025sdcbs}\textsuperscript{$\dagger$} & $\kappa\eta$ & -- & -- & $\kappa\eta$ & -- & -- & -- \\
    Our Muon diffusion (\Cref{thm:sde_scaling}) & $\kappa\eta$ & $\mu^\kappa$ & $\lambda$ & -- & -- & -- & -- \\
    \midrule
    \multicolumn{8}{@{}l}{\textit{Bound minimization}} \\
    \citet{shulgin2026deriving} & $\sqrt\kappa\,\eta$ & $\mu$ & -- & -- & -- & -- & -- \\
    \citet{bu2026convexdominance} & $\sqrt\kappa\,\eta$ & -- & -- & $\sqrt\kappa\,\eta$ & -- & -- & -- \\
    \citet{schaipp2025surprising} & -- & -- & -- & $\sqrt\kappa\,\eta$ & -- & -- & -- \\
    Our Frank--Wolfe bound (\Cref{sec:convex}) & $\sqrt\kappa\,\eta$ & $\mu$ & $\lambda$ & -- & -- & -- & -- \\
    \midrule
    \multicolumn{8}{@{}l}{\textit{Stability}} \\
    \citet{granziol2022learningrates} & -- & -- & -- & $\sqrt\kappa\,\eta$ & -- & -- & -- \\
    \midrule
    \multicolumn{8}{@{}l}{\textit{Empirical selections}} \\
    \citet{islamov2026role}\textsuperscript{\S} & $\eta$ & $\mu$ & $\lambda$ & -- & -- & -- & -- \\
    LLM: $216$-rule selection & $\eta$ & $\mu$ & $\kappa\lambda$ & $\eta$ & $\beta_1$ & $\beta_2^\kappa$ & $\lambda$ \\
    CIFAR-5M: $648$-rule selection & $\sqrt\kappa\,\eta$ & $\mu$ & $\lambda$ & $\sqrt\kappa\,\eta$ & $\beta_1$ & $\beta_2$ & $\lambda$ \\
    \bottomrule
  \end{tabular}%
  }
  \caption{\textbf{Finite-change batch-size prescriptions at $B'=\kappa B$.}
    Each cell gives the target value in terms of the reference value for its parameter group; dashes mean unspecified.
    Adam SDE and Completed transfer also scale $\epsilon'=\epsilon/\sqrt\kappa$.
    Our experimental SDE candidate instead uses $\beta_i'=\beta_i^\kappa$ and fixed $\epsilon$.
    $*$: the Power Lines product constraint leaves $p$ free; its favored choice is $p=0$.
    Its MuonW entries are our extension with fixed momentum, while fixed Adam moments are an empirical choice adopted from the paper.
    Fixed $\beta_1$ in the second-moment row is also an empirical choice adopted from the paper.
    $\dagger$: our fixed-token-budget conversion of the heuristic in \citet{cesista2025sdcbs}, assuming per-step update-noise variance is proportional to $\eta^2/B$ with a fixed proportionality factor; the source's empirical square-root rule is distinct from this conversion.
    Our Muon diffusion theorem holds $\omega$ fixed independently of $\mu$; using it for tied MuonW is an additional approximation.
    Our bound rule assumes an exact oracle and fixed $\omega$, before the effective step-size cap is active.
    $\ddagger$: our fixed-budget conversion preserves a common target bound on expected suboptimality, with $0<\kappa(1-\mu)<1$ and the source theorem's assumptions, budget, and initialization conditions satisfied at both batches; the oracle is exact and Nesterov mixing is absent.
    \S: the same learning rate is best in the source's Scion grid at $512$--$4096$ sequences per batch, with sequence length $1024$ and a $1.3$B-token budget; momentum and weight decay are fixed by design, with the entries expressed in MuonW coordinates.
    \citet{shulgin2026deriving} use a large-horizon bound at fixed momentum for an exact oracle without Nesterov or weight decay.
    \Cref{sec:scaling_rule_survey} details these conditions.
  }
  \label{tab:five-recipes}
\end{table}

The Adam SDE prescription of \citet{malladi2022sdes} scales \(\eta'=\sqrt\kappa\,\eta\), \(\beta_i'=1-\kappa(1-\beta_i)\) for \(i=1,2\), and the denominator stabilizer \(\epsilon'=\epsilon/\sqrt\kappa\).
Completed Hyperparameter Transfer adds the AdamW rule \(\lambda'=\sqrt\kappa\,\lambda\) at fixed model size and token budget \citep{mlodozeniec2025completed}.
These prescriptions match a continuous-time model; their use for discrete trajectories depends on the respective SDE approximation conditions.
Our experimental SDE candidate instead uses the exact-decay variant \(\beta_i'=\beta_i^\kappa\) and keeps \(\epsilon\) fixed.
Its moment transport agrees with the source prescription up to first order in the refresh rates and keeps the $\beta_i$ positive, and it does not apply the prescribed \(\epsilon\) transport.
The tables omit \(\epsilon\).

\citet{compagnoni2025adaptive} preserve an asymptotic-loss upper bound and the relative speeds of the parameter and moment states when the SDE is run for infinite continuous time.
Their prescription is \(\eta'=\sqrt\kappa\,\eta\), \(\lambda'=\sqrt\kappa\,\lambda\), and \(\beta_i'=1-\sqrt\kappa(1-\beta_i)\).
The supporting bound assumes a strongly convex, smooth objective with optimum at zero and isotropic noise in a noise-dominated regime where squared gradient coordinates are \(O(\eta)\).
It therefore supplies a different, conditional invariant from the trajectory criterion above.

\citet[Section~5]{cesista2025sdcbs} propose preserving the accumulated update-noise variance for steepest descent under arbitrary norms.
Their argument assumes a locally Lipschitz map from momentum to update direction, a heuristic second-moment noise estimate, and additive noise variance across steps.
The assumed per-step update-noise variance is proportional to \(\eta^2/B\), with the proportionality factor held fixed.
For \(T\) optimizer steps, preserving the accumulated variance therefore holds \(T\eta^2/B\) fixed, giving
\[
    \frac{\eta'}{\eta}=\sqrt{\frac{B'/B}{T'/T}}.
\]
At fixed step count, this gives \(\eta'=\sqrt\kappa\,\eta\).
At fixed token budget \(\tau=BT\), however, \(B'=\kappa B\) implies \(T'=T/\kappa\), so the same heuristic gives \(\eta'=\kappa\eta\).
The corresponding row in \Cref{tab:five-recipes} is our fixed-budget conversion, not a separate theorem for AdamW or a joint prescription for momentum, moment coefficients, and weight decay.
The blog's experiments instead support square-root learning-rate scaling for AdamW and Muon at fixed token budgets \citep[Section~6]{cesista2025sdcbs}; that empirical prescription should be distinguished from the fixed-budget consequence of its heuristic invariant.

For batch-size transfer, \citet{marek2025smallbatch} prescribe \(\beta_2'=\beta_2^\kappa\) to preserve second-moment memory in processed tokens.
When batch size counts sequences, this conversion assumes fixed sequence length.
They keep \(\beta_1\) fixed in their experiments; that choice does not follow from the second-moment invariant.
The rule does not specify a universal learning-rate or weight-decay scaling.

\phantomsection\label{sec:powerline_invariant}
Power Lines constrains the learning-rate--weight-decay product through parameter averaging \citep{bergsma2025powerlines}.
Writing AdamW's adaptive direction as \(u_t=\widehat m_t/(\sqrt{\widehat v_t}+\epsilon)\), with bias-corrected first and second moments \(\widehat m_t,\widehat v_t\), its update for \(0<\eta\lambda<1\) is
\begin{equation}
    \theta_{t+1}=(1-\eta\lambda)\theta_t
                 +\eta\lambda(-u_t/\lambda).
    \label{eq:powerline_parameter_ema}
\end{equation}
This is an exponential moving average of rescaled updates with nominal timescale \(1/(\eta\lambda)\) steps.
Relative to the \(\tau/B\) training steps, the timescale is
\begin{equation}
    \tau_{\text{Power Lines}}=\frac{B}{\eta\lambda\tau}.
    \label{eq:powerline_timescale}
\end{equation}
Preserving it at fixed model size and budget requires \(\eta'\lambda'=\kappa\eta\lambda\).
We express this product constraint as the family
\begin{equation}
    \eta'=\kappa^p\eta,\qquad \lambda'=\kappa^{1-p}\lambda,
    \label{eq:powerline_scaling_family}
\end{equation}
where \(p\) is a free exponent; Power Lines favors \(p=0\), or fixed learning rate and linearly scaled weight decay.
The prescription uses peak learning rate with a common schedule shape.
Its nominal timescale approximates discrete decay when \(\eta\lambda\) is small; exact retention involves products of \(1-\eta_t\lambda\).
The preferred timescale also drifts empirically at large batches and with tokens per model parameter.
The same averaging identity applies to MuonW after substituting its direction for \(u_t\).
The MuonW table entries are our extension of this identity with fixed \(\mu\).

Optimization-bound prescriptions depend on the optimizer and budget assumptions used to derive them.
At fixed momentum, the large-horizon analysis with an exact linear minimization oracle by \citet{shulgin2026deriving} gives \(\eta\propto\sqrt B\) at fixed data budget.
Its analyzed recursion has neither Nesterov mixing nor decoupled weight decay; holding momentum fixed is a condition of this comparison, not a derived momentum optimum.
\citet{kovalev2025understanding,pethick2025training} provide related normalized-gradient and linear-minimization-oracle analyses.
For \citet{bu2026convexdominance,schaipp2025surprising}, the table's dagger marks our conversion of an asymptotic peak-learning-rate rule \(\eta_{\rm peak}\propto n^{-1/2}\) using \(n=\tau/B\).
This yields \(\eta_{\rm peak}\propto\sqrt B\) only when the schedule shape and the coefficients in the bound or fitted loss model are independent of batch size.
Both works prove bounds for convex SGD and test predictions with AdamW; \citet{bu2026convexdominance} also study Muon with normalized SGD on non-matrix parameters.

We obtain two fixed-budget prescriptions from \citet{islamov2026role}, one by preserving a target convergence guarantee and the other from their experiments.
Their convergence analysis assumes a Kurdyka--\L{}ojasiewicz condition.

To compare both prescriptions with our rules, we first express their parameters in MuonW coordinates.
Their batch size and sequence length product, denoted $BS$ in their paper, corresponds to our token batch size $B$.
Let $\gamma_{\rm FW}$ be their Frank--Wolfe step size, $\alpha$ their momentum refresh rate, and $R$ the radius of the spectral-norm ball used for a matrix update.
For an exact spectral-norm oracle, a fixed shape factor $a_{\rm shape}$, and no Nesterov mixing, the MuonW parameters are
\begin{equation}
    \eta=\frac{R\gamma_{\rm FW}}{a_{\rm shape}},\qquad
    \lambda=\frac{a_{\rm shape}}{R},\qquad
    \mu=1-\alpha.
    \label{eq:islamov_coordinates}
\end{equation}

For the theoretical prescription, fix the sample budget $\tau$ and a target upper bound $\varepsilon$ on expected suboptimality.
Let $\Delta$ be the initial suboptimality, and let $c,C_\alpha,C_R>0$ collect constants held fixed across batch sizes.
With $\tau/B$ updates and gradient-noise variance proportional to $1/B$, the full theorem in their Appendix~D prescribes
\begin{equation}
    \gamma_{\rm FW}=\frac{cB}{\tau},\qquad
    \alpha=\min\{1,C_\alpha B\varepsilon^2\},\qquad
    R=C_R\log(2\Delta/\varepsilon).
    \label{eq:islamov_fixed_budget_parameters}
\end{equation}
Holding $\tau$, $\varepsilon$, and $\Delta$ fixed and increasing the batch size to $B'=\kappa B$ therefore gives, before the cap on $\alpha$ is active,
\begin{equation}
    \eta'=\kappa\eta,\qquad
    \mu'=1-\kappa(1-\mu),\qquad
    \lambda'=\lambda.
    \label{eq:islamov_fixed_budget_transfer}
\end{equation}
This rule requires the original theorem's assumptions, iteration-budget conditions, and initialization conditions to hold at both batch sizes, with $0<\kappa(1-\mu)<1$.
It preserves the same target error guarantee; it does not require equality of the tightest bounds or of the training trajectories.
This is the sense in which we classify it as invariant preservation.

The rule scales the learning rate and momentum refresh rate $1-\mu$ linearly with batch size, agreeing to first order with our SDE prescription.
\Cref{tab:scaling-recipes} displays the momentum rule as $\mu^\kappa\approx1-\kappa(1-\mu)$, while \Cref{tab:five-recipes} gives the exact affine formula.
In their analysis, momentum reduces gradient noise, so larger batches allow less momentum averaging while retaining the same error guarantee.
This prediction differs from our empirical preference for fixed momentum (\Cref{sec:scaling-discussion}).

Their theoretical recommendation to keep momentum fixed concerns a different comparison, in which the sample budget also grows.
If the target error changes with batch size, \Cref{eq:islamov_fixed_budget_parameters} instead gives $\alpha'/\alpha=\kappa(\varepsilon'/\varepsilon)^2$ before saturation.
The update radius also changes with the target error.
Along their growing-budget rule $B\propto\tau^{2/3}$, the target error scales as $\tau^{-1/3}$ up to logarithmic factors, motivating fixed momentum.
Their joint batch-size and token-budget experiments in Section~6.7 study this setting.

The empirical prescription comes from their fixed-budget Scion experiments \citep[Section~6.3 and Table~1]{islamov2026role}.
They sweep the Frank--Wolfe step size for a $124$M-parameter language model trained on $1.3$B tokens with sequence length $1024$, fixing the momentum refresh rate at $0.1$ and the update radii.
At each of $512$, $1024$, $2048$, and $4096$ sequences per batch, the best tested step size is $3.6\times10^{-4}$.
By \Cref{eq:islamov_coordinates}, these choices give $\eta'=\eta$, $\mu'=\mu$, and $\lambda'=\lambda$, so the empirical prescription amounts to no scaling over this range.
Here fixed momentum and weight decay are experimental settings, while the fixed learning rate summarizes the minima of the tested grid.
The grid selects different step sizes at smaller batches.
A common best tested step size also does not imply equal performance across batches.

Our language-model sweep includes both prescriptions, and both underperform the Power Lines prescription across target batches.
The latter scales only weight decay for matrix parameters.

The stability analysis of \citet{granziol2022learningrates} bounds a different learning-rate quantity: the maximal stable learning rate.
For Adam, small damping means a small additive denominator stabilizer \(\epsilon\).
The square-root rule applies when this damping is small and directions near the edge of the Hessian spectral bulk dominate the stability restriction \citep[Section~11]{granziol2022learningrates}.
Writing \(N_{\rm data}\) for the dataset size, the batch dependence is through \(B/(1-B/N_{\rm data})\).
The table's \(\sqrt\kappa\) approximation additionally assumes \(B,B'\ll N_{\rm data}\).
This result neither specifies AdamW weight-decay transport nor minimizes the stationarity bound in \Cref{sec:convex}.

\section{Noisy Quadratic Model: Analysis and Experiments}
\label{app:nqm_details}
\label{app:nqm_experiments}

This appendix pairs the noisy quadratic analyses with their numerical
protocols and supporting figures. We begin with the exact finite-budget
risks and the proof of \Cref{thm:nqm_crossover_main}, followed by
finite-budget SignSGD batch-size transfer, then learning-rate
scaling for momentum-SignSGD and RMSProp.

\subsection{Finite-budget risks and optimizer comparisons}
\label{sec:optimizer_choice_batch_noise}
\label{sec:noisyquadratic}

We now turn from transporting a fixed optimizer to choosing an optimizer at a
given batch size, focusing on the noisy quadratic model of
\citet{zhang2019algorithmicchoicesmatterbatch}.  A batch size determines the
signal-to-noise ratio of the minibatch gradient in each direction of the
landscape.  At a fixed training horizon \(T\), it also determines the number of
optimizer updates through \(T/\batch\).  The natural comparison is therefore
not whether one scaling rule preserves a metric, but which optimizer family
achieves the smallest tuned finite-budget risk at that batch size.  The
diagonal noisy quadratic is deliberately minimal, but it already separates
these two questions because each coordinate has its own curvature, noise level,
and signal-to-noise ratio.

\subsubsection{Setup, noise ratios, and tuned finite-budget risk}

\textbf{Noisy quadratic model.}
Let \(f\) be a diagonal quadratic objective with Hessian \(\Hmat\), and let
\(h_i>0\) be the curvature in coordinate \(i\):
\begin{equation}
    f(\w) = \frac{1}{2}\w^\top \Hmat \w
    =
    \frac{1}{2}\sum_{i=1}^d h_i \w_i^2,
\end{equation}
Write \(\ghat_{\batch}(\w_k)\) for the minibatch gradient query at batch size
\(\batch\), and write \(\Sig\) for the single-example gradient-noise
covariance.  With
\(\sigma_{\batch}:=\batch^{-1/2}\), the query has the form
\begin{equation}
    \ghat_{\batch}(\w_k)
    =
    \Hmat\w_k+\sigma_{\batch} z_k,
    \qquad
    \E[z_k]=0,
    \qquad
    \Cov(z_k)=\Sig.
\end{equation}
Assume that \(\Hmat\) and \(\Sig\) are simultaneously diagonalizable.  In their
common eigenbasis, write \(h_i\) and \(c_i\) for the diagonal entries:
\[
    \Hmat=\diag(h_1,\ldots,h_d),
    \qquad
    \Sig=\diag(c_1,\ldots,c_d),
\]
Writing \(z_{k,i}=\sqrt{c_i}\zeta_{k,i}\), where
\(\zeta_{k,i}\) is standardized, the \(i\)-th minibatch gradient coordinate is
therefore
\begin{equation}
    [\ghat_{\batch}(\w_k)]_i
    =
    h_i \w_{k,i}
    +
    \sqrt{\frac{c_i}{\batch}}\,\zeta_{k,i},
    \qquad
    \E[\zeta_{k,i}]=0,
    \qquad
    \Var(\zeta_{k,i})=1.
\end{equation}
The coordinate minibatch-noise variance is \(c_i/\batch\).  Define the
curvature-normalized noise ratio by
\begin{equation}
    \chi_i:=\frac{c_i}{h_i^2}.
    \label{eq:nqm_noise_ratios}
\end{equation}
For an initial point \(\w_0\), the initial coordinate signal-to-noise ratio is
\begin{equation}
    \mathrm{SNR}_{i,0}(\batch)
    :=
    \frac{\batch h_i^2\w_{0,i}^2}{c_i}.
    \label{eq:nqm_initial_snr}
\end{equation}
Increasing the batch size raises this initial signal-to-noise ratio, but it
does so unevenly across coordinates when the ratios \(c_i/h_i^2\) differ.

\textbf{Tuned finite-budget risk.}
\label{sec:nqm_tuned_risk}

Let \(T\) denote the fixed training horizon, measured in
examples or tokens.  At batch size \(\batch\), define the update count as \(S:=T/\batch\).
We derive the finite-step formulas for integer \(S\) and ignore floor/ceiling
effects when evaluating them in batch-size sweeps.  Let \(A\) be an optimizer
family, such as SGD, SignSGD, RMSProp, or a stochastic Newton method, and let
\(\nu\in\Lambda_A\) denote its tunable hyperparameters.  Write
\(\w_S^{A,\nu}(\batch)\) for the iterate produced by running \(A\) at
batch size \(\batch\) for \(S\) minibatch-gradient queries.  The
batch-\(\batch\) tuned risk of optimizer family \(A\) at horizon \(T\) is
\begin{equation}
    \mathcal R_A^\star(\batch;T)
    :=
    \inf_{\nu\in\Lambda_A}
    \E\!\left[
        f\!\left(\w_S^{A,\nu}(\batch)\right)
    \right]
    =
    \inf_{\nu\in\Lambda_A}
    \frac{1}{2}
    \sum_{i=1}^d
    h_i
    \E\!\left[
        \left(\w_{S,i}^{A,\nu}(\batch)\right)^2
    \right].
    \label{eq:nqm_tuned_finite_budget_risk}
\end{equation}
Whenever the minimum is attained, write an optimal hyperparameter choice as
\begin{equation}
    \nu_A^*(\batch;T)
    \in
    \operatorname*{arg\,min}_{\nu\in\Lambda_A}
    \E\!\left[
        f\!\left(\w_S^{A,\nu}(\batch)\right)
    \right],
    \label{eq:nqm_tuned_hyperparameter}
\end{equation}
so that
\(\mathcal R_A^\star(\batch;T)
=
\E[f(\w_S^{A,\nu_A^*(\batch;T)}(\batch))]\).
The optimizer-choice problem is to compare
\(\mathcal R_A^\star(\batch;T)\) across optimizer families at the same
\((\batch,T)\).  This criterion deliberately allows the hyperparameters to be
retuned at each batch size.

\subsubsection{Risk decomposition and crossover proof}
\label{app:nqm_crossover_proof}

We prove \Cref{thm:nqm_crossover_main} for the initialization
\(\w_0\sim\normal(0,\identity_2)\) and independent Gaussian oracle draws
specified in the main text.
The two-dimensional construction is $H=\diag(1/(T+1),1)$ and $\Sig=\diag(1/(4T),0)$.
For a positive diagonal preconditioner $P=\diag(p_i)$, use
\begin{equation}
 m_{t+1}=\mu m_t+P\widehat g_B(\w_t),\qquad
 \w_{t+1}=\w_t-\eta m_{t+1},\qquad m_0=0.
 \label{eq:nqm_main_update}
\end{equation}
Here $m_t$ is an unnormalized heavy-ball buffer, with unit coefficient on the preconditioned gradient.
SGD uses $P=I$ and Newton uses $P=H^{-1}$.

To express this preconditioning as a change of coordinates, let
$\widetilde\w=P^{-1/2}\w$ and
$\widetilde m=P^{-1/2}m$.
The transformed gradient oracle is
$\widetilde g_B(\widetilde\w)=P^{1/2}\widehat g_B(P^{1/2}\widetilde\w)$,
so the transformed curvature and single-example noise variance in direction $i$
are $\widetilde h_i=p_i h_i$ and $\widetilde c_i=p_i c_i$.
For $c_i>0$, the transformed curvature-to-noise ratio is therefore
\begin{equation}
 \widetilde{\mathrm{CNR}}_i
 =\frac{\widetilde h_i^2}{\widetilde c_i}
 =p_i\frac{h_i^2}{c_i}
 =p_i\,\mathrm{CNR}_i.
\end{equation}
Thus each preconditioned optimizer can be represented as ordinary momentum SGD
on its corresponding transformed quadratic.
The original-coordinate ratios remain unchanged.
The initialization also transforms, with
$\E[\widetilde\w_{0,i}^2]=1/p_i$ for the initialization above.

We will first derive the risk without momentum.
We then bound Newton's noise response uniformly over learning rate and
momentum, compare it with one feasible SGD configuration at \(\batch=1\),
and show the comparison reverse at the one-update batch \(\batch=T\).

We overload the optimizer-family subscript with \(P\) to denote
preconditioned stochastic gradient descent (PSGD)~\citep{li2018preconditioned}
with fixed preconditioner \(P\) and no momentum (\(\mu=0\)):
\[
    \w_{t+1}=\w_t-\eta P\ghat_{\batch}(\w_t),
    \qquad \eta>0.
\]
Write \(\mathcal R_P(\batch,S,\eta)\) for its expected loss after \(S\) updates
at learning rate \(\eta\).  Tuning only the learning rate at fixed budget \(T\) gives
\[
    \mathcal R_P^\star(\batch;T)
    :=\inf_{\eta>0}\mathcal R_P(\batch,S,\eta),
\]

\begin{lemma}[Risk of diagonal preconditioning]
\label{lem:nqm_preconditioned_risk}
For \(P=\diag(p_i)\), \(p_i>0\), and \(\mu=0\), define
\(q_i=p_i h_i\), \(x_i=\eta q_i\),
\(a_S(x)=(1-x)^{2S}\), and
\(v_S(x)=x^2\sum_{j=0}^{S-1}(1-x)^{2j}\).
The expected loss after \(S\) updates is, for every \(\eta>0\),
\begin{equation}
 \mathcal R_P(B,S,\eta)=\frac12\sum_i\left[h_i a_S(x_i)+\frac{c_i}{Bh_i}v_S(x_i)\right].
 \label{eq:nqm_preconditioned_risk}
\end{equation}
In particular, at fixed budget \(T\),
\begin{equation}
    \mathcal R_P^\star(\batch;T)
      =\inf_{\eta>0}\frac12\sum_i
        \left[h_i a_{T/\batch}(\eta q_i)
          +\frac{c_i}{\batch h_i}v_{T/\batch}(\eta q_i)\right].
    \label{eq:nqm_preconditioned_tuned_risk}
\end{equation}
\end{lemma}
\begin{proof}
The coordinate recursion and its second moment are
\begin{align*}
    \w_{t+1,i}&=(1-x_i)\w_{t,i}
                  -\eta p_i\sqrt{c_i/\batch}\,\zeta_{t,i},\\
    \E[\w_{S,i}^2]&=(1-x_i)^{2S}
       +\eta^2p_i^2\frac{c_i}{\batch}
          \sum_{j=0}^{S-1}(1-x_i)^{2j}.
\end{align*}
Independence removes the cross terms, and \(\E[\w_{0,i}^2]=1\).
Multiplying by \(h_i/2\), using
\(\eta^2p_i^2=x_i^2/h_i^2\), and summing proves the risk formula.
Taking the infimum over the stated tuning domain gives
\eqref{eq:nqm_preconditioned_tuned_risk}.
\end{proof}

The variance term has a useful geometric interpretation.  Define the
noise-response vector
\(\kappa_S(x)=x(1,1-x,\ldots,(1-x)^{S-1})^\top\).
Its squared norm is \(v_S(x)\), and its entries sum to
\(1-(1-x)^S\).  Cauchy--Schwarz therefore gives
\begin{equation}
    v_S(x)\ge\frac{[1-(1-x)^S]^2}{S}.
    \label{eq:noise_response_cs}
\end{equation}
Reducing the initial bias requires a nonzero total response, which forces
nonzero noise variance.  The following proof extends this argument to runs with
momentum.

\begin{proof}[Proof of \Cref{thm:nqm_crossover_main}]
Set \(h=1/(T+1)\) and \(c=1/(4T)\), so the two coordinates have
curvatures \((h,1)\) and noise variances \((c,0)\).

\textbf{Small batch: \(\batch=1\).}
For Newton, every coordinate has effective curvature one.  For any
\(\eta>0\) and \(0\le\mu<1\), let \(b_j\) describe the response to a single
noise perturbation in the preconditioned gradient: an additive error
\(\varepsilon\) at update \(t\) contributes \(-b_j\varepsilon\) to that
parameter coordinate at time \(t+1+j\).
These noise-response coefficients are computed through the auxiliary recursion
\[
    r_{-1}=0,\quad r_0=1,\quad
    r_{j+1}=(1+\mu-\eta)r_j-\mu r_{j-1},\qquad
    b_j=\eta r_j.
\]
In particular, \(b_0=\eta\) is the immediate response.
Without gradient noise, the parameter after \(j\) updates is
\(\w_{j,i}=p_j\w_{0,i}\), where \(p_j=r_j-\mu r_{j-1}\).
Substitution into the recursion gives \(p_j-p_{j+1}=b_j\), hence
\(\sum_{j=0}^{T-1}b_j=1-p_T\).  Unrolling the Newton update gives
\[
    \w_{T,i}=p_T\w_{0,i}
        -\frac{\sqrt{c_i}}{h_i}
            \sum_{j=0}^{T-1}b_j\zeta_{T-1-j,i}.
\]
The same coefficients apply to both coordinates.
The independent noise draws contribute variance with weights \(b_j^2\).
Cauchy--Schwarz therefore gives
\begin{align}
    \mathcal R_{\mathrm{newton+mom}}(1,T;\eta,\mu)
       &=\frac{1+h}{2}p_T^2+\frac{c}{2h}\sum_{j=0}^{T-1}b_j^2\notag\\
       &\ge\frac{1+h}{2}p_T^2+\frac{c}{2hT}(1-p_T)^2\notag\\
       &\ge\frac{(1+h)c/(hT)}{2(1+h+c/(hT))}
        =\frac1{8T}+o(T^{-1}).
    \label{eq:newton_small_batch_simple_lb}
\end{align}
The last inequality minimizes a quadratic over all real \(p_T\).
The bound is uniform in \((\eta,\mu)\), so it applies to the tuned
risk and to the \(\mu=0\) subfamily.

For SGD, choose \(\eta=1\), \(\mu=0\).  The sharp-coordinate bias
vanishes in one step, while the flat coordinate has effective step \(h\).
By \eqref{eq:nqm_preconditioned_risk},
\begin{align*}
    \mathcal R_{\mathrm{sgd}}(1,T;1)
       &=\frac h2(1-h)^{2T}+\frac{c}{2(2-h)}[1-(1-h)^{2T}]\\
       &=\frac{1+7e^{-2}}{16T}+o(T^{-1}).
\end{align*}
Since \((1+7e^{-2})/16<1/8\), this feasible SGD risk is strictly below
the Newton lower bound for all sufficiently large \(T\).  Both SGD
families contain this configuration, proving both small-batch inequalities.

\textbf{Large batch: \(\batch=T\).}
Only one update is available.  Because \(m_0=0\), the momentum coefficient
has no effect on either method's iterate.  Newton with \(\eta=1\)
eliminates the initial bias, giving
\[
    \mathcal R_{\mathrm{newton}}^\star(T;T)
      =\mathcal R_{\mathrm{newton+mom}}^\star(T;T)
      \le\frac{c}{2hT}=\frac1{8T}+o(T^{-1}).
\]
Dropping SGD's nonnegative noise term and optimizing its bias over all real
learning rates gives the lower bound
\begin{align*}
    \mathcal R_{\mathrm{sgd}}^\star(T;T)
      &=\mathcal R_{\mathrm{sgd+mom}}^\star(T;T)\\
      &\ge\frac12\min_{\eta\in\mathbb R}
             \{h(1-\eta h)^2+(1-\eta)^2\}\\
      &=\frac{h(1-h)^2}{2(1+h^3)}
       =\frac1{2T}+o(T^{-1}).
\end{align*}
The quadratic is minimized at \(\eta=(1+h^2)/(1+h^3)\).
Its leading coefficient exceeds Newton's, proving both large-batch
inequalities for all sufficiently large \(T\).
\end{proof}

For a dyadic budget \(T=2^k\) in the theorem's range, at least one adjacent
pair of dyadic batches satisfies
\(\mathcal R_{\mathrm{sgd+mom}}^\star(B;T)
\le\mathcal R_{\mathrm{newton+mom}}^\star(B;T)\) and
\(\mathcal R_{\mathrm{sgd+mom}}^\star(2B;T)
>\mathcal R_{\mathrm{newton+mom}}^\star(2B;T)\).
Choose \(2B\) as the first dyadic batch where the difference is positive.
The preceding batch has nonpositive difference, allowing an intermediate tie.

\subsubsection{Exact risks for general initializations}

The SGD and Newton learning-rate tuning experiments in
\Cref{app:nqm_tuning_experiments} use a fixed initialization \(\w_0\).
We give the corresponding exact expected losses below, using the functions
\(a_S\) and \(v_S\) defined above.
For SGD with learning rate \(\eta\), coordinate \(i\) satisfies
\begin{equation}
    \E[\w_{S,i}^2]
    =
    a_S(\eta h_i)\w_{0,i}^2
    +
    \frac{c_i}{\batch h_i^2}v_S(\eta h_i).
    \label{eq:nqm_sgd_bias_variance}
\end{equation}
For Newton, the same calculation gives
\begin{equation}
    \E[\w_{S,i}^2]
    =
    a_S(\eta)\w_{0,i}^2
    +
    \frac{c_i}{\batch h_i^2}v_S(\eta).
    \label{eq:nqm_newton_bias_variance}
\end{equation}
Consequently, if
\[
    A_0=\sum_{i=1}^d h_i\w_{0,i}^2,
    \qquad
    V_N(\batch)=\frac1\batch\sum_{i=1}^d\frac{c_i}{h_i},
\]
then the exact Newton risk is
\begin{equation}
    \mathcal R_{\mathrm{newton}}(\batch,S;\eta)
    =
    \frac12\left[A_0a_S(\eta)+V_N(\batch)v_S(\eta)\right].
    \label{eq:nqm_newton_tuned_objective}
\end{equation}

\subsubsection{Learning-rate scaling under a finite budget}
\label{app:nqm_tuning_experiments}

\Cref{fig:nqm_learning_rate_scaling}
shows how the tuned learning rate changes with batch size for SGD, Newton,
SignSGD, and RMSProp at a fixed training budget \(T\). We compare these tuned rates with
the linear scaling rule anchored at \(\batch=2\).

\begin{figure}[t]
    \vspace{-1.5em}
    \centering
    \includegraphics[width=\textwidth]{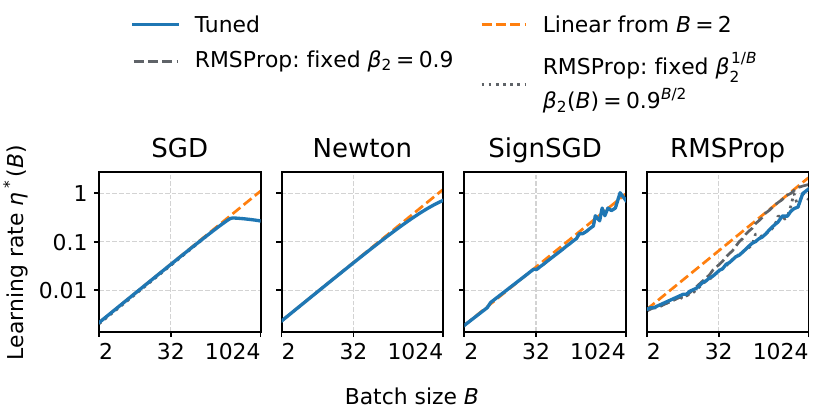}
    \caption{Tuned learning rates for SGD, Newton, SignSGD, and RMSProp
    (left to right) on the noisy quadratic with \(h=(0.2,1,5)\),
    \(c=(10,1,0.1)\), \(\w_0=(0.3,1,3)\), and \(T=4096\).
    Blue curves show per-batch tuned learning rates; orange dashed curves
    show linear scaling anchored at \(\batch=2\).
    The dotted gray SGD curve is the large-\(T\) approximation.
    For RMSProp, the blue curve jointly tunes \((\lr,\beta_2)\);
    the dashed gray curve fixes \(\beta_2=.9\), and the dotted gray curve
    keeps \(\beta_2^{1/\batch}\) fixed using \(\beta_2(\batch)=.9^{\batch/2}\).
    The RMSProp linear reference is anchored to the jointly tuned rate.}
    \label{fig:nqm_learning_rate_scaling}
\end{figure}

For SGD and Newton, we use the exact finite-\(S\) expressions in
\eqref{eq:nqm_sgd_bias_variance} and \eqref{eq:nqm_newton_tuned_objective},
respectively. The SignSGD panel uses no momentum
and no weight decay.  In coordinate \(i\), its scalar update is
\[
    \w_{k+1,i}=\w_{k,i}-\lr\operatorname{sign}
    \!\left(h_i\w_{k,i}+\sqrt{c_i/\batch}\,\zeta_{k,i}\right),
\]
where the \(\zeta_{k,i}\) are independent standard Gaussian draws.
A deterministic finite-volume recursion evaluates the terminal risk
\(\frac12\sum_i h_i\E[\w_{S,i}^2]\), with \(S=T/\batch\).
For the RMSProp panel, we use the deterministic mean-Fisher closure
\[
    \bar r_{k+1,i}
    =
    \beta_2\bar r_{k,i}
    +
    (1-\beta_2)
    \left(
        h_i^2 M_{k,i}
        +
        \frac{c_i}{\batch}
    \right),
    \qquad
    M_{k,i}:=\E[\w_{k,i}^2],
\]
together with the corresponding linear second-moment update under the
preconditioner \(\bar r_{k+1,i}^{-1/2}\).  We then tune RMSProp by a
two-dimensional grid sweep,
\[
    \inf_{\substack{\lr>0\\0\le \beta_2<1}}
    \mathcal R_{\mathrm{rms}}\!\left(\batch,\frac{T}{\batch};\lr,\beta_2\right),
\]
and plot two one-dimensional smoothing prescriptions as baselines:
\(\beta_2(\batch)\equiv .9\) and
\(\beta_2(\batch)^{1/\batch}=.9^{1/2}\), equivalently
\(\beta_2(\batch)=.9^{\batch/2}\).  Even in these simplified diagnostics, the
tuned steps do not remain on the linear reference anchored at \(\batch=2\):
Newton bends smoothly through the horizon \(T/\batch\), while SignSGD and
RMSProp show additional bends from sign reliability and the empirical-Fisher
feedback.

\subsubsection{Optimizer crossover in the noisy quadratic model}
\label{app:nqm_crossover_experiments}

We now compare optimizer families by their tuned terminal risk at a fixed
training budget \(T\). Their relative ranking changes with batch size,
exhibiting optimizer crossover even after each family is retuned.

\Cref{fig:nqm_momentum_loss_vs_batch} compares tuned SGD-M,
Newton-M, SignSGD-M, and Adam at the baseline NQM noise level;
\Cref{fig:nqm_optimizer_crossover} gives the corresponding
four-times-lower-noise variant.  In both experiments, each family is retuned at
every batch size over \(\lr\) and first-moment retention (\(\mom\) for SGD, Newton, and SignSGD; \(\beta_1\) for Adam), with Adam additionally tuning
\(\beta_2\).  Momentum improves the small-batch SignSGD behavior.

\begin{figure}[t]
    \vspace{-1.5em}
    \centering
    \includegraphics[width=\textwidth]{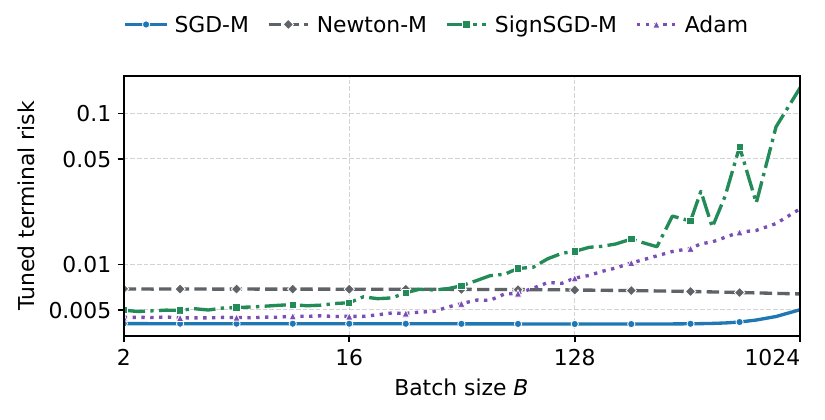}
    \caption{
    Tuned finite-budget risk versus batch size for momentum variants on the
    same three-coordinate noisy quadratic as
    \Cref{fig:nqm_learning_rate_scaling}, with \(T=4096\).  SGD-M,
    Newton-M, and SignSGD-M tune both \(\lr\) and \(\mom\) at each batch
    size; Adam tunes \(\lr\), \(\beta_1\), and \(\beta_2\).  The plot uses a dense
    local refinement around the broad-sweep winners; the linear methods are
    evaluated by the exact second-moment recursion, and the SignSGD/Adam
    candidates are independently checked by high-chain Monte Carlo validation.
    }
    \label{fig:nqm_momentum_loss_vs_batch}
\end{figure}

\FloatBarrier
\subsection{Finite-budget SignSGD batch-size transfer}

A fixed batch-size scaling rule transports one optimizer family along a
low-dimensional hyperparameter path.  This is useful when the goal is to
preserve a chosen invariant, but it is not the same as solving
\eqref{eq:nqm_tuned_finite_budget_risk}.  In the noisy quadratic model, the
target risk is a sum over coordinates whose effective tuples
\((h_i,c_i,\w_{0,i},\batch,T/\batch)\) can demand different behavior.  The SGD and
Newton formulas \eqref{eq:nqm_sgd_bias_variance}--\eqref{eq:nqm_newton_bias_variance}
already show two different bias-variance tradeoffs; SignSGD depends on
sign reliability; and RMSProp estimates an empirical Fisher whose mean depends on the stationary law
of \(\w_i\), not just on the explicit noise variance.

Equivalently, a scaling rule \(R\) from \(\batch_0\) to \(\batch\) compares
\[
    \E f\!\left(\w_{T/\batch}^{A,R_{\batch_0\to\batch}(\nu_0)}(\batch)\right)
\]
against the tuned quantity \(\mathcal R_A^\star(\batch;T)\).  It is ideal only
if the transported hyperparameters stay close to the optimizer-family optimum
for every relevant noise ratio \(\chi_i=c_i/h_i^2\) and signal-to-noise ratio
\(\mathrm{SNR}_{i,0}(\batch)\).  The SDE and convex prescriptions can both miss this target.

\subsubsection{Transfer from every reference batch against the per-batch optimum}
\label{app:nqm_oracle_transfer_protocol}

\Cref{fig:nqm_scaling_rules}(a)--(c) uses a three-directional NQM with
$H=\diag(1,0.1,0.01)$ and $\Sig=\diag(0.01,0.01,0.01)$.
Each trajectory starts from $\w_0\sim\normal(0,I_3)$ and $m_0=0$.
Momentum SignSGD uses $m_{t+1}=\mu m_t+(1-\mu)\widehat g_B(\w_t)$
and $\w_{t+1}=\w_t-\eta\operatorname{sign}(m_{t+1})$, with no weight decay.
Every run processes $T=16384$ samples.
We evaluate batches $B\in\{2^j:j=0,\ldots,14\}$, giving exactly $T/B$ updates.
The largest batch is $B=T=16384$, for which there is one update.

At each batch we tune shared hyperparameters by minimizing the sum of the
three directional terminal losses.
The coarse search uses $33$ logarithmically spaced learning rates in
$[10^{-5},2]$ and $17$ logarithmically spaced refresh coefficients
$1-\mu$ in $[10^{-5},1]$.
An additional candidate is $(\eta,1-\mu)=(0.01,0.01)$, for $562$ candidates in total.
Each candidate is evaluated with $192$ trajectories.
A local search around six promising regions uses $768$ fresh trajectories
per candidate, followed by an independent sample of $4096$ trajectories
to select among the finalists.
All candidates within a stage share initialization and oracle draws to reduce
noise in their comparison; search seeds differ across batch sizes.
Each selected recipe is the best configuration found by this search.
At $B_0=1$, the local search evaluates $318$ distinct candidates and the final
selection compares $29$ finalists, returning $\eta_0=0.0004574766$ and
$\mu_0=0.90503702$.

From each reference batch $B_0$ we transport its recipe to every target batch $B$
by $\eta=(B/B_0)\eta_0$, $\mu=\mu_0^{B/B_0}$ for the SDE prescription and by
$\eta=\sqrt{B/B_0}\,\eta_0$, $\mu=\mu_0$ for the convex one.
The latter uses the square-root prescription motivated by the convex analysis.
All transported recipes for one target batch, together with that batch's own selected
recipe, are evaluated on shared initialization and oracle draws, using eight
independent groups of $8192$ trajectories with seeds separate from tuning and selection.
Transporting a recipe to its own batch is the identity, so both rules reproduce the
selected recipe exactly at $B=B_0$.
The plotted quantity is the excess loss, the total terminal loss minus its
selected value at $B=1$, and the bands give one standard error across the eight groups.

\FloatBarrier
\subsection{Learning-rate scaling from the movement ansatz}
\label{app:nqm_movement_scaling}

\subsubsection{Shared movement ansatz}
\label{app:nqm_movement_ansatz}

For scalar parameter $w$, curvature $h>0$, noise variance $c>0$, and independent standard-normal $z_t$, the minibatch gradient at a fixed state $w$ is $g_t=hw+\sqrt{c/B}\,z_t$.
We evaluate both optimizer responses below at $w=1$, keeping their decay coefficients fixed as batch size varies.
Write an optimizer step as $w_{t+1}=w_t-\eta q_t$, where $q_t$ is the update direction computed from the gradient and optimizer state.
Holding $w$ fixed while the optimizer state evolves, let $A(B)=\E[q_t]$ under its stationary law.
The expected parameter displacement per processed sample is then $-\eta(B)A(B)/B$.
For $A(B)\ne0$, matching this displacement to its value at reference batch one gives the movement ansatz
\begin{equation}
 \frac{\eta(B)}{\eta(1)}=\frac{B A(1)}{A(B)}.
 \label{eq:nqm_movement_ansatz}
\end{equation}
This criterion matches mean movement at the frozen state, without imposing equality of update variance or terminal training loss.
The two optimizers differ in their response $A(B)$, which we compute below.

\subsubsection{Momentum-SignSGD response}
\label{app:signsgd_cnr_mechanism}

With retention $\mu\in[0,1)$ the moving average $m_t=\mu m_{t-1}+(1-\mu)g_t$ is a linear function of Gaussian variables, so its stationary law is Gaussian with
\begin{equation}
 \E[m]=hw,
 \qquad
 \Var(m)=\frac{1-\mu}{1+\mu}\cdot\frac{c}{B}.
 \label{eq:signsgd_stationary_ema}
\end{equation}
For momentum SignSGD, $q_t=\mathrm{sign}(m_t)$, so the response is
\begin{equation}
 A(B)=\E[\mathrm{sign}(m)]=\mathrm{erf}\!\left(\frac{hw}{\sqrt{2\Var(m)}}\right)
 =\mathrm{erf}\!\left(hw\sqrt{\frac{1+\mu}{1-\mu}\cdot\frac{B}{2c}}\right).
 \label{eq:signsgd_response}
\end{equation}
The Gaussian sign expectation gives this response in closed form.

Since $\mathrm{erf}(x)\approx 2x/\sqrt\pi$ for small $x$ and $\mathrm{erf}(x)\to1$ for large $x$, the response grows as $\sqrt B$ when noise dominates and saturates when signal dominates, so \eqref{eq:nqm_movement_ansatz} interpolates between $\eta\propto\sqrt B$ and $\eta\propto B$.
The noise-dominated limit requires $\Var(m)\gg h^2w^2$ at the largest batch, that is $\mathrm{CNR}\ll\frac{1-\mu}{1+\mu}B^{-1}$; at $B=256$ this is $\mathrm{CNR}\ll3.9\times10^{-3}$ for $\mu=0$ but $\mathrm{CNR}\ll2.1\times10^{-4}$ for $\mu=0.9$.
Momentum therefore does not change the two limits, only how much noise is needed to reach the square-root one.

Fitting \eqref{eq:nqm_movement_ansatz} over $B=1,\ldots,256$ at $\mu=0.9$ predicts exponents $0.591$, $0.911$, and $1.000$ for $\mathrm{CNR}=0.001$, $0.03$, and $1$, against measured $0.588$, $0.794$, and $0.904$.
The ansatz is accurate in the noise-dominated condition and overpredicts as CNR rises.
This is expected: the ansatz freezes $w$ at its initial scale, whereas a tuned trajectory spends much of the budget near the optimum, where $hw$ is small and the sign is correspondingly more noise-driven than the frozen-state calculation assumes.
The ansatz therefore bounds the measured exponents from the linear side rather than tracking them.

\subsubsection{Momentum-SignSGD learning-rate scaling across CNR conditions}
\label{app:signsgd_cnr_protocol}

\textbf{Model and tuning.}
\Cref{fig:nqm_scaling_rules}(d) uses $L(w)=w^2/2$ with $h=1$ and $c\in\{1000,100/3,1\}$, giving $\mathrm{CNR}\in\{0.001,0.03,1\}$.
The budget is $T=4096$ samples and batches are $B=2^j$, $j=0,\ldots,8$.
At update $t$, the optimizer forms $g_t=w_{t-1}+\sqrt{c/B}\,z_t$ and $m_t=\mu m_{t-1}+(1-\mu)g_t$, then updates $w_t=w_{t-1}-\eta\,\mathrm{sign}(m_t)$, with $m_0=0$ and no weight decay.
The main comparison uses $\mu=0.9$; the no-momentum control uses $\mu=0$.

The initialization is $w_0\sim\normal(0,1)$, shared across learning-rate candidates and noise conditions.
A fixed $w_0$ cannot be used here: because SignSGD moves a constant distance $\eta$ per update, the iterate stays on the lattice $w_0-k\eta$, and every $\eta$ that divides $w_0$ lands exactly on the optimum.
At $\mathrm{CNR}=100$ and $B=8$ this makes the tuning objective multimodal, with $26$ local minima over the search range and terminal loss exactly zero at $\eta=1/2$; a random $w_0$ leaves a single local minimum.
This is a structural feature of constant-step methods rather than a search failure, and it does not arise for the normalized updates of \Cref{app:rmsprop_cnr_protocol}.

At each batch and noise condition, a coarse grid of $121$ learning rates from $10^{-6}$ to $2$ is evaluated on 512 trajectories.
We refine three separated candidate neighborhoods and select a learning rate with 2048 fresh trajectories.
All refined candidates are then evaluated on eight independent groups of 1024 held-out trajectories.
Shared Gaussian draws reduce comparison noise across candidates and conditions.
The largest ratio of a selected candidate's held-out loss to the minimum on its validation grid is $1.0516$.

\textbf{Selected learning rates.}
For the three CNR conditions in increasing order, the selected learning rate increases by factors of $27.2$, $67.7$, and $135.8$ from $B=1$ to $B=256$ at $\mu=0.9$.
Fitting $\log\eta$ against $\log B$ over the nine batch sizes gives exponents $0.588$, $0.794$, and $0.904$.
The no-momentum control gives $0.505$, $0.591$, and $0.853$ in the same conditions.
Because the noise levels enter only through $\sqrt{c/B}$ applied to a shared standard normal draw, adding a noise condition leaves the others bitwise unchanged, which we verified across all cached cells.

\textbf{Frozen-state check.}
The response $A(B)$ of \Cref{app:signsgd_cnr_mechanism} is available in closed form because the moving average of Gaussian gradients is Gaussian.
We verified it against a frozen-state simulation at $w=1$ that discards 3000 updates and averages 3000 further updates over $2\times10^{5}$ trajectories; the largest absolute discrepancy is $2.0\times10^{-4}$ at $\mu=0.9$.

\subsubsection{RMSProp response}
\label{app:rmsprop_cnr_mechanism}

Write $\beta_2$ for the squared-gradient accumulator's decay coefficient, $v_t$ for that accumulator, and $\epsilon>0$ for the numerical stabilizer.
The frozen-state experiment evolves $v_t=\beta_2v_{t-1}+(1-\beta_2)g_t^2$ while holding $w=1$.
Here $q_t=g_t/(\sqrt{v_t}+\epsilon)$, so the response is the stationary mean normalized gradient $A(B)=\E[g_t/(\sqrt{v_t}+\epsilon)]$.

The stationary accumulator has mean $h^2w^2+c/B$.
Replacing the random accumulator by this mean and neglecting $\epsilon$ gives a mean-denominator approximation.
At $w=1$, with $\mathrm{CNR}=h^2/c$, it predicts
\begin{equation}
 \frac{\eta(B)}{\eta(1)}\approx
 B\sqrt{\frac{\mathrm{CNR}+1/B}{\mathrm{CNR}+1}}.
 \label{eq:rmsprop_movement_approximation}
\end{equation}
Its local logarithmic slope is $1-[2(1+B\,\mathrm{CNR})]^{-1}$, approaching $1/2$ when noise dominates and $1$ when signal dominates.
In general, $\E[g/\sqrt v]\neq\E[g]/\sqrt{\E[v]}$.
The dashed curves in \Cref{fig:rmsprop_cnr_scaling} therefore use measured stationary responses in \eqref{eq:nqm_movement_ansatz}, rather than assuming the mean-denominator approximation is exact.
The conditional gradient signal-to-noise ratio along training is $Bw^2\,\mathrm{CNR}$, which changes even though CNR remains fixed.
\subsubsection{RMSProp learning-rate scaling across CNR conditions}
\label{app:rmsprop_cnr_protocol}

The same CNR dependence appears for adaptive optimizers.
We repeat the experiment of \Cref{sec:nqm_signsgd_scaling} with RMSProp, which divides each gradient by the square root of a moving average of squared gradients.
Its one-dimensional conditional second moment is $\E[\widehat g_B(w)^2\mid w]=h^2w^2+c/B$: the signal term stays fixed at a given parameter value while the noise term falls with batch size.
When noise dominates, a larger batch shrinks the normalization denominator through $c/B$, motivating square-root scaling; when signal dominates, the denominator barely changes and linear scaling compensates for fewer updates per sample.

In \Cref{fig:rmsprop_cnr_scaling} the exponent rises from $0.48$ at $\mathrm{CNR}=0.01$ through $0.75$ at $\mathrm{CNR}=1$ to $1.02$ at $\mathrm{CNR}=100$.
Adam shows the same trend once its first moment is enabled, with the exponent rising from $0.47$ to $0.94$ at $\beta_2=0.9$ and from $0.48$ to $0.95$ at $\beta_2=0.99$ across the same three conditions (\Cref{fig:adam_cnr_scaling}).
Tuned SGD stays between $0.96$ and $0.99$ in all three, showing that the dependence comes from the adaptiveness introduced by the second-moment normalization.

\begin{figure}[b]
    \vspace{-1.5em}
\centering
\begin{minipage}[c]{0.60\textwidth}
\centering
\includegraphics[width=\linewidth]{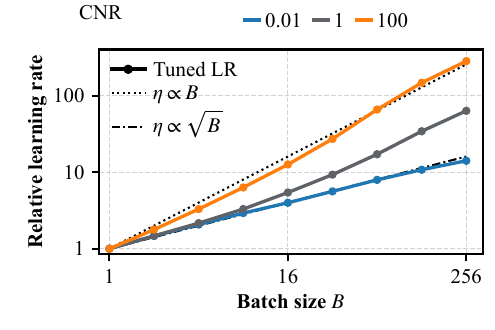}
\end{minipage}\hfill
\begin{minipage}[c]{0.36\textwidth}
\caption{\textbf{CNR changes RMSProp learning-rate scaling.}
Colors indicate CNR. Curves show learning rates tuned for terminal loss,
each relative to its own value at $B=1$.
Dotted and dash-dotted lines mark linear and square-root scaling.}
\label{fig:rmsprop_cnr_scaling}
\end{minipage}
\end{figure}

\textbf{Training and independent tuning.}
\Cref{fig:rmsprop_cnr_scaling} uses $L(w)=w^2/2$, $w_0=1$, and $c\in\{100,1,0.01\}$, giving $\mathrm{CNR}\in\{0.01,1,100\}$.
The budget is $T=4096$ samples and batches are $B=2^j$, $j=0,\ldots,8$.
At update $t$, RMSProp forms $g_t=w_{t-1}+\sqrt{c/B}\,z_t$ and $v_t=\beta_2v_{t-1}+(1-\beta_2)g_t^2$, then updates $w_t=w_{t-1}-\eta g_t/(\sqrt{v_t}+\epsilon)$.
Here $v_0=0$, $\epsilon=10^{-8}$, and the main comparison uses $\beta_2=0.9$ without bias correction, first-moment momentum, or weight decay.
The controls divide $v_t$ by $1-\beta_2^t$ before normalization, with $\beta_2=0.9$ or $0.99$.

At each batch and noise condition, a coarse grid of 101 learning rates from $10^{-5}$ to $1$ is evaluated on 512 trajectories.
We refine three separated candidate neighborhoods and select a learning rate with 2048 fresh trajectories.
All refined candidates are then evaluated on eight independent groups of 1024 held-out trajectories.
Shared Gaussian draws reduce comparison noise across candidates and conditions; trajectory indices remain independent within each candidate.
The largest ratio of a selected candidate's held-out loss to the minimum on its validation grid is $1.00838$.

For the three CNR conditions in increasing order, the selected learning rate increases by factors of $14.09$, $63.26$, and $281.84$ from $B=1$ to $B=256$.
The corresponding endpoint slopes $\log[\eta(256)/\eta(1)]/\log256$ are $0.477$, $0.748$, and $1.017$.
Candidates within $5\%$ of the minimum held-out loss at each endpoint give slope ranges $[0.422,0.518]$, $[0.720,0.768]$, and $[1.004,1.031]$.
These are descriptive loss-tolerance ranges conditional on the searched grid, rather than confidence intervals.
With bias correction, the slopes are $(0.533,0.831,1.100)$ at $\beta_2=0.9$ and $(0.554,0.865,1.128)$ at $\beta_2=0.99$.
The SGD control selects the minimum exact expected terminal loss on a dense learning-rate grid and gives slopes $(0.986,0.974,0.962)$.

\textbf{Frozen-state response measurement.}
The RMSProp response in \Cref{app:rmsprop_cnr_mechanism} is measured at $w=1$ for both decay coefficients.
We initialize the accumulator at its fixed-state mean, discard 2048 updates, and average 512 further updates in each of eight independent groups of 512 antithetic pairs.
The state remains fixed while the squared-gradient accumulator evolves.
Uncertainty is assessed across the eight groups, accounting for dependence between successive updates within a trajectory.
The measured learning-rate ratios differ from the mean-denominator approximation by at most $4.31\%$ across tested settings; their largest relative group-based standard error is $0.023\%$.
An independent scalar implementation matches the vectorized training evaluator to absolute error below $2\times10^{-16}$.
In a fresh 1024-trajectory check of the main selected configurations, the stabilizer is at most $0.017\%$ of the normalization denominator.

\begin{figure}[t]
    \vspace{-1.5em}
\centering
\includegraphics[width=\textwidth]{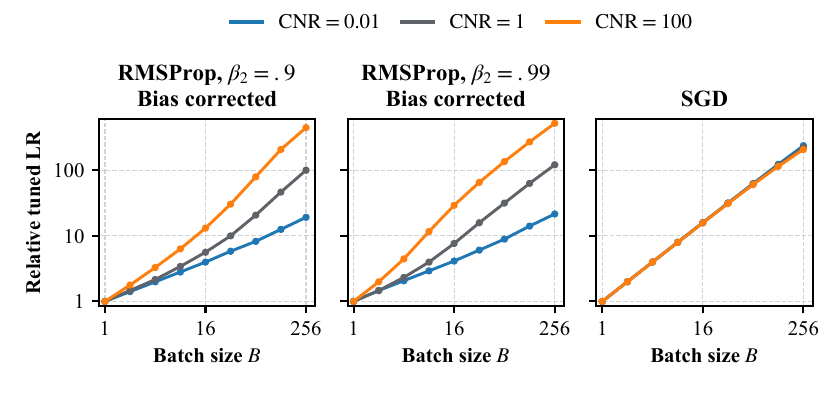}
\caption{\textbf{The dependence on the curvature-to-noise ratio (CNR) survives accumulator bias correction.}
The two RMSProp panels divide the squared-gradient exponential moving average by $1-\beta_2^t$ at
update $t$, with second-moment decay $0.9$ and $0.99$, respectively.
The SGD panel selects learning rates using the exact expected terminal loss.
All panels use the same one-dimensional NQM, initial point, sample budget,
noise variances, and batch sizes as \Cref{fig:rmsprop_cnr_scaling}.
Learning rates are normalized separately at $B=1$.
SGD's relative learning-rate curves remain close to linear in batch size in
these conditions, while RMSProp's curves separate with CNR.}
\label{fig:rmsprop_cnr_controls}
\end{figure}

\textbf{Adam.}
\Cref{fig:adam_cnr_scaling} repeats the sweep with Adam, adding a first moment
$m_t=\beta_1m_{t-1}+(1-\beta_1)g_t$ at $\beta_1=0.9$ and dividing both moments by their
bias-correction factors before the update.
The model, budget, batch grid, search grid, trajectory counts, and selection protocol are
unchanged, and search seeds are drawn separately from the RMSProp sweep.
For the three CNR conditions in increasing order, the endpoint slopes are
$(0.471,0.713,0.941)$ at $\beta_2=0.9$ and $(0.478,0.727,0.948)$ at $\beta_2=0.99$.

\begin{figure}[t]
    \vspace{-1em}
\centering
\includegraphics[width=\textwidth]{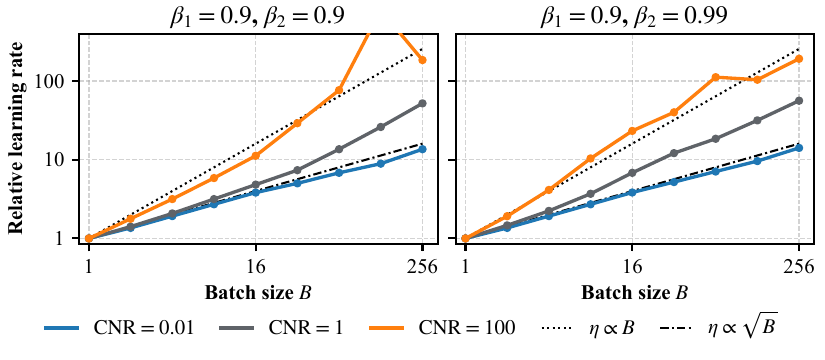}
\caption{\textbf{Adam shows the same CNR dependence as RMSProp.}
The tuning protocol, model, budget, and batch grid match \Cref{fig:rmsprop_cnr_scaling},
with Adam's first moment enabled and bias correction applied to both moments.
Colors indicate CNR; curves show learning rates tuned for terminal loss, each relative to its own value at $B=1$.
Dotted and dash-dotted lines mark linear and square-root scaling.}
\label{fig:adam_cnr_scaling}
\end{figure}

\clearpage
\section{Additional Matrix-Model Results}
\label{app:matrix_results}

The main text shows optimizer crossover on a noisy quadratic
(\Cref{fig:nqm_optimizer_crossover}).  Crossovers also occur on three matrix
regression tasks: at the higher noise level, the best optimizer changes from Muon
at small batches to EMA-Shampoo at large batches in matrix linear and logistic
regression (\Cref{fig:matrix_models_high_noise_tuned_loss_vs_batch}).  The
following figures report all tuned optimizer families on all three matrix tasks
at the baseline and higher noise levels.

\subsection{Matrix-model definitions}
\label{app:nqm_matrix_model_definitions}

\textbf{Matrix linear regression.}
Let \(X=ab^\top\) be a rank-one matrix covariate, where \(a\) and \(b\) are
independent, centered random vectors satisfying
\(\E[aa^\top]=\mathsf A\) and \(\E[bb^\top]=\mathsf C\).  Under the linear
teacher $y=\langle\Theta_\star,X\rangle+\epsilon$ and squared loss,
\begin{equation}
    \begin{aligned}
        \ell_{\mathrm{lin}}(\Theta;X,y) &=\frac12\bigl(\langle\Theta,X\rangle-y\bigr)^2
    \end{aligned}
    \label{eq:overview_matrix_linear_setting}
\end{equation}
with \(\E[\epsilon\mid X]=0\).  This model preserves a quadratic population
objective but exposes matrix-aware optimizers to left--right feature geometry
and rank-deficient per-example gradients to better simulate the neural network landscape.

\textbf{Matrix logistic regression.}
Using the same rank-one covariates, define
\(\varsigma(u)=(1+e^{-u})^{-1}\), a temperature-controlled teacher
\(p_\tau(X)=\varsigma(\langle\Theta_\star,X\rangle/\tau)\), and
a student \(p_\Theta(X)=\varsigma(\langle\Theta,X\rangle)\).
The labels and population excess risk are
\begin{equation}
    y\mid X\sim\operatorname{Bernoulli}\!\left(p_\tau(X)\right),
    \qquad
    L_{\mathrm{log}}(\Theta)-L_{\mathrm{log}}^\star
    =\E_X\!\left[
        D_{\mathrm{KL}}\!\left(
            \operatorname{Bern}\!\left(p_\tau(X)\right)
            \,\middle\|\,
            \operatorname{Bern}\!\left(p_\Theta(X)\right)
        \right)
    \right].
    \label{eq:overview_matrix_logistic_setting}
\end{equation}
Unlike the preceding objectives, its curvature depends on the current
predictive probabilities and is therefore state dependent.

\textbf{Normalized linear regression.}
Finally, let the predictor depend only on the direction of the matrix
parameter through
\(Q(\Theta)=\Theta/\|\Theta\|_{\mathrm F}\) and
\(Q_\star=\Theta_\star/\|\Theta_\star\|_{\mathrm F}\).  With the same
rank-one design and a linear teacher,
\begin{equation}
    y=\langle Q_\star,X\rangle+\epsilon,
    \qquad
    L_{\mathrm{nlin}}(\Theta)-L_{\mathrm{nlin}}^\star
    =\frac12\left\|
        \mathsf A^{1/2}\bigl(Q(\Theta)-Q_\star\bigr)\mathsf C^{1/2}
    \right\|_{\mathrm F}^2.
    \label{eq:overview_normalized_linear_setting}
\end{equation}
The resulting objective is invariant to positive rescaling of \(\Theta\),
introducing a radial null direction and optimizer-dependent angular dynamics.

\begin{figure}[t]
    \vspace{-1.5em}
    \centering
    \includegraphics[width=\textwidth]{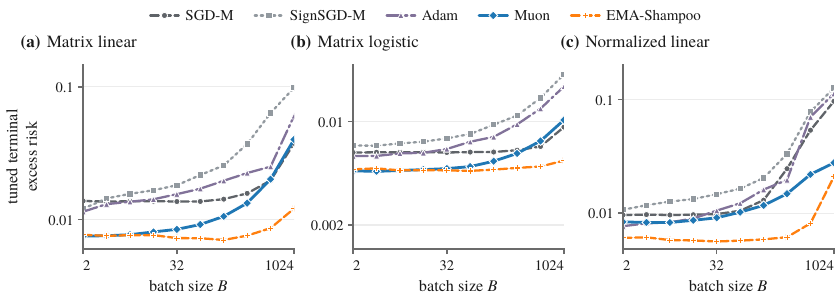}
    \caption{
    Tuned finite-budget terminal excess risk in the \(4\times4\) matrix
    prediction models.  Every point uses \(T=4096\) examples and
    \(T/\batch\) updates; each optimizer is retuned separately at every batch
    size.  Minibatch sufficient statistics are sampled exactly from the
    finite-support design, while the expectation over training trajectories is
    estimated by independent float64 validation.  Shaded bands show \(95\%\)
    Monte Carlo intervals for the exact population excess risk.
    }
    \label{fig:matrix_models_tuned_loss_vs_batch}
\end{figure}

\begin{figure}[t]
    \vspace{-1.5em}
    \centering
    \includegraphics[width=\textwidth]{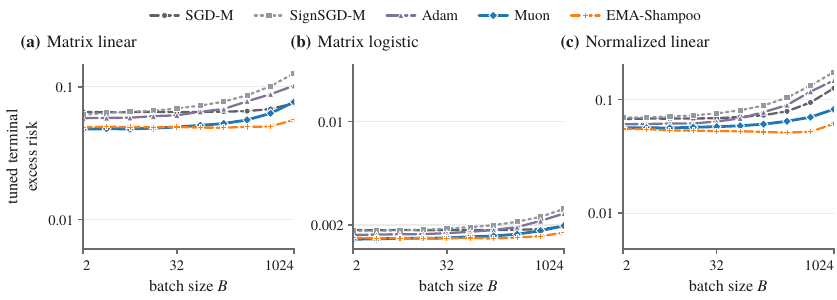}
    \caption{
    Higher-noise counterpart of
    \Cref{fig:matrix_models_tuned_loss_vs_batch}.  The squared-loss models use
    \(\sigma_\epsilon=4\), and the logistic teacher uses
    \(\tau_{\mathrm{log}}=4\).  The remaining settings and protocols, including the \(95\%\) Monte Carlo
    intervals, are unchanged.
    }
    \label{fig:matrix_models_high_noise_tuned_loss_vs_batch}
\end{figure}

\clearpage

\begingroup
\raggedbottom
\section{Scaling-rule sweeps on CIFAR-5M and language modeling}
\label{app:cifar_experiments}
\label{app:scaling_experiments}

We test how well a scaling rule transfers a tuned hyperparameter configuration across batch sizes at a fixed data budget.
The CIFAR-5M and language-model sweeps favor different rules, and the best rule also changes across target batches within each setting.
We first define the common comparison, then report each sweep, the additional language-model weight-decay experiment, and the comparison across settings.
The complete loss tables follow the results.

\subsection{Shared design and evaluation}
\label{app:s2-protocol}

\textbf{Scaling coordinates.}
Both experiments apply MuonW to hidden weight matrices and AdamW to the remaining parameters, which we call auxiliary parameters.
We write their learning rates as \(\eta_{\rm M}\) and \(\eta_{\rm A}\), their weight decays as \(\lambda_{\rm M}\) and \(\lambda_{\rm A}\), Muon momentum as \(\mu\), and Adam's first- and second-moment retention coefficients as \(\beta_1\) and \(\beta_2\).
The MuonW update is that of \Cref{def:muonw-adamw}, with \(\omega=\mu\).
In language modeling, embeddings, the output projection, and scalar parameters have distinct auxiliary learning rates; one scaling choice multiplies all three by the same factor.
Let \(B_0\) be the reference batch, \(B\) a target batch, and \(\kappaBatch=B/B_0\) their ratio.
For each learning-rate or weight-decay coordinate \(h\), with reference value \(h_0\), the choices are
\[
  h(B)=h_0\kappaBatch^p,\qquad p\in\{0,\tfrac12,1\},
\]
corresponding to fixed, square-root, and linear scaling.
For a retention coefficient \(q\in\{\mu,\beta_1,\beta_2\}\), with reference value \(q_0\), we either keep \(q(B)=q_0\) or preserve its decay in data units.
On CIFAR-5M, the latter choice is \(q(B)=q_0^{\kappaBatch}\), preserving the half-life in processed examples.
For language modeling, let \(\tau\) be the token budget and \(S(B)=\lceil\tau/B\rceil\) the number of updates.
We use \(q(B)=q_0^{S(B_0)/S(B)}\), so \(q(B)^{S(B)}=q_0^{S(B_0)}\) preserves retention over the complete run.
The two exponents agree when the step-count ratio equals the batch ratio; \Cref{app:s2-llm} specifies the rounding exception.

\textbf{Complete grids.}
CIFAR-5M varies both learning rates and both weight decays independently, giving \(3^4 2^3=648\) rules.
The language-model grid holds auxiliary weight decay at \(0.001\) and varies the other six coordinates, giving \(3^3 2^3=216\) rules.
Rule selection and the comparisons below use the complete grid in each setting, with auxiliary weight decay varied independently on CIFAR-5M.
In tables, LR and WD denote learning rate and weight decay; F, S, and L denote fixed, square-root, and linear scaling.
E, also labeled EMA in the figures and tables, denotes the retention-preserving choice defined above for the relevant experiment.
Subscripts \(\mathrm{mat}\) and \(\mathrm{aux}\) in the tables correspond to \(\mathrm M\) and \(\mathrm A\).

\textbf{Evaluation and rule selection.}
Let \(\mathcal B\) be the set of target batches and \(L_r(B)\) the validation loss of rule \(r\) at batch \(B\), using the metric specified for each experiment below.
Write \(L_{\min}(B)=\min_{r'}L_{r'}(B)\) for the smallest loss in the tested grid.
For these scaling-rule sweeps, the batch-size (BSZ) regret is the validation-loss gap to the best tested rule at the same batch size.
We reserve excess loss for the loss gap to the retuned \(B=1\) reference in \Cref{fig:nqm_scaling_rules}(a)--(c).
For each rule, we report its mean and maximum BSZ regret over the target batch-size choices,
\[
  \Delta_r(B)=L_r(B)-L_{\min}(B),\qquad
  \overline\Delta(r)=\frac{1}{|\mathcal B|}\sum_{B\in\mathcal B}\Delta_r(B),\qquad
  \Delta_{\max}(r)=\max_{B\in\mathcal B}\Delta_r(B).
\]
We select a common rule by the smallest mean BSZ regret, weighting target batch-size choices equally.
Ranking by mean loss gives the same order because all rules share the same target batches.
The coordinate panels fix one coordinate choice and average loss over every combination of the remaining choices.

\subsection{CIFAR-5M sweep}
\label{sec:cifar5m_scaling}
\label{app:s2-cifar}
\label{app:cifar5m_scaling}

\textbf{Protocol and reference configuration.}
We train a ResNet with GroupNorm in one pass over CIFAR-5M~\citep{nakkiran2021deepbootstrap}.
GroupNorm prevents changing batch size from changing the normalization statistics compared to BatchNorm.
Every run processes the same number of examples with constant learning rates; schedule quantities are expressed as fractions of processed examples.
The reference batch is \(B_0=256\), and the targets are \(B\in\{64,128,512,1024,2048,4096\}\).
Validation loss is per-example cross-entropy averaged over the final ten evaluations, equally spaced in processed examples.
We tune the hyperparameters at the reference batch size by coordinate descent with $\lambda_M = \lambda_A$, obtaining the reference configuration
\begin{align*}
  (\eta_{\rm M},\eta_{\rm A},\mu,\beta_1,\beta_2,\lambda_{\rm M},\lambda_{\rm A})_0
  ={}&(0.00681567,\,0.000980392,\,0.93002885,\\
     &0.6561,\,0.999911571,\,0.00125,\,0.00125),
\end{align*}
and its validation loss is \(0.2314310\).
We evaluate all \(648\) rules at all six targets using the same code, data stream, and schedule, giving \(3{,}888\) runs with seed \(42\).

\textbf{Square-root learning rates with the other coordinates fixed give the best common rule.}
This rule minimizes mean BSZ regret as defined in \Cref{app:s2-protocol}.
Its mean and maximum BSZ regrets are \(0.00492\) and \(0.00913\).
At \(B=(64,128,512,1024,2048,4096)\), its losses are
\[
  (0.23370,\,0.23136,\,0.23373,\,0.22778,\,0.23290,\,0.23683).
\]
\Cref{fig:cifar5m_transfer_axis_marginals} shows the mean loss across rules with each coordinate choice.
The individual winners still change with batch size.
At \(B=64\), the best rule scales both learning rates and weight decays linearly, transports Muon momentum and Adam \(\beta_2\), and keeps \(\beta_1\) fixed.
At \(B=4096\), it scales the Muon learning rate and auxiliary weight decay by the square root, transports \(\beta_2\), and keeps the remaining coordinates fixed.
\Cref{tab:cifar5m_transfer_raw} gives the complete grid.

\begin{figure}[h]
    \vspace{-1em}
    \centering
    \includegraphics[width=\textwidth]{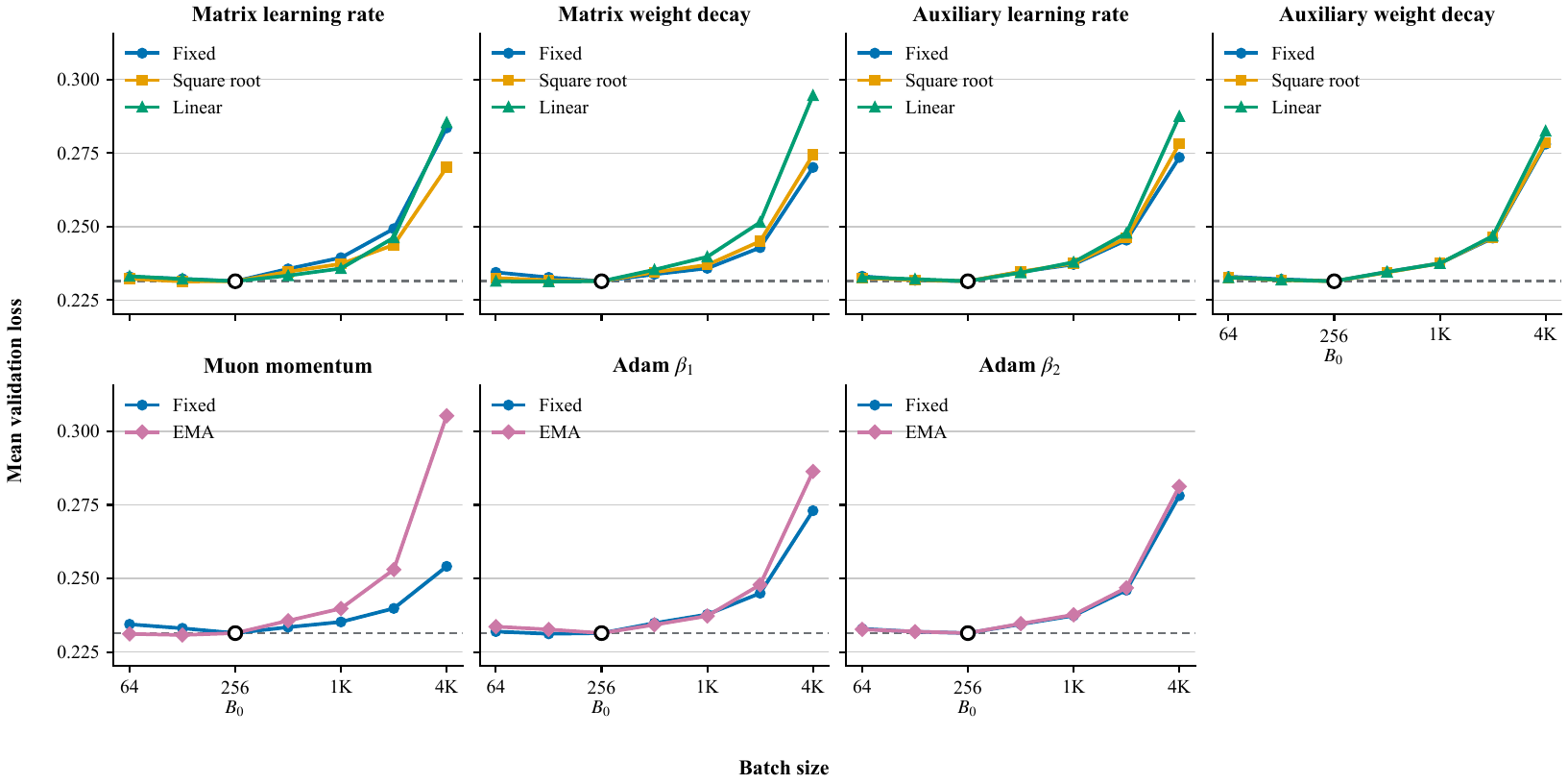}
    \caption{\textbf{All \(648\) CIFAR-5M scaling rules, summarized one
    coordinate at a time.}  In each panel, a point at a target batch is the
    arithmetic mean of the final-window validation losses over all rules that make the
    indicated choice for that coordinate, weighting all combinations of the other six equally.
    The outlined marker and dashed line mark the shared reference at
    \(B_0=256\).  Fixed, square root, and linear denote exponents
    \(0\), \(\tfrac12\), and \(1\); EMA preserves the half-life in examples.
    Lower is better; all panels share the same vertical range.}
    \label{fig:cifar5m_transfer_axis_marginals}
\end{figure}

\textbf{Independent retuning provides a comparison beyond the scaling grid.}
For the full-retuning curve in \Cref{fig:muonw_transfer_comparison}, we tune all seven hyperparameters independently at each target from \(512\) to \(4096\).
The coordinate-descent procedure, detailed in \Cref{app:llm_cd_protocol}, starts from the best grid configuration at that batch.
Learning rates and weight decays move by factors \(\{\tfrac12,\tfrac1{\sqrt2},\sqrt2,2\}\); retention coefficients use the same factors on \(1-q\).
A candidate improving loss by more than \(0.002\) on seed \(42\) is repeated on seed \(43\), and also on seed \(44\) when the gain is below \(0.006\).
We accept it only if its median improvement exceeds \(0.002\), and stop after all seven coordinates pass without an accepted move.
\Cref{tab:cifar5m_retune} reports the selected configurations.
The comparison curves add seeds \(43\) and \(44\) to the grid's seed \(42\) and report three-seed medians.
At retuned configurations, the three losses span up to \(0.0064\), so a single-seed grid minimum can lie below the repeated-run result for the same configuration.
Standard deviations over the final evaluation window describe variation during one run; they are not confidence intervals across seeds.

\begin{table}[h]
  \centering
  \small
  \setlength{\tabcolsep}{4pt}
  \begin{tabular}{rrrrrrrrr}
    \toprule
    $B$ & $\eta_{\rm M}$ & $\eta_{\rm A}$ & $\mu$ & $\beta_1$ & $\beta_2$ & $\lambda_{\rm M}$ & $\lambda_{\rm A}$ & median loss \\
    \midrule
    $256$ & $0.006816$ & $0.0009804$ & $0.93$ & $0.6561$ & $0.999912$ & $0.00125$ & $0.00125$ & $0.2331$ \\
    $512$ & $0.01363$ & $0.002773$ & $0.93$ & $0.4305$ & $0.999823$ & $0.000625$ & $0.001768$ & $0.2275$ \\
    $1024$ & $0.02726$ & $0.001961$ & $0.93$ & $0.6561$ & $0.999912$ & $0.00125$ & $0.00125$ & $0.2299$ \\
    $2048$ & $0.03856$ & $0.007843$ & $0.93$ & $0.6561$ & $0.999912$ & $0.00125$ & $0.003536$ & $0.2318$ \\
    $4096$ & $0.02726$ & $0.0009804$ & $0.9505$ & $0.6561$ & $0.998586$ & $0.00125$ & $0.005$ & $0.2337$ \\
    \bottomrule
  \end{tabular}
  \caption{\textbf{CIFAR-5M configurations selected by coordinate descent.} Losses are medians over seeds $42$--$44$. At the reference batch $B_0=256$, the hyperparameter search ties the two weight decays, $\lambda_{\rm M}=\lambda_{\rm A}$. Displayed hyperparameters are rounded.}
  \label{tab:cifar5m_retune}
\end{table}

\subsection{Language-model sweep}
\label{app:llm_scaling_experiments}
\label{sec:llm_scaling_rule_sweep}
\label{app:s2-llm}
\label{app:muonw_batch_transfer}

\textbf{Protocol and reference configuration.}
We train a decoder-only Transformer on the FineWeb10B token stream.
The model has \(12\) layers, width \(768\), \(6\) attention heads, context length \(1024\), squared-ReLU feed-forward layers of width \(3072\), and a vocabulary of \(50{,}304\) tokens.
MuonW updates the two-dimensional matrices within each Transformer block using Nesterov momentum and twelve Newton--Schulz iterations.
Embeddings, the output projection, and one-dimensional parameters use AdamW.
The grid uses bfloat16 on eight GPUs and seed \(1\).
The metric is terminal token-averaged validation cross-entropy on a fixed set of \(10{,}485{,}760\) tokens.
\Cref{tab:muonw_transfer_anchor} gives the reference configuration, whose reference audit run has loss \(3.26547\).
Batch labels use binary multiples, with \(128\mathrm K=131{,}072\) reference tokens and targets \(256\mathrm K,512\mathrm K,1\mathrm M,2\mathrm M\), or \(262{,}144\), \(524{,}288\), \(1{,}048{,}576\), and \(2{,}097{,}152\) tokens.
The budget \(\tau=1{,}703{,}936{,}000\) tokens gives \(13{,}000\) reference steps and \(6{,}500\), \(3{,}250\), \(1{,}625\), and \(813\) target steps.
Learning rates remain at their peak for the first \(20\%\) of training and decrease linearly to zero over the remaining \(80\%\).
Both weight decays start at full strength without warmup; auxiliary decay remains \(0.001\) in the main grid.
All \(216\) rules use this schedule shape at each target, giving \(864\) grid runs.
An additional sweep at \(1\mathrm M\) and \(2\mathrm M\) tokens shows that scaling auxiliary weight decay offers little benefit.

\begin{table}[h]
\centering
\small
\setlength{\tabcolsep}{4pt}
\begin{tabular}{rrrrrrrrrr}
\toprule
$B$ & $\eta_{\rm M}$ & $\eta_{\rm Emb.}$ & $\eta_{\rm Out.}$ & $\eta_{\rm Scalar}$ & $\mu$ & $\beta_1$ & $\beta_2$ & $\lambda_{\rm M}$ & $\lambda_{\rm A}$ \\
\midrule
$128\mathrm K$ & $0.0176777$ & $0.494974$ & $0.00282842$ & $0.0106066$ & $0.95$ & $0.8$ & $0.993629$ & $0.025$ & $0.001$ \\
\bottomrule
\end{tabular}
\caption{\textbf{Language-model reference configuration.} Batch size $B$ counts tokens. The auxiliary learning rates $\eta_{\rm Emb.}$, $\eta_{\rm Out.}$, and $\eta_{\rm Scalar}$ apply to embeddings, the output projection, and scalar parameters, respectively. Learning rates are peak values, and the cooldown fraction is $0.8$. Displayed values are rounded.}
\label{tab:muonw_transfer_anchor}
\end{table}

\textbf{The best learning-rate scaling changes with the target batch.}
The best rules scale both learning-rate groups by the square root at \(256\mathrm K\) and \(512\mathrm K\), but keep them fixed and scale matrix weight decay linearly at \(1\mathrm M\) and \(2\mathrm M\).
All four winners keep Muon momentum and Adam \(\beta_1\) fixed and transport Adam \(\beta_2\).
\Cref{tab:muonw_transfer_batch_best} reports their losses and choices, and \Cref{fig:muonw_transfer_axis_marginals} shows the mean loss across rules with each coordinate choice.

\begin{table}[H]
\centering
\small
\setlength{\tabcolsep}{3.2pt}
\renewcommand{\arraystretch}{1.06}
\begin{tabular}{@{}r r l l l l l l@{}}
\toprule
& & \multicolumn{3}{c}{Parameter scaling} & \multicolumn{3}{c}{EMA scaling} \\
\cmidrule(lr){3-5}\cmidrule(lr){6-8}
Batch & Loss & Matrix LR & Matrix WD & Auxiliary LR & Muon $\mu$ & Adam $\beta_1$ & Adam $\beta_2$ \\
\midrule
256K & 3.26553 & Square root & Fixed & Square root & Fixed & Fixed & EMA \\
512K & 3.28004 & Square root & Square root & Square root & Fixed & Fixed & EMA \\
1M & 3.31183 & Fixed & Linear & Fixed & Fixed & Fixed & EMA \\
2M & 3.37977 & Fixed & Linear & Fixed & Fixed & Fixed & EMA \\
\bottomrule
\end{tabular}
\caption{Lowest validation loss observed independently at each target batch size. The choice at one batch size is not required to transfer to another batch size.}
\label{tab:muonw_transfer_batch_best}
\end{table}

\begin{figure}[t]
    \vspace{-1.5em}
\centering
\includegraphics[width=\textwidth]{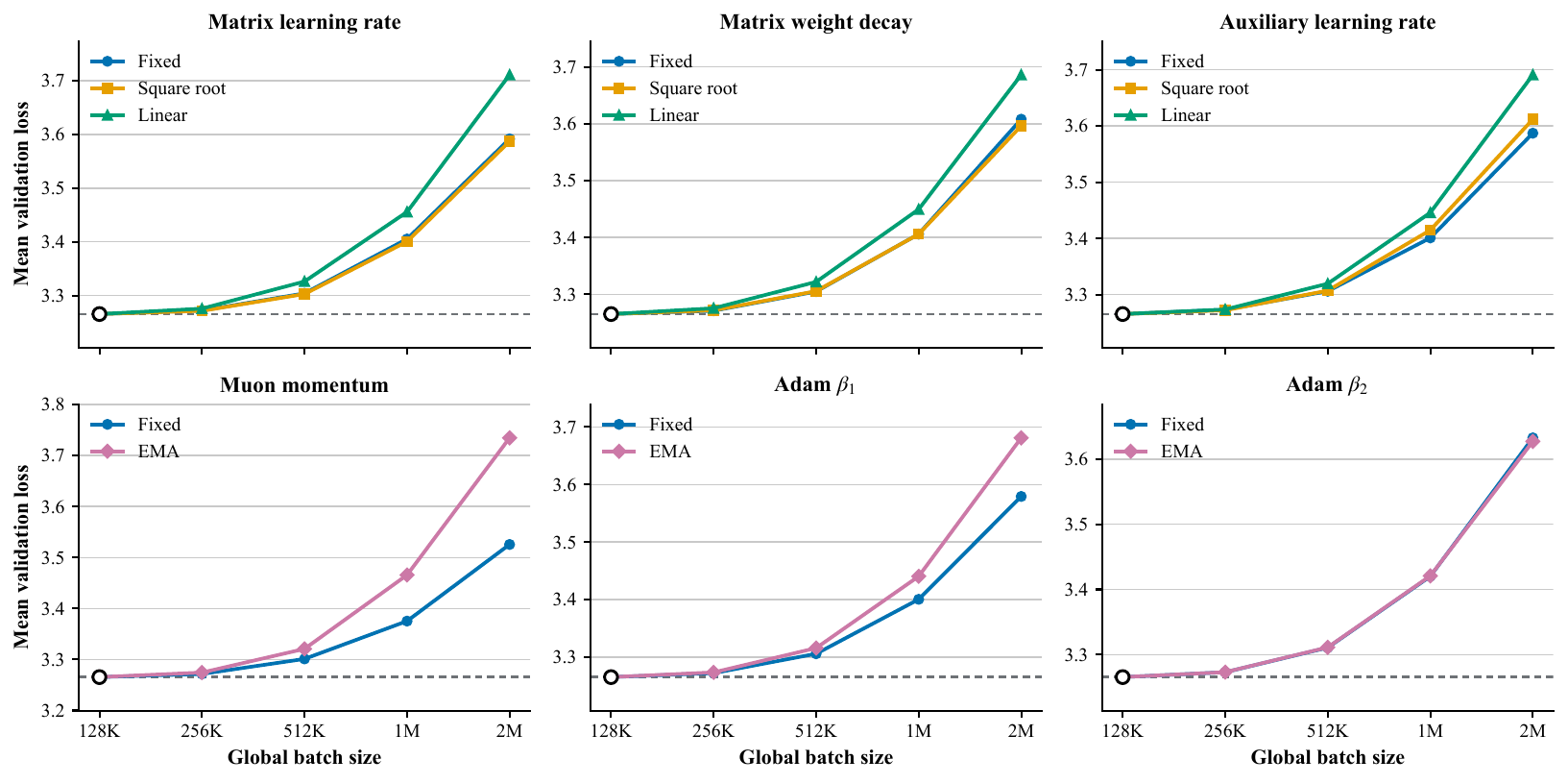}
\caption{\textbf{All \(216\) language-model scaling rules, summarized one coordinate at a time.} Each point is the arithmetic mean of the target-batch validation losses over all rules using the indicated choice on that coordinate, weighting all combinations of the other five choices equally.
Fixed, square root, and linear denote exponents \(0,\tfrac12,1\); E/EMA preserves decay over the complete token budget.
Lower is better; each panel uses its own focused vertical range.
Every rule shares the reference run at \(\batch_0=128\mathrm K\), drawn as the outlined marker and the dashed horizontal line.
The full grid is reported in \Cref{app:muonw_batch_transfer}.}
\label{fig:muonw_transfer_axis_marginals_draft}
\label{fig:muonw_transfer_axis_marginals}
\end{figure}

\textbf{Fixed learning rates and linear matrix weight decay give the smallest mean BSZ regret.}
The selected common rule keeps Muon momentum and Adam \(\beta_1\) fixed and transports only Adam \(\beta_2\).
Its mean BSZ regret is \(0.00151\) and its maximum BSZ regret is \(0.00416\).
It attains the grid minimum at \(1\mathrm M\) and \(2\mathrm M\); its gaps at \(256\mathrm K\) and \(512\mathrm K\) are \(0.00416\) and \(0.00188\).
\Cref{tab:muonw_transfer_resolved} gives the resulting hyperparameters; Emb. and Out. denote the embedding and output-projection learning rates.
\Cref{tab:muonw_transfer_coordinate} changes one choice at a time around this rule.
Transporting Muon momentum or Adam \(\beta_1\) increases loss substantially.
The rule with \(\beta_2\) also fixed ranks second by mean BSZ regret, \(0.00240\), and first by maximum BSZ regret, \(0.00399\).
Transporting \(\beta_2\) lowers mean loss by \(0.000895\), but raises loss by \(0.00017\) at \(256\mathrm K\).
It lowers loss by \(0.00022\), \(0.00089\), and \(0.00264\) at the other three targets.
\Cref{fig:muonw_transfer_regret_scatter} displays both criteria for all rules.

\begin{table}[H]
\centering
\scriptsize
\setlength{\tabcolsep}{2.8pt}
\renewcommand{\arraystretch}{1.06}
\begin{tabular}{@{}r r r r r r r r r r@{}}
\toprule
& & \multicolumn{3}{c}{Matrix hyperparameter configuration} & \multicolumn{5}{c}{Auxiliary hyperparameter configuration} \\
\cmidrule(lr){3-5}\cmidrule(lr){6-10}
Batch & Steps & LR & WD & $\mu$ & Emb. LR & Out. LR & Scalar LR & $\beta_1$ & $\beta_2$ \\
\midrule
128K & 13000 & 0.0176777 & 0.025 & 0.95 & 0.494974 & 0.00282842 & 0.0106066 & 0.8 & 0.993629 \\
256K & 6500 & 0.0176777 & 0.05 & 0.95 & 0.494974 & 0.00282842 & 0.0106066 & 0.8 & 0.987299 \\
512K & 3250 & 0.0176777 & 0.1 & 0.95 & 0.494974 & 0.00282842 & 0.0106066 & 0.8 & 0.97476 \\
1M & 1625 & 0.0176777 & 0.2 & 0.95 & 0.494974 & 0.00282842 & 0.0106066 & 0.8 & 0.950156 \\
2M & 813 & 0.0176777 & 0.4 & 0.95 & 0.494974 & 0.00282842 & 0.0106066 & 0.8 & 0.902854 \\
\bottomrule
\end{tabular}
\caption{Resolved numerical hyperparameters for the selected common rule. Values are rounded for readability. Auxiliary weight decay is 0.001 at every batch size, and the learning rate cooldown occupies the final 80 percent of training.}
\label{tab:muonw_transfer_resolved}
\end{table}

\begin{table}[H]
\centering
\small
\setlength{\tabcolsep}{5pt}
\renewcommand{\arraystretch}{1.05}
\begin{tabular}{@{}l r r r r r@{}}
\toprule
Only change from selected rule & 256K & 512K & 1M & 2M & \shortstack{Mean BSZ\\regret} \\
\midrule
\textbf{Selected common rule} & \textbf{3.26969} & \textbf{3.28192} & \textbf{3.31183} & \textbf{3.37977} & \textbf{0.00151} \\
\addlinespace[2pt]
Matrix LR to square root & 3.27140 & 3.29573 & 3.35026 & 3.46358 & 0.03595 \\
Matrix LR to linear & 3.28325 & 3.33974 & 3.44900 & 3.66492 & 0.12493 \\
Matrix WD to fixed & 3.27030 & 3.28906 & 3.32787 & 3.40266 & 0.01318 \\
Matrix WD to square root & 3.26862 & 3.28334 & 3.31700 & 3.38662 & 0.00460 \\
Auxiliary LR to square root & 3.26763 & 3.28119 & 3.31645 & 3.41001 & 0.00953 \\
Auxiliary LR to linear & 3.26985 & 3.28710 & 3.35634 & 3.45897 & 0.03377 \\
Muon momentum to EMA & 3.27375 & 3.30073 & 3.39503 & 3.58196 & 0.07858 \\
Adam $\beta_1$ to EMA & 3.27078 & 3.28823 & 3.33705 & 3.45730 & 0.02905 \\
Adam $\beta_2$ to fixed & 3.26952 & 3.28214 & 3.31272 & 3.38241 & 0.00240 \\
\bottomrule
\end{tabular}
\caption{One coordinate changes around the selected common rule. Every row changes exactly one entry and keeps the other five entries fixed. Mean BSZ regret is the equally weighted mean over the four target batch-size choices \(256\mathrm K,512\mathrm K,1\mathrm M,2\mathrm M\) of each row's loss gap to the best tested rule at the same batch size.}
\label{tab:muonw_transfer_coordinate}
\end{table}

\textbf{The selected matrix scaling agrees with a Power Lines prescription.}
It matches the fixed-learning-rate member of the Power Lines family in \eqref{eq:powerline_scaling_family}, preserving its nominal parameter-averaging timescale at fixed token budget~\citep{bergsma2025powerlines}.
The choice to transport Adam \(\beta_2\) does not follow from that timescale.

\begin{figure}[H]
    \vspace{-1em}
    \centering
    \includegraphics[width=0.72\textwidth]{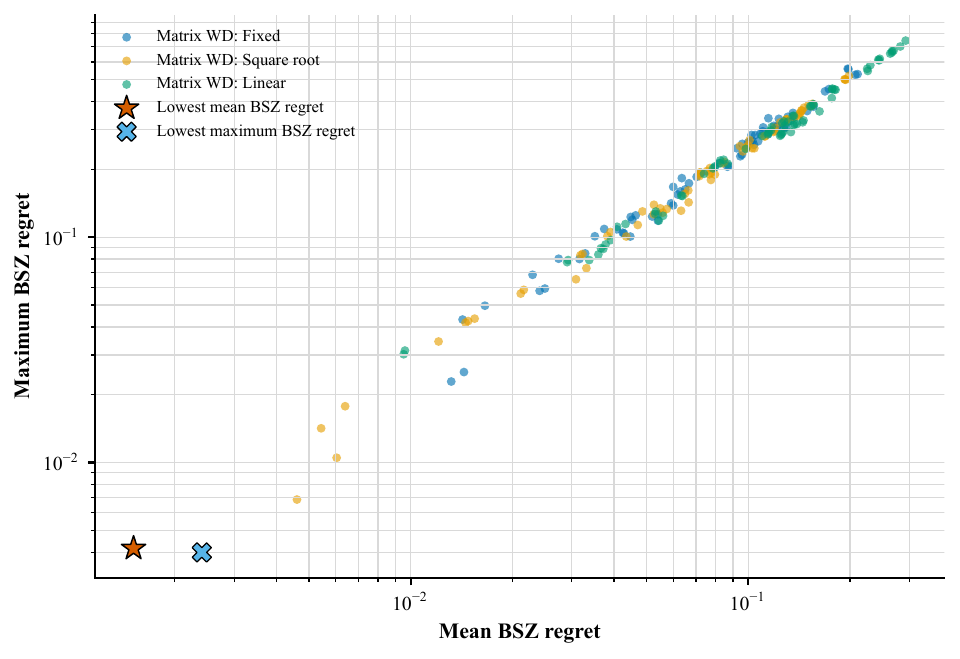}
    \caption{Mean and maximum BSZ regret for all \(216\) transfer rules.
    At each batch size, BSZ regret is the rule's validation loss minus
    the lowest validation loss among the tested rules at that same batch size.
    Each point represents one rule; the mean (horizontal axis, equal weights)
    and maximum (vertical axis) are taken over the four target batch-size
    choices \(256\mathrm K,512\mathrm K,1\mathrm M,2\mathrm M\) tokens.
    The star has the lowest mean BSZ regret, and the cross has the lowest
    maximum BSZ regret. Lower is better; both axes use logarithmic scales.}
    \label{fig:muonw_transfer_regret_scatter}
\end{figure}

\textbf{The retuning comparison combines the grid with additional searches and repeated runs.}
The language-model full-retuning curve in \Cref{fig:muonw_transfer_comparison} reports medians over \(2\)--\(4\) seeds, except at \(256\mathrm K\), where it uses the best grid loss as described below.
Other points are single runs.

At \(256\mathrm K\), we used the coordinate-descent procedure detailed in \Cref{app:llm_cd_protocol}, with an acceptance threshold of \(0.001\).
The initial matrix learning rate was \(0.0210224\), with embedding, output-projection, and scalar learning rates \((0.588627,0.00336358,0.0126134)\).
The remaining initial values were \(\mu=0.95\), \(\beta_1=0.8\), \(\beta_2=0.984992\), matrix weight decay \(0.05\), and auxiliary cooldown fraction \(0.8\).
No move passed the threshold.
Using the same seed, the search configuration's loss \(3.2678\) matched the \(3.2677\) of square-root learning-rate scaling with all other coordinates fixed.
In the original grid, that scaling rule gave \(3.2660\), within \(0.0005\) of the best grid loss \(3.2655\), which additionally transports \(\beta_2\).
The search found no improvement exceeding its threshold, and the plotted full-retuning point uses the best grid loss \(3.2655\).

\textbf{Scaling auxiliary weight decay offers little benefit.}
At \(1\mathrm M\) and \(2\mathrm M\) tokens, we extend each of the original \(216\) rules with square-root and linear scaling of auxiliary weight decay from its reference value \(0.001\), adding \(432\) single-seed runs per target.
For the original common rule, square-root scaling reduces validation loss by only \(0.0037\) and \(0.0027\) nats at these two targets, respectively.
Across all \(216\) rules, it instead raises loss by medians of \(0.001\) and \(0.005\) nats relative to fixed auxiliary decay.
Linear scaling raises loss by medians of \(0.012\) and \(0.033\) nats, respectively.
\Cref{tab:muonw_transfer_raw} reports the full measurements alongside the original rules.

\subsection{Comparison across settings}
\label{app:s2-rank-comparison}

\textbf{The complete grids favor different common rules.}
We select a common rule from all \(216\) language-model rules and all \(648\) CIFAR-5M rules, using each experiment's target batches.
The language-model sweep selects fixed learning rates and linear matrix weight decay, whereas CIFAR-5M selects square-root learning rates and fixed weight decays.
Both selected rules keep Muon momentum and Adam \(\beta_1\) fixed; the language-model rule additionally transports \(\beta_2\).

\textbf{The tested bound-minimizing Muon prescription ranks better in both settings.}
Both prescriptions fix matrix weight decay.
The bound-minimizing prescription uses square-root learning-rate scaling with fixed momentum; trajectory matching uses linear scaling with retention-preserving momentum.
Varying the unspecified AdamW choices gives \(12\) rules per family for language modeling and \(36\) for CIFAR-5M.
We rank each rule by its mean loss over \(\kappaBatch\in\{2,4,8,16\}\) within the complete grid, averaging tied ranks.
We then average ranks within each family; lower is better.
The bound-minimizing and trajectory-matching families rank \(52.1\) versus \(141.9\) among \(216\) language-model rules, and \(59.9\) versus \(347.7\) among \(648\) CIFAR-5M rules.

\textbf{The rankings depend on the setting and batch range.}
The main grids use one seed and one reference configuration per setting.
The repeated runs described above provide checks for selected configurations.
A family rank averages over its AdamW completions, whereas a coordinate panel averages losses over all compatible rules.
The dependence on setting and batch range motivates independent retuning of every optimizer at every batch size in the optimizer comparisons of \Cref{sec:llm_experiments}.

\subsection{Complete measurements}
\label{app:scaling_complete_measurements}

\Cref{tab:cifar5m_transfer_raw} reports all \(3{,}888\) CIFAR-5M measurements, with one row per rule and six target-batch losses.
Its entries are cross-entropies averaged over the final ten evaluations.
\Cref{tab:muonw_transfer_raw} combines the \(864\) original language-model terminal losses with the \(864\) additional auxiliary weight-decay measurements.
Each original rule is followed by its square-root and linear auxiliary-decay variants, which were evaluated only at \(1\mathrm M\) and \(2\mathrm M\).
All \(864\) original language-model runs reached their terminal step with finite loss.
The CIFAR-5M rows and language-model groups retain their original order by mean BSZ regret and highlight the selected common rule in each setting.
Language-model ranks and mean BSZ regrets use the four-batch original grid; the partial auxiliary-decay variants have no four-batch rank or BSZ regret.
\clearpage
\begingroup
\scriptsize
\setlength{\tabcolsep}{2.8pt}
\renewcommand{\arraystretch}{1.02}
\setlength{\LTpre}{0pt}
\setlength{\LTpost}{0pt}
\setlength{\LTleft}{\fill}
\setlength{\LTright}{\fill}
\setlength{\LTcapwidth}{\textwidth}
\begin{longtable}{@{}rcccccccrrrrrrr@{}}
\caption{Complete CIFAR-5M validation losses for all \(648\) scaling rules at six target batches, ordered by mean BSZ regret \(\overline\Delta\). This is the equally weighted mean over the six target batch-size choices of the loss gap to the best tested rule at the same batch size. Each loss is cross-entropy in nats averaged over the final ten evaluations of one run with seed \(42\); lower is better. Batch sizes count examples. Rule codes are F (fixed), S (square-root), L (linear), and E (retention preservation in processed examples), with \(q(B)=q(256)^{B/256}\). Bold marks the selected common rule. Ordering and BSZ regret use unrounded losses; displayed values have five decimal places.}\label{tab:cifar5m_transfer_raw}\\
\toprule
Rank & \multicolumn{4}{c}{Parameter scaling} & \multicolumn{3}{c}{EMA scaling} & \multicolumn{6}{c}{Validation loss} & \shortstack{Mean BSZ\\regret} \\
\cmidrule(lr){2-5}\cmidrule(lr){6-8}\cmidrule(lr){9-14}
& $\eta_{\rm mat}$ & $\lambda_{\rm mat}$ & $\eta_{\rm aux}$ & $\lambda_{\rm aux}$ & $\mu$ & $\beta_1$ & $\beta_2$ & 64 & 128 & 512 & 1024 & 2048 & 4096 & $\overline{\Delta}$ \\
\midrule
\endfirsthead
\multicolumn{15}{c}{\tablename\ \thetable\ continued}\\
\toprule
Rank & \multicolumn{4}{c}{Parameter scaling} & \multicolumn{3}{c}{EMA scaling} & \multicolumn{6}{c}{Validation loss} & \shortstack{Mean BSZ\\regret} \\
\cmidrule(lr){2-5}\cmidrule(lr){6-8}\cmidrule(lr){9-14}
& $\eta_{\rm mat}$ & $\lambda_{\rm mat}$ & $\eta_{\rm aux}$ & $\lambda_{\rm aux}$ & $\mu$ & $\beta_1$ & $\beta_2$ & 64 & 128 & 512 & 1024 & 2048 & 4096 & $\overline{\Delta}$ \\
\midrule
\endhead
\midrule
\multicolumn{15}{r}{Continued on the next page}\\
\endfoot
\bottomrule
\endlastfoot
\textbf{1} & \textbf{S} & \textbf{F} & \textbf{S} & \textbf{F} & \textbf{F} & \textbf{F} & \textbf{F} & \textbf{0.23370} & \textbf{0.23136} & \textbf{0.23373} & \textbf{0.22778} & \textbf{0.23290} & \textbf{0.23683} & \textbf{0.00492} \\
2 & S & F & F & S & F & F & E & 0.22970 & 0.23049 & 0.23440 & 0.23301 & 0.23772 & 0.23346 & 0.00533 \\
3 & L & F & S & S & F & E & E & 0.23435 & 0.23346 & 0.22739 & 0.22892 & 0.23043 & 0.24545 & 0.00553 \\
4 & S & F & S & L & F & F & F & 0.23374 & 0.23289 & 0.23135 & 0.23350 & 0.23206 & 0.23692 & 0.00561 \\
5 & S & F & S & F & F & F & E & 0.23201 & 0.23276 & 0.23128 & 0.23387 & 0.23190 & 0.23986 & 0.00581 \\
6 & L & F & F & L & F & E & E & 0.23920 & 0.23106 & 0.23005 & 0.23157 & 0.22957 & 0.24050 & 0.00586 \\
7 & S & S & S & F & F & F & F & 0.23225 & 0.23205 & 0.22913 & 0.23220 & 0.23285 & 0.24356 & 0.00587 \\
8 & L & F & S & L & F & E & F & 0.23617 & 0.23393 & 0.22979 & 0.22867 & 0.23112 & 0.24471 & 0.00626 \\
\end{longtable}
\endgroup

\clearpage
\begingroup
\scriptsize
\setlength{\tabcolsep}{2.8pt}
\renewcommand{\arraystretch}{1.02}
\setlength{\LTpre}{0pt}
\setlength{\LTpost}{0pt}
\setlength{\LTleft}{\fill}
\setlength{\LTright}{\fill}
\setlength{\LTcapwidth}{\textwidth}
\begin{longtable}{@{}rcccccccrrrrr@{}}
\caption{Terminal validation losses for the 216 original language-model rules and their auxiliary weight-decay extensions. Each group contains fixed (F), square-root (S), and linear (L) auxiliary decay; E denotes retention preservation over the complete run. Groups retain the original ranking by mean BSZ regret, the equally weighted mean over the four target batch-size choices of the loss gap to the original grid minimum at the same batch size. Dashes mark unmeasured losses and unavailable four-batch ranks and BSZ regrets. Bold marks the common rule selected from the original grid.}\label{tab:muonw_transfer_raw}\\
\toprule
Rank & \multicolumn{4}{c}{Parameter scaling} & \multicolumn{3}{c}{EMA scaling} & \multicolumn{4}{c}{Terminal validation loss} & \shortstack{Mean BSZ\\regret} \\
\cmidrule(lr){2-5}\cmidrule(lr){6-8}\cmidrule(lr){9-12}
& $\eta_{\rm mat}$ & $\lambda_{\rm mat}$ & $\eta_{\rm aux}$ & $\lambda_{\rm aux}$ & $\mu$ & $\beta_1$ & $\beta_2$ & 256K & 512K & 1M & 2M & $\overline{\Delta}$ \\
\midrule
\endfirsthead
\multicolumn{13}{c}{\tablename\ \thetable\ continued}\\
\toprule
Rank & \multicolumn{4}{c}{Parameter scaling} & \multicolumn{3}{c}{EMA scaling} & \multicolumn{4}{c}{Terminal validation loss} & \shortstack{Mean BSZ\\regret} \\
\cmidrule(lr){2-5}\cmidrule(lr){6-8}\cmidrule(lr){9-12}
& $\eta_{\rm mat}$ & $\lambda_{\rm mat}$ & $\eta_{\rm aux}$ & $\lambda_{\rm aux}$ & $\mu$ & $\beta_1$ & $\beta_2$ & 256K & 512K & 1M & 2M & $\overline{\Delta}$ \\
\midrule
\endhead
\midrule
\multicolumn{13}{r}{Continued on the next page}\\
\endfoot
\bottomrule
\endlastfoot
\textbf{1} & \textbf{F} & \textbf{L} & \textbf{F} & \textbf{F} & \textbf{F} & \textbf{F} & \textbf{E} & \textbf{3.26969} & \textbf{3.28192} & \textbf{3.31183} & \textbf{3.37977} & \textbf{0.00151} \\*
-- & F & L & F & S & F & F & E & -- & -- & 3.30815 & 3.37707 & -- \\*
-- & F & L & F & L & F & F & E & -- & -- & 3.30896 & 3.37751 & -- \\
\addlinespace[1pt]
2 & F & L & F & F & F & F & F & 3.26952 & 3.28214 & 3.31272 & 3.38241 & 0.00240 \\*
-- & F & L & F & S & F & F & F & -- & -- & 3.30956 & 3.38014 & -- \\*
-- & F & L & F & L & F & F & F & -- & -- & 3.30969 & 3.38111 & -- \\
\addlinespace[1pt]
3 & F & S & F & F & F & F & E & 3.26862 & 3.28334 & 3.31700 & 3.38662 & 0.00460 \\*
-- & F & S & F & S & F & F & E & -- & -- & 3.31425 & 3.38392 & -- \\
\end{longtable}
\endgroup

\clearpage
\endgroup

\clearpage
\begingroup
\raggedbottom
\section{Language-Model Optimizer Crossover Experiments}
\label{app:optimizer_crossover_experiments}
\label{app:llm_experiment_details}

We compare Adam, Lion, Muon, KL-SOAP, and Shampoo in the language-model
setting of \Cref{sec:llm_experiments} at batch sizes
\(B\in\{128\mathrm K,512\mathrm K,1\mathrm M,2\mathrm M\}\) tokens and a
training budget of approximately \(1.7\)B tokens.
Each optimizer is tuned independently at every batch size under both
decoupled weight decay and HyperBall norm control.
We rank the selected configurations by terminal validation loss, with lower
values indicating better performance.
The SOAP variant we use throughout is KL-SOAP~\citep{lin2026klshampoo}, which
the main text calls SOAP for brevity.
An optimizer crossover is a reversal of this ordering across batch sizes
after retuning.

\subsection{Implementations and common protocol}

We follow the notation of \Cref{app:s2-protocol}:
$\eta_{\rm M},\lambda_{\rm M}$ are the hidden-matrix learning rate and
weight decay, and $\eta_{\rm A},\lambda_{\rm A}$ are their auxiliary-Adam
counterparts.  Here the subscript $\mathrm M$ applies to whichever
hidden-matrix optimizer is being compared; in \Cref{app:s2-protocol} it
is MuonW.  The auxiliary learning rates are
$\eta_{\rm Emb.},\eta_{\rm Out.},\eta_{\rm Scalar}$ for embeddings,
the output projection, and scalar parameters, respectively.
The search coordinate $\eta_{\rm A}$ scales all three together.
As in \Cref{app:s2-protocol}, $\mu$ is Muon momentum and
$\beta_1,\beta_2$ are the auxiliary Adam retention coefficients.
We distinguish the matrix optimizer's coefficients as
$\beta_1^{\rm M},\beta_2^{\rm M}$ and write $\beta_{\rm sh}$ for
KL-SOAP's preconditioner retention coefficient.
The auxiliary learning-rate cooldown occupies the fraction $f_{\rm cd}$
of training.  We write terminal validation loss as $L_{\rm val}$.
Learning rates and weight decays in the tables are peak values;
an additional subscript $t$ denotes their values at update $t$.

All methods use the same model, token stream, validation set, precision,
distributed training code, and learning-rate schedule shape.  For a hidden
matrix parameter $W$, the WD family applies decoupled weight decay as
\[
    W_{t+1}=(1-\eta_{{\rm M},t}\lambda_{{\rm M},t})W_t-\eta_{{\rm M},t} U_t,
\]
where $U_t$ is the update produced by the named optimizer.  The HyperBall
family instead applies the corresponding norm-preserving matrix projection
and uses no matrix weight decay.  We do not tune the preconditioner frequency,
preconditioner initialization, or schedule family.

\subsection{Optimizer definitions}

For a hidden matrix $W\in\mathbb{R}^{m\times n}$ with gradient $G_t$, the
boxes define its raw direction $U_t$ and wrapper; scalar operations are
entrywise and distributed details are omitted.  Other parameter groups use
bias-corrected Adam with their tabulated learning rates and absolute decay
$\lambda_{\rm A}$.  HyperBall acts only on hidden matrices.

\begin{algorithm}[H]
\caption{WD and HyperBall matrix wrappers}
\label{alg:llm_matrix_wrappers}
\begin{tcolorbox}[optimizerblock]
\small
Let $(U_t,S_{t+1})=\mathcal{D}_{\mathsf O}(G_t,S_t)$ and
$\epsilon_H=10^{-10}$.  The two matrix families differ only in the final map:
\begin{align*}
\textsc{WD:}\quad
  W_{t+1}&=(1-\eta_{{\rm M},t}\lambda_{{\rm M},t})W_t-\eta_{{\rm M},t}U_t,\\
\textsc{HyperBall:}\quad
  r_t&=\lVert W_t\rVert_F,\qquad
  Z_t=W_t-\eta_{{\rm M},t} r_t\frac{U_t}{\max(\lVert U_t\rVert_F,\epsilon_H)},\\
  W_{t+1}&=r_t\frac{Z_t}{\max(\lVert Z_t\rVert_F,\epsilon_H)}.
\end{align*}
\end{tcolorbox}
\end{algorithm}

\begin{algorithm}[H]
\caption{Adam and Lion raw directions}
\label{alg:llm_adam_lion}
\begin{tcolorbox}[optimizerblock]
\small
\textbf{Adam} ($M_0=V_0=0$):
\begin{align*}
  M_t&=\beta_{1}^{\rm M}M_{t-1}+(1-\beta_{1}^{\rm M})G_t,\qquad
  V_t=\beta_{2}^{\rm M}V_{t-1}+(1-\beta_{2}^{\rm M})G_t^{\odot2},\\
  U_t&=\frac{M_t/(1-(\beta_{1}^{\rm M})^t)}
            {\sqrt{V_t/(1-(\beta_{2}^{\rm M})^t)}+\epsilon}.
\end{align*}
\textbf{Lion} ($M_0=0$):
\begin{align*}
  U_t&=\operatorname{sign}\!\left(\beta_{1}^{\rm M}M_{t-1}+(1-\beta_{1}^{\rm M})G_t\right),
  \qquad M_t=\beta_{2}^{\rm M}M_{t-1}+(1-\beta_{2}^{\rm M})G_t.
\end{align*}
\end{tcolorbox}
\end{algorithm}

\begin{algorithm}[H]
\caption{Muon raw direction with the Newton--Schulz map}
\label{alg:llm_muon}
\begin{tcolorbox}[optimizerblock]
\small
\begin{enumerate}[label=\arabic*.,leftmargin=1.6em,itemsep=1pt,topsep=3pt]
\item Set $\omega=\mu$ and $M_0=0$. Form $M_t=\mu M_{t-1}+(1-\mu)G_t$ and
      $X=\omega M_t+(1-\omega)G_t$.
\item Cast $X$ to bfloat16, transpose if $m>n$, and normalize
      $X\leftarrow X/(\lVert X\rVert_F+10^{-7})$.
\item For 12 iterations with $(a,b,c)=(2,-1.5,0.5)$, set
      $A=XX^\top$, $B=bA+cA^2$, and $X\leftarrow aX+BX$.
\item Undo the transpose and return
      $U_t=\sqrt{\max(1,m/n)}\,X$.
\end{enumerate}
\end{tcolorbox}
\end{algorithm}

MuonW and MuonH share this Nesterov--Newton--Schulz direction
\citep{jordan2024muon}.

\begin{algorithm}[H]
\caption{KL-SOAP raw direction}
\label{alg:llm_klsoap}
\begin{tcolorbox}[optimizerblock]
\small
\textbf{Initialize from $G_1$ and set $U_1=0$:}
$M_1=V_1=0$, $C_L=G_1G_1^\top/n$, $C_R=G_1^\top G_1/m$, and
$(Q_s,\rho_s)=(\operatorname{eigvec}_{\rm rev}(C_s),
0.1^{-1/2}\mathbf 1)$ for $s\in\{L,R\}$.

\textbf{Direction for $t\geq2$:}
\begin{align*}
\widehat G_t&=Q_L^\top G_tQ_R,\\
M_t&=\beta_{1}^{\rm M}M_{t-1}+(1-\beta_{1}^{\rm M})\widehat G_t,\\
V_t&=\beta_{2}^{\rm M}V_{t-1}+(1-\beta_{2}^{\rm M})\widehat G_t^{\odot2},\\
U_t&=Q_L\frac{M_t}{\sqrt{V_t}+\epsilon}Q_R^\top
      \qquad\text{(no bias correction)}.
\end{align*}

\textbf{Refresh.}
After forming $U_t$, update $C_L,C_R$ from the right- and left-whitened
gradient covariances and $\rho_L,\rho_R$ from
$\operatorname{diag}(\widehat G_t)$; then set
$Q_s\leftarrow\operatorname{qr}(C_sQ_s)$ and re-express $M_t$ in the new basis
for $s\in\{L,R\}$.  This refresh is performed every step.
\end{tcolorbox}
\end{algorithm}

\begin{algorithm}[H]
\caption{Distributed Shampoo raw direction}
\label{alg:llm_shampoo}
\begin{tcolorbox}[optimizerblock]
\small
For each distributed matrix block, let $A_t$ be the library's
Adam-preconditioned graft direction.  The Shampoo direction is
\begin{align*}
 M_t&=\beta_{1}^{\rm M}M_{t-1}+(1-\beta_{1}^{\rm M})G_t,\\
 C_L&\leftarrow\beta_{2}^{\rm M}C_L+(1-\beta_{2}^{\rm M})G_tG_t^\top/n,\qquad
 C_R\leftarrow\beta_{2}^{\rm M}C_R+(1-\beta_{2}^{\rm M})G_t^\top G_t/m,\\
 S_t&=C_L^{-1/4}M_tC_R^{-1/4},\qquad
 U_t=\frac{\lVert A_t\rVert_F}{\lVert S_t\rVert_F}\,S_t.
\end{align*}
\end{tcolorbox}
\end{algorithm}

Here $C_s^{-1/4}$ is the eigen-based pseudoinverse fourth root.  The fixed
Shampoo implementation hyperparameters are
\[
\begin{aligned}
\texttt{precondition\_frequency}&=1, &
\epsilon_{\rm factor}&=0,\\
\texttt{start\_preconditioning\_step}&=-1, &
\epsilon_{\rm graft}&=10^{-15},\\
\texttt{max\_preconditioner\_dim}&=8192.
\end{aligned}
\]
\subsection{Coordinate-descent protocol}
\label{app:llm_cd_protocol}

The immediate-neighbor audit below applies to the WD campaign.  Let
$B_0=512\mathrm K$ and $\kappaBatch=B/B_0$, as in \Cref{app:s2-protocol}.
For each learning-rate coordinate $\eta$ with reference value $\eta_0$,
the initial transfer uses $\eta(B)=\eta_0\sqrt{\kappaBatch}$.
First-moment coefficients are held constant
under this initial transfer, while second-moment and preconditioner
coefficients, denoted generically by $q$ with reference value $q_0=q(B_0)$,
use the time-scale anchor $q(B)=q_0^{\kappaBatch}$.  After this transfer, a one-coordinate
sweep of an ordinary positive multiplicative coordinate $x$, such as a
learning rate, evaluates
\[
  x'\in x\{1/2,1/\sqrt{2},1,\sqrt{2},2\}.
\]
The factors $1/\sqrt{2}$ and $\sqrt{2}$ are the adjacent grid points, while
$1/2$ and $2$ provide a wider search around the same center.  For any
momentum or EMA coefficient $q$ (including $\mu$ and
$\beta_{\rm sh}$), the same factors act on its refresh rate rather than on
$q$ directly:
\[
  1-q'=m(1-q),
  \qquad m\in\{1/2,1/\sqrt{2},1,\sqrt{2},2\}.
\]
The cooldown coordinate instead uses the absolute grid
\[
  f_{\rm cd}\in\{0,0.2,0.4,0.6,0.8\},
\]
and the matrix weight-decay grid is
\begin{multline*}
  \lambda_{\rm M}\in
  \{0,0.01,0.025,0.05,0.1,0.15,0.2,0.3,0.4,0.5,0.6,0.8,1.0,1.2,\\
  1.4,1.6,1.8,2.0,2.4,2.8,3.2,3.6,4.0,4.4,4.8,5.2,5.6,6.0,6.4,6.8,7.2,8.0\}.
\end{multline*}
The grid is deliberately denser near zero.  Weight decay is searched in these
absolute units rather than multiplicatively.

At each iteration, all coordinates except one are fixed at the current
hyperparameter configuration.  A candidate replaces the current value only if its median validation
loss improves by more than $\Delta=0.002$. Gains in
$[0.002,0.006)$ require three matching repeats; larger gains require two, and
every selected center has at least two repeats.  After the last accepted move,
we evaluate the immediate lower and upper grid neighbors of every active
coordinate.  We call an entry \emph{locally closed} when none of these
neighbors exceeds the acceptance threshold under the multi-seed repeat
policy.

\begin{table}[H]
\centering
\small
\setlength{\tabcolsep}{5pt}
\renewcommand{\arraystretch}{1.08}
\begin{tabularx}{\textwidth}{@{}l>{\raggedright\arraybackslash}Xc@{}}
\toprule
Method & Tuned coordinates & Count \\
\midrule
MuonW & $\eta_{\rm M},\mu,\lambda_{\rm M},\eta_{\rm A},\beta_1,\beta_2,f_{\rm cd}$ & 7 \\
\addlinespace[2pt]
KL-SOAP-W & $\eta_{\rm M},\beta_{1}^{\rm M},\beta_{2}^{\rm M},\beta_{\rm sh},\lambda_{\rm M},\eta_{\rm A},\beta_1,\beta_2,f_{\rm cd}$ & 9 \\
\addlinespace[2pt]
ShampooW & $\eta_{\rm M},\lambda_{\rm M},\eta_{\rm A},\beta_1,\beta_2,f_{\rm cd}$ & 6 \\
\addlinespace[2pt]
AdamW & $\eta_{\rm M},\beta_{1}^{\rm M},\beta_{2}^{\rm M},\lambda_{\rm M},\eta_{\rm A},\beta_1,\beta_2,f_{\rm cd}$ & 8 \\
\addlinespace[2pt]
LionW & $\eta_{\rm M},\beta_{1}^{\rm M},\beta_{2}^{\rm M},\lambda_{\rm M},\eta_{\rm A},\beta_1,\beta_2,f_{\rm cd}$ & 8 \\
\bottomrule
\end{tabularx}
\caption{Coordinates tuned in the WD campaign.}
\label{tab:llm_cd_coordinates}
\end{table}

The complete one-coordinate audit is collected in
\Cref{app:llm_cd_sweep_tables}.

\textbf{HyperBall search schedule.}
HyperBall uses the same coordinate-descent move criterion, with
optimizer-specific learning rates, momenta, and auxiliary cooldown swept  around each center.

\subsection{Final weight-decay hyperparameters}

\textbf{The lowest-loss optimizer changes from KL-SOAP to Shampoo as batch size increases.}
Under decoupled weight decay, KL-SOAP-W attains the lowest terminal validation
loss at \(128\mathrm K\), \(3.2536\), compared with \(3.2650\) for ShampooW.
At \(2\mathrm M\), ShampooW attains the lowest loss, \(3.3622\), while
KL-SOAP-W reaches \(3.3833\).
This reversal persists despite independent hyperparameter tuning at each
batch size (\Cref{fig:track3_batch_transfer,tab:llm_matrix_hparams}).

\Cref{tab:llm_matrix_hparams,tab:llm_aux_hparams} together specify the
selected configurations in absolute hyperparameter units.
The losses are terminal validation losses for the selected configurations,
and boldface marks the numerical minimum at each batch size.
Shampoo's matrix moment coefficients are fixed rather than tuned;
blank entries denote coefficients that are not part of an optimizer.
Hyperparameters are shown to three significant figures and losses to four
decimal places.

\begin{table}[H]
\centering
\footnotesize
\setlength{\tabcolsep}{2.0pt}
\renewcommand{\arraystretch}{1.10}
\begin{tabular}{@{}llrrrrrrrr@{}}
\toprule
& & & & \multicolumn{4}{c}{Momentum / preconditioner} & & \\
\cmidrule(lr){5-8}
$B$ & Method & Steps & $\eta_{\rm M}$ & $\mu$ & $\beta_{1}^{\rm M}$ & $\beta_{2}^{\rm M}$ & $\beta_{\rm sh}$ & $\lambda_{\rm M}$ & $L_{\rm val}$ \\
\midrule
128K & MuonW & 13000 & 0.0177 & 0.95 & -- & -- & -- & 0.025 & 3.2656 \\
128K & KL-SOAP-W & 13000 & 0.0045 & -- & 0.965 & 0.974 & 0.948 & 0.1 & \textbf{3.2536} \\
128K & ShampooW & 13000 & 0.005 & -- & 0.9 & 0.9 & -- & 0.05 & 3.2650 \\
128K & AdamW & 13000 & 0.00075 & -- & 0.94 & 0.949 & -- & 0.3 & 3.3151 \\
128K & LionW & 13000 & $5\!\times\!10^{-5}$ & -- & 0.9 & 0.99 & -- & 8.6 & 3.3176 \\
\addlinespace[3pt]
512K & MuonW & 3250 & 0.025 & 0.95 & -- & -- & -- & 0.05 & 3.2774 \\
512K & KL-SOAP-W & 3250 & 0.0127 & -- & 0.9 & 0.859 & 0.859 & 0.15 & \textbf{3.2761} \\
512K & ShampooW & 3250 & 0.00707 & -- & 0.9 & 0.9 & -- & 0.15 & 3.2809 \\
512K & AdamW & 3250 & 0.0015 & -- & 0.9 & 0.859 & -- & 0.6 & 3.3392 \\
512K & LionW & 3250 & 0.0001 & -- & 0.95 & 0.96 & -- & 0.05 & 3.3907 \\
\addlinespace[3pt]
1M & MuonW & 1625 & 0.025 & 0.95 & -- & -- & -- & 0.1 & 3.3108 \\
1M & KL-SOAP-W & 1625 & 0.00636 & -- & 0.929 & 0.866 & 0.866 & 0.6 & \textbf{3.2974} \\
1M & ShampooW & 1625 & 0.00707 & -- & 0.9 & 0.9 & -- & 0.3 & 3.3090 \\
1M & AdamW & 1625 & 0.0015 & -- & 0.859 & 0.931 & -- & 0.8 & 3.4028 \\
1M & LionW & 1625 & 0.000141 & -- & 0.95 & 0.944 & -- & 0.15 & 3.4452 \\
\addlinespace[3pt]
2M & MuonW & 813 & 0.025 & 0.95 & -- & -- & -- & 0.1 & 3.3659 \\
2M & KL-SOAP-W & 813 & 0.036 & -- & 0.929 & 0.828 & 0.878 & 0.4 & 3.3833 \\
2M & ShampooW & 813 & 0.01 & -- & 0.9 & 0.9 & -- & 0.5 & \textbf{3.3622} \\
2M & AdamW & 813 & 0.00212 & -- & 0.9 & 0.869 & -- & 1.8 & 3.4781 \\
2M & LionW & 813 & 0.0002 & -- & 0.9 & 0.889 & -- & 0.3 & 3.5562 \\
\bottomrule
\end{tabular}
\caption{Final WD hidden-matrix hyperparameters and terminal validation loss.}
\label{tab:llm_matrix_hparams}
\end{table}

\begin{table}[H]
\centering
\footnotesize
\setlength{\tabcolsep}{2.5pt}
\renewcommand{\arraystretch}{1.10}
\begin{tabular}{@{}llrrrrrrr@{}}
\toprule
& & \multicolumn{3}{c}{Auxiliary learning rates} & \multicolumn{2}{c}{Adam coefficients} & & \\
\cmidrule(lr){3-5}\cmidrule(lr){6-7}
$B$ & Method & $\eta_{\rm Emb.}$ & $\eta_{\rm Out.}$ & $\eta_{\rm Scalar}$ & $\beta_1$ & $\beta_2$ & $f_{\rm cd}$ & $\lambda_{\rm A}$ \\
\midrule
128K & MuonW & 0.495 & 0.00283 & 0.0106 & 0.8 & 0.994 & 0.8 & 0.001 \\
128K & KL-SOAP-W & 0.15 & 0.00156 & 0.005 & 0.434 & 0.995 & 0.4 & 0 \\
128K & ShampooW & 1.2 & 0.0125 & 0.04 & 0.52 & 0.995 & 0.6 & 0 \\
128K & AdamW & 0.849 & 0.00884 & 0.0283 & 0.717 & 0.995 & 0.4 & 0 \\
128K & LionW & 0.287 & 0.00299 & 0.00958 & 0.859 & 0.995 & 0.6 & 0 \\
\addlinespace[3pt]
512K & MuonW & 0.7 & 0.004 & 0.015 & 0.8 & 0.965 & 0.7 & 0.001 \\
512K & KL-SOAP-W & 0.849 & 0.00884 & 0.0283 & 0.8 & 0.975 & 0.4 & 0 \\
512K & ShampooW & 0.849 & 0.00884 & 0.0283 & 0.8 & 0.982 & 0.6 & 0 \\
512K & AdamW & 0.849 & 0.00884 & 0.0283 & 0.859 & 0.975 & 0.4 & 0 \\
512K & LionW & 0.212 & 0.00221 & 0.00707 & 0.717 & 0.95 & 0.2 & 0 \\
\addlinespace[3pt]
1M & MuonW & 0.99 & 0.00566 & 0.0212 & 0.8 & 0.931 & 0.8 & 0.001 \\
1M & KL-SOAP-W & 0.6 & 0.00625 & 0.02 & 0.8 & 0.951 & 0.6 & 0 \\
1M & ShampooW & 0.849 & 0.00884 & 0.0283 & 0.8 & 0.971 & 0.6 & 0 \\
1M & AdamW & 0.6 & 0.00625 & 0.02 & 0.859 & 0.902 & 0.4 & 0 \\
1M & LionW & 0.3 & 0.00313 & 0.01 & 0.6 & 0.902 & 0.2 & 0 \\
\addlinespace[3pt]
2M & MuonW & 0.7 & 0.004 & 0.015 & 0.859 & 0.907 & 0.4 & 0.001 \\
2M & KL-SOAP-W & 0.6 & 0.00625 & 0.02 & 0.8 & 0.934 & 0.4 & 0 \\
2M & ShampooW & 0.6 & 0.00625 & 0.02 & 0.8 & 0.944 & 0.6 & 0 \\
2M & AdamW & 3.39 & 0.0354 & 0.113 & 0.929 & 0.967 & 0.6 & 0 \\
2M & LionW & 0.3 & 0.00313 & 0.01 & 0.8 & 0.954 & 0.2 & 0 \\
\bottomrule
\end{tabular}
\caption{Final WD auxiliary-Adam hyperparameters.}
\label{tab:llm_aux_hparams}
\end{table}

\subsection{Optimizer crossover under HyperBall}
\label{app:llm_hyperball_crossover}

\textbf{The change in winner also occurs under HyperBall.}
With the norm-preserving map of \Cref{alg:llm_matrix_wrappers} and independent
retuning, KL-SOAP-H has the lowest loss at \(128\mathrm K\), \(3.2585\),
whereas ShampooH has the lowest loss at \(2\mathrm M\), \(3.3390\).
At \(1\mathrm M\), the two are nearly tied, with reported losses of
\(3.2962\) and \(3.2965\), respectively.
Thus the same endpoint reversal appears under both weight decay and
HyperBall (\Cref{fig:track3_batch_transfer_hyperball,tab:llm_hyperball_matrix_hparams}).

\textbf{Other optimizer pairs also reverse their ordering.}
MuonH has lower loss than ShampooH at \(128\mathrm K\), but ShampooH has
lower loss at all three larger batches.
Their difference at \(512\mathrm K\) is only about \(0.0007\).
LionH has lower loss than AdamH at \(128\mathrm K\), while AdamH has lower
loss at every larger batch.
Both remain above MuonH, KL-SOAP-H, and ShampooH in validation loss at every
tested batch size.

\subsection{Final HyperBall hyperparameters}

\Cref{tab:llm_hyperball_matrix_hparams,tab:llm_hyperball_aux_hparams}
specify the selected HyperBall configurations in absolute hyperparameter
units.
The loss column gives the terminal validation losses plotted in \Cref{fig:track3_batch_transfer_hyperball}.
These configurations use no matrix or auxiliary weight decay.
As in the weight-decay tables, hyperparameters are shown to three significant
figures and losses to four decimal places, with boldface marking the
numerical minimum at each batch size.

\begin{table}[H]
\centering
\footnotesize
\setlength{\tabcolsep}{3.5pt}
\renewcommand{\arraystretch}{1.10}
\begin{tabular}{@{}llrrrrrrr@{}}
\toprule
& & & & \multicolumn{4}{c}{Momentum / preconditioner} & \\
\cmidrule(lr){5-8}
$B$ & Method & Steps & $\eta_{\rm M}$ & $\mu$ & $\beta_{1}^{\rm M}$ & $\beta_{2}^{\rm M}$ & $\beta_{\rm sh}$ & $L_{\rm val}$ \\
\midrule
128K & MuonH & 13000 & 0.009 & 0.95 & -- & -- & -- & 3.2629 \\
128K & KL-SOAP-H & 13000 & 0.009 & -- & 0.929 & 0.974 & 0.948 & \textbf{3.2585} \\
128K & ShampooH & 13000 & 0.009 & -- & 0.922 & 0.853 & -- & 3.2660 \\
128K & AdamH & 13000 & 0.009 & -- & 0.94 & 0.949 & -- & 3.3013 \\
128K & LionH & 13000 & 0.00636 & -- & 0.929 & 0.987 & -- & 3.2948 \\
\addlinespace[3pt]
512K & MuonH & 3250 & 0.018 & 0.95 & -- & -- & -- & 3.2772 \\
512K & KL-SOAP-H & 3250 & 0.018 & -- & 0.929 & 0.859 & 0.9 & \textbf{3.2707} \\
512K & ShampooH & 3250 & 0.018 & -- & 0.929 & 0.9 & -- & 3.2765 \\
512K & AdamH & 3250 & 0.018 & -- & 0.929 & 0.9 & -- & 3.3325 \\
512K & LionH & 3250 & 0.018 & -- & 0.929 & 0.95 & -- & 3.3360 \\
\addlinespace[3pt]
1M & MuonH & 1625 & 0.0255 & 0.95 & -- & -- & -- & 3.3074 \\
1M & KL-SOAP-H & 1625 & 0.0255 & -- & 0.929 & 0.866 & 0.866 & \textbf{3.2962} \\
1M & ShampooH & 1625 & 0.0255 & -- & 0.929 & 0.848 & -- & 3.2965 \\
1M & AdamH & 1625 & 0.018 & -- & 0.965 & 0.902 & -- & 3.3716 \\
1M & LionH & 1625 & 0.036 & -- & 0.9 & 0.902 & -- & 3.4053 \\
\addlinespace[3pt]
2M & MuonH & 813 & 0.036 & 0.95 & -- & -- & -- & 3.3715 \\
2M & KL-SOAP-H & 813 & 0.036 & -- & 0.929 & 0.656 & 0.878 & 3.3505 \\
2M & ShampooH & 813 & 0.036 & -- & 0.929 & 0.725 & -- & \textbf{3.3390} \\
2M & AdamH & 813 & 0.0255 & -- & 0.965 & 0.907 & -- & 3.4861 \\
2M & LionH & 813 & 0.036 & -- & 0.9 & 0.907 & -- & 3.5540 \\
\bottomrule
\end{tabular}
\caption{Final HyperBall hidden-matrix hyperparameters and terminal validation loss.}
\label{tab:llm_hyperball_matrix_hparams}
\end{table}

\begin{table}[H]
\centering
\footnotesize
\setlength{\tabcolsep}{4.5pt}
\renewcommand{\arraystretch}{1.10}
\begin{tabular}{@{}llrrrrrr@{}}
\toprule
& & \multicolumn{3}{c}{Auxiliary learning rates} & \multicolumn{2}{c}{Adam coefficients} & \\
\cmidrule(lr){3-5}\cmidrule(lr){6-7}
$B$ & Method & $\eta_{\rm Emb.}$ & $\eta_{\rm Out.}$ & $\eta_{\rm Scalar}$ & $\beta_1$ & $\beta_2$ & $f_{\rm cd}$ \\
\midrule
128K & MuonH & 0.3 & 0.00313 & 0.01 & 0.8 & 0.994 & 0.4 \\
128K & KL-SOAP-H & 0.3 & 0.00313 & 0.01 & 0.434 & 0.994 & 0.4 \\
128K & ShampooH & 1.2 & 0.0125 & 0.04 & 0.52 & 0.995 & 0.4 \\
128K & AdamH & 0.424 & 0.00442 & 0.0141 & 0.717 & 0.995 & 0.4 \\
128K & LionH & 0.3 & 0.00313 & 0.01 & 0.717 & 0.994 & 0.4 \\
\addlinespace[3pt]
512K & MuonH & 0.6 & 0.00625 & 0.02 & 0.859 & 0.975 & 0.4 \\
512K & KL-SOAP-H & 0.6 & 0.00625 & 0.02 & 0.8 & 0.975 & 0.8 \\
512K & ShampooH & 0.849 & 0.00884 & 0.0283 & 0.8 & 0.982 & 0.4 \\
512K & AdamH & 0.6 & 0.00625 & 0.02 & 0.8 & 0.975 & 0.4 \\
512K & LionH & 0.424 & 0.00442 & 0.0141 & 0.8 & 0.983 & 0.4 \\
\addlinespace[3pt]
1M & MuonH & 0.6 & 0.00625 & 0.02 & 0.9 & 0.951 & 0.4 \\
1M & KL-SOAP-H & 0.6 & 0.00625 & 0.02 & 0.8 & 0.951 & 0.6 \\
1M & ShampooH & 0.849 & 0.00884 & 0.0283 & 0.8 & 0.971 & 0.4 \\
1M & AdamH & 0.3 & 0.00313 & 0.01 & 0.8 & 0.959 & 0.25 \\
1M & LionH & 0.6 & 0.00625 & 0.02 & 0.9 & 0.976 & 0.4 \\
\addlinespace[3pt]
2M & MuonH & 0.6 & 0.00625 & 0.02 & 0.9 & 0.934 & 0.4 \\
2M & KL-SOAP-H & 0.6 & 0.00625 & 0.02 & 0.8 & 0.907 & 0.6 \\
2M & ShampooH & 0.6 & 0.00625 & 0.02 & 0.8 & 0.944 & 0.4 \\
2M & AdamH & 0.424 & 0.00442 & 0.0141 & 0.8 & 0.921 & 0.25 \\
2M & LionH & 0.6 & 0.00625 & 0.02 & 0.9 & 0.954 & 0.2 \\
\bottomrule
\end{tabular}
\caption{Final HyperBall auxiliary-Adam hyperparameters.}
\label{tab:llm_hyperball_aux_hparams}
\end{table}

\clearpage
\subsection{Complete One-Coordinate Sweep Results}
\label{app:llm_cd_sweep_tables}

The tables below report 643 completed one-coordinate trials across 40
optimizer--batch settings, complementing the 492-run evaluation of the final
configurations.  Each trial varies one hyperparameter while keeping all other
active hyperparameters at their accepted final values, up to numerical precision
in the stored values.  We use the most recent sweep with this configuration and,
for repeated trials, the latest completed result; neither choice depends on
validation loss.  We include only trials with a recorded validation result at the
configured final step.  The first row of each optimizer--batch block gives the accepted
configuration in absolute hyperparameter units.  Each subsequent row changes
only the displayed coordinate; a dash means unchanged from the first row.

The three auxiliary learning rates are tied during search, so the tables show
the absolute embedding rate $\eta_{\rm Emb.}$ while the output-projection and
scalar rates, $\eta_{\rm Out.}$ and $\eta_{\rm Scalar}$, change by the same
factor.  The reported quantity is terminal validation loss, not the
best loss observed earlier in training.

These tables are a local sensitivity record rather than a new configuration-selection
rule.  In particular, a numerically lower single-seed arm is not accepted
unless it also satisfies the multi-seed threshold policy in
\Cref{app:llm_cd_protocol}.

\Needspace{10\baselineskip}

\subsubsection{HyperBall sweeps}

\medskip

\textbf{MuonH}

\smallskip

\begingroup
\small
\renewcommand{\arraystretch}{1.06}
\setlength{\tabcolsep}{3pt}
\setlength{\LTpre}{0pt}
\setlength{\LTpost}{0pt}
\setlength{\LTleft}{\fill}
\setlength{\LTright}{\fill}
\setlength{\LTcapwidth}{\textwidth}
\begin{longtable}{rrrrrrr}
\caption{Hyperparameter ablation for MuonH at batch size $B=128K$.}
\label{tab:llm_cd_muonh_128k} \\
\toprule
$\eta_{\rm M}$ & $\mu$ & $\eta_{\rm Emb.}$ & $\beta_1$ & $\beta_2$ & $f_{\rm cd}$ & $L_{\rm val}$ \\
\midrule
\endfirsthead
\multicolumn{7}{l}{\itshape MuonH, $B=128K$} \\
\toprule
$\eta_{\rm M}$ & $\mu$ & $\eta_{\rm Emb.}$ & $\beta_1$ & $\beta_2$ & $f_{\rm cd}$ & $L_{\rm val}$ \\
\midrule
\endhead
\midrule
\endfoot
\bottomrule
\endlastfoot
0.009 & 0.95 & 0.3 & 0.8 & 0.99363 & 0.4 & 3.26346 \\
\midrule
0.0045 & -- & -- & -- & -- & -- & 3.29906 \\
0.006364 & -- & -- & -- & -- & -- & 3.27465 \\
0.012728 & -- & -- & -- & -- & -- & 3.27016 \\
0.018 & -- & -- & -- & -- & -- & 3.29629 \\
\addlinespace[1.5pt]
-- & 0.975 & -- & -- & -- & -- & 3.27140 \\
-- & 0.96464 & -- & -- & -- & -- & 3.26435 \\
-- & 0.92929 & -- & -- & -- & -- & 3.26798 \\
-- & 0.9 & -- & -- & -- & -- & 3.27793 \\
\addlinespace[1.5pt]
-- & -- & 0.15 & -- & -- & -- & 3.26878 \\
-- & -- & 0.21213 & -- & -- & -- & 3.26640 \\
-- & -- & 0.42426 & -- & -- & -- & 3.26341 \\
-- & -- & 0.6 & -- & -- & -- & 3.26543 \\
-- & -- & 0.84853 & -- & -- & -- & 3.26400 \\
\addlinespace[1.5pt]
-- & -- & -- & 0.9 & -- & -- & 3.26636 \\
-- & -- & -- & 0.85858 & -- & -- & 3.26464 \\
-- & -- & -- & 0.71716 & -- & -- & 3.26304 \\
-- & -- & -- & 0.6 & -- & -- & 3.26358 \\
\addlinespace[1.5pt]
-- & -- & -- & -- & 0.99681 & -- & 3.26598 \\
-- & -- & -- & -- & 0.9955 & -- & 3.26300 \\
-- & -- & -- & -- & 0.99099 & -- & 3.26375 \\
-- & -- & -- & -- & 0.98726 & -- & 3.26441 \\
-- & -- & -- & -- & 0.98198 & -- & 3.26634 \\
-- & -- & -- & -- & 0.97452 & -- & 3.27010 \\
\addlinespace[1.5pt]
-- & -- & -- & -- & -- & 0 & 3.29470 \\
-- & -- & -- & -- & -- & 0.2 & 3.26584 \\
-- & -- & -- & -- & -- & 0.6 & 3.26324 \\
-- & -- & -- & -- & -- & 0.8 & 3.26470 \\
\end{longtable}
\endgroup

\medskip

\Needspace{8\baselineskip}

\begingroup
\small
\renewcommand{\arraystretch}{1.06}
\setlength{\tabcolsep}{3pt}
\setlength{\LTpre}{0pt}
\setlength{\LTpost}{0pt}
\setlength{\LTleft}{\fill}
\setlength{\LTright}{\fill}
\setlength{\LTcapwidth}{\textwidth}
\begin{longtable}{rrrrrrr}
\caption{Hyperparameter ablation for MuonH at batch size $B=512K$.}
\label{tab:llm_cd_muonh_512k} \\
\toprule
$\eta_{\rm M}$ & $\mu$ & $\eta_{\rm Emb.}$ & $\beta_1$ & $\beta_2$ & $f_{\rm cd}$ & $L_{\rm val}$ \\
\midrule
\endfirsthead
\multicolumn{7}{l}{\itshape MuonH, $B=512K$} \\
\toprule
$\eta_{\rm M}$ & $\mu$ & $\eta_{\rm Emb.}$ & $\beta_1$ & $\beta_2$ & $f_{\rm cd}$ & $L_{\rm val}$ \\
\midrule
\endhead
\midrule
\endfoot
\bottomrule
\endlastfoot
0.018 & 0.95 & 0.6 & 0.85858 & 0.975 & 0.4 & 3.27768 \\
\midrule
0.009 & -- & -- & -- & -- & -- & 3.31502 \\
0.012728 & -- & -- & -- & -- & -- & 3.28879 \\
0.025456 & -- & -- & -- & -- & -- & 3.28412 \\
0.036 & -- & -- & -- & -- & -- & 3.31009 \\
\addlinespace[1.5pt]
-- & 0.975 & -- & -- & -- & -- & 3.28664 \\
-- & 0.96464 & -- & -- & -- & -- & 3.27986 \\
-- & 0.92929 & -- & -- & -- & -- & 3.28260 \\
-- & 0.9 & -- & -- & -- & -- & 3.29317 \\
\addlinespace[1.5pt]
-- & -- & 0.3 & -- & -- & -- & 3.28083 \\
-- & -- & 0.42426 & -- & -- & -- & 3.27904 \\
-- & -- & 0.84853 & -- & -- & -- & 3.27802 \\
-- & -- & 1.2 & -- & -- & -- & 3.28040 \\
\addlinespace[1.5pt]
-- & -- & -- & 0.95 & -- & -- & 3.27875 \\
-- & -- & -- & 0.92929 & -- & -- & 3.27751 \\
-- & -- & -- & 0.9 & -- & -- & 3.27760 \\
-- & -- & -- & 0.8 & -- & -- & 3.27826 \\
-- & -- & -- & 0.71716 & -- & -- & 3.28116 \\
\addlinespace[1.5pt]
-- & -- & -- & -- & 0.9875 & -- & 3.27784 \\
-- & -- & -- & -- & 0.98232 & -- & 3.27735 \\
-- & -- & -- & -- & 0.96464 & -- & 3.27817 \\
-- & -- & -- & -- & 0.95 & -- & 3.27851 \\
-- & -- & -- & -- & 0.92929 & -- & 3.28065 \\
-- & -- & -- & -- & 0.9 & -- & 3.28232 \\
\addlinespace[1.5pt]
-- & -- & -- & -- & -- & 0 & 3.30830 \\
-- & -- & -- & -- & -- & 0.2 & 3.28022 \\
-- & -- & -- & -- & -- & 0.6 & 3.27805 \\
-- & -- & -- & -- & -- & 0.8 & 3.27882 \\
\end{longtable}
\endgroup

\medskip

\Needspace{8\baselineskip}

\begingroup
\small
\renewcommand{\arraystretch}{1.06}
\setlength{\tabcolsep}{3pt}
\setlength{\LTpre}{0pt}
\setlength{\LTpost}{0pt}
\setlength{\LTleft}{\fill}
\setlength{\LTright}{\fill}
\setlength{\LTcapwidth}{\textwidth}
\begin{longtable}{rrrrrrr}
\caption{Hyperparameter ablation for MuonH at batch size $B=1M$.}
\label{tab:llm_cd_muonh_1m} \\
\toprule
$\eta_{\rm M}$ & $\mu$ & $\eta_{\rm Emb.}$ & $\beta_1$ & $\beta_2$ & $f_{\rm cd}$ & $L_{\rm val}$ \\
\midrule
\endfirsthead
\multicolumn{7}{l}{\itshape MuonH, $B=1M$} \\
\toprule
$\eta_{\rm M}$ & $\mu$ & $\eta_{\rm Emb.}$ & $\beta_1$ & $\beta_2$ & $f_{\rm cd}$ & $L_{\rm val}$ \\
\midrule
\endhead
\midrule
\endfoot
\bottomrule
\endlastfoot
0.025456 & 0.95 & 0.6 & 0.9 & 0.95125 & 0.4 & 3.30679 \\
\midrule
0.012728 & -- & -- & -- & -- & -- & 3.34188 \\
0.018 & -- & -- & -- & -- & -- & 3.31511 \\
0.036 & -- & -- & -- & -- & -- & 3.31483 \\
0.050912 & -- & -- & -- & -- & -- & 3.33885 \\
\addlinespace[1.5pt]
-- & 0.975 & -- & -- & -- & -- & 3.31714 \\
-- & 0.96464 & -- & -- & -- & -- & 3.31027 \\
-- & 0.92929 & -- & -- & -- & -- & 3.31138 \\
-- & 0.9 & -- & -- & -- & -- & 3.32204 \\
\addlinespace[1.5pt]
-- & -- & 0.21213 & -- & -- & -- & 3.31597 \\
-- & -- & 0.3 & -- & -- & -- & 3.31402 \\
-- & -- & 0.42426 & -- & -- & -- & 3.30766 \\
-- & -- & 0.84853 & -- & -- & -- & 3.30833 \\
-- & -- & 1.2 & -- & -- & -- & 3.31302 \\
\addlinespace[1.5pt]
-- & -- & -- & 0.95 & -- & -- & 3.31025 \\
-- & -- & -- & 0.92929 & -- & -- & 3.30805 \\
-- & -- & -- & 0.85858 & -- & -- & 3.30808 \\
-- & -- & -- & 0.8 & -- & -- & 3.31023 \\
\addlinespace[1.5pt]
-- & -- & -- & -- & 0.98276 & -- & 3.30706 \\
-- & -- & -- & -- & 0.97562 & -- & 3.30668 \\
-- & -- & -- & -- & 0.96553 & -- & 3.30731 \\
-- & -- & -- & -- & 0.93106 & -- & 3.30810 \\
-- & -- & -- & -- & 0.9025 & -- & 3.30767 \\
\addlinespace[1.5pt]
-- & -- & -- & -- & -- & 0 & 3.33445 \\
-- & -- & -- & -- & -- & 0.2 & 3.30956 \\
-- & -- & -- & -- & -- & 0.6 & 3.30758 \\
-- & -- & -- & -- & -- & 0.8 & 3.31045 \\
\end{longtable}
\endgroup

\medskip

\Needspace{8\baselineskip}

\begingroup
\small
\renewcommand{\arraystretch}{1.06}
\setlength{\tabcolsep}{3pt}
\setlength{\LTpre}{0pt}
\setlength{\LTpost}{0pt}
\setlength{\LTleft}{\fill}
\setlength{\LTright}{\fill}
\setlength{\LTcapwidth}{\textwidth}
\begin{longtable}{rrrrrrr}
\caption{Hyperparameter ablation for MuonH at batch size $B=2M$.}
\label{tab:llm_cd_muonh_2m} \\
\toprule
$\eta_{\rm M}$ & $\mu$ & $\eta_{\rm Emb.}$ & $\beta_1$ & $\beta_2$ & $f_{\rm cd}$ & $L_{\rm val}$ \\
\midrule
\endfirsthead
\multicolumn{7}{l}{\itshape MuonH, $B=2M$} \\
\toprule
$\eta_{\rm M}$ & $\mu$ & $\eta_{\rm Emb.}$ & $\beta_1$ & $\beta_2$ & $f_{\rm cd}$ & $L_{\rm val}$ \\
\midrule
\endhead
\midrule
\endfoot
\bottomrule
\endlastfoot
0.036 & 0.95 & 0.6 & 0.9 & 0.93442 & 0.4 & 3.37098 \\
\midrule
0.018 & -- & -- & -- & -- & -- & 3.40605 \\
0.025456 & -- & -- & -- & -- & -- & 3.38018 \\
0.050912 & -- & -- & -- & -- & -- & 3.38112 \\
0.072 & -- & -- & -- & -- & -- & 3.40408 \\
\addlinespace[1.5pt]
-- & 0.96464 & -- & -- & -- & -- & 3.37684 \\
-- & 0.92929 & -- & -- & -- & -- & 3.37373 \\
\addlinespace[1.5pt]
-- & -- & 0.42426 & -- & -- & -- & 3.37348 \\
-- & -- & 0.84853 & -- & -- & -- & 3.37688 \\
\addlinespace[1.5pt]
-- & -- & -- & 0.92929 & -- & -- & 3.37308 \\
-- & -- & -- & 0.85858 & -- & -- & 3.37239 \\
\addlinespace[1.5pt]
-- & -- & -- & -- & 0.95363 & -- & 3.37119 \\
-- & -- & -- & -- & 0.90725 & -- & 3.37120 \\
-- & -- & -- & -- & 0.86884 & -- & 3.37315 \\
\addlinespace[1.5pt]
-- & -- & -- & -- & -- & 0.2 & 3.37237 \\
-- & -- & -- & -- & -- & 0.6 & 3.37317 \\
\end{longtable}
\endgroup

\medskip

\Needspace{10\baselineskip}

\medskip

\textbf{KL-SOAP-H}

\smallskip

\begingroup
\footnotesize
\renewcommand{\arraystretch}{1.08}
\setlength{\tabcolsep}{2.5pt}
\setlength{\LTpre}{0pt}
\setlength{\LTpost}{0pt}
\setlength{\LTleft}{\fill}
\setlength{\LTright}{\fill}
\setlength{\LTcapwidth}{\textwidth}
\begin{longtable}{rrrrrrrrr}
\caption{Hyperparameter ablation for KL-SOAP-H at batch size $B=128K$.}
\label{tab:llm_cd_klsoaph_128k} \\
\toprule
$\eta_{\rm M}$ & $\beta_{1}^{\rm M}$ & $\beta_{2}^{\rm M}$ & $\beta_{\rm sh}$ & $\eta_{\rm Emb.}$ & $\beta_1$ & $\beta_2$ & $f_{\rm cd}$ & $L_{\rm val}$ \\
\midrule
\endfirsthead
\multicolumn{9}{l}{\itshape KL-SOAP-H, $B=128K$} \\
\toprule
$\eta_{\rm M}$ & $\beta_{1}^{\rm M}$ & $\beta_{2}^{\rm M}$ & $\beta_{\rm sh}$ & $\eta_{\rm Emb.}$ & $\beta_1$ & $\beta_2$ & $f_{\rm cd}$ & $L_{\rm val}$ \\
\midrule
\endhead
\midrule
\endfoot
\bottomrule
\endlastfoot
0.009 & 0.92929 & 0.974 & 0.94801 & 0.3 & 0.43431 & 0.99363 & 0.4 & 3.25728 \\
\midrule
0.0045 & -- & -- & -- & -- & -- & -- & -- & 3.29786 \\
0.006364 & -- & -- & -- & -- & -- & -- & -- & 3.27114 \\
0.012728 & -- & -- & -- & -- & -- & -- & -- & 3.26005 \\
0.018 & -- & -- & -- & -- & -- & -- & -- & 3.27850 \\
\addlinespace[1.5pt]
-- & 0.96464 & -- & -- & -- & -- & -- & -- & 3.25890 \\
-- & 0.95 & -- & -- & -- & -- & -- & -- & 3.25503 \\
-- & 0.9 & -- & -- & -- & -- & -- & -- & 3.26334 \\
\end{longtable}
\endgroup

\medskip

\Needspace{8\baselineskip}

\begingroup
\footnotesize
\renewcommand{\arraystretch}{1.08}
\setlength{\tabcolsep}{2.5pt}
\setlength{\LTpre}{0pt}
\setlength{\LTpost}{0pt}
\setlength{\LTleft}{\fill}
\setlength{\LTright}{\fill}
\setlength{\LTcapwidth}{\textwidth}
\begin{longtable}{rrrrrrrrr}
\caption{Hyperparameter ablation for KL-SOAP-H at batch size $B=512K$.}
\label{tab:llm_cd_klsoaph_512k} \\
\toprule
$\eta_{\rm M}$ & $\beta_{1}^{\rm M}$ & $\beta_{2}^{\rm M}$ & $\beta_{\rm sh}$ & $\eta_{\rm Emb.}$ & $\beta_1$ & $\beta_2$ & $f_{\rm cd}$ & $L_{\rm val}$ \\
\midrule
\endfirsthead
\multicolumn{9}{l}{\itshape KL-SOAP-H, $B=512K$} \\
\toprule
$\eta_{\rm M}$ & $\beta_{1}^{\rm M}$ & $\beta_{2}^{\rm M}$ & $\beta_{\rm sh}$ & $\eta_{\rm Emb.}$ & $\beta_1$ & $\beta_2$ & $f_{\rm cd}$ & $L_{\rm val}$ \\
\midrule
\endhead
\midrule
\endfoot
\bottomrule
\endlastfoot
0.018 & 0.92929 & 0.85858 & 0.9 & 0.6 & 0.8 & 0.975 & 0.8 & 3.26964 \\
\midrule
0.009 & -- & -- & -- & -- & -- & -- & -- & 3.30243 \\
0.012728 & -- & -- & -- & -- & -- & -- & -- & 3.28109 \\
0.025456 & -- & -- & -- & -- & -- & -- & -- & 3.27408 \\
0.036 & -- & -- & -- & -- & -- & -- & -- & 3.29737 \\
\addlinespace[1.5pt]
-- & 0.975 & -- & -- & -- & -- & -- & -- & 3.29182 \\
-- & 0.96464 & -- & -- & -- & -- & -- & -- & 3.27817 \\
-- & 0.95 & -- & -- & -- & -- & -- & -- & 3.27292 \\
\end{longtable}
\endgroup

\medskip

\Needspace{8\baselineskip}

\begingroup
\footnotesize
\renewcommand{\arraystretch}{1.08}
\setlength{\tabcolsep}{2.5pt}
\setlength{\LTpre}{0pt}
\setlength{\LTpost}{0pt}
\setlength{\LTleft}{\fill}
\setlength{\LTright}{\fill}
\setlength{\LTcapwidth}{\textwidth}
\begin{longtable}{rrrrrrrrr}
\caption{Hyperparameter ablation for KL-SOAP-H at batch size $B=1M$.}
\label{tab:llm_cd_klsoaph_1m} \\
\toprule
$\eta_{\rm M}$ & $\beta_{1}^{\rm M}$ & $\beta_{2}^{\rm M}$ & $\beta_{\rm sh}$ & $\eta_{\rm Emb.}$ & $\beta_1$ & $\beta_2$ & $f_{\rm cd}$ & $L_{\rm val}$ \\
\midrule
\endfirsthead
\multicolumn{9}{l}{\itshape KL-SOAP-H, $B=1M$} \\
\toprule
$\eta_{\rm M}$ & $\beta_{1}^{\rm M}$ & $\beta_{2}^{\rm M}$ & $\beta_{\rm sh}$ & $\eta_{\rm Emb.}$ & $\beta_1$ & $\beta_2$ & $f_{\rm cd}$ & $L_{\rm val}$ \\
\midrule
\endhead
\midrule
\endfoot
\bottomrule
\endlastfoot
0.025456 & 0.92929 & 0.86565 & 0.86565 & 0.6 & 0.8 & 0.95125 & 0.6 & 3.29536 \\
\midrule
0.012728 & -- & -- & -- & -- & -- & -- & -- & 3.32611 \\
0.018 & -- & -- & -- & -- & -- & -- & -- & 3.30198 \\
0.036 & -- & -- & -- & -- & -- & -- & -- & 3.30162 \\
0.050912 & -- & -- & -- & -- & -- & -- & -- & 3.32237 \\
\addlinespace[1.5pt]
-- & 0.95 & -- & -- & -- & -- & -- & -- & 3.30005 \\
-- & 0.9 & -- & -- & -- & -- & -- & -- & 3.29841 \\
\addlinespace[1.5pt]
-- & -- & 0.93282 & -- & -- & -- & -- & -- & 3.29737 \\
\end{longtable}
\endgroup

\medskip

\Needspace{8\baselineskip}

\begingroup
\footnotesize
\renewcommand{\arraystretch}{1.08}
\setlength{\tabcolsep}{2.5pt}
\setlength{\LTpre}{0pt}
\setlength{\LTpost}{0pt}
\setlength{\LTleft}{\fill}
\setlength{\LTright}{\fill}
\setlength{\LTcapwidth}{\textwidth}
\begin{longtable}{rrrrrrrrr}
\caption{Hyperparameter ablation for KL-SOAP-H at batch size $B=2M$.}
\label{tab:llm_cd_klsoaph_2m} \\
\toprule
$\eta_{\rm M}$ & $\beta_{1}^{\rm M}$ & $\beta_{2}^{\rm M}$ & $\beta_{\rm sh}$ & $\eta_{\rm Emb.}$ & $\beta_1$ & $\beta_2$ & $f_{\rm cd}$ & $L_{\rm val}$ \\
\midrule
\endfirsthead
\multicolumn{9}{l}{\itshape KL-SOAP-H, $B=2M$} \\
\toprule
$\eta_{\rm M}$ & $\beta_{1}^{\rm M}$ & $\beta_{2}^{\rm M}$ & $\beta_{\rm sh}$ & $\eta_{\rm Emb.}$ & $\beta_1$ & $\beta_2$ & $f_{\rm cd}$ & $L_{\rm val}$ \\
\midrule
\endhead
\midrule
\endfoot
\bottomrule
\endlastfoot
0.036 & 0.92929 & 0.6561 & 0.87841 & 0.6 & 0.8 & 0.90725 & 0.6 & 3.35062 \\
\midrule
0.018 & -- & -- & -- & -- & -- & -- & -- & 3.38592 \\
0.025456 & -- & -- & -- & -- & -- & -- & -- & 3.36252 \\
0.050912 & -- & -- & -- & -- & -- & -- & -- & 3.35583 \\
0.072 & -- & -- & -- & -- & -- & -- & -- & 3.38261 \\
\addlinespace[1.5pt]
-- & 0.95 & -- & -- & -- & -- & -- & -- & 3.35416 \\
-- & 0.9 & -- & -- & -- & -- & -- & -- & 3.35290 \\
\addlinespace[1.5pt]
-- & -- & 0.75683 & -- & -- & -- & -- & -- & 3.34943 \\
\end{longtable}
\endgroup

\medskip

\Needspace{10\baselineskip}

\medskip

\textbf{ShampooH}

\smallskip

\begingroup
\small
\renewcommand{\arraystretch}{1.06}
\setlength{\tabcolsep}{3pt}
\setlength{\LTpre}{0pt}
\setlength{\LTpost}{0pt}
\setlength{\LTleft}{\fill}
\setlength{\LTright}{\fill}
\setlength{\LTcapwidth}{\textwidth}
\begin{longtable}{rrrrrrrr}
\caption{Hyperparameter ablation for ShampooH at batch size $B=128K$.}
\label{tab:llm_cd_shampooh_128k} \\
\toprule
$\eta_{\rm M}$ & $\beta_{1}^{\rm M}$ & $\beta_{2}^{\rm M}$ & $\eta_{\rm Emb.}$ & $\beta_1$ & $\beta_2$ & $f_{\rm cd}$ & $L_{\rm val}$ \\
\midrule
\endfirsthead
\multicolumn{8}{l}{\itshape ShampooH, $B=128K$} \\
\toprule
$\eta_{\rm M}$ & $\beta_{1}^{\rm M}$ & $\beta_{2}^{\rm M}$ & $\eta_{\rm Emb.}$ & $\beta_1$ & $\beta_2$ & $f_{\rm cd}$ & $L_{\rm val}$ \\
\midrule
\endhead
\midrule
\endfoot
\bottomrule
\endlastfoot
0.009 & 0.92222 & 0.85294 & 1.2 & 0.52 & 0.9955 & 0.4 & 3.26617 \\
\midrule
0.0045 & -- & -- & -- & -- & -- & -- & 3.30292 \\
0.006364 & -- & -- & -- & -- & -- & -- & 3.27724 \\
0.012728 & -- & -- & -- & -- & -- & -- & 3.26965 \\
0.018 & -- & -- & -- & -- & -- & -- & 3.29102 \\
\addlinespace[1.5pt]
-- & 0.9725 & -- & -- & -- & -- & -- & 3.27934 \\
-- & 0.96111 & -- & -- & -- & -- & -- & 3.27001 \\
-- & 0.945 & -- & -- & -- & -- & -- & 3.26581 \\
-- & 0.89 & -- & -- & -- & -- & -- & 3.27140 \\
\addlinespace[1.5pt]
-- & -- & 0.92647 & -- & -- & -- & -- & 3.26603 \\
-- & -- & 0.89601 & -- & -- & -- & -- & 3.26571 \\
-- & -- & 0.79203 & -- & -- & -- & -- & 3.26656 \\
\addlinespace[1.5pt]
-- & -- & -- & 0.84853 & -- & -- & -- & 3.26672 \\
-- & -- & -- & 1.6971 & -- & -- & -- & 3.26749 \\
\addlinespace[1.5pt]
-- & -- & -- & -- & 0.66059 & -- & -- & 3.26792 \\
-- & -- & -- & -- & 0.32118 & -- & -- & 3.26924 \\
\addlinespace[1.5pt]
-- & -- & -- & -- & -- & 0.99681 & -- & 3.26945 \\
-- & -- & -- & -- & -- & 0.99363 & -- & 3.26595 \\
-- & -- & -- & -- & -- & 0.99099 & -- & 3.26633 \\
\addlinespace[1.5pt]
-- & -- & -- & -- & -- & -- & 0.2 & 3.26838 \\
-- & -- & -- & -- & -- & -- & 0.6 & 3.26544 \\
-- & -- & -- & -- & -- & -- & 0.8 & 3.26565 \\
\end{longtable}
\endgroup

\medskip

\Needspace{8\baselineskip}

\begingroup
\footnotesize
\renewcommand{\arraystretch}{1.08}
\setlength{\tabcolsep}{2.5pt}
\setlength{\LTpre}{0pt}
\setlength{\LTpost}{0pt}
\setlength{\LTleft}{\fill}
\setlength{\LTright}{\fill}
\setlength{\LTcapwidth}{\textwidth}
\begin{longtable}{rrrrrrrr}
\caption{Hyperparameter ablation for ShampooH at batch size $B=512K$.}
\label{tab:llm_cd_shampooh_512k} \\
\toprule
$\eta_{\rm M}$ & $\beta_{1}^{\rm M}$ & $\beta_{2}^{\rm M}$ & $\eta_{\rm Emb.}$ & $\beta_1$ & $\beta_2$ & $f_{\rm cd}$ & $L_{\rm val}$ \\
\midrule
\endfirsthead
\multicolumn{8}{l}{\itshape ShampooH, $B=512K$} \\
\toprule
$\eta_{\rm M}$ & $\beta_{1}^{\rm M}$ & $\beta_{2}^{\rm M}$ & $\eta_{\rm Emb.}$ & $\beta_1$ & $\beta_2$ & $f_{\rm cd}$ & $L_{\rm val}$ \\
\midrule
\endhead
\midrule
\endfoot
\bottomrule
\endlastfoot
0.018 & 0.92929 & 0.9 & 0.84853 & 0.8 & 0.98232 & 0.4 & 3.27578 \\
\midrule
0.009 & -- & -- & -- & -- & -- & -- & 3.31081 \\
0.012728 & -- & -- & -- & -- & -- & -- & 3.28723 \\
0.025456 & -- & -- & -- & -- & -- & -- & 3.28195 \\
0.036 & -- & -- & -- & -- & -- & -- & 3.30230 \\
\addlinespace[1.5pt]
-- & 0.96464 & -- & -- & -- & -- & -- & 3.28247 \\
-- & 0.95 & -- & -- & -- & -- & -- & 3.27802 \\
-- & 0.9 & -- & -- & -- & -- & -- & 3.28012 \\
\end{longtable}
\endgroup

\medskip

\Needspace{8\baselineskip}

\begingroup
\footnotesize
\renewcommand{\arraystretch}{1.08}
\setlength{\tabcolsep}{2.5pt}
\setlength{\LTpre}{0pt}
\setlength{\LTpost}{0pt}
\setlength{\LTleft}{\fill}
\setlength{\LTright}{\fill}
\setlength{\LTcapwidth}{\textwidth}
\begin{longtable}{rrrrrrrr}
\caption{Hyperparameter ablation for ShampooH at batch size $B=1M$.}
\label{tab:llm_cd_shampooh_1m} \\
\toprule
$\eta_{\rm M}$ & $\beta_{1}^{\rm M}$ & $\beta_{2}^{\rm M}$ & $\eta_{\rm Emb.}$ & $\beta_1$ & $\beta_2$ & $f_{\rm cd}$ & $L_{\rm val}$ \\
\midrule
\endfirsthead
\multicolumn{8}{l}{\itshape ShampooH, $B=1M$} \\
\toprule
$\eta_{\rm M}$ & $\beta_{1}^{\rm M}$ & $\beta_{2}^{\rm M}$ & $\eta_{\rm Emb.}$ & $\beta_1$ & $\beta_2$ & $f_{\rm cd}$ & $L_{\rm val}$ \\
\midrule
\endhead
\midrule
\endfoot
\bottomrule
\endlastfoot
0.025456 & 0.92929 & 0.848 & 0.84853 & 0.8 & 0.97075 & 0.4 & 3.29641 \\
\midrule
0.012728 & -- & -- & -- & -- & -- & -- & 3.33305 \\
0.018 & -- & -- & -- & -- & -- & -- & 3.30733 \\
0.036 & -- & -- & -- & -- & -- & -- & 3.30217 \\
0.050912 & -- & -- & -- & -- & -- & -- & 3.32518 \\
\addlinespace[1.5pt]
-- & 0.96464 & -- & -- & -- & -- & -- & 3.30594 \\
-- & 0.95 & -- & -- & -- & -- & -- & 3.29847 \\
-- & 0.9 & -- & -- & -- & -- & -- & 3.30112 \\
\end{longtable}
\endgroup

\medskip

\Needspace{8\baselineskip}

\begingroup
\small
\renewcommand{\arraystretch}{1.06}
\setlength{\tabcolsep}{3pt}
\setlength{\LTpre}{0pt}
\setlength{\LTpost}{0pt}
\setlength{\LTleft}{\fill}
\setlength{\LTright}{\fill}
\setlength{\LTcapwidth}{\textwidth}
\begin{longtable}{rrrrrrrr}
\caption{Hyperparameter ablation for ShampooH at batch size $B=2M$.}
\label{tab:llm_cd_shampooh_2m} \\
\toprule
$\eta_{\rm M}$ & $\beta_{1}^{\rm M}$ & $\beta_{2}^{\rm M}$ & $\eta_{\rm Emb.}$ & $\beta_1$ & $\beta_2$ & $f_{\rm cd}$ & $L_{\rm val}$ \\
\midrule
\endfirsthead
\multicolumn{8}{l}{\itshape ShampooH, $B=2M$} \\
\toprule
$\eta_{\rm M}$ & $\beta_{1}^{\rm M}$ & $\beta_{2}^{\rm M}$ & $\eta_{\rm Emb.}$ & $\beta_1$ & $\beta_2$ & $f_{\rm cd}$ & $L_{\rm val}$ \\
\midrule
\endhead
\midrule
\endfoot
\bottomrule
\endlastfoot
0.036 & 0.92929 & 0.72488 & 0.6 & 0.8 & 0.94435 & 0.4 & 3.33928 \\
\midrule
0.018 & -- & -- & -- & -- & -- & -- & 3.37605 \\
0.025456 & -- & -- & -- & -- & -- & -- & 3.34886 \\
0.050912 & -- & -- & -- & -- & -- & -- & 3.34763 \\
0.072 & -- & -- & -- & -- & -- & -- & 3.37374 \\
\addlinespace[1.5pt]
-- & 0.95 & -- & -- & -- & -- & -- & 3.34432 \\
-- & 0.9 & -- & -- & -- & -- & -- & 3.34161 \\
\addlinespace[1.5pt]
-- & -- & 0.86244 & -- & -- & -- & -- & 3.34115 \\
-- & -- & 0.80546 & -- & -- & -- & -- & 3.34013 \\
-- & -- & 0.61092 & -- & -- & -- & -- & 3.34285 \\
\addlinespace[1.5pt]
-- & -- & -- & 0.42426 & -- & -- & -- & 3.34033 \\
-- & -- & -- & 0.84853 & -- & -- & -- & 3.34150 \\
-- & -- & -- & 1.6971 & -- & -- & -- & 3.36828 \\
-- & -- & -- & 2.4 & -- & -- & -- & 3.35629 \\
\addlinespace[1.5pt]
-- & -- & -- & -- & 0.9 & -- & -- & 3.34651 \\
-- & -- & -- & -- & 0.85858 & -- & -- & 3.34282 \\
-- & -- & -- & -- & 0.71716 & -- & -- & 3.34073 \\
\addlinespace[1.5pt]
-- & -- & -- & -- & -- & 0.97218 & -- & 3.33963 \\
-- & -- & -- & -- & -- & 0.96065 & -- & 3.33901 \\
-- & -- & -- & -- & -- & 0.9213 & -- & 3.33973 \\
\addlinespace[1.5pt]
-- & -- & -- & -- & -- & -- & 0.2 & 3.34120 \\
-- & -- & -- & -- & -- & -- & 0.6 & 3.34076 \\
\end{longtable}
\endgroup

\medskip

\Needspace{10\baselineskip}

\medskip

\textbf{AdamH}

\smallskip

\begingroup
\scriptsize
\renewcommand{\arraystretch}{1.04}
\setlength{\tabcolsep}{2.5pt}
\setlength{\LTpre}{0pt}
\setlength{\LTpost}{0pt}
\setlength{\LTleft}{\fill}
\setlength{\LTright}{\fill}
\setlength{\LTcapwidth}{\textwidth}
\begin{longtable}{rrrrrrrr}
\caption{Hyperparameter ablation for AdamH at batch size $B=128K$.}
\label{tab:llm_cd_adamh_128k} \\
\toprule
$\eta_{\rm M}$ & $\beta_{1}^{\rm M}$ & $\beta_{2}^{\rm M}$ & $\eta_{\rm Emb.}$ & $\beta_1$ & $\beta_2$ & $f_{\rm cd}$ & $L_{\rm val}$ \\
\midrule
\endfirsthead
\multicolumn{8}{l}{\itshape AdamH, $B=128K$} \\
\toprule
$\eta_{\rm M}$ & $\beta_{1}^{\rm M}$ & $\beta_{2}^{\rm M}$ & $\eta_{\rm Emb.}$ & $\beta_1$ & $\beta_2$ & $f_{\rm cd}$ & $L_{\rm val}$ \\
\midrule
\endhead
\midrule
\endfoot
\bottomrule
\endlastfoot
0.009 & 0.94 & 0.94903 & 0.42426 & 0.71716 & 0.9955 & 0.4 & 3.30144 \\
\midrule
0.0045 & -- & -- & -- & -- & -- & -- & 3.32743 \\
0.006364 & -- & -- & -- & -- & -- & -- & 3.30696 \\
0.012728 & -- & -- & -- & -- & -- & -- & 3.31867 \\
0.018 & -- & -- & -- & -- & -- & -- & 3.37667 \\
\addlinespace[1.5pt]
-- & 0.96464 & -- & -- & -- & -- & -- & 3.30978 \\
-- & 0.95757 & -- & -- & -- & -- & -- & 3.30506 \\
-- & 0.955 & -- & -- & -- & -- & -- & 3.30335 \\
\end{longtable}
\endgroup

\medskip

\Needspace{8\baselineskip}

\begingroup
\scriptsize
\renewcommand{\arraystretch}{1.04}
\setlength{\tabcolsep}{2.5pt}
\setlength{\LTpre}{0pt}
\setlength{\LTpost}{0pt}
\setlength{\LTleft}{\fill}
\setlength{\LTright}{\fill}
\setlength{\LTcapwidth}{\textwidth}
\begin{longtable}{rrrrrrrr}
\caption{Hyperparameter ablation for AdamH at batch size $B=512K$.}
\label{tab:llm_cd_adamh_512k} \\
\toprule
$\eta_{\rm M}$ & $\beta_{1}^{\rm M}$ & $\beta_{2}^{\rm M}$ & $\eta_{\rm Emb.}$ & $\beta_1$ & $\beta_2$ & $f_{\rm cd}$ & $L_{\rm val}$ \\
\midrule
\endfirsthead
\multicolumn{8}{l}{\itshape AdamH, $B=512K$} \\
\toprule
$\eta_{\rm M}$ & $\beta_{1}^{\rm M}$ & $\beta_{2}^{\rm M}$ & $\eta_{\rm Emb.}$ & $\beta_1$ & $\beta_2$ & $f_{\rm cd}$ & $L_{\rm val}$ \\
\midrule
\endhead
\midrule
\endfoot
\bottomrule
\endlastfoot
0.018 & 0.92929 & 0.9 & 0.6 & 0.8 & 0.975 & 0.4 & 3.33227 \\
\midrule
0.009 & -- & -- & -- & -- & -- & -- & 3.36880 \\
0.012728 & -- & -- & -- & -- & -- & -- & 3.34160 \\
0.025456 & -- & -- & -- & -- & -- & -- & 3.34476 \\
0.036 & -- & -- & -- & -- & -- & -- & 3.38899 \\
\addlinespace[1.5pt]
-- & 0.96464 & -- & -- & -- & -- & -- & 3.35050 \\
-- & 0.95 & -- & -- & -- & -- & -- & 3.33188 \\
-- & 0.9 & -- & -- & -- & -- & -- & 3.34178 \\
\end{longtable}
\endgroup

\medskip

\Needspace{8\baselineskip}

\begingroup
\scriptsize
\renewcommand{\arraystretch}{1.04}
\setlength{\tabcolsep}{2.5pt}
\setlength{\LTpre}{0pt}
\setlength{\LTpost}{0pt}
\setlength{\LTleft}{\fill}
\setlength{\LTright}{\fill}
\setlength{\LTcapwidth}{\textwidth}
\begin{longtable}{rrrrrrrr}
\caption{Hyperparameter ablation for AdamH at batch size $B=1M$.}
\label{tab:llm_cd_adamh_1m} \\
\toprule
$\eta_{\rm M}$ & $\beta_{1}^{\rm M}$ & $\beta_{2}^{\rm M}$ & $\eta_{\rm Emb.}$ & $\beta_1$ & $\beta_2$ & $f_{\rm cd}$ & $L_{\rm val}$ \\
\midrule
\endfirsthead
\multicolumn{8}{l}{\itshape AdamH, $B=1M$} \\
\toprule
$\eta_{\rm M}$ & $\beta_{1}^{\rm M}$ & $\beta_{2}^{\rm M}$ & $\eta_{\rm Emb.}$ & $\beta_1$ & $\beta_2$ & $f_{\rm cd}$ & $L_{\rm val}$ \\
\midrule
\endhead
\midrule
\endfoot
\bottomrule
\endlastfoot
0.018 & 0.96464 & 0.9025 & 0.3 & 0.8 & 0.95863 & 0.25 & 3.37114 \\
\midrule
0.006364 & -- & -- & -- & -- & -- & -- & 3.45697 \\
0.009 & -- & -- & -- & -- & -- & -- & 3.41310 \\
0.012728 & -- & -- & -- & -- & -- & -- & 3.38258 \\
0.025456 & -- & -- & -- & -- & -- & -- & 3.37975 \\
0.036 & -- & -- & -- & -- & -- & -- & 3.41154 \\
\addlinespace[1.5pt]
-- & 0.98232 & -- & -- & -- & -- & -- & 3.54534 \\
-- & 0.975 & -- & -- & -- & -- & -- & 3.38105 \\
\end{longtable}
\endgroup

\medskip

\Needspace{8\baselineskip}

\begingroup
\footnotesize
\renewcommand{\arraystretch}{1.04}
\setlength{\tabcolsep}{2.5pt}
\setlength{\LTpre}{0pt}
\setlength{\LTpost}{0pt}
\setlength{\LTleft}{\fill}
\setlength{\LTright}{\fill}
\setlength{\LTcapwidth}{\textwidth}
\begin{longtable}{rrrrrrrr}
\caption{Hyperparameter ablation for AdamH at batch size $B=2M$.}
\label{tab:llm_cd_adamh_2m} \\
\toprule
$\eta_{\rm M}$ & $\beta_{1}^{\rm M}$ & $\beta_{2}^{\rm M}$ & $\eta_{\rm Emb.}$ & $\beta_1$ & $\beta_2$ & $f_{\rm cd}$ & $L_{\rm val}$ \\
\midrule
\endfirsthead
\multicolumn{8}{l}{\itshape AdamH, $B=2M$} \\
\toprule
$\eta_{\rm M}$ & $\beta_{1}^{\rm M}$ & $\beta_{2}^{\rm M}$ & $\eta_{\rm Emb.}$ & $\beta_1$ & $\beta_2$ & $f_{\rm cd}$ & $L_{\rm val}$ \\
\midrule
\endhead
\midrule
\endfoot
\bottomrule
\endlastfoot
0.025456 & 0.96464 & 0.90725 & 0.42426 & 0.8 & 0.9213 & 0.25 & 3.48723 \\
\midrule
0.012728 & -- & -- & -- & -- & -- & -- & 3.52999 \\
0.018 & -- & -- & -- & -- & -- & -- & 3.49427 \\
0.036 & -- & -- & -- & -- & -- & -- & 3.50275 \\
0.050912 & -- & -- & -- & -- & -- & -- & 3.54708 \\
\addlinespace[1.5pt]
-- & 0.975 & -- & -- & -- & -- & -- & 3.50189 \\
-- & 0.95 & -- & -- & -- & -- & -- & 3.49129 \\
\addlinespace[1.5pt]
-- & -- & 0.93442 & -- & -- & -- & -- & 3.48752 \\
-- & -- & 0.86884 & -- & -- & -- & -- & 3.50311 \\
-- & -- & 0.73767 & -- & -- & -- & -- & 3.57150 \\
-- & -- & 0.62901 & -- & -- & -- & -- & 3.63153 \\
\addlinespace[1.5pt]
-- & -- & -- & 0.3 & -- & -- & -- & 3.48943 \\
-- & -- & -- & 0.6 & -- & -- & -- & 3.49050 \\
\addlinespace[1.5pt]
-- & -- & -- & -- & 0.9 & -- & -- & 3.48878 \\
-- & -- & -- & -- & 0.85858 & -- & -- & 3.48518 \\
-- & -- & -- & -- & 0.71716 & -- & -- & 3.49649 \\
\addlinespace[1.5pt]
-- & -- & -- & -- & -- & 0.98033 & -- & 3.49149 \\
-- & -- & -- & -- & -- & 0.97218 & -- & 3.48463 \\
-- & -- & -- & -- & -- & 0.94435 & -- & 3.48546 \\
-- & -- & -- & -- & -- & 0.8887 & -- & 3.48618 \\
\addlinespace[1.5pt]
-- & -- & -- & -- & -- & -- & 0 & 3.50407 \\
-- & -- & -- & -- & -- & -- & 0.2 & 3.48478 \\
-- & -- & -- & -- & -- & -- & 0.4 & 3.48505 \\
\end{longtable}
\endgroup

\medskip

\Needspace{10\baselineskip}

\medskip

\textbf{LionH}

\smallskip

\begingroup
\footnotesize
\renewcommand{\arraystretch}{1.08}
\setlength{\tabcolsep}{2.5pt}
\setlength{\LTpre}{0pt}
\setlength{\LTpost}{0pt}
\setlength{\LTleft}{\fill}
\setlength{\LTright}{\fill}
\setlength{\LTcapwidth}{\textwidth}
\begin{longtable}{rrrrrrrr}
\caption{Hyperparameter ablation for LionH at batch size $B=128K$.}
\label{tab:llm_cd_lionh_128k} \\
\toprule
$\eta_{\rm M}$ & $\beta_{1}^{\rm M}$ & $\beta_{2}^{\rm M}$ & $\eta_{\rm Emb.}$ & $\beta_1$ & $\beta_2$ & $f_{\rm cd}$ & $L_{\rm val}$ \\
\midrule
\endfirsthead
\multicolumn{8}{l}{\itshape LionH, $B=128K$} \\
\toprule
$\eta_{\rm M}$ & $\beta_{1}^{\rm M}$ & $\beta_{2}^{\rm M}$ & $\eta_{\rm Emb.}$ & $\beta_1$ & $\beta_2$ & $f_{\rm cd}$ & $L_{\rm val}$ \\
\midrule
\endhead
\midrule
\endfoot
\bottomrule
\endlastfoot
0.006364 & 0.92929 & 0.98726 & 0.3 & 0.71716 & 0.99363 & 0.4 & 3.29369 \\
\midrule
0.003182 & -- & -- & -- & -- & -- & -- & 3.33450 \\
0.0045 & -- & -- & -- & -- & -- & -- & 3.30646 \\
0.009 & -- & -- & -- & -- & -- & -- & 3.31381 \\
0.012728 & -- & -- & -- & -- & -- & -- & 3.35250 \\
\addlinespace[1.5pt]
-- & 0.96464 & -- & -- & -- & -- & -- & 3.30461 \\
-- & 0.95 & -- & -- & -- & -- & -- & 3.29967 \\
-- & 0.9 & -- & -- & -- & -- & -- & 3.30068 \\
\end{longtable}
\endgroup

\medskip

\Needspace{8\baselineskip}

\begingroup
\footnotesize
\renewcommand{\arraystretch}{1.08}
\setlength{\tabcolsep}{2.5pt}
\setlength{\LTpre}{0pt}
\setlength{\LTpost}{0pt}
\setlength{\LTleft}{\fill}
\setlength{\LTright}{\fill}
\setlength{\LTcapwidth}{\textwidth}
\begin{longtable}{rrrrrrrr}
\caption{Hyperparameter ablation for LionH at batch size $B=512K$.}
\label{tab:llm_cd_lionh_512k} \\
\toprule
$\eta_{\rm M}$ & $\beta_{1}^{\rm M}$ & $\beta_{2}^{\rm M}$ & $\eta_{\rm Emb.}$ & $\beta_1$ & $\beta_2$ & $f_{\rm cd}$ & $L_{\rm val}$ \\
\midrule
\endfirsthead
\multicolumn{8}{l}{\itshape LionH, $B=512K$} \\
\toprule
$\eta_{\rm M}$ & $\beta_{1}^{\rm M}$ & $\beta_{2}^{\rm M}$ & $\eta_{\rm Emb.}$ & $\beta_1$ & $\beta_2$ & $f_{\rm cd}$ & $L_{\rm val}$ \\
\midrule
\endhead
\midrule
\endfoot
\bottomrule
\endlastfoot
0.018 & 0.92929 & 0.95 & 0.42426 & 0.8 & 0.98333 & 0.4 & 3.33575 \\
\midrule
0.009 & -- & -- & -- & -- & -- & -- & 3.37584 \\
0.012728 & -- & -- & -- & -- & -- & -- & 3.34782 \\
0.025456 & -- & -- & -- & -- & -- & -- & 3.34516 \\
0.036 & -- & -- & -- & -- & -- & -- & 3.38208 \\
\addlinespace[1.5pt]
-- & 0.96464 & -- & -- & -- & -- & -- & 3.37497 \\
-- & 0.95 & -- & -- & -- & -- & -- & 3.33810 \\
-- & 0.9 & -- & -- & -- & -- & -- & 3.34398 \\
\end{longtable}
\endgroup

\medskip

\Needspace{8\baselineskip}

\begingroup
\footnotesize
\renewcommand{\arraystretch}{1.08}
\setlength{\tabcolsep}{2.5pt}
\setlength{\LTpre}{0pt}
\setlength{\LTpost}{0pt}
\setlength{\LTleft}{\fill}
\setlength{\LTright}{\fill}
\setlength{\LTcapwidth}{\textwidth}
\begin{longtable}{rrrrrrrr}
\caption{Hyperparameter ablation for LionH at batch size $B=1M$.}
\label{tab:llm_cd_lionh_1m} \\
\toprule
$\eta_{\rm M}$ & $\beta_{1}^{\rm M}$ & $\beta_{2}^{\rm M}$ & $\eta_{\rm Emb.}$ & $\beta_1$ & $\beta_2$ & $f_{\rm cd}$ & $L_{\rm val}$ \\
\midrule
\endfirsthead
\multicolumn{8}{l}{\itshape LionH, $B=1M$} \\
\toprule
$\eta_{\rm M}$ & $\beta_{1}^{\rm M}$ & $\beta_{2}^{\rm M}$ & $\eta_{\rm Emb.}$ & $\beta_1$ & $\beta_2$ & $f_{\rm cd}$ & $L_{\rm val}$ \\
\midrule
\endhead
\midrule
\endfoot
\bottomrule
\endlastfoot
0.036 & 0.9 & 0.9025 & 0.6 & 0.9 & 0.97562 & 0.4 & 3.40475 \\
\midrule
0.018 & -- & -- & -- & -- & -- & -- & 3.43203 \\
0.025456 & -- & -- & -- & -- & -- & -- & 3.40845 \\
0.050912 & -- & -- & -- & -- & -- & -- & 3.43559 \\
0.072 & -- & -- & -- & -- & -- & -- & 3.48983 \\
\addlinespace[1.5pt]
-- & 0.95 & -- & -- & -- & -- & -- & 3.66610 \\
-- & 0.92929 & -- & -- & -- & -- & -- & 3.49354 \\
-- & 0.85858 & -- & -- & -- & -- & -- & 3.43192 \\
\end{longtable}
\endgroup

\medskip

\Needspace{8\baselineskip}

\begingroup
\small
\renewcommand{\arraystretch}{1.06}
\setlength{\tabcolsep}{3pt}
\setlength{\LTpre}{0pt}
\setlength{\LTpost}{0pt}
\setlength{\LTleft}{\fill}
\setlength{\LTright}{\fill}
\setlength{\LTcapwidth}{\textwidth}
\begin{longtable}{rrrrrrrr}
\caption{Hyperparameter ablation for LionH at batch size $B=2M$.}
\label{tab:llm_cd_lionh_2m} \\
\toprule
$\eta_{\rm M}$ & $\beta_{1}^{\rm M}$ & $\beta_{2}^{\rm M}$ & $\eta_{\rm Emb.}$ & $\beta_1$ & $\beta_2$ & $f_{\rm cd}$ & $L_{\rm val}$ \\
\midrule
\endfirsthead
\multicolumn{8}{l}{\itshape LionH, $B=2M$} \\
\toprule
$\eta_{\rm M}$ & $\beta_{1}^{\rm M}$ & $\beta_{2}^{\rm M}$ & $\eta_{\rm Emb.}$ & $\beta_1$ & $\beta_2$ & $f_{\rm cd}$ & $L_{\rm val}$ \\
\midrule
\endhead
\midrule
\endfoot
\bottomrule
\endlastfoot
0.036 & 0.9 & 0.90725 & 0.6 & 0.9 & 0.95363 & 0.2 & 3.55651 \\
\midrule
0.018 & -- & -- & -- & -- & -- & -- & 3.60350 \\
0.025456 & -- & -- & -- & -- & -- & -- & 3.56984 \\
0.050912 & -- & -- & -- & -- & -- & -- & 3.57392 \\
0.072 & -- & -- & -- & -- & -- & -- & 3.62350 \\
\addlinespace[1.5pt]
-- & 0.92929 & -- & -- & -- & -- & -- & 3.64818 \\
-- & 0.85858 & -- & -- & -- & -- & -- & 3.61206 \\
\addlinespace[1.5pt]
-- & -- & 0.93442 & -- & -- & -- & -- & 3.59923 \\
-- & -- & 0.86884 & -- & -- & -- & -- & 3.64545 \\
\addlinespace[1.5pt]
-- & -- & -- & 0.42426 & -- & -- & -- & 3.56105 \\
-- & -- & -- & 0.84853 & -- & -- & -- & 3.56483 \\
\addlinespace[1.5pt]
-- & -- & -- & -- & 0.92929 & -- & -- & 3.55956 \\
-- & -- & -- & -- & 0.85858 & -- & -- & 3.56042 \\
\addlinespace[1.5pt]
-- & -- & -- & -- & -- & 0.97681 & -- & 3.55751 \\
-- & -- & -- & -- & -- & 0.96721 & -- & 3.55283 \\
-- & -- & -- & -- & -- & 0.93442 & -- & 3.55904 \\
\addlinespace[1.5pt]
-- & -- & -- & -- & -- & -- & 0 & 3.57606 \\
-- & -- & -- & -- & -- & -- & 0.4 & 3.55957 \\
\end{longtable}
\endgroup

\medskip

\Needspace{10\baselineskip}

\subsubsection{Weight decay sweeps}

\medskip

\textbf{MuonW}

\smallskip

\begingroup
\small
\renewcommand{\arraystretch}{1.06}
\setlength{\tabcolsep}{3pt}
\setlength{\LTpre}{0pt}
\setlength{\LTpost}{0pt}
\setlength{\LTleft}{\fill}
\setlength{\LTright}{\fill}
\setlength{\LTcapwidth}{\textwidth}
\begin{longtable}{rrrrrrrr}
\caption{Hyperparameter ablation for MuonW at batch size $B=128K$.}
\label{tab:llm_cd_muonw_128k} \\
\toprule
$\eta_{\rm M}$ & $\mu$ & $\lambda_{\rm M}$ & $\eta_{\rm Emb.}$ & $\beta_1$ & $\beta_2$ & $f_{\rm cd}$ & $L_{\rm val}$ \\
\midrule
\endfirsthead
\multicolumn{8}{l}{\itshape MuonW, $B=128K$} \\
\toprule
$\eta_{\rm M}$ & $\mu$ & $\lambda_{\rm M}$ & $\eta_{\rm Emb.}$ & $\beta_1$ & $\beta_2$ & $f_{\rm cd}$ & $L_{\rm val}$ \\
\midrule
\endhead
\midrule
\endfoot
\bottomrule
\endlastfoot
0.017678 & 0.95 & 0.025 & 0.49497 & 0.8 & 0.99363 & 0.8 & 3.26547 \\
\midrule
0.00625 & -- & -- & -- & -- & -- & -- & 3.27885 \\
0.0088388 & -- & -- & -- & -- & -- & -- & 3.27196 \\
0.0125 & -- & -- & -- & -- & -- & -- & 3.26909 \\
0.025 & -- & -- & -- & -- & -- & -- & 3.27012 \\
0.035355 & -- & -- & -- & -- & -- & -- & 3.28074 \\
\addlinespace[1.5pt]
-- & 0.975 & -- & -- & -- & -- & -- & 3.26698 \\
-- & 0.96464 & -- & -- & -- & -- & -- & 3.26617 \\
-- & 0.92929 & -- & -- & -- & -- & -- & 3.26707 \\
-- & 0.9 & -- & -- & -- & -- & -- & 3.27014 \\
\addlinespace[1.5pt]
-- & -- & 0.01 & -- & -- & -- & -- & 3.26771 \\
-- & -- & 0.05 & -- & -- & -- & -- & 3.27446 \\
\addlinespace[1.5pt]
-- & -- & -- & 0.35 & -- & -- & -- & 3.26966 \\
-- & -- & -- & 0.7 & -- & -- & -- & 3.26812 \\
-- & -- & -- & 0.98995 & -- & -- & -- & 3.27181 \\
\addlinespace[1.5pt]
-- & -- & -- & -- & 0.9 & -- & -- & 3.26904 \\
-- & -- & -- & -- & 0.85858 & -- & -- & 3.26685 \\
-- & -- & -- & -- & 0.71716 & -- & -- & 3.26623 \\
-- & -- & -- & -- & 0.6 & -- & -- & 3.26888 \\
\addlinespace[1.5pt]
-- & -- & -- & -- & -- & 0.99681 & -- & 3.26637 \\
-- & -- & -- & -- & -- & 0.9955 & -- & 3.26512 \\
-- & -- & -- & -- & -- & 0.99099 & -- & 3.26567 \\
-- & -- & -- & -- & -- & 0.98726 & -- & 3.26673 \\
\addlinespace[1.5pt]
-- & -- & -- & -- & -- & -- & 0.6 & 3.26940 \\
-- & -- & -- & -- & -- & -- & 1 & 3.26561 \\
\end{longtable}
\endgroup

\medskip

\Needspace{8\baselineskip}

\begingroup
\small
\renewcommand{\arraystretch}{1.06}
\setlength{\tabcolsep}{3pt}
\setlength{\LTpre}{0pt}
\setlength{\LTpost}{0pt}
\setlength{\LTleft}{\fill}
\setlength{\LTright}{\fill}
\setlength{\LTcapwidth}{\textwidth}
\begin{longtable}{rrrrrrrr}
\caption{Hyperparameter ablation for MuonW at batch size $B=512K$.}
\label{tab:llm_cd_muonw_512k} \\
\toprule
$\eta_{\rm M}$ & $\mu$ & $\lambda_{\rm M}$ & $\eta_{\rm Emb.}$ & $\beta_1$ & $\beta_2$ & $f_{\rm cd}$ & $L_{\rm val}$ \\
\midrule
\endfirsthead
\multicolumn{8}{l}{\itshape MuonW, $B=512K$} \\
\toprule
$\eta_{\rm M}$ & $\mu$ & $\lambda_{\rm M}$ & $\eta_{\rm Emb.}$ & $\beta_1$ & $\beta_2$ & $f_{\rm cd}$ & $L_{\rm val}$ \\
\midrule
\endhead
\midrule
\endfoot
\bottomrule
\endlastfoot
0.025 & 0.95 & 0.05 & 0.7 & 0.8 & 0.96464 & 0.7 & 3.27788 \\
\midrule
0.0125 & -- & -- & -- & -- & -- & -- & 3.28710 \\
0.017678 & -- & -- & -- & -- & -- & -- & 3.28196 \\
0.035355 & -- & -- & -- & -- & -- & -- & 3.28059 \\
0.05 & -- & -- & -- & -- & -- & -- & 3.28835 \\
\addlinespace[1.5pt]
-- & 0.96464 & -- & -- & -- & -- & -- & 3.27811 \\
-- & 0.92929 & -- & -- & -- & -- & -- & 3.27926 \\
\addlinespace[1.5pt]
-- & -- & 0.025 & -- & -- & -- & -- & 3.28061 \\
-- & -- & 0.1 & -- & -- & -- & -- & 3.28498 \\
\addlinespace[1.5pt]
-- & -- & -- & 0.49497 & -- & -- & -- & 3.27872 \\
-- & -- & -- & 0.98995 & -- & -- & -- & 3.27890 \\
\addlinespace[1.5pt]
-- & -- & -- & -- & 0.9 & -- & -- & 3.27871 \\
-- & -- & -- & -- & 0.85858 & -- & -- & 3.27724 \\
-- & -- & -- & -- & 0.71716 & -- & -- & 3.27936 \\
\addlinespace[1.5pt]
-- & -- & -- & -- & -- & 0.98232 & -- & 3.27671 \\
-- & -- & -- & -- & -- & 0.975 & -- & 3.27618 \\
-- & -- & -- & -- & -- & 0.95 & -- & 3.27863 \\
-- & -- & -- & -- & -- & 0.92929 & -- & 3.28128 \\
-- & -- & -- & -- & -- & 0.9 & -- & 3.28350 \\
\addlinespace[1.5pt]
-- & -- & -- & -- & -- & -- & 0.6 & 3.27873 \\
-- & -- & -- & -- & -- & -- & 0.8 & 3.27792 \\
\end{longtable}
\endgroup

\medskip

\Needspace{8\baselineskip}

\begingroup
\small
\renewcommand{\arraystretch}{1.06}
\setlength{\tabcolsep}{3pt}
\setlength{\LTpre}{0pt}
\setlength{\LTpost}{0pt}
\setlength{\LTleft}{\fill}
\setlength{\LTright}{\fill}
\setlength{\LTcapwidth}{\textwidth}
\begin{longtable}{rrrrrrrr}
\caption{Hyperparameter ablation for MuonW at batch size $B=1M$.}
\label{tab:llm_cd_muonw_1m} \\
\toprule
$\eta_{\rm M}$ & $\mu$ & $\lambda_{\rm M}$ & $\eta_{\rm Emb.}$ & $\beta_1$ & $\beta_2$ & $f_{\rm cd}$ & $L_{\rm val}$ \\
\midrule
\endfirsthead
\multicolumn{8}{l}{\itshape MuonW, $B=1M$} \\
\toprule
$\eta_{\rm M}$ & $\mu$ & $\lambda_{\rm M}$ & $\eta_{\rm Emb.}$ & $\beta_1$ & $\beta_2$ & $f_{\rm cd}$ & $L_{\rm val}$ \\
\midrule
\endhead
\midrule
\endfoot
\bottomrule
\endlastfoot
0.025 & 0.95 & 0.1 & 0.98995 & 0.8 & 0.93106 & 0.8 & 3.31108 \\
\midrule
0.0125 & -- & -- & -- & -- & -- & -- & 3.32240 \\
0.017678 & -- & -- & -- & -- & -- & -- & 3.31407 \\
0.035355 & -- & -- & -- & -- & -- & -- & 3.31274 \\
0.05 & -- & -- & -- & -- & -- & -- & 3.32043 \\
0.070711 & -- & -- & -- & -- & -- & -- & 3.33872 \\
\addlinespace[1.5pt]
-- & 0.96464 & -- & -- & -- & -- & -- & 3.31200 \\
-- & 0.92929 & -- & -- & -- & -- & -- & 3.31255 \\
\addlinespace[1.5pt]
-- & -- & 0.05 & -- & -- & -- & -- & 3.31761 \\
-- & -- & 0.15 & -- & -- & -- & -- & 3.31150 \\
\addlinespace[1.5pt]
-- & -- & -- & 0.49497 & -- & -- & -- & 3.30991 \\
-- & -- & -- & 0.7 & -- & -- & -- & 3.30942 \\
-- & -- & -- & 1.4 & -- & -- & -- & 3.31786 \\
\addlinespace[1.5pt]
-- & -- & -- & -- & 0.9 & -- & -- & 3.31144 \\
-- & -- & -- & -- & 0.85858 & -- & -- & 3.30962 \\
-- & -- & -- & -- & 0.71716 & -- & -- & 3.31435 \\
\addlinespace[1.5pt]
-- & -- & -- & -- & -- & 0.97562 & -- & 3.31019 \\
-- & -- & -- & -- & -- & 0.96553 & -- & 3.30994 \\
-- & -- & -- & -- & -- & 0.95125 & -- & 3.31033 \\
-- & -- & -- & -- & -- & 0.9025 & -- & 3.31181 \\
\addlinespace[1.5pt]
-- & -- & -- & -- & -- & -- & 0.4 & 3.31609 \\
-- & -- & -- & -- & -- & -- & 0.6 & 3.30963 \\
-- & -- & -- & -- & -- & -- & 1 & 3.31617 \\
\end{longtable}
\endgroup

\medskip

\Needspace{8\baselineskip}

\begingroup
\footnotesize
\renewcommand{\arraystretch}{1.08}
\setlength{\tabcolsep}{2.5pt}
\setlength{\LTpre}{0pt}
\setlength{\LTpost}{0pt}
\setlength{\LTleft}{\fill}
\setlength{\LTright}{\fill}
\setlength{\LTcapwidth}{\textwidth}
\begin{longtable}{rrrrrrrr}
\caption{Hyperparameter ablation for MuonW at batch size $B=2M$.}
\label{tab:llm_cd_muonw_2m} \\
\toprule
$\eta_{\rm M}$ & $\mu$ & $\lambda_{\rm M}$ & $\eta_{\rm Emb.}$ & $\beta_1$ & $\beta_2$ & $f_{\rm cd}$ & $L_{\rm val}$ \\
\midrule
\endfirsthead
\multicolumn{8}{l}{\itshape MuonW, $B=2M$} \\
\toprule
$\eta_{\rm M}$ & $\mu$ & $\lambda_{\rm M}$ & $\eta_{\rm Emb.}$ & $\beta_1$ & $\beta_2$ & $f_{\rm cd}$ & $L_{\rm val}$ \\
\midrule
\endhead
\midrule
\endfoot
\bottomrule
\endlastfoot
0.025 & 0.95 & 0.1 & 0.7 & 0.85858 & 0.90725 & 0.4 & 3.36617 \\
\midrule
0.0125 & -- & -- & -- & -- & -- & -- & 3.37566 \\
0.017678 & -- & -- & -- & -- & -- & -- & 3.36828 \\
0.035355 & -- & -- & -- & -- & -- & -- & 3.37356 \\
0.05 & -- & -- & -- & -- & -- & -- & 3.38618 \\
\addlinespace[1.5pt]
-- & 0.975 & -- & -- & -- & -- & -- & 3.37805 \\
-- & 0.96464 & -- & -- & -- & -- & -- & 3.36938 \\
-- & 0.92929 & -- & -- & -- & -- & -- & 3.36714 \\
\end{longtable}
\endgroup

\medskip

\Needspace{10\baselineskip}

\medskip

\textbf{KL-SOAP-W}

\smallskip

\begingroup
\footnotesize
\renewcommand{\arraystretch}{1.08}
\setlength{\tabcolsep}{2.5pt}
\setlength{\LTpre}{0pt}
\setlength{\LTpost}{0pt}
\setlength{\LTleft}{\fill}
\setlength{\LTright}{\fill}
\setlength{\LTcapwidth}{\textwidth}
\begin{longtable}{rrrrrrrrrr}
\caption{Hyperparameter ablation for KL-SOAP-W at batch size $B=128K$.}
\label{tab:llm_cd_klsoapw_128k} \\
\toprule
$\eta_{\rm M}$ & $\beta_{1}^{\rm M}$ & $\beta_{2}^{\rm M}$ & $\beta_{\rm sh}$ & $\lambda_{\rm M}$ & $\eta_{\rm Emb.}$ & $\beta_1$ & $\beta_2$ & $f_{\rm cd}$ & $L_{\rm val}$ \\
\midrule
\endfirsthead
\multicolumn{10}{l}{\itshape KL-SOAP-W, $B=128K$} \\
\toprule
$\eta_{\rm M}$ & $\beta_{1}^{\rm M}$ & $\beta_{2}^{\rm M}$ & $\beta_{\rm sh}$ & $\lambda_{\rm M}$ & $\eta_{\rm Emb.}$ & $\beta_1$ & $\beta_2$ & $f_{\rm cd}$ & $L_{\rm val}$ \\
\midrule
\endhead
\midrule
\endfoot
\bottomrule
\endlastfoot
0.0045 & 0.96464 & 0.974 & 0.94801 & 0.1 & 0.15 & 0.43431 & 0.9955 & 0.4 & 3.25388 \\
\midrule
0.00225 & -- & -- & -- & -- & -- & -- & -- & -- & 3.26253 \\
0.003182 & -- & -- & -- & -- & -- & -- & -- & -- & 3.25685 \\
0.006364 & -- & -- & -- & -- & -- & -- & -- & -- & 3.25702 \\
0.009 & -- & -- & -- & -- & -- & -- & -- & -- & 3.26629 \\
\addlinespace[1.5pt]
-- & 0.975 & -- & -- & -- & -- & -- & -- & -- & 3.25474 \\
-- & 0.95 & -- & -- & -- & -- & -- & -- & -- & 3.25607 \\
-- & 0.92929 & -- & -- & -- & -- & -- & -- & -- & 3.25725 \\
\end{longtable}
\endgroup

\medskip

\Needspace{8\baselineskip}

\begingroup
\footnotesize
\renewcommand{\arraystretch}{1.08}
\setlength{\tabcolsep}{2.5pt}
\setlength{\LTpre}{0pt}
\setlength{\LTpost}{0pt}
\setlength{\LTleft}{\fill}
\setlength{\LTright}{\fill}
\setlength{\LTcapwidth}{\textwidth}
\begin{longtable}{rrrrrrrrrr}
\caption{Hyperparameter ablation for KL-SOAP-W at batch size $B=512K$.}
\label{tab:llm_cd_klsoapw_512k} \\
\toprule
$\eta_{\rm M}$ & $\beta_{1}^{\rm M}$ & $\beta_{2}^{\rm M}$ & $\beta_{\rm sh}$ & $\lambda_{\rm M}$ & $\eta_{\rm Emb.}$ & $\beta_1$ & $\beta_2$ & $f_{\rm cd}$ & $L_{\rm val}$ \\
\midrule
\endfirsthead
\multicolumn{10}{l}{\itshape KL-SOAP-W, $B=512K$} \\
\toprule
$\eta_{\rm M}$ & $\beta_{1}^{\rm M}$ & $\beta_{2}^{\rm M}$ & $\beta_{\rm sh}$ & $\lambda_{\rm M}$ & $\eta_{\rm Emb.}$ & $\beta_1$ & $\beta_2$ & $f_{\rm cd}$ & $L_{\rm val}$ \\
\midrule
\endhead
\midrule
\endfoot
\bottomrule
\endlastfoot
0.012728 & 0.9 & 0.85858 & 0.85858 & 0.15 & 0.84853 & 0.8 & 0.975 & 0.4 & 3.27594 \\
\midrule
0.006364 & -- & -- & -- & -- & -- & -- & -- & -- & 3.28188 \\
0.009 & -- & -- & -- & -- & -- & -- & -- & -- & 3.27597 \\
0.018 & -- & -- & -- & -- & -- & -- & -- & -- & 3.28161 \\
0.025456 & -- & -- & -- & -- & -- & -- & -- & -- & 3.28798 \\
\addlinespace[1.5pt]
-- & 0.96464 & -- & -- & -- & -- & -- & -- & -- & 3.96353 \\
-- & 0.95 & -- & -- & -- & -- & -- & -- & -- & 3.29349 \\
-- & 0.92929 & -- & -- & -- & -- & -- & -- & -- & 3.27927 \\
\end{longtable}
\endgroup

\medskip

\Needspace{8\baselineskip}

\begingroup
\footnotesize
\renewcommand{\arraystretch}{1.08}
\setlength{\tabcolsep}{2.5pt}
\setlength{\LTpre}{0pt}
\setlength{\LTpost}{0pt}
\setlength{\LTleft}{\fill}
\setlength{\LTright}{\fill}
\setlength{\LTcapwidth}{\textwidth}
\begin{longtable}{rrrrrrrrrr}
\caption{Hyperparameter ablation for KL-SOAP-W at batch size $B=1M$.}
\label{tab:llm_cd_klsoapw_1m} \\
\toprule
$\eta_{\rm M}$ & $\beta_{1}^{\rm M}$ & $\beta_{2}^{\rm M}$ & $\beta_{\rm sh}$ & $\lambda_{\rm M}$ & $\eta_{\rm Emb.}$ & $\beta_1$ & $\beta_2$ & $f_{\rm cd}$ & $L_{\rm val}$ \\
\midrule
\endfirsthead
\multicolumn{10}{l}{\itshape KL-SOAP-W, $B=1M$} \\
\toprule
$\eta_{\rm M}$ & $\beta_{1}^{\rm M}$ & $\beta_{2}^{\rm M}$ & $\beta_{\rm sh}$ & $\lambda_{\rm M}$ & $\eta_{\rm Emb.}$ & $\beta_1$ & $\beta_2$ & $f_{\rm cd}$ & $L_{\rm val}$ \\
\midrule
\endhead
\midrule
\endfoot
\bottomrule
\endlastfoot
0.006364 & 0.92929 & 0.86565 & 0.86565 & 0.6 & 0.6 & 0.8 & 0.95125 & 0.6 & 3.29629 \\
\midrule
0.003182 & -- & -- & -- & -- & -- & -- & -- & -- & 3.30629 \\
0.0045 & -- & -- & -- & -- & -- & -- & -- & -- & 3.30860 \\
0.009 & -- & -- & -- & -- & -- & -- & -- & -- & 3.30143 \\
0.012728 & -- & -- & -- & -- & -- & -- & -- & -- & 3.30874 \\
\addlinespace[1.5pt]
-- & 0.96464 & -- & -- & -- & -- & -- & -- & -- & 3.32143 \\
-- & 0.95 & -- & -- & -- & -- & -- & -- & -- & 3.31439 \\
-- & 0.9 & -- & -- & -- & -- & -- & -- & -- & 3.29697 \\
\end{longtable}%
\endgroup

\clearpage
\endgroup
\section{Directional Branching in the Language Model}
\label{app:lm_directional_protocol}

This appendix describes how the branches of \Cref{sec:lm_directional} are constructed and evaluated, and lists every endpoint behind \Cref{fig:lm_directional}.
\Cref{app:lm_directional_local_nqm} establishes a local noisy-quadratic correspondence for the polar-Jacobian choice of diagnostic geometry.

\textbf{Base run and anchors.}
The base run is the tuned $128\mathrm K$ MuonW--AdamW configuration of \Cref{app:llm_experiment_details}, trained for $13{,}000$ steps of $128\mathrm K$ tokens.
Its matrix learning rate is constant until step $2{,}601$ and then decays linearly to zero at step $13{,}000$.
Anchors are placed every $1{,}000$ steps from $1{,}000$ to $12{,}000$.
A branch restores the full optimizer state at its anchor, sees the same data in the same order as the base run, and follows the base run's learning-rate schedule by global step.
Every branch trains for $1{,}024$ base steps ($134$M tokens).

\textbf{Gauss--Newton matrix.}
Let $\w$ denote the vector of hidden-matrix parameters, with all other parameters fixed at the anchor.
For a batch of $m$ prediction tokens, let $z_i(\w)$ be the vocabulary logits for token $i$, $p_i = \operatorname{softmax}(z_i(\w))$ its predicted probability vector, and $J_i = \partial z_i / \partial\w$ the logit Jacobian.
For the mean token cross-entropy loss, the generalized Gauss--Newton matrix is
\[
  G = \frac{1}{m}\sum_{i=1}^{m} J_i^\top
  \bigl(\operatorname{diag}(p_i)-p_i p_i^\top\bigr)J_i,
\]
where $\operatorname{diag}(p_i)-p_i p_i^\top$ is the Hessian of the token cross-entropy with respect to the logits, and all quantities are evaluated at the anchor.
This positive-semidefinite matrix retains the loss curvature through the logit Jacobian and omits the terms involving second derivatives of the logits with respect to the parameters.
We refer to it as the Gauss--Newton matrix below.

\textbf{Held subspace.}
Let $\Dir$ be Muon's ideal direction map on the flattened hidden matrices, including the fixed shape factors of \Cref{sec:convex}.
For matrix block $\ell$ with compact singular value decomposition $M_\ell=U_\ell\Sigma_\ell V_\ell^\top$, this map returns $a_{\mathrm{shape},\ell}U_\ell V_\ell^\top$.
Let $q$ denote the flattened momentum input to this map, including the Nesterov correction of \Cref{eq:lionk_stochastic_corrected_m}.
For anchor parameters $\w_\star$ and momentum buffer $\m_\star$, let $q_\star$ be the reference input formed using the mean gradient $\g(\w_\star)$:
\[
  q_\star=\omega\mom\m_\star+(1-\omega\mom)\g(\w_\star).
\]
Assume that all matrix blocks of this reference input have full rank.
On the full-rank domain, define the local preconditioner $P(q)$ as the derivative of the direction map with respect to this input, and freeze it at the anchor:
\begin{equation}
  P(q)=D\Dir(q),\qquad P=P(q_\star).
  \label{eq:directional_preconditioner}
\end{equation}
Thus $P\,\delta q$ is the first-order change in the update direction caused by an input perturbation $\delta q$ at the anchor.
For full-rank matrix inputs, $P$ is symmetric positive semidefinite; \Cref{app:lm_directional_local_nqm} proves this property and the resulting local noisy-quadratic correspondence.
Empirically, we evaluate $P$ at the anchor's momentum buffer $\m_\star$.
We also damp every singular value $\sigma$ of the input to $\sigma+10^{-3}$ in the closed form of $D\Dir$; the damped operator remains symmetric positive semidefinite.
At each anchor we compute the top $768$ positive-eigenvalue modes of the preconditioned Gauss--Newton matrix $P^{1/2} G P^{1/2}$ restricted to the hidden matrices.
The eigenvectors are computed iteratively from Gauss--Newton--vector products on $128\mathrm K$-token batches.
An eigenvector $e$ with positive eigenvalue $h$ corresponds to the parameter-space direction $P^{1/2}e$, since $PG(P^{1/2}e)=hP^{1/2}e$.
This mapping is the coordinate transformation of the local noisy-quadratic model, whose momentum dynamics are given in \Cref{app:lm_directional_local_nqm}.
The rank-$r$ construction requires at least $r$ positive eigenvalues.
We map the top $r$ corresponding eigenvectors this way and orthonormalize the images with two Cholesky--QR passes, giving the basis $V_r \in \mathbb{R}^{N \times r}$ of the top-$r$ held subspace, where $N$ is the number of hidden-matrix parameters and all hidden matrices are flattened into one vector.
The random control instead uses an orthonormal basis of $768$ Gaussian directions, drawn once per anchor.

\textbf{Branch update.}
Each branch forms two Muon updates from the same stream of gradients.
The \emph{base-clock} update takes one step per $128\mathrm K$-token base step at the scheduled learning rate.
The \emph{large-batch} update accumulates sixteen base gradients and takes one $2\mathrm M$ step with $\sqrt{16}$ times the mean scheduled learning rate over those sixteen steps; the momentum coefficient per update is unchanged.
The applied update is the base-clock update projected onto the held subspace plus the large-batch update projected onto its orthogonal complement.
Matrix weight decay is applied once per base step to the current parameters, independently of the two projected updates, so the decay per token matches the base run.
The fully scaled branch applies the large-batch update everywhere, and the control branch applies the base-clock update everywhere, which recovers the unmodified optimizer.
In every branch the auxiliary AdamW parameters stay on the base clock at $128\mathrm K$.
\Cref{alg:lm_branch} summarizes one branch.

\begin{algorithm}[H]
\caption{Directional branching}
\label{alg:lm_branch}
\begin{tcolorbox}[optimizerblock]
\small
\begin{enumerate}[label=\arabic*.,leftmargin=1.6em,itemsep=1pt,topsep=3pt]
\item \emph{Basis.} Freeze $P=D\Dir(q_\star)$, the derivative of Muon's ideal direction map at the anchor momentum input, as in \Cref{eq:directional_preconditioner}. In practice $P$ is evaluated at $\m_\star$ with singular values damped by $10^{-3}$. Compute the top-$r$ positive-eigenvalue modes $e_1,\dots,e_r$ of $P^{1/2} G P^{1/2}$ over the flattened hidden matrices and set $V = \operatorname{orth}(P^{1/2}e_1,\dots,P^{1/2}e_r)$.
      For the random control, $V$ is an orthonormal basis of $768$ Gaussian directions; the fully scaled branch uses $V = 0$ and the control branch $V = I$.
\item \emph{State.} Restore the anchor's parameters and optimizer state, and make two copies of the Muon state: a base-clock copy and a large-batch copy.
\item For base steps $t = 1, \dots, 1{,}024$, with the base run's batch $t$, learning rate $\eta_t$, and matrix weight-decay coefficient $\lambda_t$:
  \begin{enumerate}[label=(\alph*),leftmargin=1.4em,itemsep=0pt,topsep=1pt]
  \item Compute the $128\mathrm K$ gradient $g_t$ of the hidden matrices at the current parameters.
  \item Base clock: $\Delta_{128\mathrm K} = \eta_t\, \mathrm{Muon}_{\text{base}}(g_t)$.
  \item Large batch: add $g_t$ to an accumulator; when $t$ is a multiple of $16$, set $\Delta_{2\mathrm M} = 4\bar\eta\, \mathrm{Muon}_{\text{large}}(\bar g)$ with $\bar g$ the mean of the sixteen accumulated gradients and $\bar\eta$ their mean learning rate, and clear the accumulator; otherwise $\Delta_{2\mathrm M} = 0$.
  \item Apply $W \leftarrow (1-\eta_t\lambda_t)W - V V^\top \Delta_{128\mathrm K} - (I - V V^\top)\Delta_{2\mathrm M}$.
  \item Update the auxiliary AdamW parameters on the base clock, as in the base run.
  \end{enumerate}
\item Evaluate the validation loss on $10.5$M held-out tokens; the branch penalty is this loss minus the control branch's.
\end{enumerate}
\end{tcolorbox}
\end{algorithm}

\textbf{Evaluation.}
The branch penalty is the validation loss on $10.5$M held-out tokens at the end of a branch, minus the same quantity for the control branch from the same anchor.
The loss curves in \Cref{fig:lm_directional} are instead evaluated on $1$M held-out tokens every $32$ base steps.
To estimate run-to-run noise, we ran one branch twice (anchor $3{,}000$, $512\mathrm K$, top-$768$ held); the two penalties were $+18.9$ and $+13.5 \times 10^{-3}$ nats.
\Cref{tab:lm_directional_numbers} lists all endpoints, and \Cref{fig:lm_directional_curves} shows the trajectory of every branch.

\begin{table}[h]
  \centering
  \small
  \setlength{\tabcolsep}{5pt}
  \begin{tabular}{lrrrrrrr}
    \toprule
    anchor & fully scaled & top-$16$ & top-$64$ & top-$128$ & top-$256$ & top-$768$ & random-$768$ \\
    \midrule
    $1{,}000$ & $+68.7$ & $+58.9$ & $+55.9$ & $+53.7$ & $+42.2$ & $+41.6$ & $+70.4$ \\
    $2{,}000$ & $+50.9$ & $+52.1$ & $+45.8$ & $+39.4$ & $+38.6$ & $+30.5$ & $+50.8$ \\
    $3{,}000$ & $+36.1$ & $+34.3$ & $+27.6$ & $+28.6$ & $+21.9$ & $+18.3$ & $+36.0$ \\
    $4{,}000$ & $+40.3$ & $+29.3$ & $+27.3$ & $+25.5$ & $+23.2$ & $+19.0$ & $+39.7$ \\
    $5{,}000$ & $+28.8$ & $+21.2$ & $+18.4$ & $+17.0$ & $+14.8$ & $+11.7$ & $+28.6$ \\
    $6{,}000$ & $+25.1$ & $+21.9$ & $+16.7$ & $+14.1$ & $+12.5$ & $+10.2$ & $+24.6$ \\
    $7{,}000$ & $+25.4$ & $+20.8$ & $+16.7$ & -- & $+14.0$ & $+12.5$ & $+25.5$ \\
    $8{,}000$ & $+20.6$ & $+18.1$ & $+15.6$ & $+14.2$ & $+13.0$ & $+11.7$ & $+20.6$ \\
    $9{,}000$ & $+16.6$ & $+13.5$ & $+12.5$ & -- & $+10.9$ & $+10.8$ & $+16.6$ \\
    $10{,}000$ & $+12.6$ & $+11.5$ & $+11.3$ & $+11.3$ & $+11.2$ & $+11.1$ & $+12.6$ \\
    $11{,}000$ & $+14.3$ & $+14.3$ & $+14.1$ & -- & $+14.1$ & $+14.1$ & $+14.3$ \\
    $12{,}000$ & $+12.9$ & $+12.9$ & $+12.9$ & $+12.9$ & $+13.0$ & $+12.9$ & $+12.9$ \\
    \bottomrule
  \end{tabular}
  \caption{\textbf{Branch penalty at $2\mathrm M$ for every anchor and held subspace}, in units of $10^{-3}$ nats: the validation loss at the end of a $1{,}024$-step branch minus that of the $128\mathrm K$ control branch from the same anchor. A dash marks a branch that was not run.}
  \label{tab:lm_directional_numbers}
\end{table}

\begin{figure}[h]
  \centering
  \includegraphics[width=\textwidth]{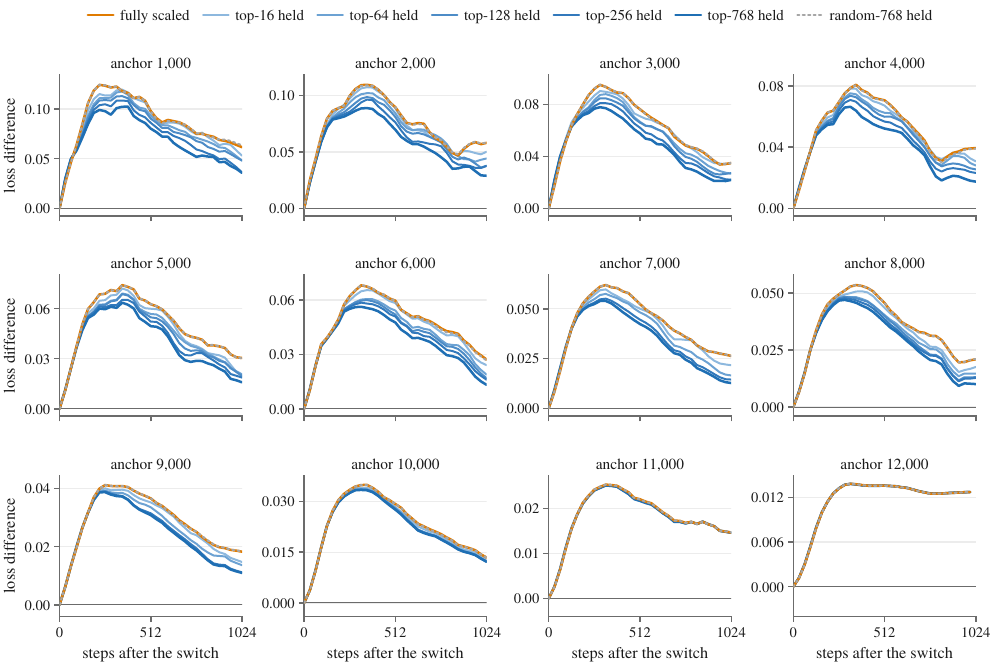}
  \caption{\textbf{Branch-penalty trajectories at every anchor.} Held-out loss of each $2\mathrm M$ branch minus that of the $128\mathrm K$ control branch from the same anchor, evaluated every $32$ base steps on $1$M held-out tokens and smoothed over three evaluations, for the fully scaled, top-$r$ held ($r = 16, 64, 128, 256, 768$), and random-$768$ held (dotted) branches. Each panel has its own vertical range. The endpoints in \Cref{tab:lm_directional_numbers} are measured on $10.5$M validation tokens instead. Top-$128$ was not run at anchors $7{,}000$, $9{,}000$, and $11{,}000$.}
  \label{fig:lm_directional_curves}
\end{figure}

\FloatBarrier
\subsection{A local noisy-quadratic model for Muon}
\label{app:lm_directional_local_nqm}

\textbf{Setup.}
We use the objective $L$, mean gradient $\g$, and minibatch oracle $\ghat_{\batch}$ of \Cref{eq:sde_gradient_model}, and apply the optimizer recursion in \Cref{eq:ema_momentum_update,eq:lionk_stochastic_corrected_m,eq:weight_update}.
The parameter vector $\w$ contains the flattened hidden matrices, with auxiliary parameters fixed, and steps are indexed from the anchor by $t=0,1,\ldots$.
We allow schedules $\lr_t,\wdc_t\geq0$ and retain $\mom\in[0,1)$ and $\omega\in[0,1]$.
Specialize the direction map $\Dir$ to the ideal blockwise polar map of \Cref{sec:convex}, including its fixed positive shape factors $a_{\mathrm{shape},\ell}$ for matrix block $\ell$.
Use the local preconditioner $P(q)=D\Dir(q)$ defined in \Cref{eq:directional_preconditioner}.
The results below are local to full-rank polar inputs and do not assume that the ideal polar map is smooth everywhere.
All vector norms are Euclidean, matrix blocks have the equivalent Frobenius norm, and operator norms are induced norms.

Use the additive oracle in \Cref{eq:sde_gradient_model}, with independent scaled noises $z_t$ of mean zero and fixed single-example covariance $\Sig_0\succeq0$.
Thus the minibatch noise is $\batch^{-1/2}z_t$.
The initial state and frozen operators below are fixed independently of these noises.
Write $q_t$ for the Nesterov input denoted $\widetilde\m_{t+1}$ in \Cref{eq:lionk_stochastic_corrected_m}.
We reserve tildes below for the preconditioned coordinates of \Cref{app:nqm_crossover_proof}.
Let $(\bar\w_t,\bar\m_t)$ be the noiseless reference trajectory with the same initial state and schedules, and let $\bar q_t$ be its Nesterov input, so $\bar q_0=q_\star$ from the subspace construction.
For each of $\w,\m,q$, use $\delta$ to denote the deviation from its reference value.
In particular, $\delta\w_0=\delta\m_0=0$.

Assume that all matrix blocks of the anchor input $\bar q_0$ have full rank.
Freeze the Gauss--Newton matrix $G_0=G$ defined above and the polar Jacobian $P_0=P(\bar q_0)$.
The frozen linear model has deviations $(\delta\w_t^{\rm lin},\delta\m_t^{\rm lin})$, zero initial state, and the same noises $z_t$.
Its linearized gradient and Nesterov-input deviations, $\delta g_t^{\rm lin}$ and $\delta q_t^{\rm lin}$, are specified by
\begin{equation}
\begin{aligned}
 \delta g_t^{\rm lin}&=G_0\delta\w_t^{\rm lin}+\batch^{-1/2}z_t,\\
 \delta\m_{t+1}^{\rm lin}&=\mom\delta\m_t^{\rm lin}+(1-\mom)\delta g_t^{\rm lin},\\
 \delta q_t^{\rm lin}&=\omega\delta\m_{t+1}^{\rm lin}+(1-\omega)\delta g_t^{\rm lin},\\
 \delta\w_{t+1}^{\rm lin}&=(1-\lr_t\wdc_t)\delta\w_t^{\rm lin}-\lr_tP_0\delta q_t^{\rm lin}.
\end{aligned}
\label{eq:local_muon_deviations}
\end{equation}
This model approximates centered trajectory deviations.
\begin{theorem}[Local noisy-quadratic correspondence for ideal Muon]
\label{thm:local_muon_nqm}
Under the setup above, $P_0$ is positive semidefinite.
Writing $P_0^{\dagger/2}$ for the positive square root of its Moore--Penrose pseudoinverse, the coordinates
\[
 \widetilde\w_t=P_0^{\dagger/2}\delta\w_t^{\rm lin},\qquad
 \widetilde\m_t=P_0^{1/2}\delta\m_t^{\rm lin}
\]
transform the frozen model in \Cref{eq:local_muon_deviations} exactly into momentum/Nesterov SGD on $\operatorname{range}P_0$, with the same schedules and decoupled weight decay, quadratic curvature $\widetilde H$, and minibatch noise covariance $\widetilde\Sig/\batch$, where
\[
 \widetilde H=P_0^{1/2}G_0P_0^{1/2},\qquad
 \widetilde\Sig=P_0^{1/2}\Sig_0P_0^{1/2}.
\]
\end{theorem}

\begin{proof}
We verify the geometry of the ideal update Jacobian and then transform the frozen dynamics.
For a full-rank matrix block with compact singular value decomposition $M=U\Sigma V^\top$, the differential of its nuclear norm in a perturbation $E$ is $\operatorname{tr}(U^\top EV)=\langle UV^\top,E\rangle_F$.
Thus its nuclear-norm gradient is its polar factor.
The ideal direction map $\Dir$ is therefore the gradient of the convex function given by the sum of the block nuclear norms weighted by $a_{\mathrm{shape},\ell}$.
This function is smooth on the full-rank neighborhood, so $P_0$ is symmetric positive semidefinite.

Every parameter increment in \Cref{eq:local_muon_deviations} is a scalar shrinkage of the current deviation plus a vector in $\operatorname{range}P_0$.
Zero initialization implies that $\delta\w_t^{\rm lin}$ stays in this range and equals $P_0^{1/2}\widetilde\w_t$.
The momentum buffer averages unpreconditioned gradients, so its coordinate transformation multiplies by $P_0^{1/2}$.
Define the transformed scaled noise by $\widetilde z_t=P_0^{1/2}z_t$.
Multiply the linear momentum recursion by $P_0^{1/2}$ and the parameter recursion by $P_0^{\dagger/2}$.
The identity $P_0^{\dagger/2}P_0=P_0^{1/2}$ then gives
\begin{equation}
\begin{aligned}
 \widetilde\m_{t+1}
 &=\mom\widetilde\m_t+(1-\mom)
       (\widetilde H\widetilde\w_t+\batch^{-1/2}\widetilde z_t),\\
 \widetilde\w_{t+1}
 &=(1-\lr_t\wdc_t)\widetilde\w_t
   -\lr_t\bigl[\omega\widetilde\m_{t+1}
       +(1-\omega)(\widetilde H\widetilde\w_t+\batch^{-1/2}\widetilde z_t)\bigr].
\end{aligned}
\label{eq:local_muon_nqm}
\end{equation}
These operations establish the claimed coordinate correspondence, which is invertible for parameter deviations in $\operatorname{range}P_0$.
The quadratic objective is $\widetilde\w^\top\widetilde H\widetilde\w/2$.
Momentum components in $\ker P_0$ do not affect the frozen parameter updates and are omitted by this reduction.
If $P_0=0$, the parameter deviations stay zero and the reduced model has dimension zero.

A fixed linear transformation preserves the independence and zero means of $z_t$, and its covariance is $\operatorname{Cov}(\widetilde z_t)=P_0^{1/2}\Sig_0P_0^{1/2}$.
This covariance identity holds for the unconditioned frozen process.
\end{proof}

\begin{remark}[Curvature modes and CNR]
For a unit eigenvector $e_i$ of $\widetilde H$ with positive eigenvalue $h_i$, the noise variance is $c_i=e_i^\top\widetilde\Sig e_i$ and the parameter-space direction is $P_0^{1/2}e_i$.
The CNR convention of \Cref{def:nqm-cnr} applies with these $h_i,c_i$.
If the transformed noises are Gaussian and $\widetilde H,\widetilde\Sig$ are simultaneously diagonalizable on $\operatorname{range}P_0$, their common eigenbasis gives the independent-coordinate NQM of \Cref{def:nqm} on its positive-curvature modes.
Without these conditions, the coordinate noises can be correlated.
With $P=P_0$ at the anchor, this is the eigenvector-to-parameter construction in \Cref{alg:lm_branch}.
Large $h_i$ alone does not imply large CNR, since $c_i$ also enters its definition.
\end{remark}

\textbf{Approximation error.}
Let $H_0=\nabla^2L(\bar\w_0)$ be the anchor Hessian.
Define the Gauss--Newton discrepancy and the two reference-trajectory drifts by
\[
 \varepsilon_G=\|H_0-G_0\|,\qquad
 d_{H,t}=\|\nabla^2L(\bar\w_t)-H_0\|,\qquad
 d_{P,t}=\|P(\bar q_t)-P_0\|.
\]
Finite-batch estimation error in $G_0$ is included in $\varepsilon_G$.
Assume that the Hessian is Lipschitz with constant $L_H$ on the parameter segments $[\bar\w_t,\w_t]$.
Assume that every matrix block remains full rank on the input segments $[\bar q_t,q_t]$, where $P$ is Lipschitz with constant $L_P$.
These assumptions are local conditions on a noise realization over the $S$ steps being compared.
They do not assert that unbounded noises stay in such a neighborhood.
The noise distribution of the frozen linear model is not conditioned on this local event.

Let the full-state error be $\mathcal E_t=(\delta\w_t-\delta\w_t^{\rm lin},\delta\m_t-\delta\m_t^{\rm lin})$, let $F_t$ be the frozen state-transition matrix, and define the one-step error bound $\epsilon_t$ by
\[
 F_t=
 \begin{pmatrix}
 (1-\lr_t\wdc_t)I-\lr_t(1-\omega\mom)P_0G_0 & -\lr_t\omega\mom P_0\\
 (1-\mom)G_0 & \mom I
 \end{pmatrix},
\]
\[
\begin{aligned}
 \epsilon_t={}&\bigl[\lr_t(1-\omega\mom)\|P_0\|+1-\mom\bigr]
 \left[(\varepsilon_G+d_{H,t})\|\delta\w_t\|+\frac{L_H}{2}\|\delta\w_t\|^2\right]\\
 &+\lr_t\left[d_{P,t}\|\delta q_t\|+\frac{L_P}{2}\|\delta q_t\|^2\right].
\end{aligned}
\]
\begin{theorem}[Finite-horizon approximation error for ideal Muon]
\label{thm:local_muon_error}
For every integer $S\geq1$ and every noise realization satisfying the local segment assumptions for $0\leq t<S$, the original Muon deviations and the frozen model satisfy
\begin{equation}
 \|\mathcal E_S\|
 \leq\sum_{t=0}^{S-1}
 \left(\prod_{j=t+1}^{S-1}\|F_j\|\right)\epsilon_t,
 \label{eq:local_muon_accumulated_error}
\end{equation}
where an empty product is one.
\end{theorem}

\begin{proof}
Let the gradient deviation be $\delta g_t=\ghat_{\batch}(\w_t)-\g(\bar\w_t)$.
Define the gradient and update remainders by
\begin{equation}
\begin{aligned}
 r^g_t&=\delta g_t-G_0\delta\w_t-\batch^{-1/2}z_t,\\
 r^u_t&=\Dir(q_t)-\Dir(\bar q_t)-P_0\delta q_t.
\end{aligned}
\label{eq:local_muon_remainders}
\end{equation}
For the error comparison, subtract the reference recursion from the noisy recursion.
The definitions of the remainders give the exact identities
\[
\begin{aligned}
 \delta g_t&=G_0\delta\w_t+\batch^{-1/2}z_t+r^g_t,\\
 \delta\m_{t+1}&=\mom\delta\m_t+(1-\mom)\delta g_t,\\
 \delta q_t&=\omega\delta\m_{t+1}+(1-\omega)\delta g_t,\\
 \delta\w_{t+1}&=(1-\lr_t\wdc_t)\delta\w_t-\lr_tP_0\delta q_t-\lr_t r^u_t.
\end{aligned}
\]
Reference subtraction cancels the baseline update $\Dir(\bar q_t)$.
This is essential because positive-scale invariance implies $P(q)q=0$ for the ideal polar map, whereas $\Dir(q)$ need not vanish.
The theorem therefore concerns perturbations, not the replacement of the absolute update by $P_0q_t$.

On each realization satisfying the local segment assumptions, the nonlinear remainders satisfy
\begin{equation}
\begin{aligned}
 \|r^g_t\|&\leq(\varepsilon_G+d_{H,t})\|\delta\w_t\|
                  +\frac{L_H}{2}\|\delta\w_t\|^2,\\
 \|r^u_t\|&\leq d_{P,t}\|\delta q_t\|+\frac{L_P}{2}\|\delta q_t\|^2.
\end{aligned}
\label{eq:local_muon_error_bounds}
\end{equation}
Integrating the Hessian along the parameter segment yields
\[
 \left\|\g(\w_t)-\g(\bar\w_t)-\nabla^2L(\bar\w_t)\delta\w_t\right\|
 \leq\frac{L_H}{2}\|\delta\w_t\|^2.
\]
Adding and subtracting $H_0\delta\w_t$ and $G_0\delta\w_t$ proves the gradient bound in \Cref{eq:local_muon_error_bounds}.
In particular, the Gauss--Newton discrepancy contributes a first-order error, not a quadratic Taylor remainder.
The Lipschitz assumption on $P$ likewise gives
\[
 \|\Dir(q_t)-\Dir(\bar q_t)-P(\bar q_t)\delta q_t\|
 \leq\frac{L_P}{2}\|\delta q_t\|^2.
\]
Adding and subtracting $P_0\delta q_t$ proves the update bound.

Let $b_t$ denote the remainder forcing in the full-state error recursion.
Since $\delta q_t=\omega\mom\delta\m_t+(1-\omega\mom)\delta g_t$, subtracting the frozen recursion gives $\mathcal E_{t+1}=F_t\mathcal E_t+b_t$, where
\[
 b_t=
 \begin{pmatrix}
 -\lr_t\bigl[(1-\omega\mom)P_0r^g_t+r^u_t\bigr]\\
 (1-\mom)r^g_t
 \end{pmatrix},\qquad
 \|b_t\|\leq\epsilon_t.
\]
Starting from $\mathcal E_0=0$, iterate this identity and take operator norms to obtain \Cref{eq:local_muon_accumulated_error}.
The bound makes no contraction assumption.
If the transition norms admit a uniform upper bound, their product is bounded by the corresponding power of that bound.
Small one-step errors alone do not imply a uniform-in-time approximation without further stability control.
The comparison uses the original state coordinates because nonlinear remainders need not remain in $\operatorname{range}P_0$.

For the finite Newton--Schulz variant, let $\Dir_{\rm NS}$ denote its direction map and let its map defect be $d=\Dir_{\rm NS}-\Dir$.
If both trajectories use $\Dir_{\rm NS}$, retain the ideal-map Jacobian $P_0=P(\bar q_0)$ along the resulting reference trajectory.
Replacing $\Dir$ by $\Dir+d$ in both trajectories adds exactly $d(q_t)-d(\bar q_t)$ to the update remainder.
The triangle inequality adds its norm to the update-remainder bound and $\lr_t\|d(q_t)-d(\bar q_t)\|$ to $\epsilon_t$.
A quantitative implementation guarantee requires a bound on that defect.
The quadratic model describes centered gradient responses, not the full loss difference from a nonstationary reference, which also has a linear term.
Moving auxiliary parameters would introduce additional gradient-coupling terms.
The local correspondence consequently motivates the diagnostic geometry without proving a branch-penalty prediction or transferring the SignSGD scaling exponents of \Cref{sec:nqm_signsgd_scaling} to Muon.
\end{proof}

\clearpage
\begingroup
\raggedbottom
\section{Batch Size and Optimizer Compute Efficiency}
\label{app:nanochat_optimizer_ce}

The optimizer crossovers in \Cref{sec:batch_size_comparison} show that batch
size changes which optimizer achieves the lowest loss at a fixed training
budget.  Here we measure the compute efficiency cost of increasing batch
size in a separate nanochat experiment with MuonW, KL-SOAP-W, ShampooW,
AdamW, and LionW.  We tune each optimizer at a reference batch of 256K tokens
and transfer its hyperparameters to larger batches.  We compare each result
with the compute needed to reach the same loss using a 256K batch, first
with the same optimizer and then with MuonW as a shared reference.

\subsection{Model, data, and training budgets}

This experiment uses the January 7
\href{https://github.com/karpathy/nanochat/commit/ccf4b7f9bf91a250aa398a0cecab270bcea56050}{nanochat snapshot}.
The decoder has depth 8, width 512, sequence length 2,048, and vocabulary
size 32,768.  The scaling parameter count, used to set the training token
budget, is 58,720,256.  Training uses FineWeb-Edu shards 0 through 238
\citep{penedo2024fineweb}.  Shard 239 is reserved for validation.  Validation
bits per byte (BPB) are computed on 10,485,760 tokens.  Every result uses seed 1.
This setting differs from the preceding MuonW scaling-rule sweep
(\Cref{app:muonw_batch_transfer}) and the independently retuned optimizer
comparison (\Cref{sec:batch_size_comparison}).  Those experiments use a
12-layer, 768-wide model on FineWeb with approximately 1.7B training tokens.

The reference global batch contains 262,144 tokens, which we abbreviate as
256K.  Each optimizer is tuned at a horizon of eight training tokens per
scaling parameter.  This horizon contains 469,762,048 tokens and takes 1,792
optimization steps at the reference batch.  We call each optimizer's
selected configuration at this batch and horizon its anchor.
The selected hyperparameter configurations are shown in
\Cref{tab:nanochat_anchor_recipes}.

We follow the notation of \Cref{app:s2-protocol,app:llm_experiment_details}.
The matrix learning rate and weight decay are $\eta_{\rm M}$ and
$\lambda_{\rm M}$, and auxiliary weight decay is $\lambda_{\rm A}$.
Muon momentum is $\mu$; the other matrix optimizers use
$\beta_1^{\rm M},\beta_2^{\rm M}$, while auxiliary Adam uses $\beta_1,\beta_2$.
KL-SOAP's preconditioner coefficient is $\beta_{\rm sh}$; ShampooW uses
$\beta_2^{\rm M}$ for its preconditioner.

The embeddings and output projection use AdamW for every optimizer.
Their learning rates $\eta_{\rm Emb.}$ and $\eta_{\rm Out.}$ have base values
0.3 and 0.004.  Both rates are multiplied by the
auxiliary multiplier in \Cref{tab:nanochat_anchor_recipes} and by the common
width factor $(512/768)^{-1/2}$.  Auxiliary AdamW uses
reference coefficients $\beta_1=0.8$ and $\beta_2=0.95$, $\epsilon=10^{-10}$,
and weight decay $\lambda_{\rm A}=0$.
The cooldown fraction is $f_{\rm cd}=0.4$; all learning rates decay linearly
to zero over this final fraction of each run.

\begin{table}[H]
\centering
\small
\setlength{\tabcolsep}{4pt}
\renewcommand{\arraystretch}{1.06}
\begin{tabular}{@{}lrrrrrr@{}}
\toprule
& \multicolumn{5}{c}{Matrix hyperparameter configuration} & \\
\cmidrule(lr){2-6}
Optimizer & $\eta_{\rm M}$ & $\lambda_{\rm M}$ & $\mu$ / $\beta_1^{\rm M}$ & $\beta_2^{\rm M}$ & $\beta_{\rm sh}$ & Aux. LR mult. \\
\midrule
MuonW & 0.02 & 0.0473004 & 0.95 & -- & -- & 1.41421 \\
KL-SOAP-W & 0.0127279 & 0.05 & 0.95 & 0.9 & 0.9 & 1.41421 \\
ShampooW & 0.005 & 0.2 & 0.9 & 0.9 & -- & 2 \\
AdamW & 0.00178381 & 0.6 & 0.9 & 0.95 & -- & 1 \\
LionW & 0.000336359 & 12 & 0.9 & 0.99 & -- & 1.41421 \\
\bottomrule
\end{tabular}
\caption{Final anchor hyperparameter configurations at a global batch of 256K tokens and eight tokens per scaling parameter. The first-moment column uses $\mu$ for MuonW and $\beta_1^{\rm M}$ for the other optimizers. The coefficient $\beta_{\rm sh}$ applies only to KL-SOAP-W; ShampooW uses $\beta_2^{\rm M}$ for its preconditioner. The auxiliary learning rate multiplier acts on both $\eta_{\rm Emb.}$ and $\eta_{\rm Out.}$.}
\label{tab:nanochat_anchor_recipes}
\end{table}

\subsection{Anchor search}

To select each optimizer's anchor, we first search a grid of matrix learning
rates and matrix weight decays at 256K tokens and eight training tokens per
scaling parameter.  We then
vary one hyperparameter at a time, extending a search range when its best
observed value lies at a boundary.  This single-seed search produced 178
unique completed observations with the same training implementation and
data.  \Cref{tab:nanochat_anchor_cd} reports the number of observations for
each optimizer and the improvement from the initial matrix grid to the
selected anchor.
These search-stage observations are distinct from the observations used
for the reference fits and batch comparisons in \Cref{tab:nanochat_ce_raw},
including their 256K, eight-token-per-parameter measurements.

\begin{table}[H]
\centering
\small
\renewcommand{\arraystretch}{1.06}
\begin{tabular}{@{}lrrrr@{}}
\toprule
Optimizer & Completed points & Initial BPB & Final BPB & Improvement \\
\midrule
MuonW & 47 & 1.0730 & 1.0711 & 0.0019 \\
KL-SOAP-W & 23 & 1.0715 & 1.0712 & 0.0003 \\
ShampooW & 40 & 1.0777 & 1.0750 & 0.0027 \\
AdamW & 25 & 1.1094 & 1.1080 & 0.0014 \\
LionW & 43 & 1.1510 & 1.1369 & 0.0141 \\
\bottomrule
\end{tabular}
\caption{Anchor search observations. Initial BPB is the selected result after the first matrix learning rate and weight decay grid. Final BPB is the selected result after the subsequent coordinate searches and boundary extensions.}
\label{tab:nanochat_anchor_cd}
\end{table}

\subsection{Transfer across the token horizon}

To estimate a reference loss--compute curve for each optimizer, we vary the
training token budget while holding the global batch at 256K tokens.
Let $N=58{,}720{,}256$ denote the parameter count used for this budget and let
$\tau$ denote the number of training tokens.  We evaluate
\[
    \frac{\tau}{N}\in\{1,2,4,8,12,16,24\}.
\]
The anchor has $\tau/N=8$.  For the relative horizon $r=(\tau/N)/8$, we multiply
every peak learning rate by $r^{-0.32}$ and matrix weight decay by
$r^{-0.16}$.

We obtain this scaling by combining two empirical scaling
results.  For peak learning rate $\eta$, \citet{bjorck2025scalinghorizons}
fit $\eta\propto \tau^{-0.32}$ at fixed model size in their large-model
experiments.  For batch size $B$, matrix peak learning rate $\eta_{\rm M}$,
and matrix weight decay $\lambda_{\rm M}$,
the nominal parameter-averaging timescale relative to the training horizon
is $\tau_{\text{Power Lines}}=B/(\eta_{\rm M}\lambda_{\rm M}\tau)$
(\Cref{eq:powerline_timescale}).  \citet{bergsma2025powerlines} fit this
timescale as proportional to $(\tau/N)^{-0.520}$ for AdamW.
Combining these two prescriptions at fixed $N$ and $B$ gives
\[
    \lambda_{\rm M}\propto
    \frac{1}{\eta_{\rm M}\tau\tau_{\text{Power Lines}}}
    \propto \tau^{0.32-1+0.520}=\tau^{-0.16}.
\]

We have not established that their combination is the best scaling rule
across token horizons or optimizers in our setting.  The reference curves
therefore measure performance under this prescribed transfer, without
independent retuning at each token budget.

\subsection{Scaling moving-average coefficients across batch sizes}

\textbf{KL-SOAP-W and ShampooW favor scaling only auxiliary AdamW $\beta_2$.}
For both optimizers, the best tested coefficient-transfer rule shared across
512K, 1M, and 2M scales auxiliary AdamW $\beta_2$ while keeping all other
moving-average coefficients fixed.  We select this common rule by its mean
validation-loss gap to the best tested rule at each batch.
The selected pattern matches our earlier MuonW sweep, which keeps Muon
momentum and auxiliary Adam $\beta_1$ fixed and scales auxiliary Adam
$\beta_2$ (\Cref{app:muonw_batch_transfer}).  We describe the comparison
and the resulting transfer protocol below.

\textbf{Batch comparison.}
We fix $\tau/N=8$ and compare global batches of 256K, 512K, 1M, and 2M
tokens.  Let $B$ denote the global batch in tokens, $B_0=262{,}144$ the
reference batch, and $\kappaBatch=B/B_0$ their ratio.  The update count
$S(B)=\lceil\tau/B\rceil$ is 1,792, 896, 448, and 224, respectively.
For every matrix or auxiliary peak learning rate
$\eta\in\{\eta_{\rm M},\eta_{\rm Emb.},\eta_{\rm Out.}\}$, we use
\[
    \eta(B)=\eta(B_0),
    \qquad
    \lambda_{\rm M}(B)=\kappaBatch\lambda_{\rm M}(B_0).
\]
Thus, peak learning rates remain fixed across batches and matrix weight
decay increases linearly.  Learning rates still decay to zero over the
final 40 percent of each run.

These two choices follow the common rule with the smallest mean batch-size
regret in the earlier MuonW sweep (\Cref{app:muonw_batch_transfer}).
That rule attains the grid minimum at 1M and 2M; the individual winners at
256K and 512K instead use square-root learning-rate scaling.
The fixed-learning-rate, linear-weight-decay choice also matches the
corresponding Power Lines prescription \citep{bergsma2025powerlines}.
We apply these two choices to all five optimizers in the different nanochat
setting described above, without independently retuning learning rates or
weight decay at each batch.

\textbf{Coefficient-transfer search.}
For a moving-average coefficient $q_0$ at the reference batch, let $q(B)$
denote its value at the target batch.  We test two choices,
\[
    q(B)=q_0
    \qquad\text{or}\qquad
    q(B)=q_0^{\kappaBatch},
\]
called fixed and EMA transfer, respectively.  EMA transfer preserves the
moving average's decay over the same number of processed tokens.
For KL-SOAP-W, we independently choose between these transfers for
$\beta_1^{\rm M}$, $\beta_2^{\rm M}$, $\beta_{\rm sh}$, $\beta_1$, and
$\beta_2$.  Its $2^5$ choices at three target batches give 96 runs.
ShampooW shares $\beta_2^{\rm M}$ between its matrix second moment and
preconditioner, so its $2^4$ choices give 48 runs.  The search therefore
contains 144 runs in total.

For each optimizer, a rule's batch-size (BSZ) regret is its validation BPB
minus the lowest BPB among that optimizer's tested rules at the same
batch.  We select one rule for all three target batches by minimizing
mean BSZ regret across 512K, 1M, and 2M.  Both optimizers select
\emph{auxiliary-only EMA transfer}, which applies EMA transfer only to
auxiliary AdamW $\beta_2$.  Its mean BSZ regret is 0.001233 BPB for
KL-SOAP-W and 0.000400 BPB for ShampooW.

\textbf{Transferring matrix coefficients as well gives higher loss at 2M.}
We compare the selected rule with an invariant-preserving rule that also
applies EMA transfer to $\beta_2^{\rm M}$ and, for KL-SOAP-W,
$\beta_{\rm sh}$.  Both rules keep $\beta_1^{\rm M}$ and auxiliary
$\beta_1$ fixed.  The invariant-preserving rule maintains the decay of
second-moment and preconditioner moving averages in token units.
Nevertheless, auxiliary-only EMA transfer gives lower loss for both
optimizers at 2M (\Cref{tab:nanochat_beta_comparison}).
\Cref{tab:nanochat_beta_leaderboard} reports all 48 coefficient-transfer
rules.

\begin{table}[H]
\centering
\small
\setlength{\tabcolsep}{5pt}
\renewcommand{\arraystretch}{1.06}
\begin{tabular}{@{}llrrr@{}}
\toprule
Optimizer & Transfer rule & 512K BPB & 1M BPB & 2M BPB \\
\midrule
KL-SOAP-W & \textbf{Auxiliary-only EMA} & \textbf{1.0812} & \textbf{1.1035} & \textbf{1.1614} \\
 & Invariant-preserving & 1.0867 & 1.1594 & 1.3735 \\
\addlinespace[2pt]
ShampooW & \textbf{Auxiliary-only EMA} & \textbf{1.0847} & \textbf{1.1039} & \textbf{1.1525} \\
 & Invariant-preserving & 1.0844 & 1.1044 & 1.6558 \\
\bottomrule
\end{tabular}
\caption{Coefficient transfer across batch sizes. Auxiliary-only EMA transfers auxiliary Adam $\beta_2$ and keeps all other coefficients fixed; it minimizes mean BSZ regret across the three target batches for each optimizer. The invariant-preserving rule also transfers matrix $\beta_2^{\rm M}$ and KL-SOAP-W's preconditioner coefficient $\beta_{\rm sh}$, preserving their decay in token units. Both rules keep $\beta_1^{\rm M}$ and $\beta_1$ fixed. EMA transfer means $q(B)=q_0^{\kappaBatch}$, where $\kappaBatch=B/B_0$, with $B_0=256\mathrm K$ tokens.}
\label{tab:nanochat_beta_comparison}
\end{table}

\textbf{Coefficient transfers for the remaining optimizers.}
For MuonW, we keep $\mu$ and auxiliary Adam $\beta_1$ fixed and apply EMA
transfer to auxiliary Adam $\beta_2$.  For AdamW and LionW, we keep
$\beta_1^{\rm M}$ and auxiliary $\beta_1$ fixed and apply EMA transfer to
both $\beta_2^{\rm M}$ and auxiliary $\beta_2$.
These are prescribed choices for the remaining optimizers; the exhaustive
coefficient-transfer search above covers KL-SOAP-W and ShampooW.

\subsection{Compute efficiency metric}

We measure training compute as the number of training tokens multiplied by
352,321,536 floating point operations per token.  This common compute proxy
excludes optimizer-specific update costs and does not measure wall-clock
time.  As in \Cref{app:llm_experiment_details}, $L_{\rm val}$ denotes
terminal validation loss, measured here in BPB.
Let $C$ denote compute and $L_{{\rm val},o}(C)$ the fitted validation loss
for optimizer $o$ at a batch of 256K.  With a fitted
asymptotic loss $E_o$, positive amplitude $A_o$, and positive exponent
$\alpha_o$, the reference curve is
\[
    L_{{\rm val},o}(C)=E_o+A_o\left(\frac{C}{10^{18}}\right)^{-\alpha_o}.
\]
The fitted values are reported in \Cref{tab:nanochat_ce_fits}.
Each fit uses that optimizer's seven reference observations, whose losses
decrease monotonically with compute.

\begin{table}[H]
\centering
\small
\renewcommand{\arraystretch}{1.06}
\begin{tabular}{@{}lrrrr@{}}
\toprule
Optimizer & Irreducible BPB & Amplitude & Exponent & RMSE \\
\midrule
MuonW & 0.983714 & 0.025145 & 0.672383 & 0.002624 \\
KL-SOAP-W & 0.992348 & 0.018555 & 0.805057 & 0.001106 \\
ShampooW & 1.011685 & 0.011108 & 0.907607 & 0.006164 \\
AdamW & 0.987890 & 0.030050 & 0.778845 & 0.005971 \\
LionW & 0.918750 & 0.095158 & 0.492865 & 0.013589 \\
\bottomrule
\end{tabular}
\caption{Parameters of the fitted compute loss curves.}
\label{tab:nanochat_ce_fits}
\end{table}

For a larger batch result with loss $L_{\rm val}$, equivalent compute is the compute
on the fitted reference curve that reaches the same loss.  We denote this
quantity by $C_{o,\mathrm{eq}}(L_{\rm val})$ and obtain it by inverting the fitted
curve for $L_{\rm val}>E_o$,
\[
    C_{o,\mathrm{eq}}(L_{\rm val})
    =10^{18}\left(\frac{A_o}{L_{\rm val}-E_o}\right)^{1/\alpha_o}.
\]
Let $C_{\mathrm{actual}}$ be the compute used by the observed run.
We define its compute efficiency $\mathrm{CE}_o$ as the ratio of equivalent
compute to actual compute, and its penalty $P_o$ as one minus this ratio,
\[
    \mathrm{CE}_o(L_{\rm val})
    =\frac{C_{o,\mathrm{eq}}(L_{\rm val})}{C_{\mathrm{actual}}},
    \qquad
    P_o(L_{\rm val})=1-\mathrm{CE}_o(L_{\rm val}).
\]
A penalty of zero means that the larger batch is as efficient as the 256K
reference at the same loss.  A positive penalty means that the larger batch
uses more compute to reach that loss.  For example, a penalty of 0.6 means
that the reference curve reaches the observed loss with 40 percent of the
run's compute.  In the reported comparisons, each optimizer's own 256K
anchor is assigned a penalty of zero by convention; the fitted curve need
not pass exactly through that observation.

\begin{figure}[H]
    \vspace{-1em}
    \centering
    \includegraphics[width=0.78\textwidth]{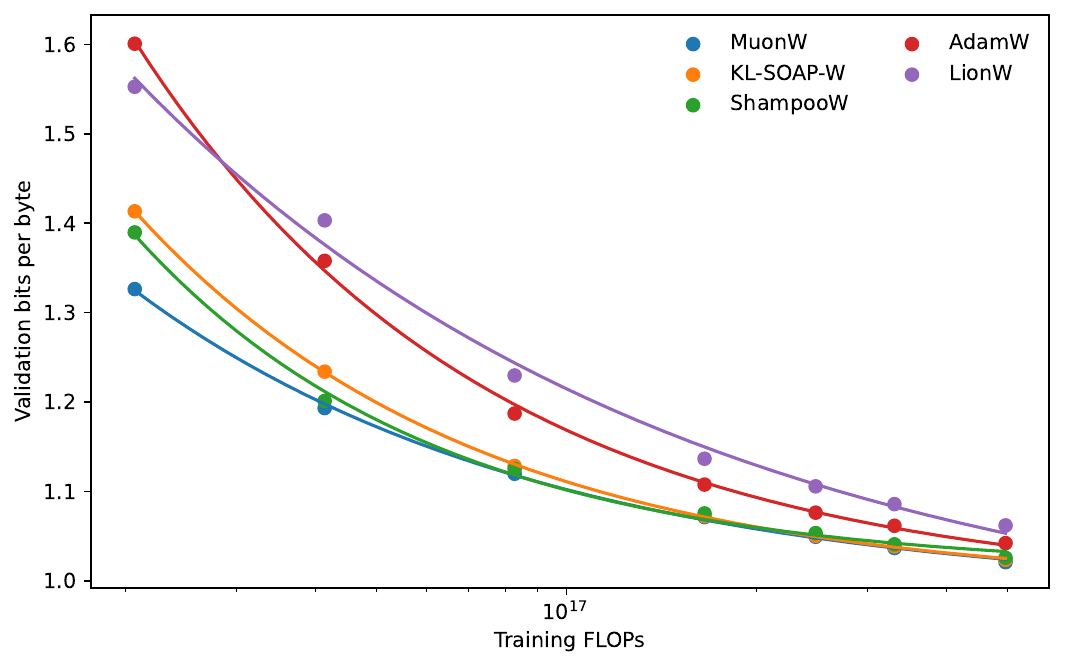}
    \caption{Validation losses across token budgets when the batch isze is 256K and the fitted loss--compute curve
    for each optimizer.  Each curve is fit only to that optimizer's seven
    observations.}
    \label{fig:nanochat_loss_compute}
\end{figure}

Using each optimizer's own 256K curve measures the efficiency lost when
transferring that optimizer to a larger batch.  Because the reference
changes with the optimizer, these penalties do not rank absolute
performance across optimizers.  For a common comparison, we instead invert
every observed loss through the MuonW 256K curve and divide by the run's
actual compute.  At the fixed budget used here, this shared reference
preserves the ordering of absolute losses.  Only the MuonW 256K anchor is
assigned zero penalty by convention in the shared comparison; the other
optimizers' 256K results are evaluated through the fitted MuonW curve.

\subsection{Results}

\begin{figure}[H]
    \vspace{-1em}
    \centering
    \includegraphics[width=\textwidth]{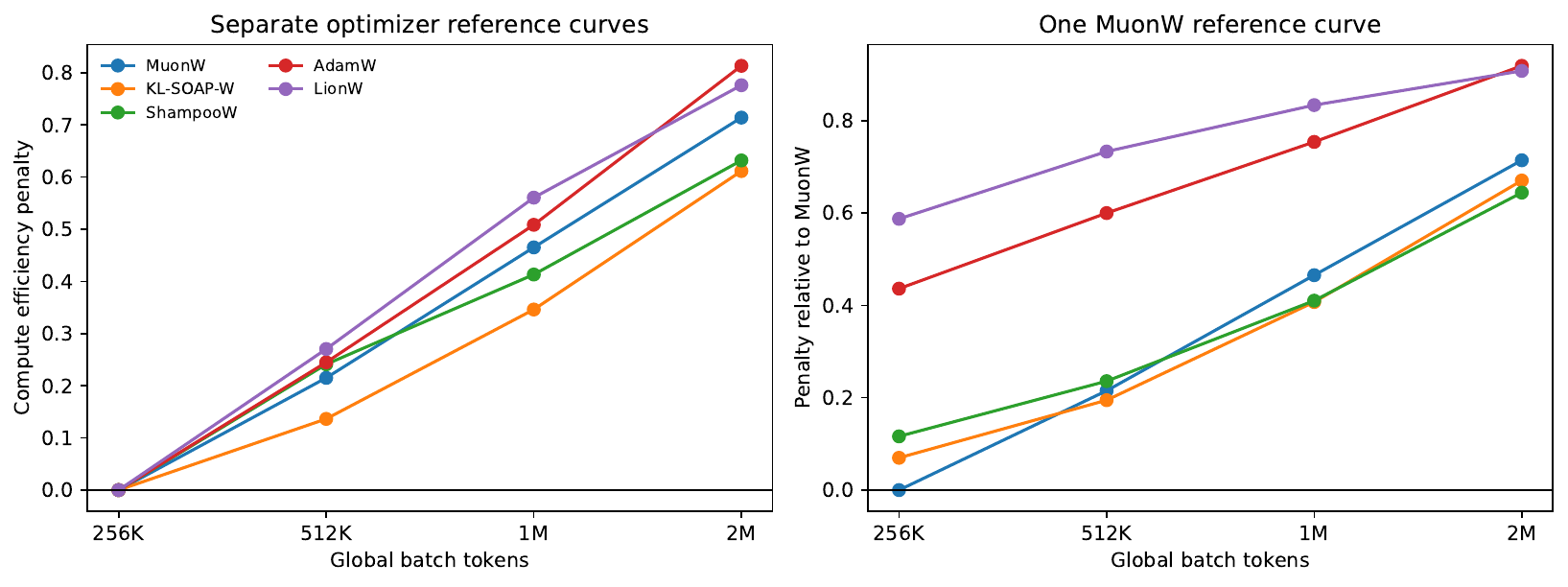}
    \caption{Compute efficiency penalties at a fixed horizon of eight tokens
    per scaling parameter.  The left panel fits and uses a separate reference
    curve for every optimizer.  The right panel maps every loss through the
    MuonW curve and therefore retains absolute loss differences.  Lower is
    better.}
    \label{fig:nanochat_ce_penalties}
\end{figure}

\textbf{KL-SOAP-W loses the least efficiency relative to its own reference.}
Its penalties are the smallest at every larger batch, reaching
0.1365 at 512K, 0.3462 at 1M, and
0.6116 at 2M.  ShampooW is close at 2M with a penalty of 0.6320.  MuonW reaches
0.7141, LionW reaches 0.7761, and AdamW reaches 0.8133 at the same batch.
The fitted KL-SOAP-W reference therefore reaches its 2M-batch loss using
about 39 percent of the compute spent by that run.

\textbf{ShampooW achieves the lowest loss at the largest batch.}
KL-SOAP-W has
the smallest common reference penalty at 512K and 1M.  At 2M, ShampooW has
penalty 0.6440 and KL-SOAP-W has penalty 0.6702.  Their absolute validation
losses are 1.1525 and 1.1614.  Thus, retaining more efficiency relative to
one's own reference does not imply achieving the lowest absolute loss.

This change in winner agrees qualitatively with the KL-SOAP-to-Shampoo
crossover in \Cref{sec:batch_size_comparison}.  The present experiment
provides evidence in a different setting under the specified transfer
protocol.  Its penalties measure the combined outcome of changing batch
size and transferring hyperparameters, rather than the efficiency achievable
after independent retuning at every batch.

\begin{table}[H]
\centering
\small
\setlength{\tabcolsep}{5pt}
\renewcommand{\arraystretch}{1.04}
\begin{tabular}{@{}rlrrr@{}}
\toprule
& & & \multicolumn{2}{c}{Compute-efficiency penalty} \\
\cmidrule(lr){4-5}
$B$ & Optimizer & $L_{\rm val}$ (BPB) & Own reference & Muon reference \\
\midrule
256K & MuonW & 1.0712 & 0.0000 & 0.0000 \\
 & KL-SOAP-W & 1.0722 & 0.0000 & 0.0699 \\
 & ShampooW & 1.0753 & 0.0000 & 0.1163 \\
 & AdamW & 1.1076 & 0.0000 & 0.4361 \\
 & LionW & 1.1365 & 0.0000 & 0.5872 \\
\addlinespace
512K & MuonW & 1.0829 & 0.2151 & 0.2151 \\
 & KL-SOAP-W & 1.0812 & 0.1365 & 0.1947 \\
 & ShampooW & 1.0847 & 0.2412 & 0.2358 \\
 & AdamW & 1.1397 & 0.2449 & 0.5997 \\
 & LionW & 1.1885 & 0.2705 & 0.7330 \\
\addlinespace
1M & MuonW & 1.1121 & 0.4653 & 0.4653 \\
 & KL-SOAP-W & 1.1035 & 0.3462 & 0.4072 \\
 & ShampooW & 1.1039 & 0.4133 & 0.4101 \\
 & AdamW & 1.2000 & 0.5086 & 0.7538 \\
 & LionW & 1.2651 & 0.5607 & 0.8335 \\
\addlinespace
2M & MuonW & 1.1793 & 0.7141 & 0.7141 \\
 & KL-SOAP-W & 1.1614 & 0.6116 & 0.6702 \\
 & ShampooW & 1.1525 & 0.6320 & 0.6440 \\
 & AdamW & 1.4387 & 0.8133 & 0.9185 \\
 & LionW & 1.4016 & 0.7761 & 0.9076 \\
\bottomrule
\end{tabular}
\caption{Compute efficiency penalties at the eight token per parameter horizon. Own curve penalty uses the compute loss curve fitted separately for each optimizer. Muon curve penalty maps every observed loss through the single MuonW curve.}
\label{tab:nanochat_compute_efficiency}
\end{table}

All optimizer specific inversions in \Cref{tab:nanochat_compute_efficiency}
fall within the loss range covered by their respective seven reference
observations.  Under the common MuonW curve, the AdamW and LionW results at 2M lie
outside the directly observed MuonW loss range and are extrapolations.  The
remaining common reference values are interpolations.

\subsection{Complete results}

\Cref{tab:nanochat_ce_raw} contains the 50 observations used for the reference
fits and batch comparisons.  \Cref{tab:nanochat_beta_leaderboard} gives the
complete ranking of coefficient-transfer rules for KL-SOAP-W and ShampooW.

\clearpage
\begingroup
\footnotesize
\renewcommand{\arraystretch}{0.96}
\setlength{\tabcolsep}{6pt}
\setlength{\LTpre}{0pt}
\setlength{\LTpost}{0pt}
\setlength{\LTleft}{\fill}
\setlength{\LTright}{\fill}
\setlength{\LTcapwidth}{\textwidth}
\begin{longtable}{@{}lrrrr@{}}
\caption{Complete observations across token budgets and batch sizes used in the compute efficiency analysis.}\label{tab:nanochat_ce_raw}\\
\toprule
Optimizer & $\tau/N$ & $B$ (tokens) & $S(B)$ & $L_{\rm val}$ (BPB) \\
\midrule
\endfirsthead
\multicolumn{5}{c}{\tablename\ \thetable\ continued}\\
\toprule
Optimizer & $\tau/N$ & $B$ (tokens) & $S(B)$ & $L_{\rm val}$ (BPB) \\
\midrule
\endhead
\midrule
\multicolumn{5}{r}{Continued on next page}\\
\endfoot
\bottomrule
\endlastfoot
MuonW & 1 & 262,144 & 224 & 1.3264 \\
 & 2 & 262,144 & 448 & 1.1932 \\
 & 4 & 262,144 & 896 & 1.1197 \\
 & 8 & 262,144 & 1,792 & 1.0712 \\
 & 12 & 262,144 & 2,688 & 1.0492 \\
 & 16 & 262,144 & 3,584 & 1.0366 \\
 & 24 & 262,144 & 5,376 & 1.0208 \\
 & 8 & 524,288 & 896 & 1.0829 \\
 & 8 & 1,048,576 & 448 & 1.1121 \\
 & 8 & 2,097,152 & 224 & 1.1793 \\
\addlinespace[3pt]
KL-SOAP-W & 1 & 262,144 & 224 & 1.4134 \\
 & 2 & 262,144 & 448 & 1.2340 \\
 & 4 & 262,144 & 896 & 1.1285 \\
 & 8 & 262,144 & 1,792 & 1.0722 \\
 & 12 & 262,144 & 2,688 & 1.0503 \\
 & 16 & 262,144 & 3,584 & 1.0384 \\
 & 24 & 262,144 & 5,376 & 1.0234 \\
 & 8 & 524,288 & 896 & 1.0812 \\
 & 8 & 1,048,576 & 448 & 1.1035 \\
 & 8 & 2,097,152 & 224 & 1.1614 \\
\addlinespace[3pt]
ShampooW & 1 & 262,144 & 224 & 1.3898 \\
 & 2 & 262,144 & 448 & 1.2012 \\
 & 4 & 262,144 & 896 & 1.1251 \\
 & 8 & 262,144 & 1,792 & 1.0753 \\
 & 12 & 262,144 & 2,688 & 1.0533 \\
 & 16 & 262,144 & 3,584 & 1.0407 \\
 & 24 & 262,144 & 5,376 & 1.0257 \\
 & 8 & 524,288 & 896 & 1.0847 \\
 & 8 & 1,048,576 & 448 & 1.1039 \\
 & 8 & 2,097,152 & 224 & 1.1525 \\
\addlinespace[3pt]
AdamW & 1 & 262,144 & 224 & 1.6010 \\
 & 2 & 262,144 & 448 & 1.3580 \\
 & 4 & 262,144 & 896 & 1.1871 \\
 & 8 & 262,144 & 1,792 & 1.1076 \\
 & 12 & 262,144 & 2,688 & 1.0762 \\
 & 16 & 262,144 & 3,584 & 1.0615 \\
 & 24 & 262,144 & 5,376 & 1.0421 \\
 & 8 & 524,288 & 896 & 1.1397 \\
 & 8 & 1,048,576 & 448 & 1.2000 \\
 & 8 & 2,097,152 & 224 & 1.4387 \\
\addlinespace[3pt]
LionW & 1 & 262,144 & 224 & 1.5528 \\
 & 2 & 262,144 & 448 & 1.4033 \\
 & 4 & 262,144 & 896 & 1.2297 \\
 & 8 & 262,144 & 1,792 & 1.1365 \\
 & 12 & 262,144 & 2,688 & 1.1056 \\
 & 16 & 262,144 & 3,584 & 1.0857 \\
 & 24 & 262,144 & 5,376 & 1.0619 \\
 & 8 & 524,288 & 896 & 1.1885 \\
 & 8 & 1,048,576 & 448 & 1.2651 \\
 & 8 & 2,097,152 & 224 & 1.4016 \\
\end{longtable}
\endgroup

\clearpage
\begingroup
\scriptsize
\renewcommand{\arraystretch}{1.02}
\setlength{\tabcolsep}{2.2pt}
\setlength{\LTpre}{0pt}
\setlength{\LTpost}{0pt}
\setlength{\LTleft}{\fill}
\setlength{\LTright}{\fill}
\setlength{\LTcapwidth}{\textwidth}
\begin{longtable}{@{}rcccccrrrr@{}}
\caption{Complete beta-transfer rule rankings. F keeps a coefficient fixed, and E applies EMA transfer. For ShampooW, the preconditioner shares the matrix $\beta_2^{\rm M}$ coefficient, so no separate preconditioner choice is shown. Mean BSZ regret averages the gap to the best tested rule for the same optimizer and batch size over the three target batches.}\label{tab:nanochat_beta_leaderboard}\\
\toprule
Rank & $\beta_1^{\rm M}$ & $\beta_2^{\rm M}$ & $\beta_{\rm sh}$ & $\beta_1$ & $\beta_2$ & 512K & 1M & 2M & Mean BSZ regret \\
\midrule
\endfirsthead
\multicolumn{10}{c}{\tablename\ \thetable\ continued}\\
\toprule
Rank & $\beta_1^{\rm M}$ & $\beta_2^{\rm M}$ & $\beta_{\rm sh}$ & $\beta_1$ & $\beta_2$ & 512K & 1M & 2M & Mean BSZ regret \\
\midrule
\endhead
\midrule
\multicolumn{10}{r}{Continued on next page}\\
\endfoot
\bottomrule
\endlastfoot
\multicolumn{10}{l}{\bfseries KL-SOAP-W} \\
\textbf{1} & \textbf{F} & \textbf{F} & \textbf{F} & \textbf{F} & \textbf{E} & \textbf{1.0812} & \textbf{1.1035} & \textbf{1.1614} & \textbf{0.00123} \\
2 & F & F & F & F & F & 1.0787 & 1.1033 & 1.1650 & 0.00153 \\
3 & F & E & F & F & F & 1.0821 & 1.1042 & 1.1659 & 0.00327 \\
4 & F & E & F & F & E & 1.0829 & 1.1045 & 1.1660 & 0.00367 \\
5 & F & F & F & E & E & 1.0793 & 1.1069 & 1.2031 & 0.01563 \\
6 & F & E & F & E & E & 1.0809 & 1.1067 & 1.2054 & 0.01687 \\
7 & F & E & F & E & F & 1.0801 & 1.1087 & 1.2173 & 0.02123 \\
8 & F & F & F & E & F & 1.0803 & 1.1119 & 1.2171 & 0.02230 \\
9 & E & E & F & F & E & 1.0786 & 1.1097 & 1.2297 & 0.02520 \\
10 & E & E & E & F & E & 1.0791 & 1.1126 & 1.2302 & 0.02650 \\
11 & E & E & E & F & F & 1.0799 & 1.1136 & 1.2304 & 0.02717 \\
12 & E & E & F & F & F & 1.0780 & 1.1115 & 1.2345 & 0.02720 \\
13 & E & E & F & E & E & 1.0777 & 1.1175 & 1.2868 & 0.04653 \\
14 & E & E & E & E & E & 1.0794 & 1.1209 & 1.2907 & 0.04953 \\
15 & E & E & E & E & F & 1.0799 & 1.1204 & 1.2935 & 0.05047 \\
16 & E & E & F & E & F & 1.0798 & 1.1191 & 1.2994 & 0.05197 \\
17 & E & F & F & F & E & 1.0794 & 1.1372 & 1.3338 & 0.06933 \\
18 & F & F & E & F & E & 1.0831 & 1.1271 & 1.3433 & 0.07037 \\
19 & E & F & F & F & F & 1.0788 & 1.1380 & 1.3478 & 0.07407 \\
20 & E & F & F & E & E & 1.0808 & 1.1429 & 1.3607 & 0.08067 \\
21 & F & E & E & F & F & 1.0871 & 1.1519 & 1.3797 & 0.09210 \\
22 & F & F & E & F & F & 1.0820 & 1.1449 & 1.3918 & 0.09210 \\
23 & F & E & E & F & E & 1.0867 & 1.1594 & 1.3735 & 0.09240 \\
24 & E & F & F & E & F & 1.0818 & 1.1460 & 1.3972 & 0.09420 \\
25 & E & F & E & F & E & 1.0795 & 1.1449 & 1.4224 & 0.10147 \\
26 & E & F & E & E & E & 1.0807 & 1.1505 & 1.4273 & 0.10537 \\
27 & E & F & E & F & F & 1.0806 & 1.1521 & 1.4477 & 0.11267 \\
28 & F & F & E & E & E & 1.0807 & 1.1799 & 1.4467 & 0.12163 \\
29 & E & F & E & E & F & 1.0807 & 1.1594 & 1.5018 & 0.13317 \\
30 & F & E & E & E & E & 1.0844 & 1.1929 & 1.5375 & 0.15747 \\
31 & F & F & E & E & F & 1.0812 & 1.2512 & 1.5237 & 0.17123 \\
32 & F & E & E & E & F & 1.0854 & 1.2774 & 1.6103 & 0.21023 \\
\addlinespace[3pt]
\multicolumn{10}{l}{\bfseries ShampooW} \\
\textbf{1} & \textbf{F} & \textbf{F} & \textbf{--} & \textbf{F} & \textbf{E} & \textbf{1.0847} & \textbf{1.1039} & \textbf{1.1525} & \textbf{0.00040} \\
2 & F & F & -- & F & F & 1.0845 & 1.1039 & 1.1568 & 0.00177 \\
3 & F & F & -- & E & E & 1.0848 & 1.1083 & 1.1885 & 0.01390 \\
4 & F & F & -- & E & F & 1.0846 & 1.1091 & 1.2043 & 0.01937 \\
5 & E & E & -- & F & E & 1.0842 & 1.1150 & 1.1991 & 0.01947 \\
6 & E & E & -- & F & F & 1.0835 & 1.1161 & 1.2009 & 0.02020 \\
7 & E & E & -- & E & E & 1.0851 & 1.1208 & 1.2316 & 0.03253 \\
8 & E & E & -- & E & F & 1.0849 & 1.1227 & 1.2399 & 0.03587 \\
9 & E & F & -- & F & E & 1.0877 & 1.1461 & 1.2772 & 0.05703 \\
10 & E & F & -- & E & E & 1.0877 & 1.1453 & 1.3130 & 0.06870 \\
11 & E & F & -- & F & F & 1.0869 & 1.1445 & 1.3184 & 0.06997 \\
12 & E & F & -- & E & F & 1.0879 & 1.1536 & 1.3394 & 0.08033 \\
13 & F & E & -- & F & F & 1.0842 & 1.1046 & 1.4766 & 0.10850 \\
14 & F & E & -- & E & E & 1.0841 & 1.1060 & 1.6090 & 0.15307 \\
15 & F & E & -- & F & E & 1.0844 & 1.1044 & 1.6558 & 0.16823 \\
16 & F & E & -- & E & F & 1.0847 & 1.1073 & 2.0020 & 0.28470 \\
\end{longtable}
\endgroup

\endgroup


\newpage



\begin{thebibliography}{67}
\providecommand{\natexlab}[1]{#1}
\providecommand{\url}[1]{\texttt{#1}}
\expandafter\ifx\csname urlstyle\endcsname\relax
  \providecommand{\doi}[1]{doi: #1}\else
  \providecommand{\doi}{doi: \begingroup \urlstyle{rm}\Url}\fi

\bibitem[Amari et~al.(2021)Amari, Ba, Grosse, Li, Nitanda, Suzuki, Wu, and
  Xu]{amari2021preconditioning}
Shun-ichi Amari, Jimmy Ba, Roger Grosse, Xuechen Li, Atsushi Nitanda, Taiji
  Suzuki, Denny Wu, and Ji~Xu.
\newblock When does preconditioning help or hurt generalization?
\newblock In \emph{International Conference on Learning Representations}, 2021.
\newblock URL \url{https://openreview.net/forum?id=S724o4_WB3}.

\bibitem[Bergsma et~al.(2025)Bergsma, Dey, Gosal, Gray, Soboleva, and
  Hestness]{bergsma2025powerlines}
Shane Bergsma, Nolan Dey, Gurpreet Gosal, Gavia Gray, Daria Soboleva, and Joel
  Hestness.
\newblock {Power Lines}: Scaling laws for weight decay and batch size in {LLM}
  pre-training.
\newblock In \emph{Advances in Neural Information Processing Systems}, 2025.
\newblock URL \url{https://openreview.net/forum?id=bFXbLQzRoZ}.

\bibitem[Bjorck et~al.(2025)Bjorck, Benhaim, Chaudhary, Wei, and
  Song]{bjorck2025scalinghorizons}
Johan Bjorck, Alon Benhaim, Vishrav Chaudhary, Furu Wei, and Xia Song.
\newblock Scaling optimal {LR} across token horizons.
\newblock In \emph{International Conference on Learning Representations}, 2025.
\newblock URL \url{https://arxiv.org/abs/2409.19913}.

\bibitem[Bu et~al.(2026)Bu, Xu, and Mao]{bu2026convexdominance}
Zhiqi Bu, Shiyun Xu, and Jialin Mao.
\newblock Convex dominance in deep learning {I}: A scaling law of loss and
  learning rate.
\newblock In \emph{International Conference on Learning Representations}, 2026.
\newblock URL \url{https://openreview.net/forum?id=dSdLqg02tx}.

\bibitem[Cesista \& Wen(2025)Cesista and Wen]{cesista2025sdcbs}
Franz~Louis Cesista and Kaiyue Wen.
\newblock Critical batch size for steepest descent under arbitrary norms,
  November 2025.
\newblock URL
  \url{https://leloykun.github.io/ponder/steepest-descent-crit-bz/}.
\newblock Blog post, November 22.

\bibitem[Chen et~al.(2023)Chen, Liang, Huang, Real, Wang, Pham, Dong, Luong,
  Hsieh, Lu, and Le]{chen2024lion}
Xiangning Chen, Chen Liang, Da~Huang, Esteban Real, Kaiyuan Wang, Hieu Pham,
  Xuanyi Dong, Thang Luong, Cho-Jui Hsieh, Yifeng Lu, and Quoc~V. Le.
\newblock Symbolic discovery of optimization algorithms.
\newblock In \emph{Advances in Neural Information Processing Systems},
  volume~36, pp.\  49205--49233, 2023.
\newblock URL
  \url{https://proceedings.neurips.cc/paper_files/paper/2023/hash/9a39b4925e35cf447ccba8757137d84f-Abstract-Conference.html}.

\bibitem[Choi et~al.(2019)Choi, Shallue, Nado, Lee, Maddison, and
  Dahl]{choi2019empirical}
Dami Choi, Christopher~J. Shallue, Zachary Nado, Jaehoon Lee, Chris~J.
  Maddison, and George~E. Dahl.
\newblock On empirical comparisons of optimizers for deep learning, 2019.
\newblock URL \url{https://arxiv.org/abs/1910.05446}.

\bibitem[Dahl et~al.(2023)Dahl, Schneider, Nado, Agarwal, Sastry, Hennig,
  Medapati, Eschenhagen, Kasimbeg, Suo, Bae, Gilmer, Peirson, Khan, Anil,
  Rabbat, Krishnan, Snider, Amid, Chen, Maddison, Vasudev, Badura, Garg, and
  Mattson]{dahl2023algoperf}
George~E. Dahl, Frank Schneider, Zachary Nado, Naman Agarwal,
  Chandramouli~Shama Sastry, Philipp Hennig, Sourabh Medapati, Runa
  Eschenhagen, Priya Kasimbeg, Daniel Suo, Juhan Bae, Justin Gilmer, Abel~L.
  Peirson, Bilal Khan, Rohan Anil, Mike Rabbat, Shankar Krishnan, Daniel
  Snider, Ehsan Amid, Kongtao Chen, Chris~J. Maddison, Rakshith Vasudev, Michal
  Badura, Ankush Garg, and Peter Mattson.
\newblock Benchmarking neural network training algorithms, 2023.
\newblock URL \url{https://arxiv.org/abs/2306.07179}.

\bibitem[{Essential AI} et~al.(2025){Essential AI}, Shah, Polloreno, Stratos,
  Monk, Chaluvaraju, Hojel, Ma, Thomas, Tanwer, Shah, Nguyen, Smith, Callahan,
  Pust, Parmar, Rushton, Mazarakis, Kapila, Srivastava, Singla, Romanski,
  Vanjani, and Vaswani]{shah2025practicalefficiency}
{Essential AI}, Ishaan Shah, Anthony~M. Polloreno, Karl Stratos, Philip Monk,
  Adarsh Chaluvaraju, Andrew Hojel, Andrew Ma, Anil Thomas, Ashish Tanwer,
  Darsh~J. Shah, Khoi Nguyen, Kurt Smith, Michael Callahan, Michael Pust, Mohit
  Parmar, Peter Rushton, Platon Mazarakis, Ritvik Kapila, Saurabh Srivastava,
  Somanshu Singla, Tim Romanski, Yash Vanjani, and Ashish Vaswani.
\newblock Practical efficiency of {Muon} for pretraining, 2025.
\newblock URL \url{https://arxiv.org/abs/2505.02222}.

\bibitem[Everett et~al.(2024)Everett, Xiao, Wortsman, Alemi, Novak, Liu, Gur,
  Sohl-Dickstein, Kaelbling, Lee, and Pennington]{everett2024scalingexponents}
Katie Everett, Lechao Xiao, Mitchell Wortsman, Alexander~A. Alemi, Roman Novak,
  Peter~J. Liu, Izzeddin Gur, Jascha Sohl-Dickstein, Leslie~Pack Kaelbling,
  Jaehoon Lee, and Jeffrey Pennington.
\newblock Scaling exponents across parameterizations and optimizers.
\newblock In \emph{Proceedings of the 41st International Conference on Machine
  Learning}, volume 235, 2024.
\newblock URL \url{https://arxiv.org/abs/2407.05872}.

\bibitem[Goyal et~al.(2017)Goyal, Doll{\'a}r, Girshick, Noordhuis, Wesolowski,
  Kyrola, Tulloch, Jia, and He]{goyal2017accurate}
Priya Goyal, Piotr Doll{\'a}r, Ross Girshick, Pieter Noordhuis, Lukasz
  Wesolowski, Aapo Kyrola, Andrew Tulloch, Yangqing Jia, and Kaiming He.
\newblock Accurate, large minibatch {SGD}: Training {ImageNet} in 1 hour, 2017.
\newblock URL \url{https://arxiv.org/abs/1706.02677}.

\bibitem[Granziol et~al.(2022)Granziol, Zohren, and
  Roberts]{granziol2022learningrates}
Diego Granziol, Stefan Zohren, and Stephen Roberts.
\newblock Learning rates as a function of batch size: A random matrix theory
  approach to neural network training.
\newblock \emph{Journal of Machine Learning Research}, 23\penalty0
  (173):\penalty0 1--65, 2022.
\newblock URL \url{https://jmlr.org/papers/v23/20-1258.html}.

\bibitem[Gupta et~al.(2018)Gupta, Koren, and Singer]{gupta2018shampoo}
Vineet Gupta, Tomer Koren, and Yoram Singer.
\newblock Shampoo: Preconditioned stochastic tensor optimization.
\newblock In \emph{Proceedings of the 35th International Conference on Machine
  Learning}, volume~80 of \emph{Proceedings of Machine Learning Research}, pp.\
   1842--1850, 2018.
\newblock URL \url{https://proceedings.mlr.press/v80/gupta18a.html}.

\bibitem[Hilton et~al.(2022)Hilton, Cobbe, and
  Schulman]{hilton2022batchsizeinvariance}
Jacob Hilton, Karl Cobbe, and John Schulman.
\newblock Batch size-invariance for policy optimization.
\newblock In \emph{Advances in Neural Information Processing Systems},
  volume~35, pp.\  17086--17098, 2022.
\newblock URL \url{https://arxiv.org/abs/2110.00641}.

\bibitem[Islamov et~al.(2026)Islamov, Machacek, Lucchi, Silveti-Falls,
  Gorbunov, and Cevher]{islamov2026role}
Rustem Islamov, Roman Machacek, Aurelien Lucchi, Antonio Silveti-Falls, Eduard
  Gorbunov, and Volkan Cevher.
\newblock On the role of batch size in stochastic conditional gradient methods,
  2026.
\newblock URL \url{https://arxiv.org/abs/2603.21191}.

\bibitem[Jastrz{\k{e}}bski et~al.(2017)Jastrz{\k{e}}bski, Kenton, Arpit,
  Ballas, Fischer, Bengio, and Storkey]{jastrzebski2017three}
Stanis{\l}aw Jastrz{\k{e}}bski, Zachary Kenton, Devansh Arpit, Nicolas Ballas,
  Asja Fischer, Yoshua Bengio, and Amos Storkey.
\newblock Three factors influencing minima in {SGD}, 2017.
\newblock URL \url{https://arxiv.org/abs/1711.04623}.

\bibitem[Johnson et~al.(2020)Johnson, Agrawal, Gu, and
  Guestrin]{johnson2020adascale}
Tyler~B. Johnson, Pulkit Agrawal, Haijie Gu, and Carlos Guestrin.
\newblock {AdaScale SGD}: A user-friendly algorithm for distributed training.
\newblock In \emph{Proceedings of the 37th International Conference on Machine
  Learning}, volume 119, pp.\  4911--4920, 2020.
\newblock URL \url{https://proceedings.mlr.press/v119/johnson20a.html}.

\bibitem[Jordan(2026)]{jordan2026moddednanogpt}
Keller Jordan.
\newblock {Modded-NanoGPT Optimization Benchmark}, 2026.
\newblock URL
  \url{https://github.com/KellerJordan/modded-nanogpt/tree/master/records/track_3_optimization}.

\bibitem[Jordan et~al.(2024)Jordan, Jin, Boza, You, Cesista, Newhouse, and
  Bernstein]{jordan2024muon}
Keller Jordan, Yuchen Jin, Vlado Boza, Jiacheng You, Franz Cesista, Laker
  Newhouse, and Jeremy Bernstein.
\newblock Muon: An optimizer for hidden layers in neural networks, 2024.
\newblock URL \url{https://kellerjordan.github.io/posts/muon/}.

\bibitem[Khona et~al.(2026)Khona, Vavre, Wang, Fu, Wu, Chrzanowski, Catanzaro,
  Mudigere, Pool, Lightstone, Shoeybi, Patwary, Tajbakhsh, and
  Blankevoort]{khona2026soapmuonbeyond}
Mikail Khona, Aditya Vavre, Boxiang Wang, Deyu Fu, Hao Wu, Mike Chrzanowski,
  Bryan Catanzaro, Dheevatsa Mudigere, Jeff Pool, Michael Lightstone, Mohammad
  Shoeybi, Mostofa Patwary, Nima Tajbakhsh, and Tijmen Blankevoort.
\newblock {SOAP}, {Muon}, and beyond: Pushing {LLM} pretraining scales, July
  2026.
\newblock URL \url{https://arxiv.org/abs/2607.20548}.

\bibitem[Kingma \& Ba(2015)Kingma and Ba]{kingma2015adam}
Diederik~P. Kingma and Jimmy Ba.
\newblock Adam: A method for stochastic optimization.
\newblock In \emph{International Conference on Learning Representations}, 2015.
\newblock URL \url{https://arxiv.org/abs/1412.6980}.

\bibitem[Kovalev(2025)]{kovalev2025understanding}
Dmitry Kovalev.
\newblock Understanding gradient orthogonalization for deep learning via
  non-{E}uclidean trust-region optimization, 2025.
\newblock URL \url{https://arxiv.org/abs/2503.12645}.

\bibitem[Kunstner et~al.(2019)Kunstner, Balles, and
  Hennig]{kunstner2019limitations}
Frederik Kunstner, Lukas Balles, and Philipp Hennig.
\newblock Limitations of the empirical fisher approximation for natural
  gradient descent.
\newblock In \emph{Advances in Neural Information Processing Systems},
  volume~32, 2019.
\newblock URL
  \url{https://proceedings.neurips.cc/paper/2019/hash/46a558d97954d0692411c861cf78ef79-Abstract.html}.

\bibitem[Kunstner et~al.(2023)Kunstner, Chen, Lavington, and
  Schmidt]{kunstner2023noisenotmain}
Frederik Kunstner, Jacques Chen, Jonathan~Wilder Lavington, and Mark Schmidt.
\newblock Noise is not the main factor behind the gap between {SGD} and {Adam}
  on transformers, but sign descent might be.
\newblock In \emph{International Conference on Learning Representations}, 2023.
\newblock URL \url{https://openreview.net/forum?id=a65YK0cqH8g}.

\bibitem[Lee et~al.(2022)Lee, Cheng, Paquette, and
  Paquette]{lee2022trajectoryminibatch}
Kiwon Lee, Andrew~N. Cheng, Elliot Paquette, and Courtney Paquette.
\newblock Trajectory of mini-batch momentum: Batch size saturation and
  convergence in high dimensions.
\newblock In \emph{Advances in Neural Information Processing Systems},
  volume~35, 2022.
\newblock URL
  \url{https://proceedings.neurips.cc/paper_files/paper/2022/hash/efcb76ac1df9231a24893a957fcb9001-Abstract-Conference.html}.

\bibitem[Li et~al.(2018)Li, Tai, and E]{li2018stochastic}
Qianxiao Li, Cheng Tai, and Weinan E.
\newblock Stochastic modified equations and dynamics of stochastic gradient
  algorithms {I:} mathematical foundations.
\newblock \emph{CoRR}, abs/1811.01558, 2018.
\newblock URL \url{http://arxiv.org/abs/1811.01558}.

\bibitem[Li et~al.(2024)Li, Zhao, Zhang, Sun, Wu, Jiao, Wang, Liu, Fang, Xue,
  Tao, Cui, and Wang]{li2024surge}
Shuaipeng Li, Penghao Zhao, Hailin Zhang, Xingwu Sun, Hao Wu, Dian Jiao, Weiyan
  Wang, Chengjun Liu, Zheng Fang, Jinbao Xue, Yangyu Tao, Bin Cui, and Di~Wang.
\newblock Surge phenomenon in optimal learning rate and batch size scaling.
\newblock In \emph{Advances in Neural Information Processing Systems},
  volume~37, 2024.
\newblock URL
  \url{https://proceedings.neurips.cc/paper_files/paper/2024/hash/ef74413c7bf1d915c3e45c72e19a5d32-Abstract-Conference.html}.

\bibitem[Li(2018)]{li2018preconditioned}
Xi-Lin Li.
\newblock Preconditioned stochastic gradient descent.
\newblock \emph{IEEE Transactions on Neural Networks and Learning Systems},
  29\penalty0 (5):\penalty0 1454--1466, 2018.
\newblock \doi{10.1109/TNNLS.2017.2672978}.
\newblock URL \url{https://arxiv.org/abs/1512.04202}.

\bibitem[Li et~al.(2021)Li, Malladi, and Arora]{li2021validity}
Zhiyuan Li, Sadhika Malladi, and Sanjeev Arora.
\newblock On the validity of modeling {SGD} with stochastic differential
  equations ({SDEs}).
\newblock In \emph{Advances in Neural Information Processing Systems},
  volume~34, 2021.
\newblock URL
  \url{https://papers.nips.cc/paper/2021/hash/69f62956429865909921fa916d61c1f8-Abstract.html}.

\bibitem[Lin et~al.(2026)Lin, Lowe, Dangel, Eschenhagen, Xu, and
  Grosse]{lin2026klshampoo}
Wu~Lin, Scott~C. Lowe, Felix Dangel, Runa Eschenhagen, Zikun Xu, and Roger~B.
  Grosse.
\newblock Understanding and improving {Shampoo} and {SOAP} via
  {Kullback--Leibler} minimization.
\newblock In \emph{International Conference on Learning Representations}, 2026.
\newblock URL \url{https://openreview.net/forum?id=pQQuC1nIQq}.

\bibitem[Lu et~al.(2026)Lu, Bian, Wen, Wu, Chen, Wang, and
  Li]{lu2026momentumrivervalley}
Miao Lu, Zeyu Bian, Kaiyue Wen, Beining Wu, Siyu Chen, Tianhao Wang, and
  Zhiyuan Li.
\newblock Towards understanding momentum acceleration in river-valley loss
  landscape, 2026.
\newblock URL \url{https://arxiv.org/abs/2609.30957}.

\bibitem[Malladi et~al.(2022)Malladi, Lyu, Panigrahi, and
  Arora]{malladi2022sdes}
Sadhika Malladi, Kaifeng Lyu, Abhishek Panigrahi, and Sanjeev Arora.
\newblock On the {SDE}s and scaling rules for adaptive gradient algorithms.
\newblock In \emph{Advances in Neural Information Processing Systems},
  volume~35, 2022.
\newblock URL
  \url{https://proceedings.neurips.cc/paper_files/paper/2022/hash/32ac710102f0620d0f28d5d05a44fe08-Abstract-Conference.html}.

\bibitem[Marek et~al.(2025)Marek, Lotfi, Somasundaram, Wilson, and
  Goldblum]{marek2025smallbatch}
Martin Marek, Sanae Lotfi, Aditya Somasundaram, Andrew~Gordon Wilson, and Micah
  Goldblum.
\newblock Small batch size training for language models: When vanilla {SGD}
  works, and why gradient accumulation is wasteful.
\newblock In \emph{Advances in Neural Information Processing Systems}, 2025.
\newblock URL \url{https://arxiv.org/abs/2507.07101}.

\bibitem[McCandlish et~al.(2018)McCandlish, Kaplan, Amodei, and {OpenAI Dota
  Team}]{mccandlish2018empirical}
Sam McCandlish, Jared Kaplan, Dario Amodei, and {OpenAI Dota Team}.
\newblock An empirical model of large-batch training, 2018.
\newblock URL \url{https://arxiv.org/abs/1812.06162}.

\bibitem[Merrill et~al.(2025)Merrill, Arora, Groeneveld, and
  Hajishirzi]{merrill2025criticalbatchsizerevisited}
William Merrill, Shane Arora, Dirk Groeneveld, and Hannaneh Hajishirzi.
\newblock Critical batch size revisited: A simple empirical approach to
  large-batch language model training.
\newblock In \emph{Advances in Neural Information Processing Systems}, 2025.
\newblock URL \url{https://openreview.net/forum?id=XUKUx7Xu89}.

\bibitem[Meterez et~al.(2026)Meterez, Nair, Morwani, Pehlevan, Kakade, and
  Damian]{meterez2026defensequadraticmodel}
Alexandru Meterez, Pranav~Ajit Nair, Depen Morwani, Cengiz Pehlevan, Sham
  Kakade, and Alex Damian.
\newblock A defense of the quadratic model, 2026.
\newblock URL \url{https://arxiv.org/abs/2607.21716}.

\bibitem[Mlodozeniec et~al.(2026)Mlodozeniec, Ablin, B{\'e}thune, Busbridge,
  Klein, Ramapuram, and Cuturi]{mlodozeniec2025completed}
Bruno~Kacper Mlodozeniec, Pierre Ablin, Louis B{\'e}thune, Dan Busbridge,
  Michal Klein, Jason Ramapuram, and Marco Cuturi.
\newblock Completed hyperparameter transfer across modules, width, depth, batch
  and duration.
\newblock In \emph{International Conference on Learning Representations}, 2026.
\newblock URL
  \url{https://proceedings.iclr.cc/paper_files/paper/2026/hash/4b0c1645f3d6a1730931e65ecbf91ac3-Abstract-Conference.html}.

\bibitem[Monzio~Compagnoni et~al.(2025)Monzio~Compagnoni, Liu, Islamov, Proske,
  Orvieto, and Lucchi]{compagnoni2025adaptive}
Enea Monzio~Compagnoni, Tianlin Liu, Rustem Islamov, Frank~Norbert Proske,
  Antonio Orvieto, and Aurelien Lucchi.
\newblock Adaptive methods through the lens of {SDE}s: Theoretical insights on
  the role of noise.
\newblock In \emph{International Conference on Learning Representations}, 2025.
\newblock URL \url{https://arxiv.org/abs/2411.15958}.

\bibitem[Morwani et~al.(2025)Morwani, Shapira, Vyas, Malach, Kakade, and
  Janson]{morwani2024shampooperspective}
Depen Morwani, Itai Shapira, Nikhil Vyas, Eran Malach, Sham~M. Kakade, and
  Lucas Janson.
\newblock A new perspective on {Shampoo}'s preconditioner.
\newblock In \emph{International Conference on Learning Representations}, 2025.
\newblock URL \url{https://arxiv.org/abs/2406.17748}.

\bibitem[Nakkiran et~al.(2021)Nakkiran, Neyshabur, and
  Sedghi]{nakkiran2021deepbootstrap}
Preetum Nakkiran, Behnam Neyshabur, and Hanie Sedghi.
\newblock The deep bootstrap framework: Good online learners are good offline
  generalizers.
\newblock In \emph{International Conference on Learning Representations}, 2021.
\newblock URL \url{https://openreview.net/forum?id=guetrIHLFGI}.

\bibitem[Penedo et~al.(2024)]{penedo2024fineweb}
Guilherme Penedo et~al.
\newblock The {FineWeb} datasets: Decanting the web for the finest text data at
  scale, 2024.
\newblock URL \url{https://arxiv.org/abs/2406.17557}.

\bibitem[Pethick et~al.(2025)Pethick, Xie, Antonakopoulos, Zhu, Silveti-Falls,
  and Cevher]{pethick2025training}
Thomas Pethick, Wanyun Xie, Kimon Antonakopoulos, Zhenyu Zhu, Antonio
  Silveti-Falls, and Volkan Cevher.
\newblock Training deep learning models with norm-constrained {LMO}s.
\newblock In \emph{Proceedings of the 42nd International Conference on Machine
  Learning}, volume 267 of \emph{Proceedings of Machine Learning Research},
  pp.\  49069--49104, 2025.
\newblock URL \url{https://proceedings.mlr.press/v267/pethick25a.html}.

\bibitem[Qiu et~al.(2025)Qiu, Chen, Phan, Lei, and
  Wilson]{qiu2025matrixpreconditioned}
Shikai Qiu, Zixi Chen, Hoang Phan, Qi~Lei, and Andrew~Gordon Wilson.
\newblock Hyperparameter transfer enables consistent gains of
  matrix-preconditioned optimizers across scales.
\newblock In \emph{Advances in Neural Information Processing Systems}, 2025.
\newblock URL \url{https://arxiv.org/abs/2512.05620}.

\bibitem[Sato et~al.(2025)Sato, Naganuma, and
  Iiduka]{sato2025muoncriticalbatch}
Naoki Sato, Hiroki Naganuma, and Hideaki Iiduka.
\newblock Convergence bound and critical batch size of {Muon} optimizer, 2025.
\newblock URL \url{https://arxiv.org/abs/2507.01598}.

\bibitem[Schaipp et~al.(2025)Schaipp, H{\"a}gele, Taylor, Simsekli, and
  Bach]{schaipp2025surprising}
Fabian Schaipp, Alexander H{\"a}gele, Adrien Taylor, Umut Simsekli, and Francis
  Bach.
\newblock The surprising agreement between convex optimization theory and
  learning-rate scheduling for large model training.
\newblock In \emph{Proceedings of the 42nd International Conference on Machine
  Learning}, volume 267, pp.\  53267--53294, 2025.
\newblock URL \url{https://proceedings.mlr.press/v267/schaipp25a.html}.

\bibitem[Schaul et~al.(2013)Schaul, Zhang, and LeCun]{schaul2013nomorepesky}
Tom Schaul, Sixin Zhang, and Yann LeCun.
\newblock No more pesky learning rates.
\newblock In \emph{Proceedings of the 30th International Conference on Machine
  Learning}, volume~28, pp.\  343--351, 2013.
\newblock URL \url{https://proceedings.mlr.press/v28/schaul13.html}.

\bibitem[Schmidt et~al.(2021)Schmidt, Schneider, and
  Hennig]{schmidt2021crowdedvalley}
Robin~M. Schmidt, Frank Schneider, and Philipp Hennig.
\newblock Descending through a crowded valley: Benchmarking deep learning
  optimizers.
\newblock In \emph{Proceedings of the 38th International Conference on Machine
  Learning}, volume 139, pp.\  9367--9376, 2021.
\newblock URL \url{https://proceedings.mlr.press/v139/schmidt21a.html}.

\bibitem[Semenov et~al.(2025)Semenov, Pagliardini, and
  Jaggi]{semenov2025benchmarkingoptimizers}
Andrei Semenov, Matteo Pagliardini, and Martin Jaggi.
\newblock Benchmarking optimizers for large language model pretraining, 2025.
\newblock URL \url{https://arxiv.org/abs/2509.01440}.

\bibitem[Sfyraki \& Wang(2025)Sfyraki and Wang]{sfyraki2025lions}
Maria-Eleni Sfyraki and Jun-Kun Wang.
\newblock Lions and muons: Optimization via stochastic frank-wolfe, 2025.
\newblock URL \url{https://arxiv.org/abs/2506.04192}.

\bibitem[Shallue et~al.(2019)Shallue, Lee, Antognini, Sohl-Dickstein, Frostig,
  and Dahl]{shallue2019measuring}
Christopher~J. Shallue, Jaehoon Lee, Joseph Antognini, Jascha Sohl-Dickstein,
  Roy Frostig, and George~E. Dahl.
\newblock Measuring the effects of data parallelism on neural network training.
\newblock \emph{Journal of Machine Learning Research}, 20\penalty0
  (112):\penalty0 1--49, 2019.
\newblock URL \url{https://jmlr.org/papers/v20/18-789.html}.

\bibitem[Shulgin et~al.(2026)Shulgin, von R{\"u}tte, Zhang, Ajroldi,
  Sch{\"o}lkopf, and Orvieto]{shulgin2026deriving}
Egor Shulgin, Dimitri von R{\"u}tte, Tianyue~H. Zhang, Niccol{\`o} Ajroldi,
  Bernhard Sch{\"o}lkopf, and Antonio Orvieto.
\newblock Deriving hyperparameter scaling laws via modern optimization theory,
  2026.
\newblock URL \url{https://arxiv.org/abs/2603.15958}.

\bibitem[Smith \& Le(2018)Smith and Le]{smith2018bayesian}
Samuel~L. Smith and Quoc~V. Le.
\newblock A {B}ayesian perspective on generalization and stochastic gradient
  descent.
\newblock In \emph{International Conference on Learning Representations}, 2018.
\newblock URL \url{https://arxiv.org/abs/1710.06451}.

\bibitem[Smith et~al.(2018)Smith, Kindermans, Ying, and Le]{smith2018dont}
Samuel~L. Smith, Pieter-Jan Kindermans, Chris Ying, and Quoc~V. Le.
\newblock Don't decay the learning rate, increase the batch size.
\newblock In \emph{International Conference on Learning Representations}, 2018.
\newblock URL \url{https://openreview.net/forum?id=B1Yy1BxCZ}.

\bibitem[Sre{\'c}kovi{\'c} et~al.(2025)Sre{\'c}kovi{\'c}, Geiping, and
  Orvieto]{sreckovic2025batchsize}
Teodora Sre{\'c}kovi{\'c}, Jonas Geiping, and Antonio Orvieto.
\newblock Is your batch size the problem? {R}evisiting the {Adam}-{SGD} gap in
  language modeling, 2025.
\newblock URL \url{https://arxiv.org/abs/2506.12543}.

\bibitem[Vyas et~al.(2025)Vyas, Morwani, Zhao, Shapira, Brandfonbrener, Janson,
  and Kakade]{vyas2024soap}
Nikhil Vyas, Depen Morwani, Rosie Zhao, Itai Shapira, David Brandfonbrener,
  Lucas Janson, and Sham~M. Kakade.
\newblock {SOAP}: Improving and stabilizing {Shampoo} using {Adam} for language
  modeling.
\newblock In \emph{International Conference on Learning Representations}, 2025.
\newblock URL \url{https://openreview.net/forum?id=IDxZhXrpNf}.

\bibitem[Wadia et~al.(2021)Wadia, Duckworth, Schoenholz, Dyer, and
  Sohl-Dickstein]{wadia2021whitening}
Neha Wadia, Daniel Duckworth, Samuel~S. Schoenholz, Ethan Dyer, and Jascha
  Sohl-Dickstein.
\newblock Whitening and second order optimization both make information in the
  dataset unusable during training, and can reduce or prevent generalization.
\newblock In \emph{Proceedings of the 38th International Conference on Machine
  Learning}, volume 139, pp.\  10617--10629, 2021.
\newblock URL \url{https://proceedings.mlr.press/v139/wadia21a.html}.

\bibitem[Wang et~al.(2024)Wang, Malladi, Wang, Lyu, and
  Li]{wang2024marginalmomentum}
Runzhe Wang, Sadhika Malladi, Tianhao Wang, Kaifeng Lyu, and Zhiyuan Li.
\newblock The marginal value of momentum for small learning rate {SGD}.
\newblock In \emph{International Conference on Learning Representations}, 2024.
\newblock URL \url{https://openreview.net/forum?id=3JjJezzVkT}.

\bibitem[Wang \& Aitchison(2024)Wang and Aitchison]{wang2024batchinvariant}
Xi~Wang and Laurence Aitchison.
\newblock Batch size invariant {Adam}, 2024.
\newblock URL \url{https://arxiv.org/abs/2402.18824}.

\bibitem[Wen et~al.(2025)Wen, Li, Wang, Hall, Liang, and
  Ma]{wen2025riverValley}
Kaiyue Wen, Zhiyuan Li, Jason Wang, David Hall, Percy Liang, and Tengyu Ma.
\newblock Understanding warmup-stable-decay learning rates: A river valley loss
  landscape view.
\newblock In \emph{International Conference on Learning Representations}, 2025.
\newblock URL \url{https://arxiv.org/abs/2410.05192}.

\bibitem[Wen et~al.(2026{\natexlab{a}})Wen, Dang, Lyu, Ma, and
  Liang]{wen2026hyperball}
Kaiyue Wen, Xingyu Dang, Kaifeng Lyu, Tengyu Ma, and Percy Liang.
\newblock Fantastic pretraining optimizers and where to find them {II}:
  Hyperball optimization, 2026{\natexlab{a}}.
\newblock URL \url{https://arxiv.org/abs/2606.16899}.

\bibitem[Wen et~al.(2026{\natexlab{b}})Wen, Hall, Ma, and
  Liang]{wen2025fantastic}
Kaiyue Wen, David Hall, Tengyu Ma, and Percy Liang.
\newblock Fantastic pretraining optimizers and where to find them.
\newblock In \emph{International Conference on Learning Representations},
  2026{\natexlab{b}}.
\newblock URL \url{https://openreview.net/forum?id=2J51qUZ0iG}.

\bibitem[Wu et~al.(2018)Wu, Ren, Liao, and Grosse]{wu2018shorthorizon}
Yuhuai Wu, Mengye Ren, Renjie Liao, and Roger Grosse.
\newblock Understanding short-horizon bias in stochastic meta-optimization.
\newblock In \emph{International Conference on Learning Representations}, 2018.
\newblock URL \url{https://openreview.net/forum?id=H1MczcgR-}.

\bibitem[Yang et~al.(2021)Yang, Hu, Babuschkin, Sidor, Liu, Farhi, Ryder,
  Pachocki, Chen, and Gao]{yang2021tensorprogramsv}
Greg Yang, Edward~J. Hu, Igor Babuschkin, Szymon Sidor, Xiaodong Liu, David
  Farhi, Nick Ryder, Jakub Pachocki, Weizhu Chen, and Jianfeng Gao.
\newblock Tensor programs {V}: Tuning large neural networks via zero-shot
  hyperparameter transfer.
\newblock In \emph{Advances in Neural Information Processing Systems},
  volume~34, 2021.
\newblock URL \url{https://arxiv.org/abs/2203.03466}.

\bibitem[Zhang et~al.(2026)Zhang, Islamov, Monzio~Compagnoni, Pang, Lucchi, and
  Orvieto]{zhang2026beyond}
Chenxiang Zhang, Rustem Islamov, Enea Monzio~Compagnoni, Jun Pang, Aurelien
  Lucchi, and Antonio Orvieto.
\newblock Beyond a single explanation of the {Adam}--{SGD} gap, 2026.
\newblock URL \url{https://arxiv.org/abs/2606.14259}.

\bibitem[Zhang et~al.(2019)Zhang, Li, Nado, Martens, Sachdeva, Dahl, Shallue,
  and Grosse]{zhang2019algorithmicchoicesmatterbatch}
Guodong Zhang, Lala Li, Zachary Nado, James Martens, Sushant Sachdeva,
  George~E. Dahl, Christopher~J. Shallue, and Roger Grosse.
\newblock Which algorithmic choices matter at which batch sizes? {I}nsights
  from a noisy quadratic model.
\newblock In \emph{Advances in Neural Information Processing Systems},
  volume~32, 2019.
\newblock URL
  \url{https://proceedings.neurips.cc/paper/2019/hash/e0eacd983971634327ae1819ea8b6214-Abstract.html}.

\bibitem[Zhang et~al.(2025)Zhang, Morwani, Vyas, Wu, Zou, Ghai, Foster, and
  Kakade]{zhang2024criticalbatchsize}
Hanlin Zhang, Depen Morwani, Nikhil Vyas, Jingfeng Wu, Difan Zou, Udaya Ghai,
  Dean~P. Foster, and Sham~M. Kakade.
\newblock How does critical batch size scale in pre-training?
\newblock In \emph{International Conference on Learning Representations}, 2025.
\newblock URL \url{https://arxiv.org/abs/2410.21676}.

\bibitem[Zhao et~al.(2025)Zhao, Morwani, Brandfonbrener, Vyas, and
  Kakade]{zhao2025deconstructing}
Rosie Zhao, Depen Morwani, David Brandfonbrener, Nikhil Vyas, and Sham~M.
  Kakade.
\newblock Deconstructing what makes a good optimizer for autoregressive
  language models.
\newblock In \emph{International Conference on Learning Representations}, 2025.
\newblock URL \url{https://arxiv.org/abs/2407.07972}.

\end{thebibliography}
\end{document}